\documentclass[11pt]{article}

\usepackage[utf8x]{inputenc}
\usepackage{comment}
\usepackage{microtype}
\usepackage{graphicx}
\usepackage{subfigure}
\usepackage{booktabs} 
\usepackage{scalerel}
\usepackage[margin=1in]{geometry}
\usepackage{natbib}

\usepackage{amssymb}
\usepackage{amsthm} 
\usepackage{amsmath}
\usepackage{mathtools}

\usepackage{nicefrac}

\usepackage{algorithm}
\usepackage{algpseudocode}
\usepackage[shortlabels]{enumitem}

\usepackage{url}
\usepackage{float}

\usepackage[hidelinks]{hyperref}

\title{MAML and ANIL Provably Learn Representations}

\author{Liam Collins\thanks{Department of Electrical and Computer Engineering, 
The University of Texas at Austin, Austin, TX,  USA. \qquad\qquad\{liamc@utexas.edu, mokhtari@austin.utexas.edu, sanjay.shakkottai@utexas.edu\}.},\quad  Aryan Mokhtari$^*$, \quad Sewoong Oh\thanks{School of Computer Science and Engineering, University of Washington, Seattle, WA, USA. \qquad\qquad \{sewoong@cs.washington.edu\}.} ,\quad Sanjay Shakkottai$^*$}

\date{}

\usepackage{setspace}

\begin{document}

\maketitle

\vskip 0.3in



\newcommand{\ones}{\mathbf{1}}
\newcommand{\integers}{{\mbox{\bf Z}}}
\newcommand{\symm}{{\mbox{\bf S}}}  

\newcommand{\nullspace}{{\mathcal N}}
\newcommand{\range}{{\mathcal R}}
\newcommand{\Rank}{\mathop{\bf Rank}}
\newcommand{\Tr}{\mathop{\bf Tr}}
\newcommand{\diag}{\mathop{\bf diag}}
\newcommand{\card}{\mathop{\bf card}}
\newcommand{\rank}{\mathop{\bf rank}}
\newcommand{\conv}{\mathop{\bf conv}}
\newcommand{\prox}{\mathbf{prox}}

\newcommand{\ind}{\mathds{1}}
\newcommand{\E}{\mathbb{E}}
\newcommand{\Prob}{\mathbb{P}}
\newcommand{\bigO}{\mathcal{O}}
\newcommand{\B}{\mathcal{B}}
\newcommand{\s}{\mathcal{S}}
\newcommand{\Ev}{\mathcal{E}}
\newcommand{\R}{\mathbb{R}}
\newcommand{\Co}{{\mathop {\bf Co}}} 
\newcommand{\dist}{\mathop{\bf dist{}}}
\newcommand{\argmin}{\mathop{\rm argmin}}
\newcommand{\argmax}{\mathop{\rm argmax}}
\newcommand{\epi}{\mathop{\bf epi}} 
\newcommand{\Vol}{\mathop{\bf vol}}
\newcommand{\dom}{\mathop{\bf dom}} 
\newcommand{\intr}{\mathop{\bf int}}
\newcommand{\sign}{\mathop{\bf sign}}

\newcommand{\cf}{{\it cf.}}
\newcommand{\eg}{{\it e.g.}}
\newcommand{\ie}{{\it i.e.}}
\newcommand{\etc}{{\it etc.}}

\newtheorem{innercustomthm}{Theorem}
\newenvironment{customthm}[1]
  {\renewcommand\theinnercustomthm{#1}\innercustomthm}
  {\endinnercustomthm}

\newtheorem{thm}{Theorem}
\newtheorem{lemma}{Lemma}
\newtheorem{cor}{Corollary}
\newtheorem{rem}{Remark}
\newtheorem{conj}{Conjecture}
\newtheorem{assumption}{Assumption}
\newtheorem{definition}{Definition}
\newtheorem{fact}{Fact}

\makeatletter
\newcommand*\rel@kern[1]{\kern#1\dimexpr\macc@kerna}
\newcommand*\widebar[1]{%
  \begingroup
  \def\mathaccent##1##2{%
    \rel@kern{0.8}%
    \overline{\rel@kern{-0.8}\macc@nucleus\rel@kern{0.2}}%
    \rel@kern{-0.2}%
  }%
  \macc@depth\@ne
  \let\math@bgroup\@empty \let\math@egroup\macc@set@skewchar
  \mathsurround\z@ \frozen@everymath{\mathgroup\macc@group\relax}%
  \macc@set@skewchar\relax
  \let\mathaccentV\macc@nested@a
  \macc@nested@a\relax111{#1}%
  \endgroup
}
\makeatother

\newcommand{\numberthis}{\addtocounter{equation}{1}\tag{\theequation}}

\newcommand{\simiid}{\overset{\text{i.i.d.}}{\sim}}
\def\dist{\operatorname{dist}}
\def\col{\operatorname{col}}
\def\Cov{\operatorname{Cov}}
\def\polylog{\operatorname{polylog}}
\def\poly{\operatorname{poly}}
\def\Tr{\operatorname{Tr}}
\def\del{\boldsymbol{\Delta}}
\def\deld{\boldsymbol{\bar{\Delta}}}
\def\deldin{\boldsymbol{\bar{\Delta}}_{t,i}^{in}}
\def\lam{\boldsymbol{\Lambda}}
\def\b{\mathbf{B}_t}
\def\w{\mathbf{w}_t}
\def\teta{\tilde{\eta}_\ast}
\def\sout{\mathbf{\Sigma}_{t,i}^{out}}
\def\sin{\mathbf{\Sigma}_{t,i}^{in}}
\def\zout{\mathbf{z}_{t,i}^{out}}
\def\zin{\mathbf{z}_{t,i}^{in}}
\def\bsbin{\mathbf{B}_{t}^{\top}\mathbf{\Sigma}_{t,i}^{in}\mathbf{B}_t}
\def\bstin{\mathbf{B}_{\ast}^{\top}\mathbf{\Sigma}_{t,i}^{in}\mathbf{B}_t}
\def\rank{\operatorname{rank}}
\def\dim{\operatorname{dim}}
\def\vec{\operatorname{vec}}
\def\st{\operatorname{s.t.}}
\def\spann{\operatorname{span}}

\begin{abstract}
Recent empirical evidence has driven conventional wisdom to believe that
gradient-based meta-learning (GBML) methods perform well at few-shot learning because they learn an expressive data representation that is shared across tasks. However, the mechanics of GBML have remained largely mysterious from a theoretical perspective. 
In this paper{, we prove that two well-known GBML methods, MAML and ANIL,} as well as their first-order approximations, are capable of learning common representation among a set of given tasks. Specifically, in the well-known multi-task linear representation learning setting, they are able to recover the ground-truth representation at an  exponentially fast rate. Moreover, our analysis illuminates that the driving force causing MAML and ANIL to recover the underlying representation is that 
they {adapt the final layer of their model,} which harnesses the underlying task diversity to improve the representation in all directions of interest.
To the best of our knowledge, these are the first results to show that MAML and/or ANIL learn expressive representations and to rigorously explain why they do so.
\end{abstract}


\newpage


\renewcommand{\baselinestretch}{0.5}\normalsize
\tableofcontents
\renewcommand{\baselinestretch}{1.0}\normalsize


\newpage

\section{Introduction} \label{sec:intro}

A widely popular approach to achieve fast adaptation in multi-task learning settings 
is to learn a representation that extracts the important features shared across tasks
\citep{maurer2016benefit}. However, our understanding of {\em how} multi-task representation learning should be done and {\em why} certain methods work well is still nascent.


Recently, a paradigm known as {\em meta-learning} has emerged as a powerful means of learning multi-task representations. This was sparked in large part by the introduction of Model-Agnostic Meta-Learning (MAML) \citep{finn2017model}, 
which achieved impressive results in few-shot image classification and reinforcement learning scenarios, and led to a series of related  gradient-based meta-learning (GBML) methods
\citep{raghu2019rapid,nichol2018reptile,antoniou2018train,hospedales2020meta}.
Surprisingly, MAML does not explicitly try to learn a useful representation; instead,
it aims to find a good initialization for a small number of task-specific gradient descent steps, agnostic of whether the learning model contains a representation. 
Nevertheless, \citet{raghu2019rapid} empirically argued that MAML's impressive performance on neural networks is likely due to its tendency to learn a shared representation across tasks. To make this argument, they noticed that MAML's representation does not change significantly when adapted to each task. Moreover, they showed that a modified version of MAML that freezes the representation during local adaptation, known as the Almost-No-Inner-Loop algorithm (ANIL), typically performs at least as well as MAML on few-shot image classification tasks. Yet it is still not well understood why these algorithms 
that search for a good initialization for gradient descent should find useful a global representation among tasks.
Thus, in this paper, we aim to address the following questions:

\vspace{-1mm}
\begin{quote}
\em{Do MAML and ANIL provably learn high-quality representations? If so, why?}
\end{quote}
\vspace{-1mm}

To answer these questions we consider the multi-task linear representation learning setting \citep{maurer2016benefit,tripuraneni2020provable,du2020fewshot} in which each task is a noisy linear regression problem in $\mathbb{R}^d$ with optimal solution lying in a shared $k$-dimensional subspace, where $k \ll d$. The learning model is a two-layer linear network consisting of a representation (the first layer of the model) and head {(the last layer)}. {The goal is to learn a representation that projects} data onto the shared subspace so as to reduce the number of samples needed to find the optimal regressor for a new task from $\Omega(d)$ to $\Omega(k)$.

\textbf{Main contributions}.
{\em We prove, for the first time, that both MAML and ANIL, as well their first-order approximations, are capable of representation learning and recover the ground-truth subspace in this setting.}
Our analysis reveals that ANIL and MAML's distinctive adaptation updates for the last layer of the learning model are critical to their recovery of the ground-truth representation. 
Figure 1 visualizes this observation: all meta-learning approaches (Exact ANIL, MAML, and their first-order (FO) versions that ignore second-order derivatives) approach the ground truth exponentially fast, while a non-meta learning baseline of average loss minimization {empirically} fails to recover the ground-truth. We show that the inner loop updates of the head exploit \textit{task diversity} to make the outer loop updates bring the representation  closer to the ground-truth.
 However, MAML's inner loop updates for the representation can inhibit this behavior, thus, our results for MAML require an initialization with error {\em related to task diversity}, whereas ANIL requires only {\em constant} error. We also show that ANIL learns the ground-truth representation with only $\tilde{O}(\tfrac{k^3d}{n}+k^3)\ll d$ samples per task, demonstrating that ANIL's representation learning is {\em sample-efficient}.



\begin{figure}[t]
\begin{center}
\vspace{-2mm}
\centerline{\includegraphics[width=0.52\columnwidth]{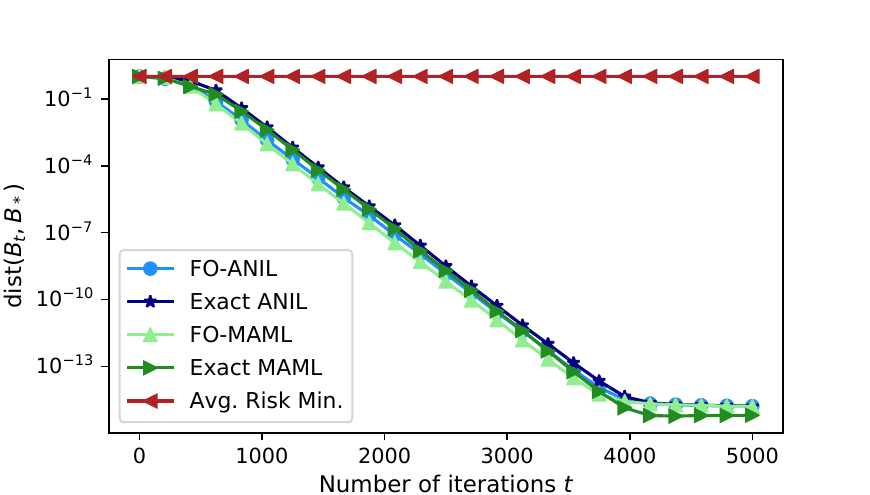}
}
\vspace{-2mm}
\caption{Distance of learned representation from the ground-truth for ANIL, MAML and average risk minimization run on task population losses in multi-task linear representation learning setting.}
\vspace{-8mm}
\label{fig:intro}
\end{center}
\end{figure}

\textbf{Related work.} Several works have studied why 
meta-learning algorithms are effective; please see 
Appendix \ref{app:related_work} for a comprehensive discussion. 
Building off \citet{raghu2019rapid}, most of these works have studied meta-learning from a representation learning perspective \citep{goldblum2020unraveling,kao2022maml,saunshi2021representation,arnold2021maml,wang2021bridging}. Among these, 
 \citet{ni2021data,bouniot2020towards,setlur2020support} and \citet{kumar2021effect} showed mixed  empirical impacts of training task diversity on model performance. Most related to our work is \citet{saunshi2020sample}, which proved that the continuous version of a first-order GBML method, Reptile \citep{nichol2018reptile}, learns a 1-D linear representation in a two-task setting with a specific initialization, explicit regularization, and infinite samples per task.
Other works studied multi-task  representation learning  in the linear setting we consider from a statistical perspective \citep{maurer2016benefit,du2020fewshot,tripuraneni2020provable}. \citet{collins2021exploiting} and \citet{thekumparampil2021statistically} further gave optimization results for gradient-based methods in this setting. 
However, the algorithms they studied are customized for the assumed low-dimensional linear representation model, which makes it relatively easy to learn the correct representation efficiently. A more challenging task is to understand how general purpose and model-agnostic meta-learning algorithms perform, such as MAML and ANIL we study.

\noindent \textbf{Notations.} We use bold capital letters for matrices and bold lowercase letters for vectors.  
We use $\mathcal{O}^{d\times k}$ to denote the set of matrices in $\mathbb{R}^{d \times k}$ with orthonormal columns. 
A hat above a matrix, e.g. $\mathbf{\hat{B}} \in \mathbb{R}^{d\times k}$ implies $\mathbf{\hat{B}}\in\mathcal{O}^{d\times k}$. We let $\col(\mathbf{B})$ denote the column space of $\mathbf{B}$ and $\col(\mathbf{B})^{\perp}$ denote the orthogonal complement to this space.
 $\mathcal{N}(0, \sigma^2)$ denotes the Gaussian distribution with mean 0 and variance $\sigma^2$. $O(\cdot)$ and $\Omega(\cdot)$ hide constant factors, and $\tilde{O}(\cdot)$ and $\tilde{\Omega}(\cdot)$ hide log factors.

\section{Problem Formulation}


We employ the linear representation learning framework studied in \cite{maurer2016benefit,tripuraneni2020provable,du2020fewshot}. We index tasks by $(t,i)$, corresponding to the $i$-th task sampled on iteration $t$.
Each task in this setting is a linear regression problem in $\mathbb{R}^d$. The inputs $\mathbf{x}_{t,i}\in \mathbb{R}^d$ and labels $y_{t,i}\in \mathbb{R}$ for the $(t,i)$-th task are sampled from a distribution on $\mathbb{R}^d\times\mathbb{R}$ such that:
\begin{equation}
    \mathbf{x}_{t,i} \sim p, \quad z_{t,i} \sim \mathcal{N}(0,\sigma^2), \quad y_{t,i} = \langle \boldsymbol{\theta}_{\ast,t,i}, \mathbf{x}_{t,i}\rangle + z_{t,i} \nonumber
\end{equation}
where  $\boldsymbol{\theta}_{\ast,t,i} \in \mathbb{R}^d$ is the ground-truth regressor for task $(t,i)$, and $p$  is a distribution over $\mathbb{R}^d$ satisfying $\mathbb{E}_{\mathbf{x}\sim p}[\mathbf{x}]=\mathbf{0}_d$ and $\mathbb{E}_{\mathbf{x}\sim p}[\mathbf{x}\mathbf{x}^\top]=\mathbf{I}_d$.
To account for shared information across tasks, we suppose there exists a matrix $\mathbf{B}_\ast \in \mathcal{O}^{d \times k}$ 
such that the ground-truth regressors $\{\boldsymbol{\theta}_{\ast,t,i}\}_i$ for all tasks lie in $\col(\mathbf{B}_\ast)$, so they can be written as $\boldsymbol{\theta}_{\ast,t,i} = \mathbf{B}_\ast \mathbf{w}_{\ast,t,i}$ for all $i$.
The task environment $(\mathbf{B}_\ast, \nu)$ consists of  $\mathbf{B}_\ast$ and a distribution $\nu$ on $\mathbb{R}^{k}$ from which the ground-truth heads, {i.e. last or predictive layers}, $\mathbf{w}_{\ast,t,i}$ are drawn.
With knowledge of $\col{(\mathbf{B}_\ast)}$, we can reduce  number of samples needed to solve a task from $\Omega(d)$ to $\Omega(k)$ by projecting the task data onto $\col{(\mathbf{B}_\ast)}$, then learning a head in $\mathbb{R}^k$. The question becomes how to learn $\col{(\mathbf{B}_\ast)}$.

The learning model consists of a  representation $\mathbf{B} \in \mathbb{R}^{d\times k}$ and a head $\mathbf{w}\in \mathbb{R}^{k}$.
We would like the column space of $\mathbf{B}$ to be close to that of $\mathbf{B}_\ast$, measured as follows. 
\begin{definition}[Principle angle distance]
 Let $\mathbf{\hat{B}}\in \mathcal{O}^{d\times k}$ and $ \mathbf{\hat{B}}_{\ast,\perp}\in \mathcal{O}^{d\times (d-k)}$ denote orthonormal matrices whose columns span $\col(\mathbf{B})$ and $\col(\mathbf{B}_\ast)^{\perp}$, respectively. Then the principle angle distance between $\mathbf{B}$ and  $
\mathbf{B}_\ast$ is
\begin{align}
    \dist(\mathbf{B}, \mathbf{B}_\ast) &\coloneqq \|\mathbf{\hat{B}}_{\ast,\perp}^\top \mathbf{\hat{B}}\|_2.
\end{align}
For shorthand, we denote $\dist_t \coloneqq \dist(\mathbf{B}_t, \mathbf{B}_\ast)$.
\end{definition}
Notice that $\dim(\col(\mathbf{B}_\ast)) = k$. Thus, the learned representation $\mathbf{B}$ must extract $k$ orthogonal directions belonging to $\col(\mathbf{B}_\ast)$. We will soon see that MAML and ANIL's task-specific adaptation of the head critically leverages task diversity to learn $k$ such directions.

\section{Algorithms}
Here we formally state the implementation of ANIL and MAML for the problem described above. 
First, letting $\boldsymbol{\theta} \coloneqq [\mathbf{w}^\top; \vec(\mathbf{B})]\in \mathbb{R}^{(d+1)k}$ denote the vector of model parameters, 
we define the population loss for task $(t,i)$:
\begin{align}
    \mathcal{L}_{t,i}(\boldsymbol{\theta}) \coloneqq \tfrac{1}{2}\mathbb{E}_{\mathbf{x}_{t,i},y_{t,i}} \left[(\langle \mathbf{B} \mathbf{w}, \mathbf{x}_{t,i}\rangle - y_{t,i})^2 \right]. 
\end{align}
Often we approximate this loss with the finite-sample loss  for a dataset $\mathcal{D}_{t,i}$:
\begin{align}
 \hat{\mathcal{L}}_{t,i}(\boldsymbol{\theta};\mathcal{D}_{t,i})  &\coloneqq   \tfrac{1}{2|\mathcal{D}_{t,i}|}\sum_{j}  (\langle \mathbf{Bw}, \mathbf{x}_{t,i,j}\rangle - y_{t,i,j})^2. \nonumber
 \label{fs_loss}
 \end{align}

\noindent\textbf{MAML.} 
MAML  minimizes the average loss across tasks {\em after} a small number of task-specific gradient updates. Here, we consider that the task-specific updates are one step of SGD for simplicity.  
Specifically, the loss function that MAML minimizes is
\begin{equation}
  \min_{\boldsymbol{\theta}} \mathcal{L}_{\text{\tiny{\textit{MAML}}}}(\boldsymbol{\theta}) \coloneqq
   \mathbb{E}_{\mathbf{w}_{\ast,t,i}, \mathcal{D}_{t,i}}[
  \mathcal{L}_{t,i}( \boldsymbol{\theta}- \alpha\nabla_{\boldsymbol{\theta}} \hat{\mathcal{L}}_{t,i}(\boldsymbol{\theta};   \mathcal{D}_{t,i}))
  )
   ]. \label{maml_obj_main}
\end{equation}
MAML essentially solves this objective with minibatch SGD. At iteration $t$, it samples $n$ tasks and two i.i.d. datasets $\mathcal{D}_{t,i}^{in},$ $\mathcal{D}_{t,i}^{out}$ for each task consisting of $m_{in}$ labeled inner loop samples and $m_{out}$ labeled outer loop samples, respectively. For the $i$-th sampled task, in what is known as the inner loop, MAML takes the task-specific gradient step from the initial model $(\mathbf{B}_t, \mathbf{w}_t)$ with step size $\alpha$ using the $m_{in}$ samples for  both the head and representation:
\begin{align*}
\boldsymbol{\theta}_{t,i}&= \begin{bmatrix} \mathbf{w}_{t,i} \\
\vec(\mathbf{B}_{t,i}) \\
\end{bmatrix} \gets 
\begin{bmatrix}
\mathbf{w}_{t} - \alpha \nabla_{\mathbf{w}} \hat{\mathcal{L}}_{i}(\mathbf{{B}}_{t}, \mathbf{w}_{t}; \mathcal{D}_{t,i}^{in})  \\ 
\vec(\mathbf{B}_{t}) - \alpha \nabla_{\vec(\mathbf{B})} \hat{\mathcal{L}}_{i}(\mathbf{{B}}_{t},
\mathbf{w}_{t}; \mathcal{D}_{t,i}^{in})  \\
\end{bmatrix}
\end{align*}
Then, in the outer loop, the new parameters are computed by taking a minibatch SGD step with respect to the loss after task-specific adaptation, using step size $\beta$ and the samples $\mathcal{D}_{t,i}^{out}$ for each task. Specifically, 
\begin{align}
    \boldsymbol{\theta}_{t+1}&\gets \boldsymbol{\theta}_t-\tfrac{\beta}{n}\sum_{i=1}^n (\mathbf{I}-\alpha \nabla^2_{\boldsymbol{\theta}}\hat{\mathcal{L}}_{t,i} (\boldsymbol{\theta}_t; \mathcal{D}_{t,i}^{out})) \nabla\hat{\mathcal{L}}_{t,i}(\boldsymbol{\theta}_{t,i}; \mathcal{D}_{t,i}^{out}) \nonumber
\end{align}
Note that the above Exact MAML update requires expensive second-order derivative computations. In practice, FO-MAML, which drops the Hessian,  is often used, {since it typically achieves similar performance \citep{finn2017model}}.

\textbf{ANIL.} 
Surprisingly, \citet{raghu2019rapid} noticed that training neural nets with a modified version of MAML that lacks inner loop updates for the representation resulted in models that matched and sometimes even exceeded the performance of models trained by MAML on few-shot image classification tasks. 
This modified version is ANIL, and its  inner loop updates in our linear case are given as follows:
\begin{align*}
\boldsymbol{\theta}_{t,i}&
=
\begin{bmatrix} \mathbf{w}_{t,i} \\
\vec(\mathbf{B}_{t,i}) \\
\end{bmatrix}  =
\begin{bmatrix}
\mathbf{w}_{t} - \alpha \nabla_{\mathbf{w}} \hat{\mathcal{L}}_{t,i}(\mathbf{{B}}_{t}, \mathbf{w}_{t}; \mathcal{D}_{t,i}^{in}) \\
\vec(\mathbf{B}_{t})  \\
\end{bmatrix}
\end{align*}
In the outer loop, ANIL again takes a minibatch SGD step with respect to the loss after the inner loop update. 
 Then, the outer loop updates for  Exact ANIL are given by:
\begin{align*}
\boldsymbol{\theta}_{t+1}\gets \boldsymbol{\theta}_{t} - \frac{\beta}{n}\sum_{i=1}^n 
\hat{\mathbf{H}}_{t,i}(\boldsymbol{\theta}_{t}; \mathcal{D}_{t,i}^{out})
\nabla_{\boldsymbol{\theta}}\hat{\mathcal{L}}_{t,i}(\boldsymbol{\theta}_{t,i},  \mathcal{D}_{t,i}^{out})
\end{align*}
where, for Exact ANIL,
\begin{align*}
\hat{\mathbf{{H}}}_{t,i}&(\mathbf{{B}}_{t}, \mathbf{w}_{t}; \mathcal{D}_{t,i}^{out})
\coloneqq \begin{bmatrix}
\mathbf{I}_k - \alpha \nabla_{\mathbf{w}}^2 \hat{\mathcal{L}}_{t,i}( \boldsymbol{\theta}_{t}; \mathcal{D}_{t,i}^{out}) &  \mathbf{0} \\
  -\alpha \tfrac{\partial^2}{\partial \vec({\mathbf{B}}) \partial{\mathbf{w}}} \hat{\mathcal{L}}_{t,i}( \boldsymbol{\theta}_{t}; \mathcal{D}_{t,i}^{out})   &  \mathbf{I}  
\end{bmatrix}
\end{align*} 
To avoid computing second order derivatives, we can instead treat $\hat{\mathbf{{H}}}_{t,i}$ as the identity operator, in which case we call the algorithm FO-ANIL. 



\subsection{Role of Adaptation} \label{sec:intuition}
Now we present new intuition for MAML and ANIL's representation learning ability which motivates our proof structure.
The key observation is that the outer loop gradients for the representation are {\em evaluated at the inner loop-adapted parameters}; this harnesses the power of task diversity  to improve the representation in {\em all} $k$ directions. 
This is easiest to see in the FO-ANIL case with $m_{in}=m_{out}=\infty$. In this case, using that the input data distribution $p$ satisfies $\mathbb{E}_{\mathbf{x}\sim p}[\mathbf{x}]=\mathbf{0}_d$ and $\mathbb{E}_{\mathbf{x}\sim p}[\mathbf{x}\mathbf{x}^\top]=\mathbf{I}_d$, the update for the representation is given as:
\begin{align}
    \mathbf{B}_{t+1} &=
   \mathbf{{B}}_{t}  \Big(\underbrace{  \mathbf{I}_k - \tfrac{\beta}{n} \sum_{i=1}^n \mathbf{w}_{t,i} \mathbf{w}_{t,i}^\top}_{\text{FO-ANIL prior weight}}\Big) + \mathbf{{B}}_\ast\underbrace{\tfrac{\beta}{n} \sum_{i=1}^n   \mathbf{w}_{\ast,t,i} \mathbf{w}_{t,i}^\top}_{\text{FO-ANIL signal weight}} \label{intuition}
\end{align}
We would like the `prior weight' to be small and the `signal weight' to be large so the update replaces energy from $\col(\mathbf{B}_t)$ with energy from $\col({\mathbf{B}_\ast})$. 
Roughly, this is true as long as  $\boldsymbol{\Psi}_t\coloneqq \tfrac{1}{n} \sum_{i=1}^n \mathbf{w}_{t,i} \mathbf{w}_{t,i}^\top$ is well-conditioned, i.e. 
the $\mathbf{w}_{t,i}$'s are diverse. Assuming
$\mathbf{w}_{t,i}\approx \mathbf{w}_{\ast,t,i}$ for each task,  
then $\boldsymbol{\Psi}_t$ 
is well-conditioned if and only if the tasks are diverse.
Thus, we can see how task diversity 
causes the column space of the representation learned by FO-ANIL to approach the ground-truth. For FO-MAML, we see similar behavior, with a caveat. The representation update is:
\begin{align}
    \mathbf{B}_{t+1} &\stackrel{(a)}{=}
    {\mathbf{{B}}_{t} - \tfrac{\beta}{n} \sum_{i=1}^n \mathbf{{B}}_{t,i}\mathbf{w}_{t,i} \mathbf{w}_{t,i}^\top} +  {\mathbf{{B}}_\ast\tfrac{\beta}{n} \sum_{i=1}^n   \mathbf{w}_{\ast,t,i} \mathbf{w}_{t,i}^\top} \nonumber \\
    \vspace{-1mm}
     &\stackrel{(b)}{=} 
    {\mathbf{{B}}_{t}\Big(\underbrace{\mathbf{I}_k - (\mathbf{I}_k -\alpha \mathbf{w}_t\mathbf{w}_t^\top) \tfrac{\beta}{n} \sum_{i=1}^n \mathbf{w}_{t,i} \mathbf{w}_{t,i}^\top}_{\text{FO-MAML prior weight}}\Big)} + {\mathbf{{B}}_\ast\Big(\underbrace{\tfrac{\beta}{n} \sum_{i=1}^n  (1 - \alpha\langle \mathbf{w}_{t,i},\mathbf{w}_t\rangle) \mathbf{w}_{\ast,t,i} \mathbf{w}_{t,i}^\top}_{\text{FO-MAML signal weight}}\Big)} \nonumber
\end{align}
Equation $(a)$ is similar to \eqref{intuition} except that one $\mathbf{B}_t$ is replaced by the $\mathbf{B}_{t,i}$'s resulting from inner loop adaptation. Expanding $\mathbf{B}_{t,i}$ in $(b)$, we see that the prior weight is at least as large as in \eqref{intuition}, since $\lambda_{\max}(\mathbf{I}_k -\alpha \mathbf{w}_t\mathbf{w}_t^\top)\leq 1$, but it can still be small as long as the $\mathbf{w}_{t,i}$'s are diverse and $\|\mathbf{w}_t\|_2$ is small. Thus we can see that FO-MAML should also learn the representation, yet its inner loop adaptation complicates its ability to do so.

{\textbf{Comparison with no inner-loop adaptation.}} Compare these updates to the case when there is {no inner loop adaptation}, i.e. we run SGD on the non-adaptive objective $\min_{\boldsymbol{\theta}} \mathbb{E}_{\mathbf{w}_{\ast,t,i}} [\mathcal{L}_{t,i}(\boldsymbol{\theta})]$ instead of \eqref{maml_obj_main}. In this case, $\mathbf{B}_{t+1}$ is:
\begin{align}
    \mathbf{B}_{t+1} &=
    {\mathbf{{B}}_{t} \big(  \mathbf{I}_k - {\beta}\mathbf{w}_{t} \mathbf{w}_{t}^\top\big)} +  \beta{\mathbf{{B}}_\ast \mathbf{\bar{w}}_{\ast,t} \mathbf{w}_{t}^\top} \label{noadapt}
\end{align}
where $\mathbf{\bar{w}}_{\ast,t} \coloneqq \frac{1}{n}\sum_{i=1}^n\mathbf{w}_{\ast,t,i}$.
Observe that the coefficient of $\b$ in the update is rank $k\!-\!1$, while the coefficient of $\mathbf{B}_\ast$ is rank 1. Thus, $\col(\mathbf{B}_{t+1})$ can approach $\col(\mathbf{B}_{\ast})$ in at most one direction on any iteration. Empirically, $\mathbf{w}_t$ points in roughly the same direction throughout training, preventing this approach from learning  $\col(\mathbf{B}_\ast)$ (e.g. see Figure \ref{fig:intro}).

%
%
%
%

\textbf{Technical challenges.} 
The intuition on the role of adaptation, while appealing, makes strong assumptions; most notably that the $\mathbf{w}_{t,i}$'s are diverse enough to improve the representation and that the algorithm dynamics are stable. 
To show these points, we observe that $\mathbf{w}_{t,i}$ can be written as the linear combination of a vector distinct for each task and a vector that is shared across all tasks at time $t$. Showing that the shared vector is small implies the $\mathbf{w}_{t,i}$'s are diverse, and we can control the magnitude of the shared vector by controlling $\|\mathbf{w}_t\|_2$ and $\|\mathbf{I}_k - \alpha \mathbf{B}_t^\top \mathbf{B}_t\|_2$. 
Showing that these  quantities are small at all times also ensures the stability of the algorithms.
Meanwhile, we must show that $\|\mathbf{B}_{\ast,\perp}^\top\mathbf{B}_t\|_2$ and $\dist_t=\|\mathbf{\hat{B}}_{\ast,\perp}^\top \mathbf{\hat{B}}_t\|_2 $ are contracting. 
It is not obvious that any of these conditions hold individually; in fact, they require a novel 
multi-way inductive argument to show that they  hold simultaneously for each $t$ (see Section \ref{sec:sketch}).
\section{Main Results}

In this section we formalize our intuition discussed previously and prove that both MAML and ANIL and their first-order approximations are capable learning the column space of the ground-truth representation.
To do so, we first make the following assumption concerning the diversity of the sampled ground-truth heads.

\begin{assumption}[Task diversity]\label{assump:tasks_diverse_main}
The eigenvalues of the symmetric matrix  $\boldsymbol{{\Psi}}_{\ast,t} \coloneqq \tfrac{1}{n}\sum_{i=1}^n \mathbf{w}_{\ast,t,i}\mathbf{w}_{\ast,t,i}^\top$ are almost surely\footnote{We could instead assume the ground-truth heads are sub-gaussian and use standard concentration results show that with $n\!=\!\Omega(k+\log(T))$, the set of ground-truth heads $\{\mathbf{w}_{\ast,t,i}\}_{i=1}^n$ sampled on iteration $t$ are $(1+O(1),1-O(1))$-diverse for all $T$ iterations with high probability, but instead we assume almost-sure bounds for simplicity.} uniformly bounded below and above by $\mu_\ast^2$ and $L_\ast^2$, respectively, i.e., $\mu_\ast^2 \mathbf{I} \preceq \boldsymbol{{\Psi}}_{\ast,t} \preceq L_\ast^2 \mathbf{I}$, for all $ t \in [T]$.
\end{assumption}

The lower bound on the eigenvalues of the matrix $\boldsymbol{{\Psi}}_{\ast,t}$ ensures that the $k\times k$ matrix $\boldsymbol{{\Psi}}_{\ast,t}$ is full rank and hence the vectors $\{\mathbf{w}_{\ast,t,i}\}_{i=1}^n$ span $\mathbb{R}^k$, therefore they are diverse. However, the diversity level of the tasks is defined by ratio of the eigenvalues of the  matrix $\boldsymbol{{\Psi}}_{\ast,t}$, i.e., $\kappa_\ast\coloneqq \frac{L_{\ast}}{\mu_\ast}$. If this ratio is close to 1, then the ground-truth heads are very diverse and have equal energy in all directions. On the other hand, if $\kappa_\ast$ is large, then the ground-truth heads are not very diverse as their energy is mostly focused in  a specific direction. Hence, as we will see in our results, smaller  $\kappa_\ast$ leads to faster convergence for ANIL and MAML.



{
Now we are ready to state our main results for the ANIL and FO-ANIL algorithms in the infinite sample case. 
\begin{thm}\label{thm:anil_pop_main}
Consider the infinite sample case for ANIL and FO-ANIL, where  $m_{in}\!=\!m_{out}\!= \!\infty$. Further, suppose the conditions in Assumption \ref{assump:tasks_diverse_main} hold, the initial weights are selected as $\mathbf{w}_0 = \mathbf{0}_k$ and ${\alpha}\mathbf{B}_0^\top\mathbf{B}_0  = \mathbf{I}_k$. Let the step sizes are chosen as $\alpha = O(\nicefrac{1}{L_{\ast}})$ and 
 $\beta = O({\alpha}{\kappa_{\ast}^{-4}})$ for ANIL and $\beta = O({\alpha}{\kappa_{\ast}^{-4}\min(1,{\nicefrac{\mu_\ast^2}{\eta_\ast^{2}}}}))$ for FO-ANIL, where ${\eta}_\ast$ satisfies $\|\tfrac{1}{n}\sum_{i=1}^n\mathbf{w}_{\ast,t,i}\|_2 \leq \eta_\ast$ for all times $t\in [T]$ almost surely. If the initial error satisfies the condition $\dist_0 \leq \sqrt{0.9}$, then almost surely for both  ANIL and FO-ANIL we have,
\begin{align}
    \dist(\mathbf{{B}}_T, \mathbf{{B}}_\ast) \leq \left(1 - 0.5\beta \alpha E_0 \mu_\ast^2   \right)^{T-1},
\end{align}
where 
 $E_0 \coloneqq  0.9 - \dist_0^2$.
\end{thm}
Theorem \ref{thm:anil_pop_main} shows that both FO-ANIL and Exact ANIL learn a representation that approaches the ground-truth exponentially fast as long as the initial representation $\mathbf{B}_0$ is normalized and is a constant distance away from the ground-truth, the initial head $\mathbf{w}_0=\mathbf{0}$, and the sampled tasks are diverse. {Note that $\beta$ is larger for ANIL and hence its convergence is faster, demonstrating the benefit of second-order updates.}

Next, we state our results for FO-MAML and Exact MAML for the same infinite sample setting. 
Due to the adaptation of both the representation and head, the MAML and FO-MAML updates involve  third- and fourth-order products of the ground-truth heads, unlike the ANIL and FO-ANIL updates which involve at most second-order products. 
To analyze the higher-order terms, we assume that the 
 energy in each ground-truth head is balanced. 

\begin{assumption}[Task incoherence]\label{assump:tasks_main}
    For all times $t\in [T]$ and tasks $i \in [n]$, we almost surely have
$
\|\mathbf{w}_{\ast,t,i}\|_2\!\leq\! c{\sqrt{k}}L_\ast 
$, where $c$ is a constant.
\end{assumption} 

Next, as discussed in Section \ref{sec:intuition}, MAML's adaptation of the representation complicates its ability to learn the ground-truth subspace. As a result, we require an additional condition to show that MAML learns the representation: the distance of the initialization to the ground-truth must small in the sense that it must scale with the task diversity and inversely with $k$. We formalize this in the following theorem. 

\begin{thm}\label{thm:maml_pop_main} Consider the infinite sample case for MAML, where $m_{in}=m_{out}= \infty$. Further, suppose the conditions in Assumptions \ref{assump:tasks_diverse_main} and \ref{assump:tasks_main} hold, the initial weights are selected as $\mathbf{w}_0 = \mathbf{0}_k$ and ${\alpha}\mathbf{B}_0^\top\mathbf{B}_0  = \mathbf{I}_k$, and the step sizes satisfy $\alpha = O({k^{-2/3} L_\ast^{-1} T^{-1/4}})$ and $ \beta = 
O({\alpha }{ \kappa_\ast^{-4} })$. If $\dist_0 =O({k}^{-0.75}\kappa_{\ast}^{-1.5})$, then almost surely 
\begin{align*}
    \dist(\mathbf{{B}}_T, \mathbf{{B}}_\ast) \leq \left(1 - 0.5\beta \alpha E_0 \mu_\ast^2   \right)^{T-1} ,
\end{align*}
where $E_0 \coloneqq  0.9 - \dist_0^2$.
\end{thm}
We can see that the initial representation learning error for MAML must scale as $O({k}^{-0.75}\kappa_{\ast}^{-1.5})$, which can be much smaller than the constant scaling that is sufficient for ANIL to learn the representation (see Theorem \ref{thm:anil_pop_main}). Next we give the main result for FO-MAML, which requires an additional condition that the norm of the average of the ground-truth heads sampled on each iteration is small. This condition arises due to the fact that the FO-MAML updates are approximations of the exact MAML updates, and thus have a bias that depends on the average of the ground-truth heads. Without control of this bias, $\mathbf{w}_t$ the dynamics will diverge. 
\begin{thm}\label{thm:fomaml_pop_main} Consider the infinite sample case for FO-MAML, where $m_{in}=m_{out}= \infty$. Further, suppose the conditions in Assumptions \ref{assump:tasks_diverse_main} and \ref{assump:tasks_main} hold, the initial weights are selected as $\mathbf{w}_0 = \mathbf{0}_k$ and ${\alpha}\mathbf{B}_0^\top\mathbf{B}_0  = \mathbf{I}_k$, and the step sizes satisfy $\alpha = O(\tfrac{1}{\sqrt{k}L_\ast})$ and $\beta = O({\alpha}{\kappa_{\ast}^{-4}})$. If the initial error satisfies $\dist_0 =O({k}^{-0.5}\kappa_{\ast}^{-1})$, and the average of the true heads  almost surely satisfies $  \|\tfrac{1}{n}\sum_{i=1}^n\mathbf{w}_{\ast,t,i}\|_2= O(k^{-1.5}\kappa_{\ast}^{-3} \mu_\ast) $ for all times $t$, then almost surely
    $$
     \dist(\mathbf{{B}}_T, \mathbf{{B}}_\ast) \leq \left(1 - 0.5\beta \alpha E_0 \mu_\ast^2   \right)^{T-1},
    $$
    where $E_0 \coloneqq  0.9 - \dist_0^2$.
\end{thm}
Theorem \ref{thm:fomaml_pop_main} shows that FO-MAML learns $\col(\mathbf{B}_\ast)$ as long as the initial principal angle is small and $ \|\tfrac{1}{n}\sum_{i=1}^n\mathbf{w}_{\ast,t,i}\|_2= O(k^{-1.5}\kappa_{\ast}^{-3} \mu_\ast) $ on all iterations, due to the biased updates.
Note that the FO-ANIL updates are also biased, but this bias scales with $\|\mathbf{I}_k-\alpha\mathbf{B}_t^\top \mathbf{B}_t\|_2$, which is eventually decreasing quickly enough to make the cumulative error induced by the bias negligible without any additional conditions.  In contrast, $\|\mathbf{I}_k\!-\!\alpha\mathbf{B}_t^\top \mathbf{B}_t\|_2$ is not guaranteed to decrease for FO-MAML due to the inner loop adaptation of the representation, so we need the additional condition.

To the best of our knowledge, the above theorems  are the first results to show that ANIL, MAML, and their first-order approximations learn representations in any setting. Moreover, they are the first to show
how {\em task diversity} plays a key role in representation learning from an {optimization} perspective, to the best of our knowledge. 
Due the the restrictions on $\beta$ and $\alpha$, Theorems \ref{thm:anil_pop_main} and \ref{thm:maml_pop_main} show that the rate of contraction of principal angle distance diminishes with less task diversity. Thus, the more diverse the tasks, i.e. the smaller $\kappa_\ast$, the faster that ANIL and MAML learn the representation. Additionally, the less diverse the tasks, the more accurate initialization that MAML requires, and the tighter that the true heads must be centered around zero to control the FO-MAML bias.

}

\subsection{Finite-sample results}

Thus far we have only considered the infinite sample case, i.e.,  $m_{in}\!=\!m_{out}\!=\!\infty$, to highlight the reasons that the adaptation updates in MAML and ANIL are essential for representation learning. Next, we study the finite sample setting.
Indeed, establishing our results for the finite sample case is more challenging, but the mechanisms by which ANIL and MAML learn representations for finite $m_{in}$ and $m_{out}$ are very similar to the infinite-sample case, and the finite-sample problem reduces to showing concentration of the  updates to the infinite-sample updates.

For MAML, this concentration requires assumptions on  sixth and eighth-order products of the data which arise due to the inner-loop updates. In light of this, for the sake of  readability we  only give the finite-sample result for ANIL and FO-ANIL, whose analyses require only standard assumptions on the data, as we state below.
 \begin{assumption}[Sub-gaussian feature distribution]
\label{assump:data}
For $\mathbf{x}\sim p$,  $\mathbb{E}[\mathbf{x}]\!=\!\mathbf{0}$ and  $\Cov(\mathbf{x})\!=\!\mathbf{I}_d$. Moreover, $\mathbf{x}$ is $\mathbf{I}_d$-sub-gaussian in the sense that $\mathbb{E}[\exp(\mathbf{v}^\top \mathbf{x})]\!\leq\! \exp(\tfrac{\|\mathbf{v}\|_2^2}{2})$ $\forall\; \mathbf{v}$.
\end{assumption}
Under this assumption, we can show the following.
\begin{thm}[ANIL Finite Samples]\label{thm:anil_fs} Consider the finite-sample case for  ANIL and FO-ANIL.
Suppose Assumptions \ref{assump:tasks_diverse_main}, \ref{assump:tasks_main} and \ref{assump:data} hold,  $\alpha = O((\sqrt{k}L_\ast+\sigma)^{-1})$ 
and $\beta$ is chosen as in Theorem \ref{thm:anil_pop_main}.
For some $\delta \!>\!0$ to be defined later, let $E_0\! =\! 0.9\!-\!\dist^2_0-\delta$ and 
assume $E_0$ is lower bounded by a positive constant. Suppose the sample sizes  satisfy  $m_{in} = \tilde{\Omega}(M_{in})$ and $m_{out} = \tilde{\Omega}(M_{out})$ for some expressions $M_{in},M_{out}$ to be defined later. 
Then 
both ANIL and FO-ANIL satisfy:
\vspace{-1mm}
\begin{align}
  \dist(\mathbf{{B}}_{T}, \mathbf{{B}}_\ast ) &\leq \left(1 - 0.5\beta \alpha \mu_{\ast}^2\right)^{T-1} + \tilde{O}(\delta) \nonumber
\end{align}
\vspace{-2mm}
where for ANIL,
\vspace{-0.2mm}
\begin{align}
   M_{in}&=  k^3+\tfrac{k^3d}{n}, \quad M_{out}= k^2 + \tfrac{dk+k^3}{n}, \quad \delta=  (\sqrt{k}\kappa_{\ast}^2+\tfrac{\kappa_\ast\sigma}{\mu_\ast}+\tfrac{\sigma^2}{\mu_\ast^2})({k}\!+\!\tfrac{\sqrt{dk}}{\sqrt{n}})(\tfrac{1}{\sqrt{m_{in}}} \!+\!\tfrac{1}{\sqrt{m_{out}}}) \nonumber
\end{align}
\vspace{-2mm}
and for FO-ANIL,
\vspace{-0.2mm}
\begin{align}
  M_{in}&=  k^2
    , \quad M_{out}= \tfrac{dk+  k^3}{n}, \quad \delta= (\sqrt{k}\kappa_{\ast}^2+\tfrac{\kappa_\ast\sigma}{\mu_\ast}+ \tfrac{\sigma^2}{\mu_\ast^2\sqrt{m_{in}}})\tfrac{\sqrt{dk}}{\sqrt{nm_{out}}} \nonumber
\end{align}
\vspace{-0mm}
with probability at least $1\!-\!\tfrac{T}{\poly(n)}\!-\!\tfrac{T}{\poly(m_{in})}\!-O(Te^{-90k})$.
\end{thm}
For ease of presentation, the $\tilde{\Omega}()$ notation excludes log factors and all parameters besides $k,d$ and $n$; please see Theorem \ref{thm:anil_fs_app} in Appendix \ref{app:anil_fs} for the full statement. 
We focus on dimension parameters and $n$ here to highlight the sample complexity benefits conferred by ANIL and FO-ANIL compared to solving each task separately ($n\!=\!1$).
Theorem \ref{thm:anil_fs} shows that ANIL requires only  $m_{in}+m_{out} = \tilde{\Omega}(\frac{k^3d}{n}+ k^3)$  samples per task to reach a neighborhood of the ground-truth solution. 
Since $k\ll d$ and $n$ can be large, this sample complexity is far less than the 
$\tilde{\Omega}(d)$ required to solve each task individually \citep{hsu2012random}. Note that more samples are required for Exact ANIL because the second-order updates involve higher-order products of the data, which have heavier tails than the analogous terms for FO-ANIL.

\section{Proof sketch} \label{sec:sketch}

We now  discuss how we prove the results in greater detail. We focus on the FO-ANIL case because the presentation is simplest yet still illuminates the key ideas used in all proofs.

\subsection{Theorem 1 (FO-ANIL)}

\noindent {\textbf{Intuition.}} 
Our goal is to show that the distance between the column spaces of $\mathbf{B}_{t}$ and $\mathbf{B}_\ast$, i.e. $\dist_{t} \coloneqq \|\mathbf{\hat{B}}_{\ast,\perp}^\top\mathbf{\hat{B}}_{t}\|_2$ is converging to zero at a linear rate for all $t$. 
We will use an inductive argument in which we assume favorable conditions to hold up to time $t$, and will prove they continue to hold at time $t+1$.
To show $\dist_{t+1}$ is linearly decaying, it is helpful to first consider the non-normalized energy in the subspace orthogonal to the ground-truth, namely $\|\mathbf{\hat{B}}_{\ast,\perp}^\top\mathbf{{B}}_{t+1}\|_2$.  
 We have seen in equation \eqref{intuition} that if the inner-loop adapted heads $\mathbf{w}_{t,i}$ at time $t$ are diverse, then the FO-ANIL update of the representation subtracts energy from the previous representation and adds energy from the ground-truth representation.
Examining \eqref{intuition} closer, we see that the only energy in the column space of the new representation that can be orthogonal to the ground-truth subspace is contributed by the previous representation, and this energy 
is contracting at a rate proportional to  the condition number of the matrix formed by the adapted heads. In particular, if we define the matrix $\boldsymbol{\Psi}_t\coloneqq \tfrac{1}{n} \sum_{i=1}^n \mathbf{w}_{t,i} \mathbf{w}_{t,i}^\top$, then we have
\begin{align}
    \|\mathbf{B}_{\ast,\perp}^\top \mathbf{B}_{t+1}\|_2 &=\|\mathbf{B}_{\ast,\perp}^\top \mathbf{B}_{t}(\mathbf{I}-\beta \mathbf{\Psi}_t)\|_2 
    \leq ( 1 - \beta \lambda_{\min}(\mathbf{\Psi}_t))\|\mathbf{B}_{\ast,\perp}^\top \mathbf{B}_{t}\|_2, \label{psisi}
\end{align}
as long as $\beta \leq \nicefrac{1}{\lambda_{\max}(\mathbf{\Psi}_t)}$. 
Therefore, to show that the normalized energy $\|\mathbf{\hat{B}}_{\ast,\perp}^\top\mathbf{\hat{B}}_{t+1}\|_2$ approaches zero, we aim to show: (I) The condition number of $\mathbf{\Psi}_t$ continues to stay controlled and finite, which implies linear convergence of the non-normalized energy in $\col(\mathbf{B}_\ast)^\perp$ according to \eqref{psisi}; and (II) The minimum singular value of the representation $\mathbf{B}_{t+1}$ is staying the same. Otherwise, the energy orthogonal to the ground-truth subspace could be decreasing, but the representation could be becoming singular, which would mean the distance to the ground-truth subspace is not decreasing.

To show (I), note that the adapted heads are given by:
\begin{align}
    \mathbf{w}_{t,i}=
    \underbrace{\del_t\mathbf{w}_t}_{\text{non-unique}}+\alpha \underbrace{\mathbf{B}_{t}^\top \mathbf{B}_{\ast}\mathbf{w}_{\ast,t,i}}_{\text{unique}}, \label{updw}
    \vspace{-2mm}
\end{align}
where $\del_t \coloneqq \mathbf{I}_k - \alpha \b^\top \b$. 
The vector $\del_t \mathbf{w}_t$ is present in every $\mathbf{w}_{t,i}$,
so we refer to it as the non-unique part of $\mathbf{w}_{t,i}$. On the other hand, $\alpha \mathbf{B}_{t}^\top \mathbf{B}_{\ast}\mathbf{w}_{\ast,t,i}$ is the unique part of $\mathbf{w}_{t,i}$. 
Equation \eqref{updw} shows that if the non-unique part of each $\mathbf{w}_{t,i}$ is relatively small compared to the unique part, then $\mathbf{\Psi}_t \approx \alpha^2 \mathbf{B}_t^\top \mathbf{B}_\ast \mathbf{\Psi}_{\ast,t} \mathbf{B}_\ast^\top \mathbf{B}_t$, 
meaning the $\mathbf{w}_{t,i}$'s are almost as diverse as the ground-truth heads. So we aim to show $\|\del_t\|_2$ and $\|\mathbf{w}_t\|_2$ remain small for all $t$. We specifically need to show they are small compared to $\sigma_{\min}^2(\mathbf{B}_{t}^\top \mathbf{B}_\ast)$, since this quantity roughly lower bounds the energy in the diverse part of $\mathbf{w}_{t,i}$. {One can show that $\sigma_{\min}^2(\mathbf{\hat{B}}_{t}^\top \mathbf{B}_\ast) = 1 - \dist_t^2$, so we need to use that $\dist_t$ is decreasing in order to lower bound the energy in the unique part of $\mathbf{w}_{t,i}$.}


It is also convenient to track $\|\del_t\|_2$ in order to show (II), since $\|\del_{t+1}\|_2\leq \varepsilon$ implies  $\sigma_{\min}(\mathbf{B}_{t+1})\geq  \tfrac{\sqrt{1-\varepsilon}}{\sqrt{\alpha}}$. Note that for (II), we need control of $\|\del_{t+1}\|_2$, whereas to show (I) we needed control of $\|\del_t\|_2$. This difference in time indices is accounted for by the induction we will soon discuss.


We can now see why it makes sense to initialize with $\|\del_0\|_2=0$ and $\|\mathbf{w}_t\|_2=0$ (in fact, they do not have to be exactly zero; any initialization with $\|\mathbf{w}_0\|_2=O(\sqrt{\alpha})$ and  $\|\del_t\|_2=O({\alpha}^2)$ would work). 
However, proving that $\|\del_t\|_2$ and $\|\mathbf{w}_t\|_2$ remain small  is difficult
because the algorithm lacks explicit regularization or a normalization step after each round. 
Empirically, $\sigma_{\min}(\mathbf{B}_t)$ may decrease and $\|\mathbf{w}_t\|_2$ may increase on any particular round, so it is not clear why $\sigma_{\min}(\mathbf{B}_t)$ does not go to zero (i.e. $\|\del_t\|_2$ does not go to 1) and $\|\mathbf{w}_t\|_2$ does not blow up. To address these issues, one could add an explicit regularization term to the loss functions or an orthonormalization step to the algorithm, but doing so is empirically unnecessary and  would not be consistent with the ANIL formulation or algorithm.

\begin{figure*}[t]
\vspace{3mm}
\begin{center}
\centerline{\includegraphics[width=\textwidth]{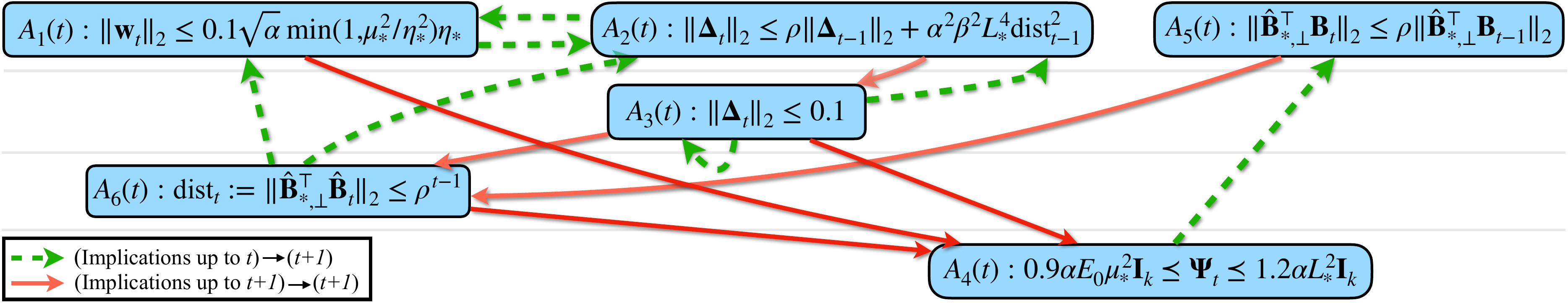}
}
\vspace{-2mm}
\caption{Logical flow of the proof. Note that there are no cycles among the implications from $t+1$ to $t+1$, so the logic is consistent.}
\vspace{-8mm}
\label{fig:flow}
\end{center}
\end{figure*}

\textbf{Inductive structure.}
We overcome the aforementioned challenges by executing a multi-way induction that involves the following six inductive hypotheses:
\begin{enumerate}
\vspace{-2mm}
\item $A_1(t) \coloneqq \{\|\mathbf{w}_t\|_2 = O(\sqrt{\alpha} \min(1,\tfrac{\mu_\ast^2}{\eta_\ast^2})\eta_\ast  )\}$,
   \item $A_2(t) \!\coloneqq\! \!\big\{\|\del_t\|_2 \!\leq\! \rho\|\del_{t-1}\|_2
 \!+\! O(\beta^2\alpha^2 L_{\ast}^4 \dist^2_{t-1})\big\}$, 
       \item $A_3(t) \coloneqq \big\{\|\del_t  \|_2 \leq \tfrac{1}{10}\big\}$, 
       \item $A_4(t) \coloneqq \big\{0.9\alpha E_0 \mu_\ast\mathbf{I}_k \preceq \mathbf{\Psi}_t \preceq 1.2\alpha L_\ast^2 \mathbf{I}_k \big\}$, 
        \item $A_5(t) \coloneqq \big\{\|\mathbf{{B}}_{\ast,\perp}^\top \mathbf{{B}}_{t} \|_2  \leq \rho \|\mathbf{{B}}_{\ast,\perp}^\top \mathbf{{B}}_{t-1} \|_2  \big\}$,  
     \item  $A_6(t) \coloneqq \{\dist_{t}=\|\mathbf{\hat{B}}_{\ast,\perp}^\top\mathbf{\hat{B}}_t\|_2 \leq \rho^{t-1} \}$,  
\end{enumerate}
\vspace{-2mm}
where $\rho\!=\!1 - 0.5\beta \alpha E_0 \mu_\ast^2$.  Our previous intuition motivates our choice of inductive hypotheses $A_1(t),\dots,A_5(t)$ as intermediaries to ultimately show that $\dist_t$ linearly converges to zero.  More specifically,  $A_1(t),A_2(t),$ and $A_3(t)$ bound  $\|\mathbf{w}_t\|_2$ and  $\|\del_t\|_2$, $A_4(t)$ controls the diversity of the inner loop-adapted heads, and  $A_5(t)$ and $A_6(t)$ confirm that the learned representation approaches the ground-truth.
We employ two upper bounds on $\|\del_t\|_2$ because we need to use that $\{\|\del_t\|_2\}_t$ is both summable $(A_2(t))$ and uniformly small $(A_3(t))$ to complete different parts of the induction. In particular, if true for all $t$, $A_2(t)$ shows that $\|\del_t\|_2$ may initially increase, but eventually linearly converges to zero due to the linear convergence of $\dist_t$. The initialization implies each inductive hypothesis holds at time $t=1$. We must show they hold at time $t+1$ if they hold up to time $t$.

To do this, we employ the logic visualized in 
Figure \ref{fig:flow}. 
The top level events ($A_1(t+1), A_2(t+1), A_5(t+1)$) are most ``immediate'' in the sense that they follow directly from other events at all times up to and including $t$ (via the dashed green arrows). The proofs of all other events at time $t\!+\!1$ require the occurrence of other events at time $t+1$, with more logical steps needed as one moves down the graph, and solid red arrows denoting implications from and to time $t+1$. In particular, $A_3(t+1)$ requires the events up to and including time $t$ {\em and} a top-level event at $t+1$, namely $A_2(t+1)$, so it is in the second level. Similarly, $A_6(t+1)$ requires events up to and including  time $t$ and the second-level event at $t\!+\!1$, so it is in the third level, and so on. 

Recall that our intuition is that diverse adapted heads leads to contraction of the non-normalized representation distance. We see this logic in the implication $A_4(t)\!\implies\! A_5(t+1)$. 
 We then reasoned that contraction of the non-normalized distance leads to linear convergence of the distance as long as the minimum singular value of the representation is controlled from below. This intuition is captured in the implication  $A_5(t+1)\cap A_3(t+1)\! \implies \!A_6(t+1)$.

We also discussed that the diversity of the adapted heads depends on the global head being small, the representation being close to a scaled orthonormal matrix, and the representation distance being bounded away from 1 at the start of that iteration. This is seen in the implication showing that the adapted heads are again diverse on iteration $t+1$, in particular
 $A_1(t+1)\cap A_3(t+1)\cap A_6(t+1)\! \implies\! A_4(t+1)$.
The other implications in the graph are technical and needed to control $\|\mathbf{w}_{t+1}\|_2$ and $\|\del_{t+1}\|_2$.

\textbf{Showing the implications.}
We now formally discuss each implication, starting with the top level. Full proofs are provided in Appendix \ref{app:anil}.
\begin{itemize}[leftmargin=10pt]
        \item $A_4(t)\!\implies \! A_5(t\!+\!1)$. This is true by equation \eqref{psisi}.
    \item $A_1(t)\cap A_3(t)\cap A_6(t) \!\implies\! A_2(t\!+\!1)$. It can be shown that $\del_{t+1}$ is of the form:
    \begin{align}
        \del_{t+1}&= \del_t(\mathbf{I}_k- \beta \alpha^2  \b^\top \mathbf{B}_\ast \mathbf{\Psi}_{\ast, t} \mathbf{B}_\ast^\top \b) + \mathbf{N}_{t} \nonumber
    \end{align}
    for some matrix $\mathbf{N}_{t}$ whose norm is upper bounded by a linear combination of $\|\del_t\|_2$ and $\dist_t$. We next use
    \begin{align}
        \lambda_{\min}(\b^\top \mathbf{B}_\ast \mathbf{\Psi}_{\ast, t} \mathbf{B}_\ast^\top \b ) &\geq  \mu_\ast^2 \sigma_{\min}^2(\b^\top \mathbf{B}_\ast)\nonumber \\
        &\geq  \tfrac{0.9}{\alpha} \mu_\ast^2 (1 - \dist_t^2 )\label{eoo}
    \end{align}
    where \eqref{eoo} follows by $\sigma_{\min}^2(\mathbf{\hat{B}}_t^\top \mathbf{B}_\ast)=1-\dist_t^2$ and $A_3(t)$.  The proof follows by applying $A_6(t)$ to control $1-\dist_t^2$.
    \item $\left( \cap_{s=1}^{t}A_2(s) \cap A_6(s)\right)\!\implies\! A_1(t\!+\!1)$. This is the most difficult induction to show. The FO-ANIL dynamics are such that $\|\mathbf{w}_t\|_2$ may increase on every iteration throughout the entire execution of the algorithm. However, we can exploit the fact that the amount that it increases is proportional to $\|\del_t\|_2$, which we can show is summable due 
    to the linear convergence of $\dist_t$. First, we have
    \begin{equation*}
        \mathbf{w}_{t+1} = (\mathbf{I}_k - \beta \mathbf{B}_t^\top \mathbf{B}_t\del_t)\mathbf{w}_t +\frac{\beta}{n}\sum_{i=1}^n \del_t \mathbf{B}_t^\top \mathbf{B}_\ast \mathbf{w}_{\ast,t,i}
    \end{equation*}
    which implies $\|\mathbf{w}_t\|_2$ increases on each iteration by $O(\frac{\beta}{\sqrt{\alpha}}\|\del_{t}\|_2 \eta_\ast)$. In particular,
     \begin{align}
        \|\mathbf{w}_{t+1}\|_2 &\stackrel{(a)}{\leq} (1 +\tfrac{2\beta}{\alpha}\|\del_t\|_2)\|\mathbf{w}_{t}\|_2 + \tfrac{2\beta L_\ast}{\sqrt{\alpha}}\|\del_t\|_2 \nonumber \\
        &\stackrel{(b)}{\leq} \sum_{s=0}^{t-1}\tfrac{2\beta \eta_\ast}{\sqrt{\alpha}}\|\del_s\|_2 \prod_{r=s}^{t-1}(1 +\tfrac{2\beta}{\alpha}\|\del_r\|_2)  \nonumber \\
        &\stackrel{(c)}{\leq} \sum_{s=0}^{t-1}\tfrac{2\beta \eta_\ast}{\sqrt{\alpha}}\|\del_s\|_2 \Big(1 +\tfrac{1}{t-s}\sum_{r=s}^{t-1}\tfrac{2\beta}{\alpha}\|\del_r\|_2\Big)^{t-s} \nonumber 
    \end{align}
    where $(b)$ follows by recursively applying $(a)$ for $t,t\!-\!1,...$. and $(c)$ follows by the AM-GM inequality.
    Next, for any $s\in[t]$, recursively apply $A_2(s)$,$A_2(s\!-\!1),...$ and use $A_6(r)\; \forall r \in [s]$ to obtain, for an absolute constant $c$,
    \begin{align}
        \|\del_{s}\|_2 
        &\stackrel{(d)}{\leq} c\sum_{r=0}^{s-1} \rho^{s-r}\beta^2 \alpha^2 L_\ast^4 \dist_{r}^2 \leq c \rho^{s} \sum_{r=0}^{s-1}\rho^r \beta^2 \alpha^2 L_\ast^4 \nonumber
    \end{align}
    Plugging $(d)$ into $(c)$, computing the sum of geometric series, and applying the choice of $\beta$ completes the proof.
    \item $A_2(t\!+\!1) \cap A_3(t)\! \implies\! A_3(t\!+\!1)$. This follows straightforwardly since $\beta$ is chosen sufficiently small.
    \item $A_3(t\!+\!1) \cap \left(\cap_{s=1}^{t+1} A_5(s)\right) \cap A_6(t)\! \implies \! A_6(t\!+\!1)$.  Using the definition of the principal angle distance, the Cauchy-Schwarz inequality, and $\cap_{s=1}^{t+1} A_5(s)$, we can show
    \begin{align*}
        \dist_{t+1} &\leq  \tfrac{1}{\sigma_{\min}(\mathbf{B}_{t+1})}\|\mathbf{\hat{B}}_{\ast,\perp}^\top \mathbf{{B}}_{t+1}\|_2
        \leq \tfrac{\sigma_{\max}(\mathbf{B}_{0})}{\sigma_{\min}(\mathbf{B}_{t+1})}\rho^t\dist_0 
    \end{align*}
    from which the proof follows after applying $A_3(t\!+\!1)$ and the initial conditions. Note that here we have normalized the representation only once at time $t\!+\!1$ and used the contraction of the non-normalized energy to recurse from $t\!+\!1$ to $0$, resulting in a $\tfrac{\sigma_{\max}(\mathbf{B}_0)}{\sigma_{\min}(\mathbf{B}_{t+1})}$ scaling error. If we instead tried to directly show the contraction of distance and thereby normalized analytically on every round, we would obtain $\dist_{t+1}\leq \big(\prod_{s=0}^t \tfrac{\sigma_{\max}(\mathbf{B}_{s})}{\sigma_{\min}(\mathbf{B}_{s+1})}\big) \rho^t \dist_0$, meaning a $\prod_{s=0}^t \tfrac{\sigma_{\max}(\mathbf{B}_{s})}{\sigma_{\min}(\mathbf{B}_{s+1})}$ scaling error, which is too large because $\mathbf{B}_s$ is in fact not normalized on every round.
    \item $A_1(t+1) \cap  A_3(t+1) \cap A_6(t+1)\!\implies \! A_4(t+1)$. This follows by expanding each $\mathbf{w}_{t,i}$ as in \eqref{updw}, and using similar logic as in \eqref{eoo}. 
\end{itemize}

\subsection{Other results -- ANIL, FO-MAML, and MAML}

For ANIL, the inductive structure is nearly identical. The only meaningful change in the proof is that the second-order updates imply $\|\mathbf{w}_{t+1}\|_2-\|\mathbf{w}_{t}\|_2= O(\|\del_t\|_2^2)$, which is smaller than the $O(\|\del_t\|_2)$ for FO-ANIL, and thereby allows to control $\|\mathbf{w}_{t+1}\|_2$ with  a potentially larger $\beta$. 

For FO-MAML and MAML, recall that the inner loop update of the representation weakens the benefit of adapted head diversity (see Section~\ref{sec:intuition}). Thus, {larger} adapted head diversity is needed to learn $\col(\mathbf{B}_\ast)$.
Specifically, we require a tighter bound of $\|\del_t\|_2=O(\alpha^2)$, compared to the  $\|\del_t\|_2=O(1)$ bound in ANIL, and for FO-MAML, we also require a tighter bound on $\|\w\|_2$ (recall from Section \ref{sec:sketch} that smaller $\|\del_t\|_2$ and $\|\w\|_2$ improves adapted head diversity). Moreover, to obtain tight bounds on $\|\mathbf{w}_{t+1}\|_2$ we can no longer use that $\|\mathbf{w}_{t+1}\|_2-\|\mathbf{w}_{t}\|_2$ is controlled by $\|\del_t\|_2$ due to to additional terms in the outer loop update. To overcome these issues, we must make 
stricter assumptions on the initial distance, and in the case of FO-MAML, on the average ground-truth head. See Appendix \ref{app:maml} for details.




\begin{figure}[t]
\begin{center}
\centerline{\includegraphics[width=0.6\columnwidth]{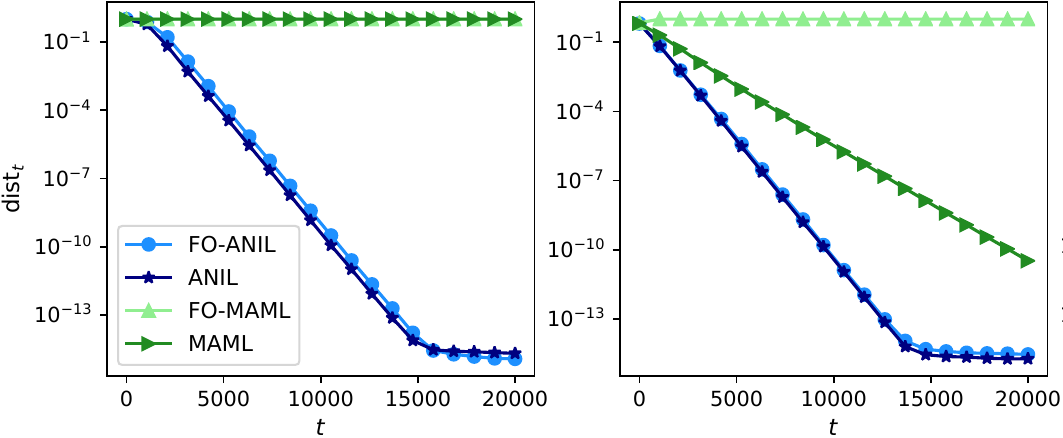}
}
\vspace{-2mm}
\caption{
(Left) Random $\mathbf{B}_0$. (Right) Methodical $\mathbf{B}_0$. In both cases, the mean ground-truth head is far from zero. 
}
\vspace{-8mm}
\vspace{-2mm}
\label{fig:sims1}
\end{center}
\end{figure}



Finally, the proof of Theorem \ref{thm:anil_fs} relies on showing concentration of the finite-sample gradients to the population gradients. 
The principal challenge is showing this concentration  for fourth-order products of the data that arise in the ANIL updates, since we cannot apply standard methods to these higher-order products while maintaining $o(d)$ samples per task. To get around this, we leverage the low-rankness of the products by applying a truncated version of the concentration result for low-rank random matrices from \cite{magen2011low}. We also use the L4-L2-hypercontractiveness of the data to control the bias in these higher-order products. Details are found in Appendix \ref{app:anil_fs}.
\section{Numerical simulations} \label{sec:sims}

{
We next show that the additional conditions required for MAML and FO-MAML to learn the ground-truth representation are empirically necessary.  That is, (i) MAML and FO-MAML require a good initialization relative to the underlying task diversity, and (ii) FO-MAML further requires  the ground-truth heads to be concentrated around zero.
To test these conditions,  we set $d\!=\!20$, $n\!=\!k=\!3$, randomly draw $\mathbf{B}_\ast$, and use the task population losses. {The ground-truth heads are drawn as $\mathbf{w}_{\ast,t,i}
{\sim} \mathcal{N}(10 \mathbf{1}_k, \mathbf{I}_k)$.}
Ground-truth task diversity is thus low, since most of the energy points in the direction $\mathbf{1}_k$. In Figure \ref{fig:sims1} (left), we use a random Gaussian initialization of $\mathbf{B}_0$, which has $\dist_0\approx 0.99$. In \ref{fig:sims1} (right), we initialize with a noisy version of $\mathbf{B}_\ast$ satisfying $\dist_0 \in [0.65,0.7]$. The plots show that in this low-diversity setting, MAML requires good initialization to achieve linear convergence, whereas FO-MAML cannot obtain it even with good initialization, as $\|\mathbb{E}[\mathbf{w}_{\ast,t,i}]\|\!\gg\! 0$. {Lastly, note that in the same setting except $\|\mathbb{E}[\mathbf{w}_{\ast,t,i}]\|\!=\!0$ as in Figure \ref{fig:intro}, all four GBML approaches  learn $\col(\mathbf{B}_\ast)$, as expected.}





}

\section{Conclusion}

Our analysis reveals that ANIL, MAML and their first-order approximations exploit task diversity via inner adaptation steps of the head to recover the ground-truth representation in the multi-task linear representation learning setting. 
Further, task diversity helps these algorithms to exhibit an implicit regularization that keeps the learned representation well-conditioned. 
However, the inner adaptation of the representation plays a restrictive role, inhibiting MAML and FO-MAML from achieving global convergence. To the best of our knowledge, these are the first results showing that GBML algorithms can learn a $k$-dimensional subspace. 

\section*{Acknowledgements}
This research is supported in part by NSF Grants 2127697, 2019844, 2107037, and 2112471, ARO Grant W911NF2110226, ONR Grant N00014-19-1-2566, the Machine Learning Lab (MLL) at UT Austin, and the Wireless Networking and Communications Group (WNCG) Industrial Affiliates Program.




\bibliography{refs}

\begin{thebibliography}{57}
\providecommand{\natexlab}[1]{#1}
\providecommand{\url}[1]{\texttt{#1}}
\expandafter\ifx\csname urlstyle\endcsname\relax
  \providecommand{\doi}[1]{doi: #1}\else
  \providecommand{\doi}{doi: \begingroup \urlstyle{rm}\Url}\fi

\bibitem[Antoniou et~al.(2019)Antoniou, Edwards, and
  Storkey]{antoniou2018train}
Antreas Antoniou, Harri Edwards, and Amos Storkey.
\newblock How to train your {MAML}.
\newblock In \emph{Seventh International Conference on Learning
  Representations}, 2019.

\bibitem[Arnold et~al.(2021)Arnold, Iqbal, and Sha]{arnold2021maml}
S{\'e}bastien Arnold, Shariq Iqbal, and Fei Sha.
\newblock When {MAML} {C}an {A}dapt {F}ast and {H}ow to {A}ssist {W}hen it
  {C}annot.
\newblock In \emph{International Conference on Artificial Intelligence and
  Statistics}, pages 244--252. PMLR, 2021.

\bibitem[Balcan et~al.(2019)Balcan, Khodak, and Talwalkar]{balcan2019provable}
Maria-Florina Balcan, Mikhail Khodak, and Ameet Talwalkar.
\newblock Provable {G}uarantees for {G}radient-{B}ased {M}eta-{L}earning.
\newblock In \emph{International Conference on Machine Learning}, pages
  424--433. PMLR, 2019.

\bibitem[Baxter(2000)]{baxter2000model}
Jonathan Baxter.
\newblock A {M}odel of {I}nductive {B}ias {L}earning.
\newblock \emph{Journal of Artificial Intelligence Research}, 12:\penalty0
  149--198, 2000.

\bibitem[Bernacchia(2020)]{bernacchia2020meta}
Alberto Bernacchia.
\newblock {M}eta-{L}earning with {N}egative {L}earning {R}ates.
\newblock In \emph{International Conference on Learning Representations}, 2020.

\bibitem[Bertinetto et~al.(2018)Bertinetto, Henriques, Torr, and
  Vedaldi]{bertinetto2018meta}
Luca Bertinetto, Joao~F Henriques, Philip Torr, and Andrea Vedaldi.
\newblock {M}eta-{L}earning with {D}ifferentiable {C}losed-form {S}olvers.
\newblock In \emph{International Conference on Learning Representations}, 2018.

\bibitem[Bouniot et~al.(2020)Bouniot, Redko, Audigier, Loesch, Zotkin, and
  Habrard]{bouniot2020towards}
Quentin Bouniot, Ievgen Redko, Romaric Audigier, Ang{\'e}lique Loesch, Yevhenii
  Zotkin, and Amaury Habrard.
\newblock Towards {B}etter {U}nderstanding {M}eta-learning {M}ethods through
  {M}ulti-task {R}epresentation {L}earning {T}heory.
\newblock \emph{arXiv preprint arXiv:2010.01992}, 2020.

\bibitem[Bullins et~al.(2019)Bullins, Hazan, Kalai, and
  Livni]{bullins2019generalize}
Brian Bullins, Elad Hazan, Adam Kalai, and Roi Livni.
\newblock Generalize {A}cross {T}asks: {E}fficient {A}lgorithms for {L}inear
  {R}epresentation {L}earning.
\newblock In \emph{Algorithmic Learning Theory}, pages 235--246. PMLR, 2019.

\bibitem[Caruana(1997)]{caruana1997multitask}
Rich Caruana.
\newblock Multitask {L}earning.
\newblock \emph{Machine learning}, 28\penalty0 (1):\penalty0 41--75, 1997.

\bibitem[Chen et~al.(2020)Chen, Wu, Li, Li, Zhan, and Chung]{chen2020closer}
Jiaxin Chen, Xiao-Ming Wu, Yanke Li, Qimai Li, Li-Ming Zhan, and Fu-lai Chung.
\newblock A {C}loser {L}ook at the {T}raining {S}trategy for {M}odern
  {M}eta-{L}earning.
\newblock \emph{Advances in Neural Information Processing Systems}, 33, 2020.

\bibitem[Chua et~al.(2021)Chua, Lei, and Lee]{chua2021fine}
Kurtland Chua, Qi~Lei, and Jason~D Lee.
\newblock How fine-tuning allows for effective {M}eta-{L}earning.
\newblock \emph{Advances in Neural Information Processing Systems}, 34, 2021.

\bibitem[Collins et~al.(2021)Collins, Hassani, Mokhtari, and
  Shakkottai]{collins2021exploiting}
Liam Collins, Hamed Hassani, Aryan Mokhtari, and Sanjay Shakkottai.
\newblock Exploiting {S}hared {R}epresentations for {P}ersonalized {F}ederated
  {L}earning.
\newblock In \emph{International Conference on Machine Learning}, pages
  2089--2099. PMLR, 2021.

\bibitem[Collins et~al.(2022)Collins, Mokhtari, and
  Shakkottai]{collins2020does}
Liam Collins, Aryan Mokhtari, and Sanjay Shakkottai.
\newblock How {D}oes the {T}ask {L}andscape {A}ffect {MAML} performance?
\newblock \emph{arxiv preprint arXiv:2010.14672}, 2022.

\bibitem[Denevi et~al.(2018)Denevi, Ciliberto, Stamos, and
  Pontil]{denevi2018incremental}
G~Denevi, C~Ciliberto, D~Stamos, and M~Pontil.
\newblock Incremental {L}earning-to-{L}earn with {S}tatistical {G}uarantees.
\newblock In \emph{34th Conference on Uncertainty in Artificial Intelligence
  2018, UAI 2018}, volume~34, pages 457--466. AUAI, 2018.

\bibitem[Du et~al.(2020)Du, Hu, Kakade, Lee, and Lei]{du2020fewshot}
Simon~Shaolei Du, Wei Hu, Sham~M Kakade, Jason~D Lee, and Qi~Lei.
\newblock Few-{S}hot {L}earning via {L}earning the {R}epresentation,
  {P}rovably.
\newblock In \emph{International Conference on Learning Representations}, 2020.

\bibitem[Fallah et~al.(2020{\natexlab{a}})Fallah, Mokhtari, and
  Ozdaglar]{fallah2020convergence}
Alireza Fallah, Aryan Mokhtari, and Asuman Ozdaglar.
\newblock On the {C}onvergence {T}heory of {G}radient-{B}ased
  {M}odel-{A}gnostic {M}eta-{L}earning {A}lgorithms.
\newblock In \emph{International Conference on Artificial Intelligence and
  Statistics}, pages 1082--1092. PMLR, 2020{\natexlab{a}}.

\bibitem[Fallah et~al.(2020{\natexlab{b}})Fallah, Mokhtari, and
  Ozdaglar]{fallah2020personalized}
Alireza Fallah, Aryan Mokhtari, and Asuman Ozdaglar.
\newblock Personalized {F}ederated learning with {T}heoretical {G}uarantees: A
  {M}odel-{A}gnostic {M}eta-{L}earning {A}pproach.
\newblock In \emph{Advances in Neural Information Processing Systems},
  volume~33, pages 3557--3568, 2020{\natexlab{b}}.

\bibitem[Fallah et~al.(2021)Fallah, Mokhtari, and
  Ozdaglar]{fallah2021generalization}
Alireza Fallah, Aryan Mokhtari, and Asuman Ozdaglar.
\newblock Generalization of {M}odel-{A}gnostic {M}eta-{L}earning {A}lgorithms:
  {R}ecurring and {U}nseen {T}asks.
\newblock \emph{Advances in Neural Information Processing Systems}, 34, 2021.

\bibitem[Finn et~al.(2017)Finn, Abbeel, and Levine]{finn2017model}
Chelsea Finn, Pieter Abbeel, and Sergey Levine.
\newblock Model-{A}gnostic {M}eta-{L}earning for {F}ast {A}daptation of {D}eep
  {N}etworks.
\newblock In \emph{Proceedings of the 34th International Conference on Machine
  Learning-Volume 70}, pages 1126--1135. JMLR. org, 2017.

\bibitem[Finn et~al.(2018)Finn, Xu, and Levine]{finn2018probabilistic}
Chelsea Finn, Kelvin Xu, and Sergey Levine.
\newblock Probabilistic {M}odel-{A}gnostic {M}eta-{L}earning.
\newblock \emph{Advances in Neural Information Processing Systems}, 31, 2018.

\bibitem[Finn et~al.(2019)Finn, Rajeswaran, Kakade, and Levine]{finn2019online}
Chelsea Finn, Aravind Rajeswaran, Sham Kakade, and Sergey Levine.
\newblock Online {M}eta-{L}earning.
\newblock In \emph{International Conference on Machine Learning}, pages
  1920--1930. PMLR, 2019.

\bibitem[Goldblum et~al.(2020)Goldblum, Reich, Fowl, Ni, Cherepanova, and
  Goldstein]{goldblum2020unraveling}
Micah Goldblum, Steven Reich, Liam Fowl, Renkun Ni, Valeria Cherepanova, and
  Tom Goldstein.
\newblock Unraveling {M}eta-{L}earning: {U}nderstanding {F}eature
  {R}epresentations for {F}ew-{S}hot {T}asks.
\newblock In \emph{International Conference on Machine Learning}, pages
  3607--3616. PMLR, 2020.

\bibitem[Hospedales et~al.(2021)Hospedales, Antoniou, Micaelli, and
  Storkey]{hospedales2020meta}
Timothy~M Hospedales, Antreas Antoniou, Paul Micaelli, and Amos~J Storkey.
\newblock Meta-{L}earning in {N}eural {N}etworks: {A} {S}urvey.
\newblock \emph{IEEE Transactions on Pattern Analysis and Machine
  Intelligence}, 2021.

\bibitem[Hsu et~al.(2012)Hsu, Kakade, and Zhang]{hsu2012random}
Daniel Hsu, Sham~M Kakade, and Tong Zhang.
\newblock Random {D}esign {A}nalysis of {R}idge {R}egression.
\newblock In \emph{Conference on learning theory}, pages 9--1. JMLR Workshop
  and Conference Proceedings, 2012.

\bibitem[Ji et~al.(2020{\natexlab{a}})Ji, Lee, Liang, and
  Poor]{ji2020convergence}
Kaiyi Ji, Jason~D Lee, Yingbin Liang, and H~Vincent Poor.
\newblock Convergence of {M}eta-{L}earning with {T}ask-{S}pecific {A}daptation
  over {P}artial {P}arameters.
\newblock \emph{Advances in Neural Information Processing Systems},
  33:\penalty0 11490--11500, 2020{\natexlab{a}}.

\bibitem[Ji et~al.(2020{\natexlab{b}})Ji, Yang, and Liang]{ji2020multi}
Kaiyi Ji, Junjie Yang, and Yingbin Liang.
\newblock Multi-{S}tep {M}odel-{A}gnostic {M}eta-{L}earning: {C}onvergence and
  {I}mproved algorithms.
\newblock \emph{CoRR}, abs/2002.07836, 2020{\natexlab{b}}.

\bibitem[Jiang et~al.(2019)Jiang, Kone{\v{c}}n{\`y}, Rush, and
  Kannan]{jiang2019improving}
Yihan Jiang, Jakub Kone{\v{c}}n{\`y}, Keith Rush, and Sreeram Kannan.
\newblock Improving {F}ederated {L}earning {P}ersonalization via {M}odel
  {A}gnostic {M}eta {L}earning.
\newblock \emph{arXiv preprint arXiv:1909.12488}, 2019.

\bibitem[Kao et~al.(2022)Kao, Chiu, and Chen]{kao2022maml}
Chia~Hsiang Kao, Wei-Chen Chiu, and Pin-Yu Chen.
\newblock {MAML} is a noisy contrastive learner in classification.
\newblock In \emph{International Conference on Learning Representations}, 2022.

\bibitem[Kumar et~al.(2021)Kumar, Deleu, and Bengio]{kumar2021effect}
Ramnath Kumar, Tristan Deleu, and Yoshua Bengio.
\newblock Effect of {D}iversity in {M}eta-{L}earning.
\newblock In \emph{Fifth Workshop on Meta-Learning at the Conference on Neural
  Information Processing Systems}, 2021.

\bibitem[Lee et~al.(2019)Lee, Maji, Ravichandran, and Soatto]{lee2019meta}
Kwonjoon Lee, Subhransu Maji, Avinash Ravichandran, and Stefano Soatto.
\newblock Meta-{L}earning with {D}ifferentiable {C}onvex {O}ptimization.
\newblock In \emph{Proceedings of the IEEE/CVF Conference on Computer Vision
  and Pattern Recognition}, pages 10657--10665, 2019.

\bibitem[Li et~al.(2017)Li, Zhou, Chen, and Li]{li2017meta}
Zhenguo Li, Fengwei Zhou, Fei Chen, and Hang Li.
\newblock Meta-{SGD}: {L}earning to {L}earn {Q}uickly for {F}ew-{S}hot
  {L}earning.
\newblock \emph{arXiv preprint arXiv:1707.09835}, 2017.

\bibitem[Magen and Zouzias(2011)]{magen2011low}
Avner Magen and Anastasios Zouzias.
\newblock Low {R}ank {M}atrix-{V}alued {C}hernoff {B}ounds and {A}pproximate
  {M}atrix {M}ultiplication.
\newblock In \emph{Proceedings of the twenty-second annual ACM-SIAM symposium
  on Discrete Algorithms}, pages 1422--1436. SIAM, 2011.

\bibitem[Maurer et~al.(2016)Maurer, Pontil, and
  Romera-Paredes]{maurer2016benefit}
Andreas Maurer, Massimiliano Pontil, and Bernardino Romera-Paredes.
\newblock The {B}enefit of {M}ultitask {R}epresentation {L}earning.
\newblock \emph{Journal of Machine Learning Research}, 17\penalty0
  (81):\penalty0 1--32, 2016.

\bibitem[McNamara and Balcan(2017)]{mcnamara2017risk}
Daniel McNamara and Maria-Florina Balcan.
\newblock Risk {B}ounds for {T}ransferring {R}epresentations {W}ith and
  {W}ithout {F}ine-{T}uning.
\newblock In \emph{International Conference on Machine Learning}, pages
  2373--2381. PMLR, 2017.

\bibitem[Ni et~al.(2021)Ni, Goldblum, Sharaf, Kong, and Goldstein]{ni2021data}
Renkun Ni, Micah Goldblum, Amr Sharaf, Kezhi Kong, and Tom Goldstein.
\newblock Data augmentation for {M}eta-{L}earning.
\newblock In \emph{International Conference on Machine Learning}, pages
  8152--8161. PMLR, 2021.

\bibitem[Nichol and Schulman(2018)]{nichol2018reptile}
Alex Nichol and John Schulman.
\newblock Reptile: {A} {S}calable {M}eta-{L}earning {A}lgorithm.
\newblock \emph{arXiv preprint arXiv:1803.02999}, 2:\penalty0 2, 2018.

\bibitem[Oh et~al.(2020)Oh, Yoo, Kim, and Yun]{oh2020boil}
Jaehoon Oh, Hyungjun Yoo, ChangHwan Kim, and Se-Young Yun.
\newblock {BOIL}: {T}owards {R}epresentation {C}hange for {F}ew-{S}hot
  {L}earning.
\newblock In \emph{International Conference on Learning Representations}, 2020.

\bibitem[Raghu et~al.(2020)Raghu, Raghu, Bengio, and Vinyals]{raghu2019rapid}
Aniruddh Raghu, Maithra Raghu, Samy Bengio, and Oriol Vinyals.
\newblock Rapid {L}earning or {F}eature {R}euse? {T}owards {U}nderstanding the
  {E}ffectiveness of {MAML}.
\newblock In \emph{International Conference on Learning Representations}, 2020.

\bibitem[Rajeswaran et~al.(2019)Rajeswaran, Finn, Kakade, and
  Levine]{rajeswaran2019meta}
Aravind Rajeswaran, Chelsea Finn, Sham~M Kakade, and Sergey Levine.
\newblock Meta-{L}earning with {I}mplicit {G}radients.
\newblock \emph{Advances in Neural Information Processing Systems}, 32, 2019.

\bibitem[Ravi and Larochelle(2016)]{ravi2016optimization}
Sachin Ravi and Hugo Larochelle.
\newblock Optimization as a {M}odel for {F}ew-{S}hot {L}earning.
\newblock 2016.

\bibitem[Saunshi et~al.(2020)Saunshi, Zhang, Khodak, and
  Arora]{saunshi2020sample}
Nikunj Saunshi, Yi~Zhang, Mikhail Khodak, and Sanjeev Arora.
\newblock A {S}ample {C}omplexity {S}eparation {B}etween {N}on-{C}onvex and
  {C}onvex {M}eta-{L}earning.
\newblock In \emph{International Conference on Machine Learning}, pages
  8512--8521. PMLR, 2020.

\bibitem[Saunshi et~al.(2021)Saunshi, Gupta, and Hu]{saunshi2021representation}
Nikunj Saunshi, Arushi Gupta, and Wei Hu.
\newblock A {R}epresentation {L}earning {P}erspective on the {I}mportance of
  {T}rain-{V}alidation {S}plitting in {M}eta-{L}earning.
\newblock In \emph{International Conference on Machine Learning}, pages
  9333--9343. PMLR, 2021.

\bibitem[Schmidhuber(1987)]{schmidhuber1987evolutionary}
J{\"u}rgen Schmidhuber.
\newblock \emph{Evolutionary {P}rinciples in {S}elf-{R}eferential {L}earning,
  or on {L}earning how to {L}earn: the {M}eta-{M}eta-...}
\newblock PhD thesis, Technische Universit{\"a}t M{\"u}nchen, 1987.

\bibitem[Setlur et~al.(2020)Setlur, Li, and Smith]{setlur2020support}
Amrith Setlur, Oscar Li, and Virginia Smith.
\newblock Is {S}upport {S}et {D}iversity {N}ecessary for {M}eta-{L}earning?
\newblock \emph{arXiv preprint arXiv:2011.14048}, 2020.

\bibitem[Snell et~al.(2017)Snell, Swersky, and Zemel]{snell2017prototypical}
Jake Snell, Kevin Swersky, and Richard Zemel.
\newblock Prototypical networks for few-shot learning.
\newblock In \emph{Advances in Neural Information Processing Systems}, pages
  4077--4087, 2017.

\bibitem[Thekumparampil et~al.(2021)Thekumparampil, Jain, Netrapalli, and
  Oh]{thekumparampil2021statistically}
Kiran~K Thekumparampil, Prateek Jain, Praneeth Netrapalli, and Sewoong Oh.
\newblock Statistically and {C}omputationally {E}fficient {L}inear
  {M}eta-{R}epresentation {L}earning.
\newblock \emph{Advances in Neural Information Processing Systems}, 34, 2021.

\bibitem[Tripuraneni et~al.(2020)Tripuraneni, Jordan, and
  Jin]{tripuraneni2020theory}
Nilesh Tripuraneni, Michael Jordan, and Chi Jin.
\newblock On the {T}heory of {T}ransfer {L}earning: {T}he {I}mportance of
  {T}ask {D}iversity.
\newblock \emph{Advances in Neural Information Processing Systems},
  33:\penalty0 7852--7862, 2020.

\bibitem[Tripuraneni et~al.(2021)Tripuraneni, Jin, and
  Jordan]{tripuraneni2020provable}
Nilesh Tripuraneni, Chi Jin, and Michael Jordan.
\newblock Provable {M}eta-{L}earning of {L}inear {R}epresentations.
\newblock In \emph{International Conference on Machine Learning}, pages
  10434--10443. PMLR, 2021.

\bibitem[Vershynin(2018)]{vershynin2018high}
Roman Vershynin.
\newblock \emph{High-{D}imensional {P}robability: {A}n {I}ntroduction with
  {A}pplications in {D}ata {S}cience}, volume~47.
\newblock Cambridge University Press, 2018.

\bibitem[Vinyals et~al.(2016)Vinyals, Blundell, Lillicrap, Wierstra,
  et~al.]{vinyals2016matching}
Oriol Vinyals, Charles Blundell, Timothy Lillicrap, Daan Wierstra, et~al.
\newblock Matching {N}etworks for {O}ne {S}hot {L}earning.
\newblock \emph{Advances in Neural Information Processing Systems},
  29:\penalty0 3630--3638, 2016.

\bibitem[Wang et~al.(2021{\natexlab{a}})Wang, Zhao, and Li]{wang2021bridging}
Haoxiang Wang, Han Zhao, and Bo~Li.
\newblock Bridging {M}ulti-task {L}earning and {M}eta-{L}earning: {T}owards
  {E}fficient {T}raining and {E}ffective {A}daptation.
\newblock In \emph{International Conference on Machine Learning}. PMLR,
  2021{\natexlab{a}}.

\bibitem[Wang et~al.(2020)Wang, Cai, Yang, and Wang]{wang2020global}
Lingxiao Wang, Qi~Cai, Zhuoran Yang, and Zhaoran Wang.
\newblock On the {G}lobal {O}ptimality of {M}odel-{A}gnostic {M}eta-{L}earning.
\newblock In \emph{International Conference on Machine Learning}, pages
  9837--9846. PMLR, 2020.

\bibitem[Wang et~al.(2021{\natexlab{b}})Wang, Yuan, Wu, and
  Ge]{wang2021guarantees}
Xiang Wang, Shuai Yuan, Chenwei Wu, and Rong Ge.
\newblock Guarantees for {T}uning the {S}tep {S}ize {U}sing a
  {L}earning-to-{L}earn {A}pproach.
\newblock In \emph{International Conference on Machine Learning}, pages
  10981--10990. PMLR, 2021{\natexlab{b}}.

\bibitem[Xu and Tewari(2021)]{xu2021representation}
Ziping Xu and Ambuj Tewari.
\newblock Representation {L}earning {B}eyond {L}inear {P}rediction {F}unctions.
\newblock \emph{Advances in Neural Information Processing Systems}, 34, 2021.

\bibitem[Yoon et~al.(2018)Yoon, Kim, Dia, Kim, Bengio, and
  Ahn]{yoon2018bayesian}
Jaesik Yoon, Taesup Kim, Ousmane Dia, Sungwoong Kim, Yoshua Bengio, and Sungjin
  Ahn.
\newblock Bayesian {M}odel-{A}gnostic {M}eta-{L}earning.
\newblock In \emph{Proceedings of the 32nd International Conference on Neural
  Information Processing Systems}, pages 7343--7353, 2018.

\bibitem[Zhou et~al.(2019)Zhou, Yuan, Xu, Yan, and Feng]{zhou2019efficient}
Pan Zhou, Xiaotong Yuan, Huan Xu, Shuicheng Yan, and Jiashi Feng.
\newblock Efficient {M}eta {L}earning via {M}inibatch {P}roximal {U}pdate.
\newblock In \emph{Advances in Neural Information Processing Systems}, pages
  1532--1542, 2019.

\bibitem[Zintgraf et~al.(2019)Zintgraf, Shiarli, Kurin, Hofmann, and
  Whiteson]{zintgraf2019fast}
Luisa Zintgraf, Kyriacos Shiarli, Vitaly Kurin, Katja Hofmann, and Shimon
  Whiteson.
\newblock Fast {C}ontext {A}daptation via {M}eta-{L}earning.
\newblock In \emph{International Conference on Machine Learning}, pages
  7693--7702. PMLR, 2019.

\end{thebibliography}
\bibliographystyle{plainnat}

\appendix

\newpage







\section{Additional Related Work} \label{app:related_work}

\textbf{Meta-learning background.} Multi-task representation learning and meta-learning have been of theoretical interest for many years \citep{schmidhuber1987evolutionary,caruana1997multitask,baxter2000model}. Recently, meta-learning methods have garnered much attention due to successful implementations in few-shot learning scenarios with deep networks. These modern approaches are roughly grouped into three categories: model-based \citep{ravi2016optimization}, metric-based \citep{snell2017prototypical,vinyals2016matching}, and gradient-based \citep{finn2017model}. In this paper we focus on gradient-based methods.

\textbf{Gradient-based meta-learning and MAML.} The practicality and simplicity of model-agnostic meta-learning (MAML) \citep{finn2017model}  has led to  many experimental and theoretical studies of gradient-based meta-learning in addition to those mentioned in Section \ref{sec:intro}. There have been numerous algorithms proposed as extensions of MAML \citep{li2017meta,finn2018probabilistic,yoon2018bayesian,antoniou2018train,nichol2018reptile,rajeswaran2019meta,zhou2019efficient,raghu2019rapid,zintgraf2019fast}, and MAML has been applied to  online \citep{finn2019online}
and federated \citep{fallah2020personalized,jiang2019improving} learning settings. Theoretical analyses of MAML and related methods have included sample complexity
guarantees in online settings \citep{balcan2019provable,denevi2018incremental}, 
 general convergence guarantees \citep{fallah2020convergence,ji2020multi,ji2020convergence}, and landscape analysis \citep{wang2020global,collins2020does}. Other works have studied the choice of inner loop step size \citep{wang2021guarantees,bernacchia2020meta} and generalization \citep{chen2020closer,fallah2021generalization}, all without splitting model parameters.

\textbf{Gradient-based meta-learning and representation learning.} A growing line of research has endeavored to develop and understand gradient-based meta-learning with a representation learning perspective.
Besides ANIL, multiple other meta-learning methods fix the representation in the inner loop \citep{lee2019meta, bertinetto2018meta}.
\citet{goldblum2020unraveling} showed that these meta-learners
learn representations that empirically exhibit the desirable behavior of clustering features by class. However, they also gave evidence suggesting this is not true for MAML since it adapts the feature extractor during the inner loop. Meanwhile, other works have argued for the benefits of adapting the representation in the inner loop both experimentally, when the head is fixed \citep{oh2020boil}, and theoretically, when the task optimal solutions may not share a representation \citep{chua2021fine}.

Two recent works have argued that ANIL behaves similarly to empirically successful approaches for representation learning.
\citet{wang2021bridging} showed that the models learned by ANIL and multi-task learning with a shared representation and unique heads are close in function space for sufficiently wide and deep ReLU networks, when the inner loop learning rate and number of inner adaptation steps for ANIL is small.
 \citet{kao2022maml} noticed that ANIL with the global head initialized at zero at the start of each round is a ``noisy contrastive learner'' in the sense that outer loop update for the representation is a gradient step with respect to a contrastive loss, which suggests that ANIL should learn quality representations. Moreover, they showed that zeroing the global head at the start of each round  empirically improves the performance  of both ANIL and MAML.
However, neither work shows that ANIL, let alone MAML, can in fact learn expressive representations. Additionally, our analysis rigorously explains the observation from \citet{kao2022maml} that having small $\|\mathbf{w}_t\|_2$ aids representation learning.

\textbf{Meta-learning and task diversity.} Initial efforts to empirically understand the effects of meta-training task diversity on meta-learning performance with neural networks have shown a promising connection between the two, although the picture is not yet clear. \citet{ni2021data} and \citet{bouniot2020towards} made modifications to the the meta-training task distribution and the meta-learning objective, respectively, to improve the effective task diversity, and both resulted in significant improvements in performance for a range of meta-learners. On the other hand, \citet{setlur2020support} and  \citet{kumar2021effect} empirically argued that reducing the overall diversity of the meta-training dataset does not restrict meta-learning. However, \citet{setlur2020support} only considered reducing intra-task data diversity, not the diversity of the tasks themselves (as no classes were dropped from the meta-training dataset), and the results due to \citet{kumar2021effect} showed that reducing the overall number of tasks seen during meta-training hurts performance for most meta-learners, including MAML. 

\textbf{Multi-task linear representation learning.} Several works have studied a similar multi-task linear representation learning setting as ours \citep{saunshi2021representation,thekumparampil2021statistically,collins2021exploiting, du2020fewshot, tripuraneni2020provable, bullins2019generalize,maurer2016benefit, mcnamara2017risk}, but did not analyze MAML or ANIL. Moreover, multiple works have shown that task diversity is necessary to learn generalizable representations from a statistical perspective \citep{du2020fewshot, tripuraneni2020theory,xu2021representation,tripuraneni2020provable}. Our work complements these by showing the benefit of task diversity to gradient-based meta-learning methods from an optimization perspective.

\newpage

\section{General Lemmas}

First we define the following notations used throughout the proofs.
\begin{table}[h]
\begin{tabular}{|c|l|}
\hline
\textbf{Notation}                                                                                                                                & \textbf{Definition}                                \\ \hline
$\del_t$                                                    & $\mathbf{I}_k - \alpha \mathbf{B}_t^\top \mathbf{B}_t$                        \\ 
$\deld_t$                & $\mathbf{I}_d - \alpha \mathbf{B}_t \mathbf{B}_t^\top$ \\ 
$L_\ast$                                                                                      & $\max_{t\in [T]} \sigma_{\max}^{0.5}\left(\tfrac{1}{n}\sum_{i=1}^n \mathbf{w}_{\ast,t,i}\mathbf{w}_{\ast,t,i}^\top\right)\leq L_\ast$   \\ 
$\mu_\ast$                                                                                      & $0< \mu_\ast \leq \min_{t\in [T]} \sigma_{\min}^{0.5}\left(\tfrac{1}{n}\sum_{i=1}^n \mathbf{w}_{\ast,t,i}\mathbf{w}_{\ast,t,i}^\top\right)$   \\
$\eta_\ast$                                                                                      & $\max_{t\in [T]} \left\|\tfrac{1}{n}\sum_{i=1}^n \mathbf{w}_{\ast,t,i}\right\|_2 \leq \eta_\ast \leq L_\ast$   \\
$L_{\max}$                                                                                      & $\max_{t\in [T],i\in[n]}\|\mathbf{w}_{\ast,t,i}\|_2 \leq L_{\max}\leq c\sqrt{k} L_\ast$  for constant $c$ \\ 
$\kappa_\ast$                                                                                      & $\nicefrac{L_\ast}{\mu_\ast}$  \\ 
$\kappa_{\ast,\max}$                                                                                      & $\nicefrac{L_{\max}}{\mu_\ast}$  \\ \hline
\end{tabular}
\end{table}

Now we have the following general lemmas.
\begin{lemma} \label{lem:sigminE} Suppose Assumption \ref{assump:tasks_diverse_main} holds and for some $t$, $\dist_t^2 \leq \frac{1}{1-\tau} \dist_0^2$. Also, suppose $\|\del_s\|_2 \leq \tau$ for all $s\in[t]$. Then, for $E_0 \coloneqq 1 - \tau - \dist_0^2$,
\begin{align}
    \sigma_{\min}\left( \tfrac{1}{n}\sum_{i=1}^n\mathbf{B}_t^\top\mathbf{B}_\ast \mathbf{w}_{\ast,t,i}\mathbf{w}_{\ast,t,i}^\top  \mathbf{B}_\ast^\top \mathbf{B}_t \right) \geq \frac{E_0 \mu_\ast^2}{\alpha}
\end{align}
\end{lemma}
\begin{proof}
First note that since $\sigma_{\min}(\mathbf{A}_1 \mathbf{A}_2) \geq \sigma_{\min}(\mathbf{A}_1) \sigma_{\min}( \mathbf{A}_2)$ for any two square matrices $\mathbf{A}_1$, $\mathbf{A}_2$, we have
\begin{align}
    \sigma_{\min}\left( \mathbf{B}_t^\top\mathbf{B}_\ast\tfrac{1}{n}\sum_{i=1}^n \mathbf{w}_{\ast,t,i}\mathbf{w}_{\ast,t,i}^\top  \mathbf{B}_\ast^\top \mathbf{B}_t \right) &\geq \sigma^2_{\min}(\mathbf{B}_\ast^\top \mathbf{B}_t)\sigma_{\min}\left(\tfrac{1}{n}\sum_{i=1}^n \mathbf{w}_{\ast,t,i}\mathbf{w}_{\ast,t,i}^\top   \right)\nonumber \\
    &\geq \sigma_{\min}^2(\mathbf{B}_\ast^\top \mathbf{B}_t) \mu_\ast^2 \nonumber \\
    &\geq \sigma_{\min}^2(\mathbf{B}_\ast^\top \mathbf{\hat{B}}_t)\sigma_{\min}^2(\mathbf{{R}}_t)\mu_\ast^2 \nonumber \\
    &\geq \sigma_{\min}^2(\mathbf{B}_\ast^\top \mathbf{\hat{B}}_t)\tfrac{1 - \tau}{\alpha}\mu_\ast^2 \label{eqqq}
\end{align}
where $\mathbf{B}_t = \mathbf{\hat{B}}_t \mathbf{R}_t$ is the QR-factorization of $\mathbf{B}_t$. 
Next, observe that
\begin{align}
   \dist_t^2 \coloneqq  \|\mathbf{B}_{\ast,\perp}^\top \mathbf{\hat{B}}_t\|_2^2 &=     \|(\mathbf{I}_d - \mathbf{B}_{\ast}\mathbf{B}_{\ast}^\top) \mathbf{\hat{B}}_t\|_2^2 \nonumber \\
    &= \max_{\mathbf{u} \in \mathbb{R}^k: \|\mathbf{u}\|_2=1}   \mathbf{u}^\top\mathbf{\hat{B}}_t^\top (\mathbf{I}_d - \mathbf{B}_{\ast}\mathbf{B}_{\ast}^\top) (\mathbf{I}_d - \mathbf{B}_{\ast}\mathbf{B}_{\ast}^\top) \mathbf{\hat{B}}_t \mathbf{u} \nonumber\\
    &= \max_{\mathbf{u} \in \mathbb{R}^k: \|\mathbf{u}\|_2=1}   \mathbf{u}^\top\mathbf{\hat{B}}_t^\top (\mathbf{I}_d - \mathbf{B}_{\ast}\mathbf{B}_{\ast}^\top) \mathbf{\hat{B}}_t \mathbf{u} \nonumber\\
    &= \max_{\mathbf{u} \in \mathbb{R}^k: \|\mathbf{u}\|_2=1}   \mathbf{u}^\top(\mathbf{I}_k - \mathbf{\hat{B}}_t^\top\mathbf{B}_{\ast}\mathbf{B}_{\ast}^\top \mathbf{\hat{B}}_t )\mathbf{u} \nonumber\\
    &= 1 - \sigma_{\min}^2(\mathbf{B}_{\ast}^\top\mathbf{\hat{B}}_t) \nonumber \\
    \implies \sigma_{\min}^2(\mathbf{B}_{\ast}^\top\mathbf{\hat{B}}_t) &= 1 - \dist_t^2 \nonumber \\
    &\geq 1 - \tfrac{1}{1- \tau}\dist_0^2.
\end{align}
which gives the result after combining with \eqref{eqqq}.
\end{proof}

Note that all four algorithms considered (FO-ANIL, Exact ANIL, FO-MAML, Exact MAML) execute the same inner loop update procedure for the head. The following lemma characterizes the diversity of the inner loop-updated heads for all four algorithms, under some assumptions on the behavior of $\dist_t$ and $\|\del_t\|_2$ which we will show are indeed satisfied later.
\begin{lemma} \label{lem:sigmin}
Suppose Assumption \ref{assump:tasks_diverse_main} holds and that on some iteration $t$, FO-ANIL, Exact ANIL, FO-MAML, and Exact MAML satisfy $\dist_t^2 \leq \frac{1}{1-\tau} \dist_0^2$ and $\|\del_s\|_2 \leq \tau$ for all $s\in[t]$. Then the inner loop-updated heads on iteration $t$ satisfy:
\begin{align}
    \sigma_{\min}\left(\frac{1}{n}\sum_{i=1}^n \mathbf{w}_{t,i}\mathbf{w}_{t,i}^\top \right) &\geq \alpha E_0 \mu_\ast^2 - 2(1+\|\del_t\|_2)\sqrt{ \alpha}\|\del_t\|_2\|\mathbf{w}_t\|_2 \eta_{\ast}  \\
    \text{ and }\quad  \sigma_{\max}\left(\frac{1}{n}\sum_{i=1}^n \mathbf{w}_{t,i}\mathbf{w}_{t,i}^\top \right) &\leq (\|\del_t\|_2\|\mathbf{w}_t\|_2+ \sqrt{\alpha (1+\|\del_t\|_2)} L_\ast )^2. 
\end{align}
\end{lemma}
\begin{proof}
We first lower bound the minimum singular value. Observe that $\frac{1-\|\del_t\|_2}{\alpha} \leq \sigma_{\min}^2(\mathbf{B}_t)\leq \sigma_{\max}^2(\mathbf{B}_t) \leq \frac{1+\|\del_t\|_2}{\alpha}$ by Weyl's inequality. Next, by expanding each $\mathbf{w}_{t,i}$ we have 
\begin{align}
\sigma_{\min}\left(\frac{1}{n} \sum_{i=1}^n \mathbf{w}_{t,i} \mathbf{w}_{t,i}^\top\right)
&= \sigma_{\min}\Bigg(\frac{1}{n} \sum_{i=1}^n (\mathbf{I}_k - \alpha \mathbf{{B}}_t^\top \mathbf{B}_t)\mathbf{w}_{t} \mathbf{w}_{t}^\top  ( \mathbf{I}_k - \alpha\mathbf{{B}}_t^\top \mathbf{B}_t ) \nonumber \\
&\quad \quad \quad + \alpha (\mathbf{I}_k -\alpha\mathbf{{B}}_t^\top \mathbf{B}_t) \mathbf{w}_{t} \mathbf{w}_{\ast,t,i}^\top\mathbf{{B}}_\ast^\top \mathbf{{B}}_t  + \alpha \mathbf{{B}}_t^\top \mathbf{{B}}_\ast \mathbf{w}_{\ast,t,i}\mathbf{w}_{t}^\top (\mathbf{I}_k - \alpha \mathbf{{B}}_t^\top \mathbf{B}_t)  \nonumber \\
&\quad \quad \quad + \alpha^2 \mathbf{{B}}_t^\top \mathbf{{B}}_\ast \mathbf{w}_{\ast,t,i}\mathbf{w}_{\ast,t,i}^\top\mathbf{{B}}_\ast^\top \mathbf{{B}}_t  \Bigg) \nonumber \\
&\geq \sigma_{\min}\Bigg(\frac{1}{n} \sum_{i=1}^n \alpha^2 \mathbf{{B}}_t^\top \mathbf{{B}}_\ast \mathbf{w}_{\ast,t,i}\mathbf{w}_{\ast,t,i}^\top\mathbf{{B}}_\ast^\top \mathbf{{B}}_t  \Bigg) - 2\alpha  \left\| \del_t\mathbf{w}_{t}\frac{1}{n}\sum_{i=1}^n \mathbf{w}_{\ast,t,i}^\top\mathbf{{B}}_\ast^\top \mathbf{{B}}_t \right\|_2 \label{weyl0} \\
&\geq  \alpha E_0 \mu_\ast^2  - 2\alpha  \left\| \del_t\mathbf{w}_{t}\frac{1}{n}\sum_{i=1}^n \mathbf{w}_{\ast,t,i}\mathbf{{B}}_\ast^\top \mathbf{{B}}_t \right\|_2   \label{weyl1} \\
&\geq  \alpha E_0 \mu_\ast^2  - 2(1+\|\del_t\|_2)\sqrt{ \alpha}\|\del_t\|_2\|\mathbf{w}_t\|_2 \eta_{\ast} \label{weyl3}  
\end{align}
where \eqref{weyl0} follows by Weyl's inequality and the fact that $\mathbf{{B}}_t^\top \mathbf{B}_t\mathbf{w}_t\mathbf{w}_t^\top\mathbf{{B}}_t^\top \mathbf{B}_t \succeq \mathbf{0}$, \eqref{weyl1} follows by 
Lemma \ref{lem:sigminE}, and \eqref{weyl3} follows by the Cauchy-Schwarz inequality.


Now we upper bound the maximum singular value of $\frac{1}{n}\sum_{i=1}^n \mathbf{w}_{t,i}\mathbf{w}_{t,i}^\top$.
We have
\begin{align}
    \sigma_{\max}\left(\frac{1}{n}\sum_{i=1}^n \mathbf{w}_{t,i}\mathbf{w}_{t,i}^\top \right) &\leq \left\|(\mathbf{I}_k -\alpha \mathbf{B}_t^\top \mathbf{B}_t)\mathbf{w}_{t}\mathbf{w}_{t}^\top (\mathbf{I}_k -\alpha \mathbf{B}_t^\top \mathbf{B}_t) \right\|_2  \nonumber \\
    &\quad + 2 \alpha \left\|(\mathbf{I}_k -\alpha \mathbf{B}_t^\top \mathbf{B}_t)\mathbf{w}_{t}\right\|_2 \left\|\frac{1}{n}\sum_{i=1}^n \mathbf{B}_t^\top \mathbf{B}_\ast \mathbf{w}_{\ast,t,i} \right\|_2  + \alpha^2 \left\|\frac{1}{n}\sum_{i=1}^n \mathbf{B}_t^\top \mathbf{B}_\ast \mathbf{w}_{\ast,t,i} \right\|_2^2 \nonumber \\
    &\leq \|\del_t\|_2^2 \|\mathbf{w}_t\|_2^2 + 2  \|\del_t\|_2  \|\mathbf{w}_t\|_2\sqrt{ \alpha(1+\|\del_t\|_2)} \eta_{\ast} + \alpha (1+\|\del_t\|_2) L_\ast^2 \nonumber \\
    &\leq  (\|\del_t\|_2\|\mathbf{w}_t\|_2+ \sqrt{\alpha(1+\|\del_t\|_2)} L_\ast )^2.
\end{align}
\end{proof}

\begin{lemma} \label{lem:gen_w}
Suppose the sequence $\{\mathbf{w}_s\}_{s=0}^{t+1}$ satisfies:
\begin{align}
    \|\mathbf{w}_0\|_2 &= 0, \nonumber \\
    \|\mathbf{w}_{s+1}\|_2 &\leq (1+\xi_{1,s})\|\mathbf{w}_{s}\|_2 + \xi_{2,s}
\end{align}
where $\xi_{1,s}\geq 0$, $\xi_{2,s}\geq 0$ for all $s\in[t]$ and  $\sum_{s=1}^t \xi_{1,s}\leq 1$. Then:
\begin{align}
   \|\mathbf{w}_{t+1}\|_2 &\leq \sum_{s=1}^t \xi_{2,s} \left(1+ 2\sum_{r=s}^t \xi_{1,r}\right)
\end{align}
\end{lemma}
\begin{proof}
We have 
\begin{align}
   \| \mathbf{w}_{t+1}\|_2 &\leq (1+\xi_{1,t})\| \mathbf{w}_{t}\|_2 + \xi_{2,t} \nonumber \\
   &\leq (1+\xi_{1,t})^2\| \mathbf{w}_{t-1}\|_2 + \xi_{2,t}(1-\xi_{1,t}) + \xi_{2,t} \nonumber \\
   &\quad\vdots \nonumber \\
   &\leq \| \mathbf{w}_{0}\|_2\prod_{s=1}^{t}(1+\xi_{1,s}) + \sum_{s=1}^t\xi_{2,s}\prod_{r=s}^{t-1}(1+\xi_{1,r}) \nonumber \\
   &= \sum_{s=1}^t\xi_{2,s}\prod_{r=s}^{t-1}(1+\xi_{1,r}) \label{000} \\
   &\leq \sum_{s=1}^t \xi_{2,s} \left(1+\frac{1}{t-s}\sum_{r=s}^t \xi_{1,r}\right)^{{t-s}} \label{amgm} 
   \end{align}
 where \eqref{000} is due to $\|\mathbf{w}_0\|_2=0$ and \eqref{amgm} follows from the AM-GM inequality. Next, note that $\left(1+\frac{1}{t-s}\sum_{r=s}^t \xi_{1,r}\right)^{{t-s}}$ is of the form $\left(1+\frac{a}{x}\right)^{x}$, where $x=t-s$ and $a= \sum_{r=s}^t \xi_{1,r}$. Since $\left(1+\frac{a}{x}\right)^{x}$ is an increasing function of $x$, we can upper bound it by its limit as $x\rightarrow \infty$, which is $e^a$. Thus we have
   \begin{align}
   \| \mathbf{w}_{t+1}\|_2 &\leq 
   \sum_{s=1}^t \xi_{2,s}\exp \left(\sum_{r=s}^t \xi_{1,r}\right) \nonumber \\
   &\leq \sum_{s=1}^t \xi_{2,s} \left(1+ 2\sum_{r=s}^t \xi_{1,r}\right) \label{12x}
\end{align}
where  \eqref{12x} follows from the numerical inequality $\exp(x) \leq 1+2x$ for all $x\in [0,1]$.
\end{proof}

\begin{lemma} \label{lem:gen_del}
Suppose that $\mathbf{B}_{t+1} = \mathbf{B}_t - \beta \mathbf{G}_t$ and 
\begin{align}
    \mathbf{G}_t = -\deld_t \mathbf{S}_t\mathbf{B}_t - \chi \mathbf{S}_t \mathbf{B}_t \del_t + \mathbf{N}_{t}
\end{align}
for $ \mathbf{N}_{t}\in \mathbb{R}^{d\times k}$ and a positive semi-definite matrix $\mathbf{S}_t \in \mathbb{R}^{k\times k}$. Then
\begin{align}
    \|\del_{t+1}\|_2 \leq \|\del_t\|_2 \left(1- (1+\chi) \beta \alpha\sigma_{\min}( \mathbf{B}_t^\top \mathbf{S}_{t}\mathbf{B}_t) \right) + 2 \beta \alpha \|\mathbf{B}_t^\top \mathbf{N}_{t}\|_2 +\beta^2 \alpha \|\mathbf{G}_t\|_2^2 
\end{align}
\end{lemma}
 \begin{proof}
  By expanding $\del_{t+1}$, $\mathbf{B}_{t+1}$, and $\mathbf{G}_t$, we obtain
    \begin{align}
      \del_{t+1}&= \mathbf{I}- \alpha \mathbf{B}_{t+1}^\top \mathbf{B}_{t+1} \nonumber \\
      &= \mathbf{I}- \alpha \mathbf{B}_{t}^\top \mathbf{B}_t + \beta \alpha \mathbf{B}^\top \mathbf{G}_t + \beta\alpha \mathbf{G}_t^\top \mathbf{B}_t -\beta^2 \alpha  \mathbf{G}_t^\top \mathbf{G}_t \nonumber \\
      &=  \del_t - \beta \alpha \del_{t} \mathbf{B}_t^\top \mathbf{S}_{t}\mathbf{B}_t  - \chi \beta \alpha\mathbf{B}_t^\top \mathbf{S}_{t}\mathbf{B}_t \del_{t} + \beta \alpha \mathbf{B}_t^\top \mathbf{N}_{t} \nonumber \\
      &\quad - \chi \beta\alpha  \mathbf{B}_t^\top \mathbf{S}_{t}\mathbf{B}_t \del_{t} - \beta\alpha \del_{t}  \mathbf{B}_t^\top \mathbf{S}_{t}\mathbf{B}_t + \beta \alpha\mathbf{N}_{t}^\top\mathbf{B}_t 
      -\beta^2 \alpha  \mathbf{G}_t^\top \mathbf{G}_t \\
      &= \tfrac{1}{2}\del_t\left(\mathbf{I}_k - (1+\chi)\beta \alpha \mathbf{B}_t^\top \mathbf{S}_{t}\mathbf{B}_t  \right) \nonumber \\
      &\quad + \tfrac{1}{2}\left(\mathbf{I}_k - (1+\chi) \beta \alpha \mathbf{B}_t^\top \mathbf{S}_{t}\mathbf{B}_t \right)\del_t + \beta \alpha (\mathbf{B}_t^\top \mathbf{N}_{t} +  \mathbf{N}_{t}^\top \mathbf{B}_t )  -\beta^2 \alpha  \mathbf{G}_t^\top \mathbf{G}_t
  \end{align}
 Therefore,
 \begin{align}
\|\del_{t+1}\|_2 &\leq \|\del_t\|_2 \left\|\mathbf{I}_k - (1+\chi) \beta \alpha \mathbf{B}_t^\top \mathbf{S}_{t}\mathbf{B}_t \right\|_2 + 2 \beta \alpha \|\mathbf{B}_t^\top \mathbf{N}_{t}\|_2 +\beta^2 \alpha \|\mathbf{G}_t\|_2^2 \nonumber \\
&\leq \|\del_t\|_2 \left(1- (1+\chi) \beta \alpha\sigma_{\min}( \mathbf{B}_t^\top \mathbf{S}_{t}\mathbf{B}_t) \right) + 2 \beta \alpha \|\mathbf{B}_t^\top \mathbf{N}_{t}\|_2 +\beta^2 \alpha \|\mathbf{G}_t\|_2^2 
 \end{align}
  where the last inequality follows by the triangle and Weyl inequalities.
 \end{proof}

\section{ANIL Infinite Samples} \label{app:anil}

We start by considering the infinite sample case, wherein $m_{in}\!=\!m_{out}\!=\!\infty$. 
Let $E_0 \coloneqq 0.9 -\dist_0^2$.
We restate Theorem \ref{thm:anil_pop_main} here with full constants.
\begin{thm}[ANIL Infinite Samples]
Let $m_{in}\!=\!m_{out}\!= \!\infty$ and define 
 $E_0\!\coloneqq\! 0.9 -\! \dist_0^2$. Suppose Assumption \ref{assump:tasks_diverse_main} holds and $\dist_0 \leq \sqrt{0.9}$. 
Let $\alpha < \frac{1}{L_{\ast}}$, $\alpha \mathbf{B}_0^\top \mathbf{B}_0 = \mathbf{I}_k$ and $\mathbf{w}_t = \mathbf{0}$. 
Then FO-ANIL with $\beta\leq \frac{\alpha E_0^3 \mu_\ast}{180 \kappa_{\ast}^4}\min(1, \tfrac{\mu_\ast^2}{\eta_\ast^2})$ and Exact ANIL with $\beta \leq \frac{\alpha E_0^2 }{40 \kappa_\ast^4}$ both satisfy that after $T$ iterations,
\begin{align}
    \dist(\mathbf{{B}}_T, \mathbf{{B}}_\ast) \leq  \left(1 - 0.5\beta \alpha E_0 \mu_\ast^2   \right)^{T-1}
\end{align}
\end{thm}
\begin{proof}
 
The proof uses an inductive argument with the following six inductive hypotheses:
 \begin{enumerate}
\item $A_1(t) \coloneqq \{\|\mathbf{w}_{t}\|_2 \leq \frac{\sqrt{\alpha} E_0 }{10} \min(1, \tfrac{\mu_\ast^2}{\eta_\ast^2}) \eta_\ast\}$ 
 \item 
$A_2(t) \coloneqq \{\| \del_t \|_2 \leq (1 - 0.5 \beta \alpha E_0 \mu_\ast^2 )\|\del_{t-1} \|_2   + \tfrac{5}{4}\alpha^2 \beta^2 L_\ast^4 \dist_{t-1}^2\},$
\item $A_3(t) \coloneqq \{\| \del_t\|_2 \leq \frac{1}{10}\}$,
\item $A_4(t) \coloneqq \{0.9 \alpha E_0 \mu_\ast^2 \mathbf{I}_k \preceq  \tfrac{1}{n}\sum_{i=1}^n\mathbf{w}_{t,i}\mathbf{w}_{t,i}^\top \preceq 1.2 \alpha L_\ast^2 \mathbf{I}_k \}$,
     \item 
      $A_5(t) \coloneqq \{ \|\mathbf{{B}}_{\ast,\perp}^\top \mathbf{B}_{t}\|_2 \leq \left(1 -  0.5 \beta\alpha E_0 \mu_\ast^2  \right) \|\mathbf{{B}}_{\ast,\perp}^\top \mathbf{B}_{t-1}\|_2\}$,
      \item$A_6(t) \coloneqq \{ \dist_{t} \leq \left(1 - 0.5 \beta  \alpha E_0 \mu_\ast^2  \right)^t \}$.
 \end{enumerate}
These conditions hold for iteration $t=0$ due to the choice of initialization $(\mathbf{B}_0,\mathbf{w}_0)$ satisfying $\mathbf{I}_k - \alpha \mathbf{B}_0^\top \mathbf{B}_0 =\mathbf{0}$ and $\mathbf{w}_0=\mathbf{0}$. We will show that if they hold for all iterations up to and including  iteration $t$ for an arbitrary $t$, then they hold at iteration $t+1$. 

\begin{enumerate}
\item $\bigcap_{s=0}^t \{A_2(s) \cap A_6(s)\}   \implies A_1(t+1)$. This is Lemma \ref{lem:w_anil_fss} for FO-ANIL and Lemma \ref{lem:ea_pop_a1} for Exact ANIL.

\item $A_1(t) \cap A_3(t) \cap A_5(t)  \implies A_2(t+1)$. This is Lemma \ref{lem:reg} for FO-ANIL and Lemma \ref{lem:ea_pop_a} for Exact ANIL.

\item $A_2(t\!+\!1) \cap A_3(t)  \implies A_3(t+1)$.
This is Corollary \ref{cor:reg} for FO-ANIL and Corollary \ref{cor:reg_exact} for Exact ANIL.

\item $A_1(t+1)\cap A_3(t+1) \cap A_6(t+1)  \implies A_4(t+1)$.
This is Lemma \ref{lem:foanil_pop_diverse} for FO-ANIL and Lemma \ref{lem:ea_pop_a4} for Exact ANIL.

\item FO-ANIL: $A_4(t) \implies A_5(t+1)$. This is Lemma  \ref{lem:contractt}.

Exact ANIL: $A_1(t)\cap A_3(t) \cap A_4(t) \implies A_5(t+1)$. This is Lemma  \ref{lem:exactanil_pop_diverse}. The slight discrepancy here in the implications is due to the extra terms in the outer loop representation update for Exact ANIL.

\item $A_3(t\!+\!1) \cap\big\{ \bigcap_{s=0}^{t+1} A_5(s)\big\} \implies A_6(t+1)$. Recall $\dist_{t+1}=\|\mathbf{B}_{\ast,\perp}^\top \mathbf{\hat{B}}_{t+1}\|_2 $ where $\mathbf{\hat{B}}_{t+1}$ is the orthogonal matrix resulting from the QR factorization of $\mathbf{B}_{t+1}$, i.e. $\mathbf{B}_{t+1} = \mathbf{\hat{B}}_{t+1}\mathbf{R}_{t+1}$ for an upper triangular matrix $\mathbf{R}_{t+1}$.
By $A_3(t+1)$ and $\cap_{s=0}^{t+1} A_5(s)$ we have
\begin{align}
\tfrac{\sqrt{1-\|\del_{t+1}\|_2}}{\sqrt{\alpha}} \dist_{t+1}  &= \tfrac{\sqrt{1-\|\del_{t+1}\|_2}}{\sqrt{\alpha}}\|\mathbf{{B}}_{\ast,\perp}^\top \mathbf{{B}}_{t+1}\|_2 \nonumber \\
&\leq \sigma_{\min}(\mathbf{{B}}_{t+1}) \|\mathbf{{B}}_{\ast,\perp}^\top \mathbf{{B}}_{t+1}\|_2\;  \nonumber \\
 &\leq \|\mathbf{{B}}_{\ast,\perp}^\top \mathbf{B}_{t+1}\|_2   \nonumber \\
    &\leq \left(1 - 0.5 \beta\alpha E_0 \mu_\ast^2 \right)^t  \|\mathbf{{B}}_{\ast,\perp}^\top \mathbf{B}_{0}\|_2 
    \nonumber \\
    &\leq \tfrac{1}{\sqrt{\alpha}}\left(1  - 0.5 \beta  \alpha E_0 \mu_\ast^2 \right)^t \|\mathbf{{B}}_{\ast,\perp}^\top \mathbf{{B}}_{0}\|_2 \nonumber \\
    &= \tfrac{1}{\sqrt{\alpha}}\left(1 - 0.5 \beta \alpha E_0 \mu_\ast^2 \right)^t \dist_0. \nonumber 
\end{align}
Dividing both sides by $\tfrac{\sqrt{1-\|\del_{t+1}\|_2}}{\sqrt{\alpha}}$ and using the facts that $\dist_0 \leq \tfrac{3}{\sqrt{10}}$ and $\|\del_{t+1}\|_2 \leq \tfrac{1}{10}$ yields
\begin{align}
    \dist_{t+1}  &\leq \tfrac{1}{\sqrt{1-\|\del_{t+1}\|_2}}\left(1  - 0.5 \beta  \alpha E_0 \mu_\ast^2  \right)^t \dist_0 \nonumber \\
    &\leq \tfrac{\sqrt{10}}{3}\left(1  - 0.5 \beta  \alpha E_0 \mu_\ast^2  \right)^t \dist_0 \nonumber \\
    &\leq \left(1  - 0.5 \beta  \alpha E_0 \mu_\ast^2 \right)^t,
\end{align}
as desired.
\end{enumerate}
\end{proof}

\subsection{FO-ANIL}

First note that the inner loop updates for FO-ANIL can be written as:
\begin{align}
    \mathbf{w}_{t,i} 
    &= \mathbf{w}_t - \alpha \nabla_{\mathbf{w}} \mathcal{L}_{t,i}(\mathbf{B}_t, \mathbf{w}_t) \nonumber \\
    &=  (\mathbf{I}_k- \alpha\mathbf{{B}}_t^\top \mathbf{{B}}_t)\mathbf{{w}}_t + \alpha \mathbf{{B}}_t^\top \mathbf{{B}}_\ast\mathbf{{w}}_{\ast,t,i}, \label{upd_head_anil_inf}
\end{align}
while the outer loop updates for the head and representation are:
\begin{align}
\mathbf{w}_{t+1} &= \mathbf{w}_t - \frac{\beta}{n}\sum_{i=1}^n \nabla_{\mathbf{w}} \mathcal{L}_{t,i}(\mathbf{{B}}_t, \mathbf{w}_{t,i}) \nonumber \\
&=  \mathbf{{w}}_{t} - \frac{\beta}{n}\sum_{i=1}^{n}\mathbf{{B}}_t^\top \mathbf{{B}}_t\mathbf{{w}}_{t,i} + \frac{\beta}{n}\sum_{i=1}^n \mathbf{{B}}_t^\top  \mathbf{{B}}_\ast\mathbf{w}_{\ast,t,i} \label{upd_head_out_anil_inf} \\
    \mathbf{{B}}_{t+1} &= \mathbf{{B}}_t  - \frac{\beta}{n}\sum_{i=1}^n \nabla_{\mathbf{B}} \mathcal{L}_{t,i}(\mathbf{{B}}_t, \mathbf{w}_{t,i}) \nonumber \\
    &= \mathbf{{B}}_t  - \frac{\beta}{n}\sum_{i=1}^n  \mathbf{{B}}_t  \mathbf{w}_{t,i}\mathbf{w}_{t,i}^\top + \frac{\beta}{n } \sum_{i=1}^n  \mathbf{{B}}_\ast\mathbf{w}_{\ast,t,i} \mathbf{w}_{t,i}^\top 
    \label{upd_rep_anil_inf} \\
    &= \mathbf{{B}}_t  - \beta(\mathbf{I}_d - \alpha \mathbf{{B}}_t \mathbf{{B}}_t^\top)\frac{1}{n}\sum_{i=1}^n(  \mathbf{{B}}_t  \mathbf{w}_{t} -  \mathbf{{B}}_\ast\mathbf{w}_{\ast,t,i} )\mathbf{w}_{t,i}^\top 
\end{align}

\begin{lemma}[FO-ANIL $A_1(t+1)$]\label{lem:w_anil_fss}
Suppose 
we are in the setting of Theorem \ref{thm:anil_pop_main},  and that the events $A_2(s)$ and $A_6(s)$
hold for all $s\in[t]$.
Then
\begin{align}
    \|\mathbf{w}_{t+1}\|_2 \leq \tfrac{1}{10} {\sqrt{\alpha}}E_0 \min(1, \tfrac{\mu_\ast^2}{\eta_\ast^2})\eta_\ast. 
\end{align}
\end{lemma}

\begin{proof}
For all $s=1,\dots,t$, the outer loop updates for FO-ANIL can be written as:
\begin{align}
\mathbf{w}_{s+1} &= \mathbf{w}_s - \frac{\beta}{n}\sum_{i=1}^n \nabla_{\mathbf{w}} \mathcal{L}_{s,i}(\mathbf{{B}}_s, \mathbf{w}_{s,i}) = \mathbf{w}_{s} -  \frac{ \beta}{n}\sum_{i=1}^n \mathbf{B}_s^\top \mathbf{B}_s \mathbf{w}_{s,i} + \frac{\beta}{n}\sum_{i=1}^n \mathbf{{B}}_s^\top \mathbf{{B}}_\ast \mathbf{w}_{\ast,s,i} \label{upd_head_anil_out_pop}
\end{align}
Substituting the definition of $\mathbf{w}_{s,i}$, we have \begin{align}
\mathbf{w}_{s+1}&= \mathbf{w}_{s}\! - \! \beta \mathbf{B}_s^\top \mathbf{B}_s (\mathbf{I}- \alpha\mathbf{B}_s^\top \mathbf{B}_s)\mathbf{w}_s - \frac{\alpha \beta}{n}\sum_{i=1}^n \mathbf{B}_s^\top \mathbf{B}_s\mathbf{{B}}_s^\top \mathbf{{B}}_\ast \mathbf{w}_{\ast,s,i} + \frac{\beta}{n}\sum_{i=1}^n \mathbf{{B}}_s^\top \mathbf{{B}}_\ast \mathbf{w}_{\ast,s,i}  \nonumber \\
&= (\mathbf{I}_{k}\! - \! \beta (\mathbf{I}- \alpha\mathbf{B}_s^\top \mathbf{B}_s) \mathbf{B}_s^\top \mathbf{B}_s)\mathbf{w}_s
+
\beta (\mathbf{I}- \alpha\mathbf{B}_s^\top \mathbf{B}_s)\mathbf{{B}}_s^\top \mathbf{{B}}_\ast \frac{1}{n}\sum_{i=1}^n\mathbf{w}_{\ast,s,i}  \label{uuu}
\end{align}
Note that $\bigcup_{s=0}^t A_3(s)$ implies $\sigma_{\max}(\mathbf{B}_s^\top \mathbf{B}_s) \leq \frac{1+\|\del_s\|_2}{\alpha}< \frac{1.1}{\alpha}$ for all $s\in\{0,\dots,t\!+\!1\}$. Let $c\coloneqq 1.1$. Using $\sigma_{\max}(\mathbf{B}_s^\top \mathbf{B}_s) \leq \frac{c}{\alpha}$ with \eqref{uuu}, we obtain
\begin{align}
    \|\mathbf{w}_{s+1}\|_2 &\leq (1 + \tfrac{c\beta}{\alpha} \|\del_s\|_2 )\|\mathbf{w}_{s}\|_2 + \tfrac{c\beta}{\sqrt{\alpha}} \|\del_s\|_2 \eta_\ast 
    \label{forall}
\end{align}  
for all $s\in\{0,\dots,t\}$.
Therefore, by applying Lemma \ref{lem:gen_w} with $\xi_{1,s} = \tfrac{c\beta}{\alpha} \|\del_s\|_2$ and $\xi_{2,s} = \tfrac{c\beta}{\sqrt{\alpha}} \|\del_s\|_2 \eta_\ast$, we have
\begin{align}
    \|\mathbf{w}_{t+1}\|_2 &\leq \sum_{s=1}^t \tfrac{c\beta}{\sqrt{\alpha}} \|\del_s\|_2 \eta_\ast\left( 1 + 2c \sum_{r=s}^t \tfrac{\beta}{\alpha} \|\del_r\|_2 \right)
\end{align}
Next, let $\rho \coloneqq 
   1 - 0.5 \beta \alpha E_0 \mu_\ast^2 $.
   By $\bigcup_{s=0}^t A_2(s)$, we have for any $s\in [t]$
\begin{align}
    \|\del_{s}\|_2 &\leq \rho\|\del_{s-1}\|_2 + \tfrac{5}{4} \alpha^{2}\beta^2 L_\ast^4 \dist_{s-1}^2 \nonumber \\
    &\leq \rho^2\|\del_{s-2}\|_2 + \tfrac{5}{4}\alpha^{2}\beta^2\rho L_\ast^4 \dist_{s-2}^2 +  \tfrac{5}{4}\alpha^{2}\beta^2L_\ast^4 \dist_{s-1}^2 \nonumber \\
    &\;\;\vdots \nonumber \\
    &\leq  \rho^{s}\|\del_{0}\|_2 + \tfrac{5}{4} \alpha^{2}\beta^2 L_\ast^4  \sum_{r=0}^{s-1} \rho^{s-1-r} \dist_{r}^2  \nonumber \\
    &= \tfrac{5}{4}\alpha^{2}\beta^2 L_\ast^4  \sum_{r=0}^{s-1} \rho^{s-1-r} \dist_{r}^2
\end{align}
since $\|\del_{0}\|_2 = 0$ by choice of initialization.
Next, we have that $\dist_s \leq \rho^s$ for all $s\in \{0,...,t\}$ by $\bigcup_{s=0}^t A_5(s)$. 
Thus, for any $s\in \{0,...,t\}$, we can further bound $\|\del_{s}\|_2$ as
\begin{align}
    \|\del_{s}\|_2 &\leq \tfrac{5}{4} \alpha^{2}\beta^2 L_\ast^4 \sum_{r=0}^{s-1} \rho^{s-1-r} \rho^{2r} \nonumber \\
    &= \tfrac{5}{4} \alpha^{2}\beta^2 L_\ast^4  \rho^{s-1} \sum_{r=0}^{s-1} \rho^r \nonumber \\
    &\leq  \rho^{s-1} \tfrac{5\alpha^{2}\beta^2 L_\ast^4 }{4(1-\rho)} \nonumber \\
    &\leq  \rho^{s-1} \tfrac{5\beta \alpha L_\ast^4 }{2 E_0 \mu_\ast^2} ,
\end{align}
which means that
\begin{align}
   \|\mathbf{w}_{t+1}\|_2 &\leq \sum_{s=1}^t \tfrac{c\beta}{\sqrt{\alpha}}  \rho^{s-1} \tfrac{5\beta \alpha L_\ast^4 }{2 E_0 \mu_\ast^2} \eta_\ast\left( 1 + 2c \sum_{r=s}^t \tfrac{\beta}{\alpha}  \rho^{r-1} \tfrac{5\beta \alpha L_\ast^4 }{2 E_0 \mu_\ast^2} \right) \nonumber \\
   &\leq 2.5c \beta^2 \sqrt{\alpha} \frac{ L_\ast^4 \eta_\ast}{E_0 \mu_\ast^2} \sum_{s=1}^t \rho^{s-1} \left(1 + 5c \beta^2\frac{L_\ast^4}{E_0 \mu_\ast^2} \sum_{r=s}^t \rho^{r-1} \right) \nonumber \\
   &\leq 2.5c \beta^2 \sqrt{\alpha} \frac{ L_\ast^4 \eta_\ast}{E_0 \mu_\ast^2} \sum_{s=1}^t \rho^{s-1} \left(1 + 6 \beta^2\frac{L_\ast^4 \rho^{s}}{E_0 \mu_\ast^2 (1-\rho)} \right) \nonumber \\
   &\leq 3 \beta^2 \sqrt{\alpha} \frac{ L_\ast^4 \eta_\ast}{E_0 \mu_\ast^2} \sum_{s=1}^t \rho^{s-1} \left(1 + 12 \frac{\beta L_\ast^4 }{\alpha E_0^2 \mu_\ast^4} \right) \nonumber \\
     &\leq 6 \beta^2 \sqrt{\alpha} \frac{ L_\ast^4 \eta_\ast}{E_0 \mu_\ast^2} \sum_{s=1}^t 1.5\rho^{s-1}  \label{twwo} \\
   &\leq 18  \frac{\beta \kappa_\ast^4 \eta_\ast}{\sqrt{\alpha}E_0^2}  \nonumber \\
   &\leq \tfrac{1}{10}\sqrt{\alpha}E_0 \min(1, \tfrac{\mu_\ast^2}{\eta_\ast^2}) \eta_\ast 
   \label{twwwo}
\end{align}
where \eqref{twwo} and \eqref{twwwo} follow since $\beta \leq \frac{\alpha E_0^3}{180 \kappa_{\ast}^4} \min(1, \tfrac{\mu_\ast^2}{\eta_\ast^2}) \eta_\ast$.     
\end{proof}

\begin{rem}
As referred to in Section \ref{sec:sketch}, it is not necessary to start with $\|\mathbf{w}_0\|_2$ and $\|\del_0\|_2$ strictly equal to zero. Precisely, it can be shown that the above lemma still holds with $\|\mathbf{w}_0\|_2 \leq c \sqrt{\alpha}E_0 \min(1,\tfrac{\mu_\ast^2}{\eta_\ast^2})\eta_\ast$ and $\|\del_0\|_2 \leq c\tfrac{\beta \alpha L_\ast^4}{E_0\mu_\ast^2}$ for a sufficiently small absolute constant $c$. Inductive hypothesis $A_6(t+1)$ would also continue to hold under this initialization, with a slightly different constant. These are the only times we use $\|\mathbf{w}_0\|_2 = \|\del_0\|_2 = 0$, so the rest of the proof would hold. Similar statements can be made regarding the rest of the algorithms.
\end{rem}

\begin{lemma} [FO-ANIL $A_2(t+1)$] \label{lem:reg}
Suppose we are in the setting of Theorem \ref{thm:anil_pop_main}
and that  $A_1(t), A_3(t), A_6(t)$ hold. 
 Then $A_2(t+1)$ holds, i.e.
\begin{align}
    \|\del_{t+1}  \|_2 
    &\leq \left(
   1 - 0.5\beta \alpha E_0 \mu_\ast^2 \right)\|\del_{t}  \|_2 
     +\tfrac{5}{4} \beta^2 \alpha^2 L_\ast^4 \dist_t^2. \label{res:control_spectrum_tight}
\end{align}
\end{lemma}
\begin{proof}
Let $\mathbf{G}_t$ be the outer loop gradient for the representation, i.e. $\mathbf{G}_t = \frac{1}{\beta}(\mathbf{B}_t- \mathbf{B}_{t+1})$. 
We aim to apply Lemma \ref{lem:gen_del}, we write $\mathbf{G}_t$ as $-\deld_t \mathbf{S}_t \mathbf{B}_t + \mathbf{N}_t$, for some positive definite matrix $\mathbf{S}_t$ and another matrix $\mathbf{N}_t$. We have
\begin{align}
  \mathbf{G}_t &= \frac{1}{n}\sum_{i=1}^n (\mathbf{B}_t \mathbf{w}_{t,i} - \mathbf{B}_\ast \mathbf{w}_{\ast,t,i})\mathbf{w}_{t,i}^\top \nonumber \\
  &= \frac{1}{n}\sum_{i=1}^n (\mathbf{B}_t \del_t \mathbf{w}_{t} + \alpha \mathbf{B}_t \mathbf{B}_t^\top \mathbf{B}_\ast \mathbf{w}_{\ast,t,i} - \mathbf{B}_\ast \mathbf{w}_{\ast,t,i})\mathbf{w}_{t,i}^\top \nonumber \\
  &= \frac{1}{n}\sum_{i=1}^n \deld_t(\mathbf{B}_t \mathbf{w}_t - \mathbf{B}_{\ast}\mathbf{w}_{\ast,t,i})\mathbf{w}_{t,i}^\top \nonumber \\
  &= - \alpha \deld_t \mathbf{B}_{\ast}\left( \frac{1}{n}\sum_{i=1}^n\mathbf{w}_{\ast,t,i}\mathbf{w}_{\ast,t,i}^\top \right) \mathbf{B}_{\ast}^\top \mathbf{B}_t + 
  \frac{1}{n}\sum_{i=1}^n \deld_t \mathbf{B}_t \mathbf{w}_t \mathbf{w}_{t,i}^\top - \frac{1}{n}\sum_{i=1}^n \deld_t \mathbf{B}_{\ast}\mathbf{w}_{\ast,t,i} \mathbf{w}_t^\top \del_t \label{expand} \\
  &= -\deld_t \mathbf{S}_t \mathbf{B}_t + \mathbf{N}_t \nonumber  
\end{align}
where \eqref{expand} follows by expanding $\mathbf{w}_{t,i}$, and $\mathbf{S}_t = \alpha \mathbf{B}_{\ast}\left( \frac{1}{n}\sum_{i=1}^n\mathbf{w}_{\ast,t,i}\mathbf{w}_{\ast,t,i}^\top \right) \mathbf{B}_{\ast}^\top$ and \\
$\mathbf{N}_t = \frac{1}{n}\sum_{i=1}^n \deld_t \mathbf{B}_t \mathbf{w}_t \mathbf{w}_{t,i}^\top - \frac{1}{n}\sum_{i=1}^n \deld_t \mathbf{B}_{\ast}\mathbf{w}_{\ast,t,i} \mathbf{w}_t^\top \del_t$. Since $\sigma_{\min}(\mathbf{B}_t^\top \mathbf{S}_t \mathbf{B}_t) \geq  E_0 \mu_\ast^2$ (by Lemma \ref{lem:sigminE}), we have by Lemma \ref{lem:gen_del}
\begin{align}
    \|\del_{t+1}\|_2 &\leq (1 - \beta \alpha E_0 \mu_\ast^2)\|\del_t\|_2 + 2 \beta \alpha \|\mathbf{B}_t^\top \mathbf{N}_t\|_2 + \beta^2 \alpha \|\mathbf{G}_t\|_2^2
\end{align}
To bound $\|\mathbf{B}_t^\top \mathbf{N}_t\|_2$, we have
\begin{align}
  \|\mathbf{B}_t^\top \mathbf{N}_t\|_2 &= \left\|\frac{1}{n}\sum_{i=1}^n \del_t \mathbf{B}_t^\top \mathbf{B}_t \mathbf{w}_t \mathbf{w}_{t,i}^\top - \frac{1}{n}\sum_{i=1}^n \del_t \mathbf{B}_t^\top  \mathbf{B}_{\ast}\mathbf{w}_{\ast,t,i} \mathbf{w}_t^\top \del_t\right\|_2 \nonumber \\
   &\leq \left\|\del_t \mathbf{B}_t^\top \mathbf{B}_t \mathbf{w}_t \mathbf{w}_t^\top \del_t \right\|_2 + \alpha \left\|\frac{1}{n}\sum_{i=1}^n\del_t \mathbf{B}_t^\top \mathbf{B}_t \mathbf{w}_t \mathbf{w}_{\ast,t,i}^\top \mathbf{B}_\ast^\top \mathbf{B}_t \right\|_2 \nonumber \\
    &\quad + \left\|\frac{1}{n}\sum_{i=1}^n \del_t \mathbf{B}_t^\top  \mathbf{B}_{\ast}\mathbf{w}_{\ast,t,i} \mathbf{w}_t^\top \del_t \right\|_2 \nonumber \\
    &\leq \tfrac{c}{\alpha} \|\del_t\|_2^2 \|\mathbf{w}_t\|_2^2 + \tfrac{c}{\sqrt{\alpha}}(\|\del_t\|_2+\|\del_{t}\|_2^2) \|\mathbf{w}_t\|_2 \eta_\ast \nonumber \\
    &\leq 
    \frac{c E_0 \mu_\ast^2}{1000 \kappa_{\ast}^4} \|\del_t\|_2 
    + \tfrac{11}{100} c E_0 \mu_\ast^2
    \|\del_t\|_2  \nonumber \\
    &\leq \tfrac{1}{8} \mu_\ast^2 \|\del_t\|_2  
\end{align}
where we have used $A_1(t)$ and $A_3(t)$ and the fact that $\min(1,\tfrac{\mu_\ast^2}{\eta_\ast^2})\eta_\ast^2\leq \mu_\ast^2$.
To bound $\|\mathbf{G}_t\|_2^2$ we have
\begin{align}
    \|\mathbf{G}_t\|_2 &\leq \|\deld_t \mathbf{S}_t \mathbf{B}_t\|_2 + \|\mathbf{N}_t\|_2 \nonumber\\
    &\leq c\sqrt{\alpha} L_\ast^2(\|\del_t\|_2 +\dist_t) + \left\|\frac{1}{n}\sum_{i=1}^n \deld_t \mathbf{B}_t \mathbf{w}_t \mathbf{w}_{t,i}^\top\right\|_2 +\left\| \frac{1}{n}\sum_{i=1}^n \deld_t \mathbf{B}_{\ast}\mathbf{w}_{\ast,t,i} \mathbf{w}_t^\top \del_t\right\|_2 \nonumber \\
    &\leq  c\sqrt{\alpha} L_\ast^2(\|\del_t\|_2 +\dist_t) + \tfrac{c}{\sqrt{\alpha}}\left\| \del_t \right\|_2^2 \|\mathbf{w}_t\|_2^2 + 2c \|\del_t\|_2 \|\mathbf{w}_t\|_2 \eta_\ast +\left\|  \del_t\right\|_2\|\mathbf{w}_t\| \eta_\ast \dist_t \nonumber \\
    &\leq c \sqrt{\alpha} L_\ast^2  \|\del_t\|_2  + c \sqrt{\alpha}L_\ast^2 \dist_t + \tfrac{c}{1000} \sqrt{\alpha}  \mu_\ast^2   \|\del_t\|_2 + 
    \tfrac{3c }{10}  {\sqrt{\alpha}} L_\ast \eta_\ast \|\del_t\|_2 \nonumber \\
    &\leq 1.5  \sqrt{\alpha} L_\ast^2  \|\del_t\|_2  + 1.1 \sqrt{\alpha}L_\ast^2 \dist_t  \label{covid}
\end{align}
where \eqref{covid} follows since 
$\eta_\ast \leq L_\ast$.
 Therefore
\begin{align}
    \|\mathbf{G}_t\|_2^2 &\leq \alpha L_{\ast}^4 (2.5\|\del_t\|_2^2 + 3.3\|\del_t\|_2 + \tfrac{5}{4}\dist_t^2) \nonumber \\
    &\leq 4\alpha L_{\ast}^4 \|\del_t\|_2 + \tfrac{5}{4} \alpha L_\ast^4 \dist_t^2 \nonumber
\end{align}
and 
\begin{align}
    \|\del_{t+1}\|_2 &\leq 
    \left(1 - \beta \alpha E_0 \mu_\ast^2 + 0.25 \beta \alpha E_0 \mu_\ast^2 + 4 \beta^2\alpha^2 L_{\ast}^4  \right)\|\del_t\|_2 + \tfrac{5}{4} \beta^2\alpha^2 L_{\ast}^4\dist_t^2 \nonumber \\
    &\leq  \left(1 - 0.5\beta \alpha E_0 \mu_\ast^2 \right)\|\del_t\|_2 + \tfrac{5}{4} \beta^2\alpha^2 L_{\ast}^4\dist_t^2
    \label{59}
\end{align}
where in \eqref{59} we have used $\beta \leq \frac{\alpha E_0^3}{180 \kappa_{\ast}^4}$, $\alpha\leq \frac{1}{L_\ast}$, and $E_0\leq 1$.
\end{proof}


\begin{cor}[FO-ANIL $A_3(t+1)$] \label{cor:reg}
Suppose we are in the setting of Theorem \ref{thm:anil_pop_main}.
If inductive hypotheses $A_2(t+1)$ and $A_3(t)$ hold, then $A_3(t+1)$ holds, i.e.
\begin{align}
    \|\del_{t+1} \|_2 \leq \tfrac{1}{10} 
\end{align}
\end{cor}
\begin{proof}
Note that according to equation \eqref{59}, 
we have 
\begin{align}
    \|\del_{t+1}\|_2 &\leq  (1 - 0.5\beta \alpha E_0\mu_\ast^2 )\|\del_{t}\|_2 + \tfrac{5}{4} \beta^2 \alpha^2 L_\ast^4 \nonumber \\
    &\leq (1 - 0.5\beta \alpha E_0\mu_\ast^2 )\tfrac{1}{10}  + \tfrac{5}{720}  E_0 \beta\alpha \mu_\ast^2  \label{res:control_spectrumm} \\
    &\leq \tfrac{1}{10}\nonumber
\end{align}
where equation \eqref{res:control_spectrumm} is satisfied by our choice of  $\beta \leq \frac{\alpha E_0^3}{180 \kappa_{\ast}^4}$ and $\alpha 
\leq \frac{1}{ L_\ast}$ and inductive hypothesis $A_3(t)$.
\end{proof}

\begin{lemma}[FO-ANIL $A_4(t+1)$] \label{lem:foanil_pop_diverse}
Suppose the conditions  of Theorem \ref{thm:anil_pop_main} are satisfied and
 inductive hypotheses $A_1(t)$, $A_3(t)$ and $A_6(t)$ hold. Then $A_4(t+1)$ holds, i.e.
\begin{align}
    \sigma_{\min}\left(\frac{1}{n}\sum_{i=1}^n \mathbf{w}_{t+1,i}\mathbf{w}_{t+1,i}^\top \right) 
    &\geq  0.9\alpha E_0 \mu_\ast^2 \nonumber \\
    \text{ and } \sigma_{\max}\left(\frac{1}{n}\sum_{i=1}^n \mathbf{w}_{t+1,i}\mathbf{w}_{t+1,i}^\top \right) 
    &\leq 1.2 \alpha L_\ast^2 \nonumber
\end{align}
\end{lemma}
\begin{proof}
By Lemma \ref{lem:sigmin} and inductive hypotheses $A_1(t)$, $A_3(t)$ and $A_6(t)$, we have
\begin{align}
    \sigma_{\min}\left(\frac{1}{n}\sum_{i=1}^n \mathbf{w}_{t,i}\mathbf{w}_{t,i}^\top \right) &\geq \alpha E_0 \mu_\ast^2 - 0.022{ \alpha}E_0 \mu_{\ast}^2 \geq  0.9\alpha E_0 \mu_\ast^2 \nonumber \\ \sigma_{\max}\left(\frac{1}{n}\sum_{i=1}^n \mathbf{w}_{t,i}\mathbf{w}_{t,i}^\top \right) &\leq (\tfrac{1}{100}\sqrt{\alpha}E_0 \kappa_\ast^{-1}+ \sqrt{1.1\alpha} L_\ast )^2 \leq 1.2 \alpha L_\ast^2
\end{align}
where we have used the fact that $\min(1,\tfrac{\mu_\ast^2}{\eta_\ast^2})\eta_\ast^2\leq \mu_\ast^2$ to lower bound the minimum singular value.
\end{proof}

\begin{lemma}[FO-ANIL $A_5(t+1)$] \label{lem:contractt}
Suppose the conditions of Theorem \ref{thm:anil_pop_main} are satisfied.
If inductive hypotheses $A_4(t)$ holds, then $A_5(t+1)$ holds,
i.e. \begin{align}
    \|\mathbf{{B}}_{\ast,\perp}^\top \mathbf{B}_{t+1}\|_2 &\leq \left(1 - 0.5 \beta \alpha E_0 \mu_\ast^2 
    \right)  \|\mathbf{{B}}_{\ast,\perp}^\top \mathbf{B}_{t}\|_2.
\end{align}
\end{lemma}
\begin{proof}
Note from $A_4(t+1)$ that $\left({\sigma_{\max}\left(\tfrac{1}{n}\sum_{i=1}^n\mathbf{w}_{t,i}\mathbf{w}_{t,i}^\top\right))}\right)^{-1} \geq \tfrac{1}{\alpha L_\ast^2} \geq \tfrac{1}{ L_\ast}$.
Thus, since $\beta \leq  \frac{\alpha E_0^3}{180 \kappa_{\ast}^4}\leq \tfrac{1}{ L_\ast} \leq \left({\sigma_{\max}\left(\tfrac{1}{n}\sum_{i=1}^n\mathbf{w}_{t,i}\mathbf{w}_{t,i}^\top\right))}\right)^{-1}$, 
we have by Weyl's inequality that
\begin{align}
    \|\mathbf{{B}}_{\ast,\perp}^\top \mathbf{B}_{t+1}\|_2 &\leq \|\mathbf{{B}}_{\ast,\perp}^\top \mathbf{B}_{t}\|_2 \left\|\mathbf{I}_k -  \frac{\beta}{n} \sum_{i=1}^n   \mathbf{w}_{t,i}\mathbf{w}_{t,i}^\top  \right\|_2 \leq  \|\mathbf{{B}}_{\ast,\perp}^\top \mathbf{B}_{t}\|_2 \left(1 - \beta \sigma_{\min}\left(\tfrac{1}{n}\sum_{i=1}^n \mathbf{w}_{t,i}\mathbf{w}_{t,i}^\top\right) \right) \nonumber \\
    &\leq \|\mathbf{{B}}_{\ast,\perp}^\top \mathbf{B}_{t}\|_2 \left(1  - 0.5\beta   \alpha E_0 \mu_\ast^2  \right).
    \nonumber
\end{align}
\end{proof}

\subsection{Exact ANIL}

To study Exact ANIL, first note that the inner loop updates are identical to those for FO-MAML. However, the outer loop updates are different. Here, we have
\begin{align}
    \mathbf{w}_{t+1} &= \mathbf{w}_t - \frac{\beta}{n}\sum_{i=1}^n \nabla_{\mathbf{w}}F_{t,i}(\mathbf{B}_t, \mathbf{w}_{t}) \nonumber \\
    \mathbf{B}_{t+1} &= \mathbf{B}_t - \frac{\beta}{n}\sum_{i=1}^n \nabla_{\mathbf{B}}F_{t,i}(\mathbf{B}_t, \mathbf{w}_{t}) \nonumber
\end{align}
where for all $t,i$:
\begin{align}
  F_{t,i}(\mathbf{B}_t, \mathbf{w}_{t}) \coloneqq  \mathcal{L}_{t,i}(\mathbf{B}_t, \mathbf{w}_{t}-\alpha \nabla_{\mathbf{w}}\mathcal{L}_{t,i}(\mathbf{B}_t, \mathbf{w}_t)) &\coloneqq \tfrac{1}{2}\|\mathbf{v}_{t,i}\|_2^2
\end{align}
and
\begin{align}
    \mathbf{v}_{t,i} &\coloneqq \mathbf{B}_t \del_t \mathbf{w}_t + \alpha \mathbf{B}_t\mathbf{B}_t^\top \mathbf{B}_\ast \mathbf{w}_{\ast,t,i} - \mathbf{B}_\ast \mathbf{w}_{\ast,t,i} = \deld_t (\mathbf{B}_t \mathbf{w}_t  - \mathbf{B}_\ast \mathbf{w}_{\ast,t,i} )
\end{align}
Therefore,
\begin{align}
    \nabla_{\mathbf{w}}F_{t,i}(\mathbf{B}_t, \mathbf{w}_{t}) &= \mathbf{B}_t^\top\deld_t \mathbf{v}_{t,i} \label{outer_w} \\
    \nabla_{\mathbf{B}}F_{t,i}(\mathbf{B}_t, \mathbf{w}_{t})
    &= \mathbf{v}_{t,i} \mathbf{w}_{t}^\top \del_t + \alpha \mathbf{v}_{t,i}\mathbf{w}_{\ast,t,i}^\top \mathbf{B}_{\ast}^\top \mathbf{B}_t  -  \alpha \mathbf{B}_t \mathbf{w}_{t}\mathbf{v}_{t,i}^\top  \mathbf{B}_t
    - \alpha \mathbf{B}_t\mathbf{B}_t^\top \mathbf{v}_{t,i}\mathbf{w}_{t}^\top + \alpha  \mathbf{B}_{\ast}\mathbf{w}_{\ast,t,i}\mathbf{v}_{t,i}^\top \mathbf{B}_t \label{outer_grad}
\end{align}
One can observe that for $\mathbf{w}$, the outer loop gradient is the same as in the FO-ANIL case but with an extra $\alpha \mathbf{B}_t^\top\deld_t$ factor. Meanwhile, the first two terms in the outer loop gradient for $\mathbf{B}$ compose the outer loop gradient in the FO-ANIL case, while the other three terms are new. We deal with these differences in the following lemmas.



\begin{lemma}[Exact ANIL $A_1(t+1)$]\label{lem:ea_pop_a1}
Suppose the conditions of Theorem \ref{thm:anil_pop_main} are satisfied and $A_2(s)$ and $A_6(s)$ hold for all $s\in[t]$, then $A_1(t+1)$ holds, i.e.
\begin{align}
    \|\mathbf{w}_{t+1}\|_2 \leq \tfrac{\sqrt{\alpha}E_0 }{10} \min(1, \tfrac{\mu_\ast^2}{\eta_\ast^2}) \eta_\ast. 
\end{align}
\end{lemma}
\begin{proof}
Similarly to the FO-ANIL case, we can show that for any $s\in[t]$,
\begin{align}
    \mathbf{w}_{s+1}&= \mathbf{w}_{s}- \frac{\beta}{n}\sum_{i=1}^n \nabla_{\mathbf{w}} F_{s,i}(\mathbf{B}_s,\mathbf{w}_s) =(\mathbf{I}_k - \beta \del_s \mathbf{B}_s^\top \mathbf{B}_s \del_s)\mathbf{w}_s + \beta \del_s^2 \mathbf{B}_s^\top \mathbf{B}_\ast \frac{1}{n}\sum_{i=1}^n \mathbf{w}_{\ast,s,i} \label{wwweb}
\end{align}
Note that $\bigcup_{s=0}^t A_3(s)$ implies $\sigma_{\max}(\mathbf{B}_s^\top \mathbf{B}_s) \leq \frac{1+\|\del_s\|_2}{\alpha}< \frac{1.1}{\alpha}$ for all $s\in\{0,\dots,t\!+\!1\}$. Let $c\coloneqq 1.1$.

Unlike in the first-order case, the coefficient of $\mathbf{w}_s$ in \eqref{wwweb}  is the identity matrix minus a positive semi-definite matrix, so this coefficient has spectral norm at most 1 (as $\beta$ is sufficiently small). So,we can bound $ \|\mathbf{w}_{s+1}\|_2$ as:
\begin{align}
    \|\mathbf{w}_{s+1}\|_2 &\leq \|\mathbf{w}_s\|_2 + \frac{c\beta }{\sqrt{\alpha}}\|\del_s\|_2^2 \eta_\ast \label{eq64}
\end{align}
which allows us to apply Lemma \ref{lem:gen_w} with $\xi_{1,s}=0$ and $\xi_{2,s} = \frac{c \beta}{\sqrt{\alpha}} \|\del_s\|_2^2\eta_\ast$ for all $s \in [t].$ This results in:
\begin{align}
    \|\mathbf{w}_{t+1}\|_2 &\leq \sum_{s=1}^t \frac{c \beta}{\sqrt{\alpha}} \|\del_s\|_2^2\eta_\ast. \nonumber 
\end{align}
Next, note that 
\begin{align}
    \|\del_s\|_2 &\leq ( 1 - 0.5 \beta \alpha E_0 \mu_\ast^2) \|\del_{s-1}\|_2 + \tfrac{5}{4}\beta^2 \alpha^2 L_\ast^4 \dist_{s-1}^2 \nonumber \\
    &\vdots \nonumber \\
    &\leq \sum_{r = 1}^{s-1} ( 1 - 0.5 \beta \alpha E_0 \mu_\ast^2)^{s-1-r}(\tfrac{5}{4}\beta^2 \alpha^2 L_\ast^4 \dist_{r}^2) \nonumber \\
    &\leq \tfrac{5}{4} L_\ast^4 \beta^2 \alpha^2 ( 1 - 0.5 \beta \alpha E_0 \mu_\ast^2)^{s-1} \sum_{r = 1}^{s-1}( 1 - 0.5 \beta \alpha E_0 \mu_\ast^2)^{r} \nonumber \\
    &\leq \frac{5\beta \alpha  L_\ast^4}{2E_0 \mu_\ast^2}( 1 - 0.5 \beta \alpha E_0 \mu_\ast^2)^{s-1}
\end{align}
therefore
\begin{align}
    \|\mathbf{w}_{t+1}\|_2 &\leq \sum_{s=1}^t c  \frac{25\beta^3 \alpha^{1.5}  L_\ast^8}{4E_0^2 \mu_\ast^4}( 1 - 0.5 \beta \alpha E_0 \mu_\ast^2)^{2s-2} \eta_\ast \nonumber \\
    &\leq   \frac{14\beta^2 \sqrt{\alpha}  L_\ast^8}{ E_0^3 \mu_\ast^6}\eta_\ast \nonumber 
    \\
    &\leq \frac{\sqrt{\alpha}E_0}{10 \kappa_{\ast}^2} \eta_\ast \label{ssss} \\
    &\leq \frac{\sqrt{\alpha}E_0}{10 } \min(1,\tfrac{\mu_\ast^2}{\eta_\ast^2}) \eta_\ast  \label{sssss}. 
\end{align}
where \eqref{ssss} follows by choice of $\beta \leq \alpha E_0^2 /(40 \kappa_\ast^4)$ and  $\alpha\leq 1/L_\ast$, and \eqref{sssss} follows since $\eta_\ast\leq L_\ast$.
\end{proof}

\begin{lemma}[Exact ANIL $A_2(t+1)$] \label{lem:ea_pop_a}
Suppose the conditions of Theorem \ref{thm:anil_pop_main} are satisfied 
and $A_1(t), A_3(t)$ and $A_5(t)$ hold, then
$A_2(t+1)$ holds, i.e.
\begin{align}
    \|\del_{t+1}\|_2 \leq (1- 0.5 \beta \alpha E_0 \mu_\ast^2)\|\del_{t}\|_2 +\tfrac{5}{4}\beta^2 \alpha^2 L_\ast^4 \dist_t^2.
    \end{align}
\end{lemma}
\begin{proof}
Let $\mathbf{G}_t \coloneqq \frac{1}{n}\sum_{i=1}^n \nabla_{\mathbf{B}}F_{t,i}(\mathbf{B}_t, \mathbf{w}_t) = \frac{1}{\beta}(\mathbf{B}_t-\mathbf{B}_{t+1})$ again be the outer loop gradient for the representation, where $\nabla_{\mathbf{B}}F_{t,i}(\mathbf{B}_t, \mathbf{w}_t)$ is written in \eqref{outer_grad}. Note that $\mathbf{G}_t$ can be re-written as:
\begin{align}
    \mathbf{G}_t &= -\deld_t \mathbf{S}_t \mathbf{B}_t - \mathbf{S}_t\mathbf{B}_t\del_t + \mathbf{N}_t
\end{align}
where $\mathbf{S}_t \coloneqq \alpha  \mathbf{B}_{\ast}\left(\tfrac{1}{n}\sum_{i=1}^n\mathbf{w}_{\ast,t,i}\mathbf{w}_{\ast,t,i}^\top\right) \mathbf{B}_{\ast}^\top$
and \\
\begin{align}
\mathbf{N}_t &\coloneqq
\tfrac{1}{n}\sum_{i=1}^n \Big(\mathbf{v}_{t,i} \mathbf{w}_{t}^\top \del_t + \alpha \deld_t \mathbf{B}_t\mathbf{w}_t\mathbf{w}_{\ast,t,i}^\top \mathbf{B}_{\ast}^\top \mathbf{B}_t  -  \alpha \mathbf{B}_t \mathbf{w}_{t}\mathbf{v}_{t,i}^\top  \mathbf{B}_t
- \alpha \mathbf{B}_t\mathbf{B}_t^\top \mathbf{v}_{t,i}\mathbf{w}_{t}^\top \nonumber \\
&\quad \quad \quad \quad + \alpha  \mathbf{B}_{\ast}\mathbf{w}_{\ast,t,i}\mathbf{w}_{t}^\top \mathbf{B}_t^\top \mathbf{B}_t\del_t\Big)
\end{align}
Since Lemma \ref{lem:sigminE} shows that $\sigma_{\min}(\mathbf{B}_t^\top\mathbf{S}_t \mathbf{B}_t) \geq E_0 \mu_\ast^2$,  Lemma \ref{lem:gen_del} (with $\chi=1$) implies that
    \begin{align}
        \|\del_{t+1}\|_2 &\leq (1 - 2\beta \alpha E_0 \mu_\ast^2)\|\del_t\|_2 + 2 \beta \alpha \|\mathbf{B}_t^\top \mathbf{N}_t\|_2 + \beta^2 \alpha^2 \|\mathbf{G}_t\|_2^2 \label{006}
    \end{align}
    It remains to control $ \|\mathbf{B}_t^\top \mathbf{N}_t\|_2$ and $\|\mathbf{G}_t\|_2$.
Note that
\begin{align}
    \|\mathbf{B}_t^\top \mathbf{N}_t\|_2 &\leq \left\|\frac{1}{n}\sum_{i=1}^n\mathbf{B}_t^\top \mathbf{v}_{t,i} \mathbf{w}_{t}^\top \del_t\right\|_2 + \alpha \left\|\frac{1}{n}\sum_{i=1}^n\mathbf{B}_t^\top\deld_t \mathbf{B}_t\mathbf{w}_t\mathbf{w}_{\ast,t,i}^\top \mathbf{B}_{\ast}^\top \mathbf{B}_t\right\|_2  
     \nonumber \\
    &\quad +  \alpha\left\|\frac{1}{n}\sum_{i=1}^n \mathbf{B}_t^\top\mathbf{B}_t \mathbf{w}_{t}\mathbf{v}_{t,i}^\top  \mathbf{B}_t\right\|_2 + \alpha \left\|\frac{1}{n}\sum_{i=1}^n\mathbf{B}_t^\top\mathbf{B}_t\mathbf{B}_t^\top \mathbf{v}_{t,i}\mathbf{w}_{t}^\top \right\|_2 \nonumber \\
    &\quad + \alpha  \left\|\frac{1}{n}\sum_{i=1}^n\mathbf{B}_t^\top\mathbf{B}_{\ast}\mathbf{w}_{\ast,t,i}\mathbf{w}_{t}^\top \mathbf{B}_t^\top \mathbf{B}_t\del_t\right\|_2  \nonumber \\
    &\leq \tfrac{c}{\sqrt{\alpha}}\|\mathbf{w}_t\|_2\|\del_t\|_2^2\left(\tfrac{c}{\sqrt{\alpha}}\|\mathbf{w}_t\|_2+ \eta_\ast \right) +2 \tfrac{c}{\sqrt{\alpha}} \|\mathbf{w}_t\|_2\|\del_t\|_2 \eta_\ast \nonumber \\
    &\quad + 2\tfrac{c}{\sqrt{\alpha}} \|\mathbf{w}_t\|_2 \|\del_t\|_2 \left(\tfrac{c}{\sqrt{\alpha}}\|\mathbf{w}_t\|_2 + \eta_\ast  \right) \nonumber \\
    &\leq   \tfrac{5}{\sqrt{\alpha}} \|\mathbf{w}_t\|_2\|\del_t\|_2 \eta_\ast + \tfrac{3}{{\alpha}} \|\mathbf{w}_t\|_2^2 \|\del_t\|_2  \nonumber \\
    &\leq {0.6 E_0}\mu_\ast^2 \|\del_t\|_2  \label{06}
\end{align}
by inductive hypotheses $A_1(t)$ and $A_3(t)$ and the fact that $\min(1,\tfrac{\mu_\ast^2}{\eta_\ast^2})\eta_\ast^2\leq \mu_\ast^2$.
Next,
\begin{align}
    \|\mathbf{G}_t\|_2 &\leq \|\deld_t \mathbf{S}_t \mathbf{B}_t \|_2+ \|\mathbf{S}_t\mathbf{B}_t\del_t\|_2 +\|\mathbf{N}_t\|_2 \nonumber \\
    &\leq {c}{\sqrt{\alpha}}(2\|\del_t\|_2 + \dist_t)L_\ast^2 + \|\mathbf{N}_t\|_2 \nonumber \\
    &\leq {c}{\sqrt{\alpha}}(2\|\del_t\|_2 + \dist_t)L_\ast^2 + \|\mathbf{w}_t\|_2\|\del_t\|_2\left(\tfrac{c}{\sqrt{\alpha}}\|\mathbf{w}_t\|_2  \|\del_t\|_2 + (\|\del_t\|_2 +\dist_t)\eta_\ast \right)  \nonumber \\
    &\quad + 4 c \|\mathbf{w}_t\|_2 \|\del_t\|_2 \eta_\ast + 2 \tfrac{c^2}{\sqrt{\alpha}} \|\mathbf{w}_t\|_2^2 \|\del_t\|_2 \nonumber \\
    &\leq {c}{\sqrt{\alpha}}(2\|\del_t\|_2 + \dist_t)L_\ast^2   + 6  \|\mathbf{w}_t\|_2 \|\del_t\|_2 \eta_\ast +  \tfrac{3}{\sqrt{\alpha}} \|\mathbf{w}_t\|_2^2 \|\del_t\|_2 \nonumber \\
    &\leq 3 \sqrt{\alpha} L_\ast^2 \|\del_t\|_2 + c \sqrt{\alpha} L_\ast^2 \dist_t  \nonumber \\
    \implies \|\mathbf{G}_t\|_2^2 &\leq {\alpha} L_\ast^4 (9  \|\del_t\|_2^2 + 7 \|\del_t\|_2 + \tfrac{5}{4}\dist_t^2)\nonumber \\
    &\leq {\alpha} L_\ast^4 (8 \|\del_t\|_2 + \tfrac{5}{4}\dist_t^2) \label{squared}
\end{align}
Combining \eqref{006}, \eqref{06} and  \eqref{squared} yields
\begin{align}
    \|\del_{t+1}\|_2 &\leq (1 - 2\beta \alpha E_0 \mu_\ast^2)\|\del_t\|_2 + 2 \beta \alpha \|\mathbf{B}_t^\top \mathbf{N}_t\|_2 + \beta^2 \alpha \|\mathbf{G}_t\|_2^2 \nonumber \\
    &\leq (1 - 2\beta \alpha E_0 \mu_\ast^2 + 1.2 \beta \alpha E_0 \mu_\ast^2 + 8 \beta^2 \alpha^2 L_\ast^4 )\|\del_t\|_2 + \tfrac{5}{4}\beta^2 \alpha^4 \dist_t^2 \nonumber \\
    &\leq (1 - 0.5\beta \alpha E_0 \mu_\ast^2 )\|\del_t\|_2 + \tfrac{5}{4}\beta^2 \alpha^4 \dist_t^2\label{599}
\end{align}
where the last inequality follows since $\beta \leq \alpha E_0^2/(40 \kappa_\ast^4)$ and $\alpha \leq 1/L_\ast$.
\end{proof}

\begin{cor}[Exact ANIL $A_3(t+1)$] \label{cor:reg_exact}
Suppose the conditions of Theorem \ref{thm:anil_pop_main} are satisfied.
If $A_2(t+1)$ and $A_3(t)$ hold. Then $A_3(t+1)$ holds, i.e.
\begin{align}
    \|\del_{t+1}\|_2 &\leq \tfrac{1}{10}
\end{align}
\end{cor}
\begin{proof}
Note that according to equation \eqref{599}, 
we have 
\begin{align}
    \|\del_{t+1}\|_2 &\leq  (1 - 0.5\beta \alpha E_0\mu_\ast^2 )\|\del_{t}\|_2 + \tfrac{5}{4} \beta^2 \alpha^2 L_\ast^4 \nonumber \\
    &\leq (1 - 0.5\beta \alpha E_0\mu_\ast^2 )\tfrac{1}{10}  + \tfrac{5}{4}  E_0 \beta\alpha \mu_\ast^2 \label{res:control_spectrum} \\
    &\leq \tfrac{1}{10 }\nonumber
\end{align}
where equation \eqref{res:control_spectrum} is satisfied by the choice of  $\beta \leq \alpha E_0^2 /(40 \kappa_\ast^4)$. and inductive hypothesis $A_3(t)$.
\end{proof}

\begin{lemma}[Exact-ANIL $A_4(t+1)$] \label{lem:exactanil_pop_diverse}
Suppose the conditions of Theorem \ref{thm:anil_pop_main} are satisfied and that
 inductive hypotheses $A_1(t)$, $A_3(t)$ and $A_6(t)$ hold. Then $A_4(t+1)$ holds, i.e.
\begin{align}
    \sigma_{\min}\left(\frac{1}{n}\sum_{i=1}^n \mathbf{w}_{t+1,i}\mathbf{w}_{t+1,i}^\top \right) 
    &\geq  0.9\alpha E_0 \mu_\ast^2 \nonumber \\
    \text{ and } \sigma_{\max}\left(\frac{1}{n}\sum_{i=1}^n \mathbf{w}_{t+1,i}\mathbf{w}_{t+1,i}^\top \right) 
    &\leq 1.2 \alpha L_\ast^2 \nonumber
\end{align}
\end{lemma}
\begin{proof} 
The proof is identical to that of Lemma \ref{lem:foanil_pop_diverse}.
\end{proof}

\begin{lemma}[Exact ANIL $A_5(t+1)$] \label{lem:ea_pop_a4}
Suppose the conditions of Theorem \ref{thm:anil_pop_main} are satisfied.
If inductive hypothesis $A_4(t)$ holds, then $A_5(t+1)$ holds, that is
\begin{align}
    \|\mathbf{B}_{\ast,\perp}^\top \mathbf{B}_{t+1}\|_2 &\leq  (1 - 0.5 \beta \alpha E_0 \mu_\ast^2)\|\mathbf{B}_{\ast,\perp}^\top \mathbf{B}_{t}\|_2  \nonumber
\end{align}
\end{lemma}
\begin{proof}
Note that from \eqref{outer_grad}, the outer loop gradient for the $(t,i)$-th task can be re-written as:
\begin{align}
     \nabla_{\mathbf{B}}F_{t,i}(\mathbf{B}_t, \mathbf{w}_{t})
    &= \mathbf{v}_{t,i} \mathbf{w}_{t}^\top \del_t + \alpha \mathbf{v}_{t,i}\mathbf{w}_{\ast,t,i}^\top \mathbf{B}_{\ast}^\top \mathbf{B}_t  -  \alpha \mathbf{B}_t \mathbf{w}_{t}\mathbf{v}_{t,i}^\top  \mathbf{B}_t
    - \alpha \mathbf{B}_t\mathbf{B}_t^\top \mathbf{v}_{t,i}\mathbf{w}_{t}^\top + \alpha  \mathbf{B}_{\ast}\mathbf{w}_{\ast,t,i}\mathbf{v}_{t,i}^\top \mathbf{B}_t  \nonumber \\
    &= \mathbf{B}_{t}\mathbf{w}_{t,i}\mathbf{w}_{t,i}^\top - \mathbf{B}_\ast\mathbf{w}_{\ast,t,i}\mathbf{w}_{t,i}^\top  -  \alpha \mathbf{B}_t \mathbf{w}_{t}\mathbf{v}_{t,i}^\top  \mathbf{B}_t
    - \alpha \mathbf{B}_t\mathbf{B}_t^\top \mathbf{v}_{t,i}\mathbf{w}_{t}^\top + \alpha  \mathbf{B}_{\ast}\mathbf{w}_{\ast,t,i}\mathbf{v}_{t,i}^\top \mathbf{B}_t \nonumber
\end{align}
Therefore, noting $\mathbf{G}_t = \tfrac{1}{n}\sum_{i=1}^n \nabla_{\mathbf{B}}F_{t,i}(\mathbf{B}_t, \mathbf{w}_{t})$, and using  $\mathbf{B}_{\ast,\perp}^\top\mathbf{B}_\ast=\mathbf{0}$, we have
\begin{align}
    \|\mathbf{B}_{\ast,\perp}^\top \mathbf{B}_{t+1}\|_2 &= \|\mathbf{B}_{\ast,\perp}^\top(\mathbf{B}_t-\beta\mathbf{G}_t)\|_2 \nonumber \\
    &\leq \bigg\|\mathbf{B}_{\ast,\perp}^\top\mathbf{B}_t\bigg(\mathbf{I}_k-\tfrac{\beta}{n}\sum_{i=1}^n \mathbf{w}_{t,i}\mathbf{w}_{t,i}^\top + \tfrac{\beta \alpha}{n}\sum_{i=1}^n (\mathbf{w}_t \mathbf{v}_{t,i}^\top \b + \b^\top \mathbf{v}_{t,i}\w^\top)\bigg)\bigg\|_2 \nonumber \\
    &\leq \bigg\|\mathbf{B}_{\ast,\perp}^\top\mathbf{B}_t\bigg(\mathbf{I}_k-\tfrac{\beta}{n}\sum_{i=1}^n \mathbf{w}_{t,i}\mathbf{w}_{t,i}^\top\bigg)\bigg\|_2 + 2\beta \alpha \|\mathbf{B}_{\ast,\perp}^\top\mathbf{B}_t\|_2\bigg\|\tfrac{1}{n}\sum_{i=1}^n \mathbf{w}_t \mathbf{v}_{t,i}^\top \b \bigg\|_2 \nonumber \\
        &\leq  \|\mathbf{B}_{\ast,\perp}^\top\mathbf{B}_t\|_2\bigg(1- 0.9 \beta \alpha E_0 \mu_\ast^2  + \beta \alpha \tfrac{3E_0}{100}\min(1,\tfrac{\mu_\ast^2}{\eta_\ast^2})\eta_\ast^2 \bigg) \label{inducts} \\
    &\leq \|\mathbf{B}_{\ast,\perp}^\top\mathbf{B}_t\|_2(1- 0.5\beta \alpha E_0 \mu_\ast^2 ) \nonumber
\end{align}
where \eqref{inducts} follows by inductive hypotheses $A_1(t)$,  $A_3(t)$, and $A_4(t)$,  and the fact that $\min(1,\tfrac{\mu_\ast^2}{\eta_\ast^2})\eta_\ast^2\leq \mu_\ast^2$.
\end{proof}

\section{MAML Infnite Samples} \label{app:maml}





\subsection{FO-MAML}
We consider FO-MAML when $m_{in}=m_{out}= \infty$. In this case, the inner loop updates are:
\begin{align}
    \mathbf{w}_{t,i}  &= \mathbf{w}_t - \alpha \nabla_{\mathbf{w}}\mathcal{L}_{t,i}(\mathbf{B}_t, \mathbf{w}_t) \nonumber\\
    &= (\mathbf{I}_d - \alpha \mathbf{{B}}_t^\top \mathbf{{B}}_t)\mathbf{w}_t + \alpha \mathbf{{B}}_t^\top \mathbf{\hat{B}}_\ast\mathbf{w}_{\ast,t,i} \nonumber \\
    \mathbf{B}_{t,i} &= \mathbf{B}_t - \alpha \nabla_{\mathbf{B}}\mathcal{L}_{t,i}(\mathbf{B}_t, \mathbf{w}_t) \nonumber\\
    &= \mathbf{B}_t(\mathbf{I}_k - \alpha \mathbf{w}_t \mathbf{w}_t^\top) + \alpha \mathbf{B}_\ast \mathbf{w}_{\ast,t,i}\mathbf{w}_t^\top  
\end{align}
The outer loop updates are:
\begin{align}
\mathbf{w}_{t+1} &= \mathbf{w}_t - \frac{\beta}{n}\sum_{i=1}^n \nabla_{\mathbf{w}}\mathcal{L}_{t,i}(\mathbf{B}_{t,i}, \mathbf{w}_{t,i}) = \mathbf{w}_t - \frac{\beta}{n}\sum_{i=1}^n \mathbf{B}_{t,i}^\top ( \mathbf{B}_{t,i}\mathbf{w}_{t,i} -  \mathbf{B}_{\ast}\mathbf{w}_{\ast,t,i})  \nonumber \\
\mathbf{B}_{t+1} &= \mathbf{B}_t - \frac{\beta}{n}\sum_{i=1}^n \nabla_{\mathbf{B}}\mathcal{L}_{t,i}(\mathbf{B}_{t,i}, \mathbf{w}_{t,i}) = \mathbf{B}_t - \frac{\beta}{n}\sum_{i=1}^n  (\mathbf{B}_{t,i}\mathbf{w}_{t,i} -  \mathbf{B}_{\ast}\mathbf{w}_{\ast,t,i})\mathbf{w}_{t,i}^\top \label{154}
%
\end{align}

Now we state the main result for Exact MAML in the infinite sample case.
Due to third and higher-order products of the ground-truth heads that arise in the FO-MAML and MAML updates, we require an upper bound on the maximum $\mathbf{w}_{\ast,t,i}$. We define the parameter $L_{\max}$ as follows.
\begin{assumption} \label{assump:lmax}
There exists  $L_{\max}<\infty$ such that almost surely for all $t\in[T]$, we have 
\begin{align}
    \max_{i\in[n]} \|\mathbf{w}_{\ast,t,i}\|_2\leq L_{\max}
\end{align}
\end{assumption}
Note that if Assumption \ref{assump:tasks_main} holds, we have $L_{\max}=O(\sqrt{k}L_\ast)$. Here we prove a slightly more general version of Theorem \ref{thm:fomaml_pop_main} in which we allow for arbitrary finite $L_{\max}$. Note that Theorem \ref{thm:fomaml_pop_app} immediately implies Theorem \ref{thm:maml_pop_main} after applying Assumption \ref{assump:tasks_main}. First we state the following assumption, then we prove the theorem.
\begin{assumption}[Initialization and small average ground-truth heads]
\label{assump:fomaml}
The following holds almost surely:
\begin{align}
    \dist_0 \leq {\frac{ 4 \mu_{\ast}}{5 L_{\max}}} \quad \text{and, for all $t\in[T]$,}\quad  \left\|\tfrac{1}{n}\sum_{i=1}^n \mathbf{w}_{\ast,t,i}\right\|_2 \leq  \eta_\ast \leq \frac{2 E_0^2 \mu_{\ast}^4}{ L_{\max}^3} .
\end{align}
\end{assumption}
\begin{thm}[FO-MAML Infinite Samples] \label{thm:fomaml_pop_app}
Let $m_{in}=m_{out}= \infty$ and define $E_0 \coloneqq 0.9 - \dist_0^2$.
Suppose that $\alpha \leq \tfrac{1}{4 L_{\max}}$, $\beta \leq \frac{ \alpha E_0^2  }{60\kappa_\ast^4}$,
     $\alpha\mathbf{B}_t^\top \mathbf{B}_t = \mathbf{I}_k$, $\mathbf{w}_0=\mathbf{0}$ and
Assumptions \ref{assump:tasks_diverse_main}, \ref{assump:lmax}  and \ref{assump:fomaml} hold.
Then FO-MAML satisfies that for all $T\in\mathbb{Z}_+$,
\begin{align}
    \dist(\mathbf{B}_T, \mathbf{B}_\ast)\leq (1 - 0.5 \beta \alpha E_0 \mu_\ast^2)^{T-1}.
\end{align}
\end{thm}
\begin{proof}
The proof follows by showing that the following inductive hypotheses hold for all $t\in[T]$:
\begin{enumerate}
    \item $A_{1}(t)\coloneqq \{\|\mathbf{w}_{t}\|_2 \leq \tfrac{E_0^2}{10}\sqrt{\alpha} \mu_\ast \kappa_{\ast,\max}^{-3} \}$
    \item $A_{2}(t)\coloneqq \{\|\boldsymbol{\Delta}_{t}\|_2 \leq  \tfrac{E_0}{10}\alpha^2 \mu_\ast^2 \}$
    \item $A_{3}(t)\coloneqq \{\|\mathbf{B}_{\ast,\perp}^\top \mathbf{B}_{t}\|_2 \leq ( 1 - 0.5 \beta \alpha E_0 \mu_\ast^2)  \|\mathbf{B}_{\ast,\perp}^\top \mathbf{B}_{t-1}\|_2\}$
    \item $A_{4}(t)\coloneqq \{\dist_{t} \leq \frac{\sqrt{10}}{3}(1- 0.5 \beta \alpha E_0 \mu_\ast^2)^{t-1}\dist_0\}$
    \item $A_{5}(t)\coloneqq \{\dist_{t} \leq (1- 0.5 \beta \alpha E_0 \mu_\ast^2)^{t-1}\}$
\end{enumerate}
 These conditions hold for iteration $t=0$ due to the choice of initialization. Now, assuming they hold for arbitrary $t$, we will show they hold at $t+1$. 
\begin{enumerate}
    \item $A_1(t)\cap A_2(t)\cap A_4(t) \implies A_1({t+1})$.
    This is Lemma \ref{lem:maml_ind1}.
    
    \item $A_1(t)\cap A_2(t)\cap A_4(t) \implies A_2({t+1})$. This is Lemma \ref{lem:maml_ind2}.

    \item $A_1(t)\cap A_2(t)\cap A_4(t) \implies A_3(t+1)$.
    This is Lemma \ref{lem:maml_ind3}.
    
    \item $A_2(t+1)\cap \bigcap_{s=1}^{t+1} A_3(s) \implies A_4(t+1) \cap A_5(t+1)$.
    Note that $A_2(t+1)\cap \bigcap_{s=1}^{t+1} A_3(s)$ implies
\begin{align}
\tfrac{\sqrt{1-\|\del_{t+1}\|_2}}{\sqrt{\alpha}} \dist_{t+1}  &= \tfrac{\sqrt{1-\|\del_{t+1}\|_2}}{\sqrt{\alpha}}\|\mathbf{{B}}_{\ast,\perp}^\top \mathbf{\hat{B}}_{t+1}\|_2 \nonumber \\
&\leq \sigma_{\min}(\mathbf{{B}}_{t+1}) \|\mathbf{{B}}_{\ast,\perp}^\top \mathbf{\hat{B}}_{t+1}\|_2\;  \nonumber \\
 &\leq \|\mathbf{{B}}_{\ast,\perp}^\top \mathbf{B}_{t+1}\|_2   \label{tyur} \\
    &\leq \left(1 - 0.5 \beta\alpha E_0 \mu_\ast^2 \right)^t  \|\mathbf{{B}}_{\ast,\perp}^\top \mathbf{B}_{0}\|_2 
    \nonumber \\
    &\leq \tfrac{1}{\sqrt{\alpha}}\left(1  - 0.5 \beta  \alpha E_0 \mu_\ast^2 \right)^t \|\mathbf{{B}}_{\ast,\perp}^\top \mathbf{\hat{B}}_{0}\|_2 \nonumber \\
    &= \tfrac{1}{\sqrt{\alpha}}\left(1 - 0.5 \beta \alpha E_0 \mu_\ast^2 \right)^t \dist_0. \label{wdexz}
\end{align}
where \eqref{tyur} follows since $\|\mathbf{{B}}_{\ast,\perp}^\top \mathbf{B}_{t+1}\|_2 = \|\mathbf{{B}}_{\ast,\perp}^\top \mathbf{\hat{B}}_{t+1} \mathbf{R}_{t+1}\|_2 \geq \|\mathbf{{B}}_{\ast,\perp}^\top \mathbf{\hat{B}}_{t+1} \|_2 \sigma_{\min}(\mathbf{R}_{t+1})$ and $\sigma_{\min}(\mathbf{R}_{t+1})= \sigma_{\min}(\mathbf{B}_{t+1})$, recalling that $\mathbf{\hat{B}}_{t+1}\mathbf{R}_{t+1}=\mathbf{B}_{t+1}$ is the QR decomposition of $\mathbf{B}_{t+1}$.
Dividing both sides of \eqref{wdexz} by $\tfrac{\sqrt{1-\|\del_{t+1}\|_2}}{\sqrt{\alpha}}$ and using the facts that $\dist_0 \leq \tfrac{3}{\sqrt{10}}$ and $\|\del_{t+1}\|_2 \leq \tfrac{1}{10}$ yields
\begin{align}
    \dist_{t+1}  &\leq \tfrac{1}{\sqrt{1-\|\del_{t+1}\|_2}}\left(1  - 0.5 \beta  \alpha E_0 \mu_\ast^2  \right)^t \dist_0 \nonumber \\
    &\leq \tfrac{\sqrt{10}}{3}\left(1  - 0.5 \beta  \alpha E_0 \mu_\ast^2  \right)^t \dist_0 \nonumber  \\
     &\leq \left(1  - 0.5 \beta  \alpha E_0 \mu_\ast^2 \right)^t,
\end{align}
as desired.
\end{enumerate}
\end{proof}

\begin{lemma}[FO-MAML $A_1(t+1$] \label{lem:maml_ind1}
Suppose the conditions of Theorem \ref{thm:fomaml_pop_app} are satisfied 
and $A_1(t)$,$A_2(t)$ and $A_4(t)$ hold. Then $A_1(t+1)$ holds, i.e.
\begin{align}
    \|\mathbf{w}_{t+1}\|_2 \leq \tfrac{E_0^2}{10} \sqrt{\alpha}\mu_\ast \kappa_{\ast,\max}^{-3}.
\end{align}
\end{lemma}

\begin{proof}
Let $\mathbf{G}_{t,i}$ be the inner loop gradient for the representation for the $(t,i)$-th task, in particular $\mathbf{G}_{t,i}=\mathbf{B}_t \mathbf{w}_t\w^\top - \mathbf{B}_\ast \mathbf{w}_{\ast,t,i}\w^\top$. By expanding the outer loop update for the head, we obtain:
\begin{align}
    \mathbf{w}_{t+1} &=  \frac{1}{n}\sum_{i=1}^n(\mathbf{I}_{k}\! - \! \beta \mathbf{B}_{t,i}^\top \mathbf{B}_{t,i}(\mathbf{I}- \alpha\mathbf{B}_{t}^\top \mathbf{B}_{t}) )\mathbf{w}_t + \beta \frac{1}{n}\sum_{i=1}^n \mathbf{{B}}_{t,i}^\top  (\mathbf{I}- \alpha\mathbf{B}_{t,i} \mathbf{B}_{t}^\top) \mathbf{{B}}_\ast \mathbf{w}_{\ast,t,i} \nonumber \\
    &= \frac{1}{n}\sum_{i=1}^n(\mathbf{I}_{k}\! - \! \beta \mathbf{B}_{t,i}^\top \mathbf{B}_{t,i}(\mathbf{I}- \alpha\mathbf{B}_{t}^\top \mathbf{B}_{t}) )\mathbf{w}_t + \beta \alpha^2 \frac{1}{n}\sum_{i=1}^n \mathbf{{B}}_{t}^\top \mathbf{G}_{t,i}\mathbf{B}_t^\top \mathbf{{B}}_\ast \mathbf{w}_{\ast,t,i} \nonumber  \\
    &\quad + \beta \frac{1}{n}\sum_{i=1}^n \mathbf{{B}}_{t,i}^\top  (\mathbf{I}- \alpha\mathbf{B}_{t} \mathbf{B}_{t}^\top) \mathbf{{B}}_\ast \mathbf{w}_{\ast,t,i} - \beta \alpha^3 \frac{1}{n}\sum_{i=1}^n \mathbf{{G}}_{t,i}^\top  \mathbf{G}_{t,i} \mathbf{B}_{t}^\top \mathbf{{B}}_\ast \mathbf{w}_{\ast,t,i} \nonumber  \\
    &= \frac{1}{n}\sum_{i=1}^n(\mathbf{I}_{k}\! - \! \beta \mathbf{B}_{t,i}^\top \mathbf{B}_{t,i}(\mathbf{I}- \alpha\mathbf{B}_{t}^\top \mathbf{B}_{t}) )\mathbf{w}_t \nonumber \\
    &\quad + \beta \alpha^2 \frac{1}{n}\sum_{i=1}^n \mathbf{{B}}_{t}^\top (\mathbf{B}_t \mathbf{w}_t \mathbf{w}_t^\top - \mathbf{B}_\ast \mathbf{w}_{\ast,t,i} \mathbf{w}_t^\top )\mathbf{B}_t^\top\mathbf{{B}}_\ast \mathbf{w}_{\ast,t,i} \nonumber  \\
    &\quad + \beta \frac{1}{n}\sum_{i=1}^n \mathbf{{B}}_{t,i}^\top  (\mathbf{I}- \alpha\mathbf{B}_{t} \mathbf{B}_{t}^\top) \mathbf{{B}}_\ast \mathbf{w}_{\ast,t,i} - \beta \alpha^3 \frac{1}{n}\sum_{i=1}^n \mathbf{{G}}_{t,i}^\top  \mathbf{G}_{t,i} \mathbf{B}_{t}^\top \mathbf{{B}}_\ast \mathbf{w}_{\ast,t,i} \nonumber \\
    &= \frac{1}{n}\sum_{i=1}^n(\mathbf{I}_{k}\! - \! \beta \mathbf{B}_{t,i}^\top \mathbf{B}_{t,i}(\mathbf{I}- \alpha\mathbf{B}_{t}^\top \mathbf{B}_{t}) )\mathbf{w}_t - \beta \alpha^2  \mathbf{B}_t^\top \mathbf{B}_\ast \left(\frac{1}{n}\sum_{i=1}^n \mathbf{w}_{\ast,t,i} \mathbf{w}_{\ast,t,i}^\top\right) \mathbf{B}_\ast^\top \mathbf{B}_t \mathbf{w}_t  \nonumber \\
    &\quad + \beta \alpha^2  \mathbf{{B}}_{t}^\top \mathbf{B}_t \mathbf{w}_t \mathbf{w}_t^\top\mathbf{B}_t^\top\mathbf{{B}}_\ast \left(\frac{1}{n}\sum_{i=1}^n \mathbf{w}_{\ast,t,i}\right) \nonumber \\ 
    &\quad + \beta \frac{1}{n}\sum_{i=1}^n \mathbf{{B}}_{t,i}^\top  (\mathbf{I}- \alpha\mathbf{B}_{t} \mathbf{B}_{t}^\top) \mathbf{{B}}_\ast \mathbf{w}_{\ast,t,i} - \beta \alpha^3 \frac{1}{n}\sum_{i=1}^n \mathbf{{G}}_{t,i}^\top  \mathbf{G}_{t,i} \mathbf{B}_{t}^\top \mathbf{{B}}_\ast \mathbf{w}_{\ast,t,i} \nonumber \\
    &= \left(\mathbf{I}_{k}\!  - \beta \alpha^2  \mathbf{B}_t^\top \mathbf{B}_\ast \left(\frac{1}{n}\sum_{i=1}^n \mathbf{w}_{\ast,t,i} \mathbf{w}_{\ast,t,i}^\top\right) \mathbf{B}_\ast^\top \mathbf{B}_t  \right)\mathbf{w}_t + \mathbf{N}_t \label{wt_num}
\end{align}
       where $\mathbf{N}_t \coloneqq -\beta \frac{1}{n}\sum_{i=1}^n\mathbf{B}_{t,i}^\top \mathbf{B}_{t,i}\del_t\mathbf{w}_t +  \beta \alpha^2  \mathbf{{B}}_{t}^\top \mathbf{B}_t \mathbf{w}_t \mathbf{w}_t^\top\mathbf{B}_t^\top\mathbf{{B}}_\ast \frac{1}{n}\sum_{i=1}^n \mathbf{w}_{\ast,t,i}\  + \beta \frac{1}{n}\sum_{i=1}^n \mathbf{{B}}_{t,i}^\top \deld_t \mathbf{{B}}_\ast \mathbf{w}_{\ast,t,i} - \beta \alpha^3 \frac{1}{n}\sum_{i=1}^n \mathbf{{G}}_{t,i}^\top  \mathbf{G}_{t,i} \mathbf{B}_{t}^\top \mathbf{{B}}_\ast \mathbf{w}_{\ast,t,i}$.
       Since $\sigma_{\min}(\mathbf{B}_t^\top \mathbf{B}_\ast \left(\frac{1}{n}\sum_{i=1}^n \mathbf{w}_{\ast,t,i} \mathbf{w}_{\ast,t,i}^\top\right) \mathbf{B}_\ast^\top \mathbf{B}_t )\geq \frac{1}{\alpha}E_0 \mu_\ast^2$ by Lemma \ref{lem:sigminE}, and $\beta\leq \frac{1}{2\alpha L_\ast^2}$, we have
    \begin{align}
        \|\mathbf{w}_{t+1}\|_2 &\leq \left\|\mathbf{I}_{k}\!  - \beta \alpha^2  \mathbf{B}_t^\top \mathbf{B}_\ast \left(\frac{1}{n}\sum_{i=1}^n \mathbf{w}_{\ast,t,i} \mathbf{w}_{\ast,t,i}^\top\right)\mathbf{B}_\ast^\top \mathbf{B}_t \right\|_2 \|\mathbf{w}_{t}\|_2 + \|\mathbf{N}_t\|_2 \nonumber \\
        &\leq (1 - \beta \alpha E_0 \mu_\ast^2) \|\mathbf{w}_{t}\|_2 + \|\mathbf{N}_t\|_2 
    \end{align}
    The remainder of the proof deals with bounding $\|\mathbf{N}_t\|_2$. 
    First note that $\bigcup_{s=0}^t A_2(s)$ with $\alpha \leq 1/(4 L_{\max})$ implies $\sigma_{\max}(\mathbf{B}_s^\top \mathbf{B}_s) \leq \frac{1+\|\del_s\|_2}{\alpha}< \frac{1.1^{2/3}}{\alpha}$ for all $s\in\{0,\dots,t\!+\!1\}$. In turn, this means that $\alpha^{1.5}\|\mathbf{B}_s\|_2^3\leq 1.1$ Let $c\coloneqq 1.1$.

    We consider each of the four terms in $\mathbf{N}_t$ separately. Using  $\sqrt{\alpha}\|\mathbf{B}_t\|_2, \alpha^{1.5}\|\mathbf{B}_t\|_2^3 \leq c$ and the Cauchy-Schwarz and triangle inequalities,  we have
    \begin{align}
        \beta \left\|\left(\frac{1}{n}\sum_{i=1}^n\mathbf{B}_{t,i}^\top \mathbf{B}_{t,i}\right)\del_t
        \mathbf{w}_t\right\|_2 &\leq \beta \bigg(\|\mathbf{B}_t\|_2^2\!+\!2\alpha\|\mathbf{B}_t\|_2 \bigg\|\tfrac{1}{n}\sum_{i=1}^n\mathbf{w}_{\ast,t,i}\bigg\|_2\|\mathbf{w}_t\|_2\!\nonumber\\
        &\quad \quad \quad + \alpha^2 \big\|\tfrac{1}{n}\sum_{i=1}^n\mathbf{w}_{\ast,t,i}\mathbf{w}_{\ast,t,i}^\top\big\|_2\|\mathbf{w}_t\|_2^2\bigg) \|\del_t\|_2\|\mathbf{w}_t\|_2 \nonumber \\
        &\leq  \beta(\tfrac{c}{\alpha} + 2c\sqrt{\alpha}\|\mathbf{w}_t\|_2 \eta_\ast + \alpha^2 L_\ast^2\|\mathbf{w}_t\|_2^2)\|\del_t\|_2 \nonumber \\
\beta \alpha^2  \left\|\mathbf{{B}}_{t}^\top \mathbf{B}_t \mathbf{w}_t \mathbf{w}_t^\top\mathbf{B}_t^\top\mathbf{{B}}_\ast \frac{1}{n}\sum_{i=1}^n \mathbf{w}_{\ast,t,i}\right\|_2  &\leq c \beta \sqrt{\alpha} \eta_\ast \|\mathbf{w}_t\|_2^2  \nonumber \\
\beta\left\| \frac{1}{n}\sum_{i=1}^n \mathbf{{B}}_{t,i}^\top \deld_t \mathbf{{B}}_\ast \mathbf{w}_{\ast,t,i}\right\|_2&\leq \frac{c\beta}{\sqrt{\alpha}} \|\del_t\|_2 \eta_\ast + \beta \alpha L_{\max}^2  \|\mathbf{w}_t\|_2 \|\del_t\|_2 + \beta\alpha L_{\max}^2  \|\mathbf{w}_t\|_2 \dist_t^2 \label{lebel}\\
   \beta \alpha^3 \left\|\frac{1}{n}\sum_{i=1}^n \mathbf{{G}}_{t,i}^\top \mathbf{G}_{t,i}
    \mathbf{B}_{t}^\top\mathbf{{B}}_\ast\mathbf{w}_{\ast,t,i}\right\|_2 &\leq c\beta \alpha^{2.5} \left(\tfrac{c\|\mathbf{w}_t\|_2^2}{{\alpha}}\eta_\ast + \tfrac{2c\|\mathbf{w}_t\|_2 L_\ast^2}{\sqrt{\alpha}} 
    + L_{\max}^3 \right)  \|\mathbf{w}_t\|_2^2. 
    \end{align}
Note that the  $\dist_t^2$ in \eqref{lebel} is due to the fact that $\|\mathbf{B}_\ast^\top (\mathbf{I}_k - \mathbf{\hat{B}}_t \mathbf{\hat{B}}_t^\top)\mathbf{B}_\ast\|_2 = \|\mathbf{B}_\ast^\top (\mathbf{I}_k - \mathbf{\hat{B}}_t \mathbf{\hat{B}}_t^\top)(\mathbf{I}_k - \mathbf{\hat{B}}_t \mathbf{\hat{B}}_t^\top)\mathbf{B}_\ast\|_2 \leq  \dist_t^2$.
Combining these bounds and applying inductive hypotheses $A_2(t)$  and $A_3(t)$ yields
\begin{align}
   \| \mathbf{N}_t\|_2 &\leq \beta \left\|\left(\frac{1}{n}\sum_{i=1}^n\mathbf{B}_{t,i}^\top \mathbf{B}_{t,i}\right)\del_t\mathbf{w}_t\right\|_2 + \beta \alpha^2  \left\|\mathbf{{B}}_{t}^\top \mathbf{B}_t \mathbf{w}_t \mathbf{w}_t^\top\mathbf{B}_t^\top\mathbf{{B}}_\ast \frac{1}{n}\sum_{i=1}^n \mathbf{w}_{\ast,t,i}\right\|_2 \nonumber \\
   &\quad + \beta\left\| \frac{1}{n}\sum_{i=1}^n \mathbf{{B}}_{t,i}^\top \deld_t \mathbf{{B}}_\ast \mathbf{w}_{\ast,t,i}\right\|_2 + \beta \alpha^3 \left\|\frac{1}{n}\sum_{i=1}^n \mathbf{{G}}_{t,i}^\top \mathbf{G}_{t,i}
    \mathbf{B}_{t}^\top\mathbf{{B}}_\ast\mathbf{w}_{\ast,t,i}\right\|_2 \nonumber \\
    &\leq \beta(\tfrac{c}{\alpha} + 2c\sqrt{\alpha}\|\mathbf{w}_t\|_2 \eta_\ast + \alpha^2 L_\ast^2\|\mathbf{w}_t\|_2^2)\|\del_t\|_2\|\mathbf{w}_t\|_2 +c \beta \sqrt{\alpha} \eta_\ast \|\mathbf{w}_t\|_2^2+\tfrac{c\beta}{\sqrt{\alpha}} \|\del_t\|_2 \eta_\ast  \nonumber \\
    &\quad  +\beta \alpha L_{\max}^2  \|\mathbf{w}_t\|_2 \|\del_t\|_2 + \beta \alpha L_{\max}^2  \|\mathbf{w}_t\|_2 \dist_t^2 + c\beta \alpha^{2.5} \left(\tfrac{c\|\mathbf{w}_t\|_2^2}{{\alpha}}\eta_\ast + \tfrac{2c\|\mathbf{w}_t\|_2 L_\ast^2}{\sqrt{\alpha}}
    + L_{\max}^3 \right)  \|\mathbf{w}_t\|_2^2 \nonumber \\
    &\leq  \tfrac{2c}{100} \beta \alpha^{1.5} \mu_\ast^3 \kappa_{\ast,\max}^{-3}  E_0^3 + \tfrac{2c}{10}\beta \alpha^{1.5} \mu_\ast^2  \eta_\ast E_0 +\tfrac{1}{10} \beta \alpha^{1.5}  \mu_\ast^3 \kappa_{\ast,\max} E_0^2 \dist_t^2 
     \nonumber
\end{align}
    Thus we have
    \begin{align}
        \|\mathbf{w}_{t+1}\|_2 &\leq \left(1- \beta\alpha E_0 \mu_\ast^2 \right)\|\mathbf{w}_{t}\|_2 + \tfrac{2c}{100} \beta \alpha^{1.5} \mu_\ast^3 \kappa_{\ast,\max}^{-3}  E_0^3 + \tfrac{2c}{10}\beta \alpha^{1.5} \mu_\ast^2  \eta_\ast E_0 +\tfrac{1}{10} \beta \alpha^{1.5}  \mu_\ast^3 \kappa_{\ast,\max} E_0^2 \dist_t^2 \nonumber \\
        &\leq \tfrac{1}{10}E_0^2\sqrt{\alpha} \mu_\ast \kappa_{\ast,\max}^{-1} - \tfrac{1}{10}\beta\alpha^{1.5} E_0^3 \mu_\ast^3 \kappa_{\ast,\max}^{-1} + \tfrac{2c}{100} \beta \alpha^{1.5} \mu_\ast^3 \kappa_{\ast,\max}^{-3}  E_0^3 + \tfrac{2c}{10}\beta \alpha^{1.5} \mu_\ast^2  \eta_\ast E_0 \nonumber \\
        &\quad +\tfrac{1}{10} \beta \alpha^{1.5}  \mu_\ast^3 \kappa_{\ast,\max} E_0^2 \dist_0^2 \nonumber \\
        &\leq \tfrac{1}{10}E_0^2\sqrt{\alpha} \mu_\ast \kappa_{\ast,\max}^{-1} 
        \label{RHSw}
    \end{align}
    where \eqref{RHSw} follows by Assumption \ref{assump:fomaml}, namely:
    \begin{align}
         \eta_\ast \leq \frac{2 E_0^2 \mu_{\ast}^4}{ L_{\max}^3} \quad \text{ and } \quad \dist_0 \leq {\frac{ 4 \mu_{\ast}}{5 L_{\max}}}.
    \end{align}
    \end{proof}
    

\begin{lemma}[FO-MAML $A_2(t+1)$] \label{lem:maml_ind2}
Suppose the conditions of Theorem \ref{thm:fomaml_pop_app} are satisfied 
and $A_1(t)$,$A_2(t)$ and $A_4(t)$ hold. Then $A_2(t+1)$ holds almost surely, i.e.
\begin{align}
    \|\del_t\|_2\leq \tfrac{E_0}{10}\alpha^2 \mu_\ast^2.
\end{align}
\end{lemma}

\begin{proof}
We will employ Lemma \eqref{lem:gen_del}, which requires writing the outer loop gradient for the representation, i.e. $\mathbf{G}_t \coloneqq \frac{1}{\beta}(\mathbf{B}_t - \mathbf{B}_{t+1})$, as $\mathbf{G}_t=-\deld_t \mathbf{S}_t \mathbf{B}_t - \chi \mathbf{B}_t \mathbf{S}_t \del_t + \mathbf{N}_t$, for some positive definite matrix $\mathbf{S}_t$, a matrix $\mathbf{N}_t$ (note that this $\mathbf{N}_t$ is different from the $\mathbf{N}_t$ from that was used in the previous lemma) and a scalar $\chi\in \{0,1\}$. To this end, we expand the outer loop gradient:
    \begin{align}
        {\mathbf{G}}_t &\coloneqq \frac{1}{n}\sum_{i=1}^n \mathbf{B}_{t,i} \mathbf{w}_{t,i}\mathbf{w}_{t,i}^\top - \mathbf{B}_{\ast} \mathbf{w}_{\ast,t,i}\mathbf{w}_{t,i}^\top \nonumber \\ 
        &= \frac{1}{n}\sum_{i=1}^n {(\mathbf{B}_{t,i}(\mathbf{I}_k - \alpha \mathbf{B}_t^\top \mathbf{B}_{t}  ) \mathbf{w}_{t} - (\mathbf{I}_d - \alpha \mathbf{B}_{t,i} \mathbf{B}_t^\top )\mathbf{B}_{\ast} \mathbf{w}_{\ast,t,i}}
        )\mathbf{w}_{t,i}^\top \nonumber \\ 
        &= \frac{1}{n}\sum_{i=1}^n {\mathbf{B}_{t,i}\del_t \mathbf{w}_{t}\mathbf{w}_{t,i}^\top - \deld_t\mathbf{B}_{\ast} \mathbf{w}_{\ast,t,i}}\mathbf{w}_{t,i}^\top + \alpha^2 \mathbf{B}_t \mathbf{w}_t \mathbf{w}_t^\top \mathbf{B}_t^\top \mathbf{B}_{\ast} \mathbf{w}_{\ast,t,i}\mathbf{w}_{t,i}^\top \nonumber \\
        &\quad \quad \quad \quad \quad - \alpha^2 \mathbf{B}_\ast \mathbf{w}_{\ast,t,i} \mathbf{w}_t^\top \mathbf{B}_t^\top \mathbf{B}_{\ast} \mathbf{w}_{\ast,t,i}\mathbf{w}_{t,i}^\top 
         \nonumber \\
        &= -  \deld_t\mathbf{B}_{\ast}\bigg(\alpha\frac{1}{n}\sum_{i=1}^n \mathbf{w}_{\ast,t,i}\mathbf{w}_{\ast,t,i}^\top\bigg)\mathbf{B}_\ast^\top \mathbf{B}_t +
        \frac{1}{n}\sum_{i=1}^n \bigg( {\mathbf{B}_{t,i}\del_t \mathbf{w}_{t}\mathbf{w}_{t,i}^\top - \deld_t\mathbf{B}_{\ast} \mathbf{w}_{\ast,t,i}}\mathbf{w}_{t}^\top \del_t \nonumber \\
        &\quad + \alpha^2 \mathbf{B}_t \mathbf{w}_t \mathbf{w}_t^\top \mathbf{B}_t^\top \mathbf{B}_{\ast} \mathbf{w}_{\ast,t,i}\mathbf{w}_{t,i}^\top  - \alpha^2 \mathbf{B}_\ast \mathbf{w}_{\ast,t,i} \mathbf{w}_t^\top \mathbf{B}_t^\top \mathbf{B}_{\ast} \mathbf{w}_{\ast,t,i}\mathbf{w}_{t,i}^\top \bigg)
         \nonumber \\
         &= -\deld_t \mathbf{S}_t \mathbf{B}_t  + \mathbf{N}_t
    \end{align}
    where $\mathbf{S}_t \coloneqq \mathbf{B}_{\ast}\big(\alpha\frac{1}{n}\sum_{i=1}^n \mathbf{w}_{\ast,t,i}\mathbf{w}_{\ast,t,i}^\top\big)\mathbf{B}_\ast^\top$, \begin{align}\mathbf{N}_t &\coloneqq \frac{1}{n}\sum_{i=1}^n \big( {\mathbf{B}_{t,i}\del_t \mathbf{w}_{t}\mathbf{w}_{t,i}^\top - \deld_t\mathbf{B}_{\ast} \mathbf{w}_{\ast,t,i}}\mathbf{w}_{t}^\top\del_t + \alpha^2 \mathbf{B}_t \mathbf{w}_t \mathbf{w}_t^\top \mathbf{B}_t^\top \mathbf{B}_{\ast} \mathbf{w}_{\ast,t,i}\mathbf{w}_{t,i}^\top\nonumber \\
    &\quad \quad \quad \quad - \alpha^2 \mathbf{B}_\ast \mathbf{w}_{\ast,t,i} \mathbf{w}_t^\top \mathbf{B}_t^\top \mathbf{B}_{\ast} \mathbf{w}_{\ast,t,i}\mathbf{w}_{t,i}^\top \big),\end{align} and $\chi=0$. 
    Since $\sigma_{\min}(\mathbf{B}_t^\top \mathbf{S}_t \mathbf{B}_t) \geq E_0 \mu_\ast^2$ (by Lemma \ref{lem:sigminE}),  Lemma \ref{lem:gen_del} shows
    \begin{align}
        \|\del_{t+1}\|_2 &\leq (1 - \beta \alpha E_0 \mu_\ast^2)\|\del_{t}\|_2 + 2\beta \alpha \|\mathbf{B}_t^\top \mathbf{N}_t\|_2 + \beta^2 \alpha \|\mathbf{G}_t\|_2^2  
    \end{align}
   So, the remainder of the proof is to bound $\|\mathbf{B}_t^\top \mathbf{N}_t\|_2$ and $\|\mathbf{G}_t\|_2^2$. First we deal with $\|\mathbf{B}_t^\top \mathbf{N}_t\|_2$. We have
   \begin{align}
      \| \mathbf{B}_t^\top \mathbf{N}_t \|_2 & \leq \left\|\frac{1}{n}\sum_{i=1}^n  \mathbf{B}_t^\top\mathbf{B}_{t,i}\del_t \mathbf{w}_{t}\mathbf{w}_{t,i}^\top \right\|_2 + \left\|\frac{1}{n}\sum_{i=1}^n \mathbf{B}_t^\top \deld_t\mathbf{B}_{\ast} \mathbf{w}_{\ast,t,i}\mathbf{w}_{t}^\top \del_t \right\|_2  \nonumber \\
      &\quad + \left\|\frac{1}{n}\sum_{i=1}^n \alpha^2 \mathbf{B}_t^\top \mathbf{B}_t \mathbf{w}_t \mathbf{w}_t^\top \mathbf{B}_t^\top \mathbf{B}_{\ast} \mathbf{w}_{\ast,t,i}\mathbf{w}_{t,i}^\top  \right\|_2 + \left\|\frac{1}{n}\sum_{i=1}^n  \alpha^2 \mathbf{B}_t^\top \mathbf{B}_\ast \mathbf{w}_{\ast,t,i} \mathbf{w}_t^\top \mathbf{B}_t^\top \mathbf{B}_{\ast} \mathbf{w}_{\ast,t,i}\mathbf{w}_{t,i}^\top \right\|_2 \label{5fiveterms}
   \end{align}
 We consider each of the four terms in \eqref{5fiveterms} separately. 
 \begin{align}
     \left\|\frac{1}{n}\sum_{i=1}^n  \mathbf{B}_t^\top\mathbf{B}_{t,i}\del_t \mathbf{w}_{t}\mathbf{w}_{t,i}^\top \right\|_2 &\leq \left\| \mathbf{B}_t^\top\mathbf{B}_{t}\lam_t\del_t \mathbf{w}_{t}\mathbf{w}_{t}\del_t^\top \right\|_2 \nonumber \\
     &\quad + \alpha \left\|\frac{1}{n}\sum_{i=1}^n \mathbf{B}_t^\top\mathbf{B}_{t}\lam_t\del_t \mathbf{w}_{t}\mathbf{w}_{\ast,t,i}^\top\mathbf{B}_{\ast}^\top \mathbf{B}_t \right\|_2 \nonumber \\
     &\quad + \alpha\left\|\frac{1}{n}\sum_{i=1}^n  \mathbf{B}_t^\top\mathbf{B}_{\ast}\mathbf{w}_{\ast,t,i}\mathbf{w}_t^\top\del_t \mathbf{w}_{t}\mathbf{w}_{t}\del_t^\top \right\|_2 \nonumber \\
     &\quad + \alpha^2 \left\|\frac{1}{n}\sum_{i=1}^n  \mathbf{B}_t^\top\mathbf{B}_{\ast}\mathbf{w}_{\ast,t,i}\mathbf{w}_t^\top\del_t \mathbf{w}_{t}\mathbf{w}_{\ast,t,i}^\top\mathbf{B}_{\ast}^\top \mathbf{B}_t \right\|_2 \nonumber \\
     &\leq \tfrac{c}{\alpha}\|\del_t\|_2^2 \|\mathbf{w}_t\|_2^2 + \tfrac{c}{\sqrt{\alpha}} \|\del_t\|_2\|\mathbf{w}_t\|_2 \eta_\ast \nonumber \\
     &\quad \quad + c \sqrt{\alpha}\|\del_t\|_2^2\|\mathbf{w}_t\|_2^3 \eta_\ast + c\alpha \|\del_t\|_2\|\mathbf{w}_t\|_2^2 L_\ast^2 \nonumber \\
     \left\|\frac{1}{n}\sum_{i=1}^n \mathbf{B}_t^\top \deld_t\mathbf{B}_{\ast} \mathbf{w}_{\ast,t,i}\mathbf{w}_{t}^\top\del_t \right\|_2  &\leq \tfrac{c}{\sqrt{\alpha}}\|\del_t\|_2^2 \|\mathbf{w}_t\|_2 \eta_\ast  \nonumber \\
     \left\|\frac{1}{n}\sum_{i=1}^n \alpha^2 \mathbf{B}_t^\top \mathbf{B}_t \mathbf{w}_t \mathbf{w}_t^\top \mathbf{B}_t^\top \mathbf{B}_{\ast} \mathbf{w}_{\ast,t,i}\mathbf{w}_{t,i}^\top  \right\|_2 &\leq c \sqrt{\alpha} \|\mathbf{w}_t\|_2^2(\|\del_t\|_2\|\mathbf{w}_t\|_2\eta_\ast +  \sqrt{\alpha}L_\ast^2 ) \nonumber \\
    \left\|\frac{1}{n}\sum_{i=1}^n  \alpha^2 \mathbf{B}_t^\top \mathbf{B}_\ast \mathbf{w}_{\ast,t,i} \mathbf{w}_t^\top \mathbf{B}_t^\top \mathbf{B}_{\ast} \mathbf{w}_{\ast,t,i}\mathbf{w}_{t,i}^\top \right\|_2 &\leq  c\alpha \|\mathbf{w}_t\|_2( \|\del_t\|_2 \|\mathbf{w}_t\|_2 L_\ast^2 + \sqrt{\alpha} L_{\max}^3 )   \nonumber
 \end{align}
   Therefore, after applying inductive hypotheses $A_1(t)$ and $A_2(t)$, we obtain
  \begin{align}
      \|\mathbf{B}_t^\top \mathbf{N}_t\|_2&\leq 
      2 c \tfrac{E_0^2}{10} \alpha^{2} \mu_\ast^4  \nonumber .
  \end{align}
  Next we bound $\|\mathbf{G}_t\|_2^2$. Note that $\|\mathbf{G}_t\|_2 \leq \|\deld_t \mathbf{S}_t\mathbf{B}_t\|_2 + \|\mathbf{N}_t\|_2 $, and 
  \begin{align}
     \|\deld_t \mathbf{S}_t\mathbf{B}_t\|_2 &\leq {c}{\sqrt{\alpha}}L_\ast^2(\|\del_t \|_2   + \dist_t)  \nonumber \\
     &\leq  {c}{\sqrt{\alpha}}L_\ast^2( \alpha^2 \mu_\ast^2 E_0  + \dist_t) . \nonumber
  \end{align}
  Moreover,
  \begin{align}
      \|\mathbf{N}_t\|_2 &\leq \left\|\frac{1}{n}\sum_{i=1}^n  \mathbf{B}_{t,i}\del_t \mathbf{w}_{t}\mathbf{w}_{t,i}^\top \right\|_2 + \left\|\frac{1}{n}\sum_{i=1}^n  \deld_t\mathbf{B}_{\ast} \mathbf{w}_{\ast,t,i}\mathbf{w}_{t}^\top \del_t \right\|_2  \nonumber \\
      &\quad + \left\|\frac{1}{n}\sum_{i=1}^n \alpha^2  \mathbf{B}_t \mathbf{w}_t \mathbf{w}_t^\top \mathbf{B}_t^\top \mathbf{B}_{\ast} \mathbf{w}_{\ast,t,i}\mathbf{w}_{t,i}^\top  \right\|_2 + \left\|\frac{1}{n}\sum_{i=1}^n  \alpha^2 \mathbf{B}_\ast \mathbf{w}_{\ast,t,i} \mathbf{w}_t^\top \mathbf{B}_t^\top \mathbf{B}_{\ast} \mathbf{w}_{\ast,t,i}\mathbf{w}_{t,i}^\top \right\|_2 \nonumber \\
      &\leq \tfrac{3 cE_0}{10} \alpha^{2.5} \mu_\ast^4 \nonumber
  \end{align}
  thus
  \begin{align}
      \|\mathbf{G}_t\|_2^2 &\leq \left({c}{\sqrt{\alpha}}L_\ast^2(\alpha^2 \mu_\ast^2   + \dist_t)  + \tfrac{3cE_0}{10} \alpha^{2.5} \mu_\ast^4   \right)^2\nonumber \\
      &\leq 3c^2 \alpha^5 L_\ast^4 \mu_\ast^4 + 2c^2 \alpha L_\ast^4 \dist_t^2\nonumber \\
      &\leq 3 \alpha L_\ast^4
  \end{align}
 which means that
 \begin{align}
 \|\del_{t+1}\|_2 &\leq (1 - \beta \alpha E_0 \mu_\ast^2)\|\del_{t}\|_2 +   \tfrac{2 cE_0^2}{10} \beta \alpha^3 \mu_\ast^4  +  3 \beta^2 \alpha^2  L_\ast^4 \nonumber \\
&\leq \tfrac{1}{10}\alpha^2 E_0 \mu_\ast^2 - \tfrac{1}{10}\beta \alpha^3 E_0^2 \mu_\ast^4 +   \tfrac{2 cE_0^2}{10} \beta \alpha^3 \mu_\ast^4  +  3 \beta^2 \alpha^2  L_\ast^4 \nonumber \\
        &\leq \tfrac{1}{10}\alpha^2 E_0 \mu_\ast^2  \label{RHS}
    \end{align}
where \eqref{RHS} follows by choice of $\beta \leq \frac{ \alpha E_0^2  }{60\kappa_\ast^4}$.
\end{proof}

\begin{lemma}[FO-MAML $A_3(t+1)$] \label{lem:maml_ind3}
Suppose the conditions of Theorem \ref{thm:fomaml_pop_app} are satisfied and $A_1(t)$, $A_2(t)$, and $A_4(t)$ hold. Then $A_3(t+1)$ holds almost surely, i.e.
\begin{align}
    \|\mathbf{{B}}_{\ast,\perp}^\top \mathbf{B}_{t+1}\|_2\leq (1-0.5 \beta  \alpha E_0 \mu_{\ast}^2)\|\mathbf{{B}}_{\ast,\perp}^\top \mathbf{B}_{t}\|_2. \nonumber
\end{align}
\end{lemma}


\begin{proof}
Recalling the definition of $\mathbf{B}_{t+1}$ from \eqref{154} and noting that $\mathbf{{B}}_{\ast,\perp}^\top \mathbf{{B}}_\ast = \mathbf{0}$, we obtain
    \begin{align}
        \mathbf{{B}}_{\ast,\perp}^\top \mathbf{B}_{t+1} 
        &= \mathbf{{B}}_{\ast,\perp}^\top\mathbf{B}_t\left(\mathbf{I}_k-\beta (\mathbf{I}_k- \alpha \mathbf{w}_t \mathbf{w}_t^\top)\frac{1}{n}\sum_{i=1}^n\mathbf{w}_{t,i}\mathbf{w}_{t,i}^\top \right)
    \end{align}
Next, using the triangle and  Cauchy-Schwarz inequalities, we obtain
    \begin{align}
        \left\|\mathbf{I}_k - \beta\left( \mathbf{I}_k - \alpha \mathbf{w}_t \mathbf{w}_t^\top\right)\frac{1}{n}\sum_{i=1}^n \mathbf{w}_{t,i}\mathbf{w}_{t,i}^\top \right\|_2 
        &\leq   \left\|\mathbf{I}_k - \frac{\beta}{n}\sum_{i=1}^n \mathbf{w}_{t,i}\mathbf{w}_{t,i}^\top\right\| + 
        \beta \alpha \left\|\mathbf{w}_t \mathbf{w}_t^\top\frac{1}{n}\sum_{i=1}^n \mathbf{w}_{t,i}\mathbf{w}_{t,i}^\top \right\|_2 \nonumber \\
        &\leq 1 - \beta \sigma_{\min}\left(\frac{1}{n}\sum_{i=1}^n \mathbf{w}_{t,i}\mathbf{w}_{t,i}^\top \right) + 
        \beta \alpha \|\mathbf{w}_t\|_2^2 \left\|\frac{1}{n}\sum_{i=1}^n \mathbf{w}_{t,i}\mathbf{w}_{t,i}^\top \right\|_2  \nonumber \\
        &\leq 1 - \beta \left( \alpha E_0 \mu_{\ast}^2 - c \eta_\ast\sqrt{\alpha}\|\mathbf{w}_t\|_2 \|\del_t\|_2  \right) \label{diverse} \\
        &\quad + 
        c \beta \alpha \|\mathbf{w}_t\|_2^2 \left( \|\mathbf{w}_t\|_2 \|\del_t\|_2^2 + \eta_\ast\sqrt{\alpha }\|\mathbf{w}_t\|_2 \|\del_t\|_2 +  \alpha L_\ast^2 \right) \nonumber \\
        &\leq 1 -   \beta\alpha E_0 \mu_{\ast}^2 +  2\tfrac{E_0^2}{100} \beta {\alpha}^3  \mu_\ast^3 \eta_\ast \kappa_{\max,\ast}^{-3} +
        c \tfrac{E_0^4}{100} \beta \alpha^3 \mu_\ast^2  L_\ast^2 \kappa_{\ast,\max}^{-6} \nonumber \\
        &\leq 1 - 0.5 \beta\alpha E_0 \mu_{\ast}^2 \label{finall}
    \end{align}
    where \eqref{diverse} follows by the diversity of the inner loop-updated heads (Lemma \ref{lem:sigmin}) and \eqref{finall} follows from $\alpha \leq 1/(4 L_{\max})$.
\end{proof}


    

\subsection{Exact MAML}

The first step in the analysis is to compute the second-order outer loop updates for Exact MAML. To do so, we must compute the loss on task $i$ at iteration $t$ after one step of gradient descent for both the representation and head.
Let $\boldsymbol{\Lambda}_t \coloneqq \mathbf{I}_k - \alpha \mathbf{w}_t\mathbf{w}_t^\top$,  $\boldsymbol{{\Delta}}_t\coloneqq \mathbf{I}_k -\alpha  \mathbf{B}_t^\top\mathbf{B}_t$, and $\deld_t \coloneqq \mathbf{I}_d - \alpha \mathbf{B}_t \mathbf{B}_t^\top$.
Note that 
\begin{align}
    F_{t,i}(\mathbf{B}_{t}, \mathbf{w}_{t}) \coloneqq  \mathcal{L}_{t,i}(\mathbf{B}_{t}-\alpha \nabla_\mathbf{B}\mathcal{L}_{t,i}(\mathbf{B}_{t}, \mathbf{w}_{t}), \mathbf{w}_{t}-\alpha \nabla_\mathbf{w}\mathcal{L}_{t,i}(\mathbf{B}_{t}, \mathbf{w}_{t})) = \tfrac{1}{2}\|\mathbf{v}_{t,i}\|_2^2
\end{align}
where
\begin{align}
    \mathbf{v}_{t,i} &= \mathbf{B}_t\boldsymbol{\Lambda}_t\del_t \mathbf{w}_t  +\alpha \mathbf{B}_t\boldsymbol{\Lambda}_t\mathbf{B}_t^\top \mathbf{B}_\ast \mathbf{w}_{\ast,t,i} +\alpha \mathbf{B}_\ast \mathbf{w}_{\ast,t,i}\mathbf{w}_t^\top \del_t \mathbf{w}_t +\alpha^2\mathbf{B}_\ast \mathbf{w}_{\ast,t,i}\mathbf{w}_{t}^\top \mathbf{B}_t^\top \mathbf{B}_\ast \mathbf{w}_{\ast,t,i} - \mathbf{B}_\ast \mathbf{w}_{\ast,t,i} \nonumber \\
    &= \boldsymbol{\bar{\Delta}}_t (\mathbf{B}_t\mathbf{w}_t -  \mathbf{B}_\ast\mathbf{w}_{\ast,t,i}) -
    \alpha(\mathbf{B}_t\mathbf{w}_t - \mathbf{B}_\ast\mathbf{w}_{\ast,t,i})\mathbf{w}_t^\top \del_t\mathbf{w}_t -\alpha^2 \mathbf{B}_t\mathbf{w}_t\mathbf{w}_t^\top \mathbf{B}_t^\top \mathbf{B}_\ast \mathbf{w}_{\ast,t,i} \nonumber \\ 
    &\quad + \alpha^2 \mathbf{B}_\ast\mathbf{w}_{\ast,t,i} \mathbf{w}_{\ast,t,i}^\top \mathbf{B}_\ast^\top  \mathbf{B}_t\mathbf{w}_t \nonumber \\
    &= (\deld_t - (\alpha \omega_t +\alpha^2 a_{t,i}) \mathbf{I}_d) (\mathbf{B}_t \mathbf{w}_t - \mathbf{B}_\ast \mathbf{w}_{\ast,t,i}) \label{gggrad} 
\end{align}
where $a_{t,i}\coloneqq \mathbf{w}_{\ast,t,i}^\top\mathbf{B}_\ast^\top \mathbf{B}_t \mathbf{w}_t \; \forall t,i$ and $\omega_t \coloneqq \mathbf{w}_t^\top \del_t \mathbf{w}_t \; \forall t$. The outer loop updates for Exact MAML are given by:
\begin{align}
    \mathbf{w}_{t+1} &= \mathbf{w}_t - \frac{\beta}{n}\sum_{i=1}^n \nabla_{\mathbf{w}}F_{t,i}(\mathbf{B}_t, \mathbf{w}_t) \nonumber \\
    \mathbf{B}_{t+1} &= \mathbf{B}_t - \frac{\beta}{n}\sum_{i=1}^n \nabla_{\mathbf{B}}F_{t,i}(\mathbf{B}_t, \mathbf{w}_t) \nonumber
\end{align}
Again, we prove a more general version of Theorem \ref{thm:maml_pop_main} in which we allow for general $L_{\max}$. First we make the following assumption.
\begin{assumption}[Exact MAML Initialization] \label{assump:exactmaml_init}
The distance of the initial representation to the ground-truth representation satisfies:
\begin{equation}\dist_0 \leq \tfrac{1}{17}\kappa_{\ast,\max}^{-1.5}.
\end{equation}
\end{assumption}
\begin{thm}[Exact MAML Infinite Samples] \label{thm:exact_maml_pop}
Let $m_{in}=m_{out}= \infty$ and define $E_0 \coloneqq 0.9 - \dist_0^2$.
Suppose that $\alpha \leq  \tfrac{E_0^{1/4}\kappa_\ast^{3/4}(\nicefrac{L_\ast}{L_{\max}})^{1/4}}{4 L_{\max} T^{1/4} }$
 and $\beta \leq \frac{E_0\alpha}{ 10  \kappa_\ast^4 }$,
     $\alpha\mathbf{B}_0^\top \mathbf{B}_0 = \mathbf{I}_k$, $\mathbf{w}_0=\mathbf{0}$ and
Assumptions \ref{assump:tasks_diverse_main}, 
\ref{assump:exactmaml_init}, and \ref{assump:lmax}  hold.
Then Exact MAML satisfies 
\begin{align}
    \dist_T \leq (1 - 0.5 \beta \alpha E_0 \mu_\ast^2)^{T-1}
\end{align}
\end{thm}


\begin{proof}
 The proof follows by showing that the following inductive hypotheses hold for all $t\in[T]$:
\begin{enumerate}
 \item $A_1(t)\coloneqq \{\|\mathbf{w}_{t}\|_2 \leq \|\mathbf{w}_{t-1}\|_2 +16 \beta  \alpha^{3.5} L_{\max}^5 t +  3 \beta \alpha^{1.5}L_{\max}^3 \dist_t^2
 \}$
    \item $A_2(t)\coloneqq \{\|\mathbf{w}_{t}\|_2 \leq \tfrac{E_0}{20} \sqrt{\alpha} \mu_\ast \}$.
    \item $A_3(t)\coloneqq \|\del_{t}\|_2 \leq \alpha^2L_{\max}^2 $
    \item $A_4(t) \coloneqq \{\|\mathbf{B}_{\ast,\perp}^\top \mathbf{B}_t\|_2 \leq (1-0.5 \beta \alpha E_0 \mu_\ast^2)\|\mathbf{B}_{\ast,\perp}^\top \mathbf{B}_{t-1}\|_2 \}$
    \item $A_5(t) \coloneqq \{\dist_t \leq \frac{\sqrt{10}}{3} (1 -0.5 \beta \alpha E_0 \mu_\ast^2 )^{t-1} \dist_0\} $
    \item $A_6(t) \coloneqq \{\dist_t \leq  (1 -0.5 \beta \alpha E_0 \mu_\ast^2 )^{t-1} \} $
\end{enumerate}
 These conditions hold for iteration $t=0$ due to the choice of initialization. Now, assuming they hold for arbitrary $t$, we will show they hold at $t+1$. 
\begin{enumerate}
    \item $A_2(t)\cap A_3(t) \implies A_1({t+1})$.
    This is Lemma \ref{lem:maml_exact_ind1}.
    
    \item $\bigcap_{s=1}^t \{A_1(s)\cap A_5(s) \}  \implies A_2({t+1})$. This is Lemma \ref{lem:maml_exact_ind2}.

    \item $A_2(t)\cap A_3(t) \cap A_5(t) \implies A_3(t+1)$.
    This is Lemma \ref{lem:maml_exact_ind3}.
    
    \item $A_2(t)\cap A_3(t) \cap A_5(t) \implies A_4(t+1)$.
    This is Lemma \ref{lem:maml_exact_ind4}.
    
    \item $ A_3(t+1) \cap \bigcap_{s=1}^{t+1} A_4(s) \implies A_5(t+1) \cap A_6(t+1)$.
    Note that $A_3(t+1) \cap \bigcap_{s=1}^{t+1} A_4(s)$ and $\alpha \leq 1/(4L_{\max})$ implies
\begin{align}
\tfrac{\sqrt{1-0.1}}{\sqrt{\alpha}} \dist_{t+1}  &= \tfrac{\sqrt{1-0.1}}{\sqrt{\alpha}}\|\mathbf{{B}}_{\ast,\perp}^\top \mathbf{\hat{B}}_{t+1}\|_2 \nonumber \\
&\leq \sigma_{\min}(\mathbf{{B}}_{t+1}) \|\mathbf{{B}}_{\ast,\perp}^\top \mathbf{\hat{B}}_{t+1}\|_2\;  \nonumber \\
 &\leq \|\mathbf{{B}}_{\ast,\perp}^\top \mathbf{B}_{t+1}\|_2   \nonumber \\
    &\leq \left(1 - 0.5 \beta \alpha E_0 \mu_\ast^2 \right)^t  \|\mathbf{{B}}_{\ast,\perp}^\top \mathbf{B}_{0}\|_2 
    \nonumber \\
    &\leq \tfrac{1}{\sqrt{\alpha}} \left(1 - 0.5 \beta \alpha E_0 \mu_\ast^2 \right)^t  \|\mathbf{{B}}_{\ast,\perp}^\top \mathbf{\hat{B}}_{0}\|_2 \label{innit} \\
    &= \tfrac{1}{\sqrt{\alpha}}\left(1 - 0.5 \beta \alpha E_0 \mu_\ast^2 \right)^t  \dist_0, \nonumber 
\end{align}
where \eqref{innit} follows due to initialization $\|\mathbf{B}_0\|_2=\frac{1}{\sqrt{\alpha}}$.
This implies 
\begin{align}
 \dist_{t+1} \leq \tfrac{\sqrt{10}}{3}\left(1 - 0.5 \beta \alpha E_0 \mu_\ast^2 \right)^t  \dist_0 \leq \left(1 - 0.5 \beta \alpha E_0 \mu_\ast^2 \right)^t 
\end{align}
since $\alpha \leq 1/(4L_{\max})$ and $\dist_0 \leq \tfrac{3}{\sqrt{10}}$ by Assumption \ref{assump:exactmaml_init}.
\end{enumerate}
\end{proof}



Next, we complete the proof of Theorem \ref{thm:exact_maml_pop} by proving the following lemmas. 

\begin{lemma}[Exact MAML $A_1(t)$]  \label{lem:maml_exact_ind1}
Suppose Assumptions \ref{assump:tasks_diverse_main} and \ref{assump:exactmaml_init} hold, and $A_2(t)$ and $A_3(t)$ hold. Then
\begin{align}
    \|\mathbf{w}_{t+1}\|_2 \leq \|\mathbf{w}_t\|_2 +16\alpha^{3.5}L_{\max}^5 + 3 \alpha^{1.5}L_{\max}^3 \dist_t^2
\end{align}
\end{lemma}

\begin{proof}
Using \eqref{gggrad} and the chain rule (while noting that $a_{t,i}$ is a function of $\mathbf{w}_t$), we find that for all $i\in [n]$, the gradient of $F_{t,i}(\mathbf{B}_t, \mathbf{w}_t)$ with respect to $\mathbf{w}_t$ is:
\begin{align}
   \nabla_{\mathbf{w}} F_{t,i}(\mathbf{B}_t,\mathbf{w}_t)&=
   (\mathbf{B}_t\del_t -\alpha \omega_t\mathbf{B}_t -\alpha^2 a_{t,i}\mathbf{B}_t)^\top(\mathbf{B}_t\del_t-\alpha \omega_t\mathbf{B}_t -\alpha^2 a_{t,i}\mathbf{B}_t) \mathbf{w}_t \nonumber \\
   &\quad -
    \mathbf{B}_t^\top (\deld_t - \alpha\omega_t\mathbf{I}_d -\alpha^2 a_{t,i}\mathbf{I}_d)^2\mathbf{B}_\ast \mathbf{w}_{\ast,t,i} \nonumber\\
    &\quad - 2\alpha\del_t \mathbf{w}_t( \mathbf{B}_t \mathbf{w}_t  - \mathbf{B}_\ast \mathbf{w}_{\ast,t,i} )^\top \mathbf{v}_{t,i} - \alpha^2 \mathbf{B}_t^\top  \mathbf{B}_\ast \mathbf{w}_{\ast,t,i} (\mathbf{B}_t\mathbf{w}_t - \mathbf{B}_{\ast}\mathbf{w}_{\ast,t,i})^\top \mathbf{v}_{t,i} \nonumber \\
    &=  (\mathbf{B}_t\del_t -\alpha \omega_t\mathbf{B}_t -\alpha^2 a_{t,i}\mathbf{B}_t)^\top(\mathbf{B}_t\del_t -\alpha \omega_t\mathbf{B}_t-\alpha^2 a_{t,i}\mathbf{B}_t) \mathbf{w}_t  + \mathbf{N}_{t,i}
\end{align}
where $\mathbf{N}_{t,i}\!\coloneqq 
-\mathbf{B}_t^\top (\deld_t - \alpha\omega_t\mathbf{I}_d -\alpha^2 a_{t,i}\mathbf{I}_d)^2\mathbf{B}_\ast \mathbf{w}_{\ast,t,i} -  2\alpha\del_t \mathbf{w}_t( \mathbf{B}_t \mathbf{w}_t  - \mathbf{B}_\ast \mathbf{w}_{\ast,t,i} )^\top \mathbf{v}_{t,i} \\
- \alpha^2 \mathbf{B}_t^\top  \mathbf{B}_\ast \mathbf{w}_{\ast,t,i} (\mathbf{B}_t\mathbf{w}_t - \mathbf{B}_{\ast}\mathbf{w}_{\ast,t,i})^\top \mathbf{v}_{t,i}$. 
    Thus, 
    \begin{align}
        \mathbf{w}_{t+1} &= \mathbf{w}_t -  \frac{\beta}{n}\sum_{i=1}^n \nabla_{\mathbf{w}}F_{t,i}(\mathbf{B}_t, \mathbf{w}_t)\nonumber \\
        &= \left(\mathbf{I}_k - \frac{\beta}{n}\sum_{i=1}^n (\mathbf{B}_t\del_t -\alpha \omega_t\mathbf{B}_t -\alpha^2 a_{t,i}\mathbf{B}_t)^\top(\mathbf{B}_t\del_t -\alpha \omega_t\mathbf{B}_t -\alpha^2 a_{t,i}\mathbf{B}_t)\right)\mathbf{w}_t - \frac{\beta}{n}\sum_{i=1}^n \mathbf{N}_{t,i}
    \end{align}
    which implies that 
    \begin{align}
        \|\mathbf{w}_{t+1}\|_2 &\leq \left\|\mathbf{I}_k - \frac{\beta}{n}\sum_{i=1}^n (\mathbf{B}_t\del_t -\alpha \omega_t\mathbf{B}_t -\alpha^2 a_{t,i}\mathbf{B}_t)^\top(\mathbf{B}_t\del_t -\alpha \omega_t\mathbf{B}_t -\alpha^2 a_{t,i}\mathbf{B}_t) \right\|_2 \|\mathbf{w}_t\|_2 \nonumber \\
        &\quad + \beta \left\|\frac{1}{n}\sum_{i=1}^n \mathbf{N}_{t,i} \right\|_2 \nonumber \\
        &\leq \|\mathbf{w}_t\|_2 + \beta \left\|\frac{1}{n}\sum_{i=1}^n \mathbf{N}_{t,i} \right\|_2 \label{229} 
    \end{align}
    where \eqref{229} follows since $\frac{1}{n}\sum_{i=1}^n (\mathbf{B}_t\del_t -\alpha \omega_t\mathbf{B}_t -\alpha^2 a_{t,i}\mathbf{B}_t)^\top(\mathbf{B}_t\del_t -\alpha \omega_t\mathbf{B}_t -\alpha^2 a_{t,i}\mathbf{B}_t)$ is PSD and $\beta$ is sufficiently small. 
    Next, we upper bound $\left\|\frac{1}{n}\sum_{i=1}^n \mathbf{N}_{t,i} \right\|_2$, and
to do so,
we first use the triangle inequality to write
\begin{align}
    \left\|\frac{1}{n}\sum_{i=1}^n \mathbf{N}_{t,i}\right\|_2 &\leq \left\|\frac{1}{n}\sum_{i=1}^n \mathbf{B}_t^\top (\deld_t -\alpha \omega_t\mathbf{I}_d  -\alpha^2 a_{t,i}\mathbf{I}_d)^2\mathbf{B}_\ast \mathbf{w}_{\ast,t,i} \right\|_2 \nonumber \\
    &\quad +  2\left\|\frac{1}{n}\sum_{i=1}^n {\alpha  \del_t \mathbf{w}_t( \mathbf{B}_t \mathbf{w}_t  - \mathbf{B}_\ast \mathbf{w}_{\ast,t,i} )^\top \mathbf{v}_{t,i}} \right\|_2 \nonumber \\
    &\quad +     \left\|\frac{1}{n}\sum_{i=1}^n \alpha^2 \mathbf{B}_t^\top  \mathbf{B}_\ast \mathbf{w}_{\ast,t,i} (\mathbf{B}_t\mathbf{w}_t - \mathbf{B}_{\ast}\mathbf{w}_{\ast,t,i})^\top \mathbf{v}_{t,i} \right\|_2 
    \label{ggggrad}
\end{align}
We will bound each of the three terms above shortly. First, note that $A_4(t)$ and $\alpha \leq 1/(4L_{\max})$ implies
\begin{align}
    \frac{\nicefrac{15}{16}}{{\alpha}}&\leq \frac{1- \alpha^2 L_{\max}^2 }{{\alpha}} \leq\sigma_{\min}^2(\mathbf{B}_t) \leq \sigma_{\max}^2(\mathbf{B}_t) \leq \frac{1+ \alpha^2 L_{\max}^2}{{\alpha}} \leq \frac{\nicefrac{17}{16}}{{\alpha}}.
\end{align}
In turn, this implies that $\|\mathbf{B}_t\|_2^3\leq 1.1$. Let $c \coloneqq 1.1$.
 Also, note that $\mathbf{B}_t^\top \deld_t = \del_t \mathbf{B}_t^\top$ and
\begin{align}
    \|\mathbf{B}_{\ast}^\top \deld_t \mathbf{B}_{\ast}\|_2 &= \|\mathbf{B}_{\ast}^\top (\mathbf{I}_d- \alpha \mathbf{B}_t\mathbf{B}_t^\top) \mathbf{B}_{\ast}\|_2  \nonumber \\
    &\leq \|\mathbf{B}_{\ast}^\top (\mathbf{I}_d-  \mathbf{\hat{B}}_t\mathbf{\hat{B}}_t^\top) \mathbf{B}_{\ast}\|_2 + \|\mathbf{B}_{\ast}^\top ( \mathbf{\hat{B}}_t\mathbf{\hat{B}}_t^\top - \alpha \mathbf{{B}}_t\mathbf{{B}}_t^\top) \mathbf{B}_{\ast}\|_2 \nonumber \\
    &\leq  \|\mathbf{B}_{\ast}^\top (\mathbf{I}_d-  \mathbf{\hat{B}}_t\mathbf{\hat{B}}_t^\top)\|\|(\mathbf{I}_d-  \mathbf{\hat{B}}_t\mathbf{\hat{B}}_t^\top) \mathbf{B}_{\ast}\|_2 + \| \mathbf{\hat{B}}_t\mathbf{\hat{B}}_t^\top - \alpha \mathbf{{B}}_t\mathbf{{B}}_t^\top\|_2 \nonumber \\
    &= \dist_t^2 + \| \mathbf{\hat{B}}_t(\mathbf{I}_k - \alpha \mathbf{{R}}_t\mathbf{{R}}_t^\top)\mathbf{\hat{B}}_t^\top\|_2 \nonumber \\
    &\leq \dist_t^2 + \|\del_t\|_2
\end{align}
where $\mathbf{{R}}_t\in\mathbb{R}^{k\times k}$ is the upper triangular matrix resulting from the QR decomposition of $\mathbf{B}_t$. We will use these observations along with inductive hypotheses $A_2(t)$ and $A_3(t)$ and the  Cauchy-Schwarz and triangle inequalities to separately bound each of the terms from \eqref{ggggrad} as follows. Let $c_2 \coloneqq E_0/20$. Then we have:
\begin{align}
   \bigg\|\frac{1}{n}\sum_{i=1}^n \;& \mathbf{B}_t^\top (\deld_t -(\alpha\omega_t+\alpha^2 a_{t,i})\mathbf{I}_d)^2\mathbf{B}_\ast \mathbf{w}_{\ast,t,i} \bigg\|_2 \nonumber \\
   &\leq \left\|\frac{1}{n}\sum_{i=1}^n  \del_t^2 \mathbf{B}_t^\top \mathbf{B}_\ast \mathbf{w}_{\ast,t,i} \right\|_2  + 2\left\|\frac{1}{n}\sum_{i=1}^n \alpha \omega_{t}  \del_t \mathbf{B}_t^\top \mathbf{B}_\ast \mathbf{w}_{\ast,t,i} \right\|_2 \nonumber \\
    &\quad + 2\left\|\frac{1}{n}\sum_{i=1}^n \alpha^2 a_{t,i}  \del_t \mathbf{B}_t^\top \mathbf{B}_\ast \mathbf{w}_{\ast,t,i} \right\|_2 + 2\left\|\frac{1}{n}\sum_{i=1}^n \alpha^3 \omega_t a_{t,i} \mathbf{B}_t^\top \mathbf{B}_\ast \mathbf{w}_{\ast,t,i} \right\|_2  \nonumber\\
    &\quad + \left\|\frac{1}{n}\sum_{i=1}^n \alpha^2 \omega_{t}^2 \mathbf{B}_t^\top\mathbf{B}_\ast \mathbf{w}_{\ast,t,i} \right\|_2 + \left\|\frac{1}{n}\sum_{i=1}^n \alpha^4 a_{t,i}^2 \mathbf{B}_t^\top\mathbf{B}_\ast \mathbf{w}_{\ast,t,i} \right\|_2 \nonumber \\
    &\leq c \alpha^{3.5}L_{\max}^4 \eta_\ast + 2c  \alpha^{4.5} L_{\max}^4 \eta_\ast\|\mathbf{w}_t\|_2^2 + 2 c \alpha^{3} L_{\max}^2 L_{\ast}^2 \|\mathbf{w}_t\|_2 
\nonumber \\
    &\quad + 2 c  \alpha^{4} L_{\max}^2 L_\ast^2 \|\mathbf{w}_t\|_2^3 + c \alpha^{3.5} L_{\max}^4\eta_\ast \|\mathbf{w}_t\|_2^4 + c \alpha^{2.5} L_{\max}^3 \|\mathbf{w}_t\|_2^2 \nonumber \\
    &\leq  c \alpha^{3.5}L_{\max}^4 \eta_\ast + 2c c_2^2 \alpha^{5.5} L_{\max}^4 \eta_\ast \mu_\ast^2 \kappa_{\ast,\max}^{-2} + 2 cc_2 \alpha^{3.5} L_{\max}^2 L_{\ast}^2 \mu_\ast \kappa_{\ast,\max}^{-1}
\nonumber \\
    &\quad + 2 c c_2^3 \alpha^{5.5} L_{\max}^2 L_\ast^2 \mu_\ast^3 \kappa_{\ast,\max}^{-3} + c c_2^4 \alpha^{5.5} L_{\max}^4\eta_\ast \mu_\ast^4 \kappa_{\ast,\max}^{-4} + c c_2^2 \alpha^{3.5} L_{\max}^3 \mu_\ast^2 \kappa_{\ast,\max}^{-2}  \nonumber \\
    &\leq 4c \alpha^{3.5}L_{\max}^4 (\eta_\ast+\mu_\ast)
\end{align}
    \begin{align}
    \bigg\|\frac{1}{n}\sum_{i=1}^n& \alpha \del_t \mathbf{w}_{t} (\mathbf{B}_t\mathbf{w}_t - \mathbf{B}_{\ast}\mathbf{w}_{\ast,t,i})^\top \mathbf{v}_{t,i} \bigg\|_2  \nonumber \\
    \quad &\leq \left\|\frac{1}{n}\sum_{i=1}^n \alpha \del_t\mathbf{w}_t \mathbf{w}_t^\top \mathbf{B}_t^\top (\deld_t -\alpha\omega_t\mathbf{I}_d -\alpha^2 a_{t,i}\mathbf{I}_d) (\mathbf{B}_t \mathbf{w}_t - \mathbf{B}_\ast \mathbf{w}_{\ast,t,i}) \right\|_2 \nonumber \\
  \quad  &\quad+ \left\|\frac{1}{n}\sum_{i=1}^n \alpha \del_t \mathbf{w}_{t}  \mathbf{w}_{\ast,t,i}^\top \mathbf{B}_{\ast}^\top (\deld_t -\alpha\omega_t\mathbf{I}_d -\alpha^2 a_{t,i}\mathbf{I}_d) (\mathbf{B}_t \mathbf{w}_t - \mathbf{B}_\ast \mathbf{w}_{\ast,t,i}) \right\|_2 \nonumber \\
  \quad  &\leq c \alpha^4 L_{\max}^4 \|\mathbf{w}_t\|_2^3 + 2c \alpha^{4.5} L_{\max}^4\eta_\ast \|\mathbf{w}_t\|_2^2 + c \alpha^{5}L_{\max}^4 \|\mathbf{w}_t\|_2^5 +  2c \alpha^{5.5}L_{\max}^4\eta_\ast \|\mathbf{w}_t\|_2^4 \nonumber \\
   \quad &\quad + c \alpha^{3.5} L_{\max}^2 \eta_\ast \|\mathbf{w}_t\|_2^4 + 2c\alpha^{4}L_{\max}^2 L_{\ast}^2 \|\mathbf{w}_t\|_2^3 +  c \alpha^{5}  L_{\max}^6 \|\mathbf{w}_t\|_2^3  \nonumber \\
 \quad   &\quad + c \alpha^{4.5} L_{\max}^5 \|\mathbf{w}_t\|_2^2 + c \alpha^{5} L_{\max}^6 \|\mathbf{w}_t\|_2 + c \alpha^3  L_{\max}^4 \dist_t^2 \|\mathbf{w}_t\|_2 \nonumber \\
 \quad   &\leq  c c_2^3 \alpha^{5.5} L_{\max}^4 \mu_\ast^3 \kappa_{\ast,\max}^{-3} + 2cc_2^2 \alpha^{5.5} L_{\max}^4\eta_\ast \mu_\ast^2\kappa_{\ast,\max}^{-2} + c c_2^5 \alpha^{7.5}L_{\max}^4 \mu_\ast^5\kappa_{\ast,\max}^{-5}  \nonumber \\
 \quad   &\quad +  2c c_2^4 \alpha^{7.5}L_{\max}^4\eta_\ast \mu_\ast^4 \kappa_{\ast,\max}^{-4} + c c_2^4 \alpha^{5.5} L_{\max}^2 \eta_\ast \mu_\ast^4 \kappa_{\ast,\max}^{-4} + 2c c_2^3 \alpha^{5.5}L_{\max}^2 L_{\ast}^2 \mu_\ast^3 \kappa_{\ast,\max}^{-3}   \nonumber \\
  \quad  &\quad +  c c_2^3 \alpha^{7.5}  L_{\max}^6 \mu_\ast^3 \kappa_{\ast,\max}^{-3} + c c_2^2 \alpha^{5.5} L_{\max}^5 \mu_\ast^2 \kappa_{\ast,\max}^{-2} + c c_2 \alpha^{5.5} L_{\max}^6 \mu_\ast \kappa_{\ast,\max}^{-1} \nonumber \\
  \quad  &\quad + cc_2 \alpha^{3.5}  L_{\max}^4 \mu_\ast \kappa_{\ast,\max}^{-1} \dist_t^2  \nonumber \\
  \quad  &\leq 9c c_2 \alpha^{5.5} L_{\max}^5 \mu_\ast^2 + cc_2 \alpha^{3.5}  L_{\max}^3 \mu_\ast^2 \dist_t^2
\end{align}

\begin{align}
   \bigg\|\frac{1}{n}\sum_{i=1}^n &\alpha^2 \mathbf{B}_t^\top  \mathbf{B}_\ast \mathbf{w}_{\ast,t,i} (\mathbf{B}_t\mathbf{w}_t - \mathbf{B}_{\ast}\mathbf{w}_{\ast,t,i})^\top \mathbf{v}_{t,i} \bigg\|_2 \nonumber \\
   &\leq \left\|\frac{1}{n}\sum_{i=1}^n \alpha^2 \mathbf{B}_t^\top  \mathbf{B}_\ast \mathbf{w}_{\ast,t,i} \mathbf{w}_t^\top \mathbf{B}_t^\top (\deld_t  -\alpha \omega_{t}\mathbf{I}_d -\alpha^2 a_{t,i}\mathbf{I}_d) (\mathbf{B}_t \mathbf{w}_t - \mathbf{B}_\ast \mathbf{w}_{\ast,t,i}) \right\|_2 \nonumber \\
   &\quad+ \bigg\|\frac{1}{n}\sum_{i=1}^n \alpha^2 \mathbf{B}_t^\top  \mathbf{B}_\ast \mathbf{w}_{\ast,t,i}  \mathbf{w}_{\ast,t,i}^\top \mathbf{B}_{\ast}^\top (\deld_t -\alpha \omega_{t}\mathbf{I}_d -\alpha^2 a_{t,i}\mathbf{I}_d) (\mathbf{B}_t \mathbf{w}_t - \mathbf{B}_\ast \mathbf{w}_{\ast,t,i}) \bigg\|_2 \nonumber \\
   &\leq c\alpha^{2.5}L_{\max}^2 \eta_\ast \|\mathbf{w}_t\|_2^2 
   + c\alpha^{3} L_{\max}^2 L_\ast^2 \|\mathbf{w}_t\|_2 
   + c\alpha^{3.5} L_{\max}^2\eta_\ast \|\mathbf{w}_t\|_2^4 
   \nonumber \\
   &\quad  + c\alpha^{2.5}L_{\max}^3 \|\mathbf{w}_t\|_2^2 
   + c\alpha^4 L_{\max}^2 L_\ast^2 \|\mathbf{w}_t\|_2^3 
   + c \alpha^2 L_\ast^2\|\mathbf{w}_{t}\|_2^3 
   \nonumber \\
   &\quad + c\alpha^3 L_{\max}^2L_{\ast}^2 \|\mathbf{w}_t\|_2 + c\alpha^{2.5} L_{\max}^3 \|\mathbf{w}_t\|_2^2  \nonumber \\
   &\quad + c \alpha^{3.5}L_{\max}^2 L_{\max}^3 + c \alpha^{1.5} L_{\max}^3 \dist_t^2 + c \alpha^3 L_{\max}^4 \|\mathbf{w}_t\|_2 \nonumber \\
   &\quad + c\alpha^{4.5} L_{\max}^5  \|\mathbf{w}_t\|_2^2 + c \alpha^{4}L_{\max}^2L_{\ast}^2\eta_\ast \|\mathbf{w}_t\|_2^3 \nonumber \\
   &\leq c c_2^2 \alpha^{3.5}L_{\max}^2 \eta_\ast \mu_\ast^2 \kappa_{\ast,\max}^{-2}  + c c_2 \alpha^{3.5} L_{\max}^2 L_\ast^2 \mu_\ast \kappa_{\ast,\max}^{-1} + c c_2^4 \alpha^{5.5} L_{\max}^2\eta_\ast \mu_\ast^4 \kappa_{\ast,\max}^{-4} \nonumber \\
   &\quad +  cc_2^2\alpha^{3.5}L_{\max}^3 \mu_\ast^2\kappa_{\ast,\max}^{-2}  + cc_2^3\alpha^{5.5} L_{\max}^2 L_\ast^2 \mu_\ast^3 \kappa_{\ast,\max}^{-3}  + c c_2^3 \alpha^{3.5} L_\ast^2\mu_\ast^3 \kappa_{\ast,\max}^{-3} \nonumber \\
   &\quad + cc_2 \alpha^{3.5} L_{\max}^2L_{\ast}^2 \mu_\ast \kappa_{\ast,\max}^{-1} + cc_2^2\alpha^{3.5} L_{\max}^3 \mu_\ast^2\kappa_{\ast,\max}^{-2}   \nonumber \\
   &\quad + c \alpha^{3.5}L_{\max}^2 L_{\max}^3 + c \alpha^{1.5} L_{\max}^3 \dist_t^2 + cc_2 \alpha^{3.5} L_{\max}^4\mu_\ast\kappa_{\ast,\max}^{-1}   \nonumber \\
   &\quad + c c_2^2\alpha^{5.5} L_{\max}^5  \mu_\ast^2\kappa_{\ast,\max}^{-2} + c \alpha^{5.5}L_{\max}^2L_{\ast}^2\eta_\ast \mu_\ast^3 \kappa_{\ast,\max}^{-3}  \nonumber \\
   &\leq 9 c \alpha^{3.5} L_{\max}^5 + c \alpha^{1.5} L_{\max}^3 \dist_t^2
\end{align}
where we have used $c_2 = E_0/20$ and $\alpha\leq 1/(4L_{\max})$ to reduce terms.
Next, combining the above bounds with  \eqref{ggggrad} yields:
\begin{align}
    \left\|\frac{1}{n}\sum_{i=1}^n \mathbf{N}_{t,i}\right\|_2 &\leq  14c\alpha^{3.5}L_{\max}^5 + 2 c \alpha^{1.5}L_{\max}^3 \dist_t^2
\end{align}
Applying $c= 1.1$ yields the result.
\end{proof}

\begin{lemma}[Exact MAML $A_2(t+1)$]  \label{lem:maml_exact_ind2}
Suppose the conditions of Theorem \ref{thm:exact_maml_pop} are satisfied, and $\bigcap_{s=1}^{t}\{A_1(s)\cap A_5(s)\}$.
Then 
\begin{align}
    \|\mathbf{w}_{t+1}\|_2 \leq \tfrac{E_0}{20} \sqrt{\alpha} \mu_\ast .
\end{align}
\end{lemma}
\begin{proof}
 By inductive hypotheses $A_1(1),\dots,A_1(t)$, we have $\|\mathbf{w}_{s+1}\|_2 \leq \|\mathbf{w}_s\|_2 +16 \beta  \alpha^{3.5} L_{\max}^5 + 3 \beta \alpha^{1.5}L_{\max}^3 \dist_s^2$ for all $s\in[t]$, so we can invoke Lemma \ref{lem:gen_w} with $\xi_{1,s} = 0 \; \forall s\in[t]$ and $\xi_{2,s}=16 \beta  \alpha^{3.5} L_{\max}^5  + 3 \beta \alpha^{1.5}L_{\max}^3 \dist_s^2$. 
 This results in
 \begin{align}
     \|\mathbf{w}_{t+1}\|_2 &\leq \sum_{s=1}^t 16 \beta  \alpha^{3.5} L_{\max}^5 + 3 \beta \alpha^{1.5}L_{\max}^3 \dist_s^2
 \end{align}
 Next, we invoke inductive hypotheses $A_6(1),\dots,A_6(t)$ to obtain $\dist_{s}^2 \leq  \tfrac{10}{9}(1-0.5 \beta \alpha E_0 \mu_\ast^2)^{2(s-1)} \dist_0^2$ for all $s\in [t]$. Therefore
 \begin{align}
     \|\mathbf{w}_{t+1}\|_2 &\leq \sum_{s=1}^t 16 \beta  \alpha^{3.5} L_{\max}^5  + 3 \beta \alpha^{1.5}L_{\max}^3 \gamma^{2(s-1)} \nonumber \\
     &\leq 16 \beta  \alpha^{3.5} L_{\max}^5 t +  3 \beta \alpha^{1.5}L_{\max}^3 \sum_{s=1}^t \tfrac{10}{9}(1-0.5 \beta \alpha E_0 \mu_\ast^2)^{2(s-1)} \dist_0^2 \nonumber \\
     &\leq 16 \beta  \alpha^{3.5} L_{\max}^5 t +  \tfrac{10}{3} \beta \alpha^{1.5}L_{\max}^3 \tfrac{\dist_0^2}{0.5 \beta \alpha E_0 \mu_{\ast}^2} \label{geo11} \\
     &\leq 16 \beta  \alpha^{3.5} L_{\max}^5 t +  \tfrac{20}{3}  \sqrt{\alpha}L_{\max}^3 \tfrac{\dist_0^2}{E_0 \mu_{\ast}^2} \nonumber \\
     &\leq \tfrac{E_0}{20} \sqrt{\alpha} \mu_\ast  \label{usebeta}
 \end{align}
 where \eqref{geo11} is due to the sum of a geometric series and \eqref{usebeta} follows since  
 $\beta \leq \tfrac{\alpha}{10 \kappa_\ast^4} \leq \frac{{E_0} \mu_\ast}{ 640 \alpha^3 L_{\max}^5 T}$ (as $\alpha$ is sufficiently small) and the initial representation satisfies
 \begin{align}
     \tfrac{20}{3}  \sqrt{\alpha}L_{\max}^3 \frac{\dist_0^2}{E_0 \mu_{\ast}^2} &\leq \tfrac{E_0}{20} \sqrt{\alpha}\mu_\ast/2 \nonumber \\
     \iff 0  &\leq \mu_\ast^3  - (\tfrac{800}{3} L_{\max}^3 + 2 \mu_{\ast}^3){\dist_0^2} + \mu_\ast^3\dist_0^4 \nonumber
     \end{align}
     which is implied by
     \begin{align}
         \dist_0 \leq \tfrac{1}{17} \kappa_{\ast,\max}^{-1.5}
     \end{align}
\end{proof}

\begin{lemma}[Exact MAML $A_3(t+1)$]  \label{lem:maml_exact_ind3}
Suppose the conditions of Theorem \ref{thm:exact_maml_pop} are satisfied and $A_2(t)$, $A_3(t)$ and $A_5(t)$ hold. Then $A_3(t+1)$ holds, namely
\begin{align}
    \|\del_{t+1}\|_2 &\leq \alpha^2 L_{\max}^2
\end{align}
\end{lemma}

\begin{proof} According to Lemma \ref{lem:gen_del}, we can control $\del_{t+1}$ by controlling $\mathbf{G}_t$, recalling that 
$\mathbf{G}_t = \tfrac{1}{\beta}(\mathbf{B}_t-\mathbf{B}_{t+1})\in \mathbb{R}^{d\times k}$ is the outer loop gradient with respect to the representation at time $t$. Before studying $\mathbf{G}_t$, we must compute the outer loop gradient with respect to the representation for task $i$. Again we use the fact that $F_{t,i}(\mathbf{B}_t,\mathbf{w}_t)= \frac{1}{2}\|\mathbf{v}_{t,i}\|_2^2$ and apply the chain rule to obtain:
    \begin{align}
    \nabla_\mathbf{B}F_{t,i}(\mathbf{B}_{t}, \mathbf{w}_{t}) &= \mathbf{v}_{t,i}\mathbf{w}_t^\top \del_t \lam_t- \alpha \mathbf{B}_t (\mathbf{w}_t \mathbf{v}_{t,i}^\top\mathbf{B}_t \lam_t + \lam_t \mathbf{B}_t^\top \mathbf{v}_{t,i}\mathbf{w}_{t}^\top) \nonumber \\
    &\quad + \alpha(\mathbf{v}_{t,i}\mathbf{w}_{\ast,t,i}^\top\mathbf{B}_{\ast}^\top + \mathbf{B}_{\ast}\mathbf{w}_{\ast,t,i}\mathbf{v}_{t,i}^\top)\mathbf{B}_t\lam_t  \nonumber\\
    &\quad - 2\alpha^2 ( \mathbf{w}_{\ast,t,i}^\top\mathbf{B}_{\ast}^\top\mathbf{v}_{t,i} )\mathbf{B}_t\mathbf{w}_t\mathbf{w}_t^\top+ \alpha^2 \mathbf{B}_\ast \mathbf{w}_{\ast,t,i}\mathbf{v}_{t,i}^\top  \mathbf{B}_\ast\mathbf{w}_{\ast,t,i}\mathbf{w}_t^\top \label{grahd}
\end{align}
Note that $\mathbf{G}_t=  \frac{1}{n}\sum_{i=1}^n \nabla_\mathbf{B}F_{t,i}(\mathbf{B}_{t}, \mathbf{w}_{t})$. 
We aim to write $\mathbf{G}_t$ as $\mathbf{G}_t = -\deld_t\mathbf{S}_t\mathbf{B}_t  - \mathbf{S}_t \mathbf{B}_t\del_t + \mathbf{N}_{t}
$ for some positive definite $\mathbf{S}$ so that we can apply Lemma \ref{lem:gen_del}. 
It turns out that of the five terms in \eqref{grahd}, the only one with `sub'-terms that contribute to $\mathbf{S}_t$ is the third term. To see this, note that
\begin{align}
 \alpha(\mathbf{v}_{t,i}\mathbf{w}_{\ast,t,i}^\top\mathbf{B}_{\ast}^\top + \mathbf{B}_{\ast}\mathbf{w}_{\ast,t,i}\mathbf{v}_{t,i}^\top)\mathbf{B}_t\lam_t 
&= -{\alpha}(\deld_t \mathbf{B}_\ast\mathbf{w}_{\ast,t,i} \mathbf{w}_{\ast,t,i}^\top\mathbf{B}_{\ast}^\top \mathbf{B}_t +  \mathbf{B}_{\ast}\mathbf{w}_{\ast,t,i}\mathbf{w}_{\ast,t,i}^\top\mathbf{B}_\ast^\top \mathbf{B}_t \del_t)
\nonumber \\
&\quad + {\alpha}((\mathbf{v}_{t,i}+\deld_t \mathbf{B}_\ast\mathbf{w}_{\ast,t,i}) \mathbf{w}_{\ast,t,i}^\top\mathbf{B}_{\ast}^\top \nonumber \\
&\quad + \mathbf{B}_{\ast}\mathbf{w}_{\ast,t,i}(\mathbf{v}_{t,i}+\mathbf{w}_{\ast,t,i}^\top\mathbf{B}_\ast^\top  \deld_t))\mathbf{B}_t \lam_t
\nonumber \\
&\quad +{{\alpha^2}\deld_t \mathbf{B}_\ast\mathbf{w}_{\ast,t,i} \mathbf{w}_{\ast,t,i}^\top\mathbf{B}_{\ast}^\top \mathbf{B}_t +  \mathbf{B}_{\ast}\mathbf{w}_{\ast,t,i}\mathbf{w}_{\ast,t,i}^\top\mathbf{B}_\ast^\top \mathbf{B}_t \del_t) \mathbf{w}_t\mathbf{w}_t^\top} \nonumber \\
&= {-(\deld_t \mathbf{S}_t \mathbf{B}_t +  \mathbf{S}_{t}\mathbf{B}_t \del_t)}
\nonumber \\
&\quad + {\alpha}((\mathbf{v}_{t,i}+\deld_t \mathbf{B}_\ast\mathbf{w}_{\ast,t,i}) \mathbf{w}_{\ast,t,i}^\top\mathbf{B}_{\ast}^\top\nonumber \\
&\quad + \mathbf{B}_{\ast}\mathbf{w}_{\ast,t,i}(\mathbf{v}_{t,i}+\mathbf{w}_{\ast,t,i}^\top\mathbf{B}_\ast^\top  \deld_t))\mathbf{B}_t \lam_t
\nonumber \\
&\quad +{\alpha^2}\deld_t \mathbf{B}_\ast\mathbf{w}_{\ast,t,i} \mathbf{w}_{\ast,t,i}^\top\mathbf{B}_{\ast}^\top \mathbf{B}_t +  \mathbf{B}_{\ast}\mathbf{w}_{\ast,t,i}\mathbf{w}_{\ast,t,i}^\top\mathbf{B}_\ast^\top \mathbf{B}_t \del_t) \mathbf{w}_t\mathbf{w}_t^\top. \nonumber
\end{align}
where $\mathbf{S}_t \coloneqq \alpha \frac{1}{n}\sum_{i=1}^n \mathbf{B}_\ast\mathbf{w}_{\ast,t,i} \mathbf{w}_{\ast,t,i}^\top\mathbf{B}_{\ast}^\top$.
Thus we can write 
\begin{align}
    \mathbf{G}_t &= -\deld_t\mathbf{S}_t\mathbf{B}_t  - \mathbf{S}_t \mathbf{B}_t\del_t + \mathbf{N}_{t} \label{258}
\end{align}
where 
\begin{align}
    \mathbf{N}_t &\coloneqq \frac{1}{n}\sum_{i=1}^n  \mathbf{v}_{t,i}\mathbf{w}_t^\top \del_t \lam_t - \frac{1}{n}\sum_{i=1}^n \alpha^2 \mathbf{B}_t\mathbf{w}_t\mathbf{w}_{\ast,t,i}^\top\mathbf{B}_{\ast}^\top\mathbf{v}_{t,i}\mathbf{w}_t^\top + \frac{1}{n}\sum_{i=1}^n \alpha^2 \mathbf{B}_\ast \mathbf{w}_{\ast,t,i}\mathbf{w}_{\ast,t,i}^\top  \mathbf{B}_\ast^\top \mathbf{v}_{t,i} \mathbf{w}_t^\top \nonumber \\
    &\quad - \frac{1}{n}\sum_{i=1}^n \alpha \mathbf{B}_t (\mathbf{w}_t \mathbf{v}_{t,i}^\top\mathbf{B}_t \lam_t + \lam_t \mathbf{B}_t^\top \mathbf{v}_{t,i}\mathbf{w}_{t}^\top) + \frac{1}{n}\sum_{i=1}^n  \alpha(\mathbf{v}_{t,i} + \deld_t \mathbf{B}_\ast\mathbf{w}_{\ast,t,i})\mathbf{w}_{\ast,t,i}^\top\mathbf{B}_{\ast}^\top\mathbf{B}_t\lam_t \nonumber \\
    &\quad + \frac{1}{n}\sum_{i=1}^n \alpha\deld_t \mathbf{B}_\ast\mathbf{w}_{\ast,t,i}\mathbf{w}_{\ast,t,i}^\top\mathbf{B}_{\ast}^\top  \mathbf{B}_t\alpha\mathbf{w}_t\mathbf{w}_t^\top \label{NN}
\end{align}
Note that $\|\del_t\|_2 \leq \frac{1}{10}$ due to $A_3(t)$ and  choice of $\alpha \leq \frac{1}{4L_{\max}}$. Therefore, Lemma \ref{lem:sigminE} implies that $\sigma_{\min}(\mathbf{B}_t^\top \mathbf{S}_t \mathbf{B}_t)\geq E_0 \mu_{\ast}^2$ where $E_0 = 1 - \tfrac{1}{10} - \dist_0^2$. 
Thus by Lemma \ref{lem:gen_del} with $\chi=1$, we have
\begin{align}
    \|\del_{t+1}\|_2 &\leq \|\del_t\|_2(1 -  2\beta \alpha E_0\mu_\ast^2) + 2\beta \alpha \|\mathbf{B}_t^\top \mathbf{N}_t\|_2 + \beta^2 \alpha \| \mathbf{G}_t\|_2^2 \label{bylemmagendel}
\end{align}
So it remains to control $\|\mathbf{B}_t^\top \mathbf{N}_t\|_2$ and $\| \mathbf{G}_t\|_2^2$. First we deal with $\|\mathbf{B}_t^\top \mathbf{N}_t\|_2$ by upper bounding the norm of $\mathbf{B}_t^\top$ times each of the six terms in \eqref{NN}. As before, we use $c=1.1$ as an absolute constant that satisfies $ \sigma_{\max}^3(\mathbf{B}_t) \leq c/\alpha^{1.5}$. We have
\begin{align}
    &\left\|\frac{1}{n}\sum_{i=1}^n  \mathbf{B}_t^\top\mathbf{v}_{t,i}\mathbf{w}_t^\top \del_t \lam_t  \right\|_2 \nonumber \\
    &\leq \left\|\frac{1}{n}\sum_{i=1}^n  \mathbf{B}_t^\top(\deld_t - (\alpha \omega_t +\alpha^2 a_{t,i}) \mathbf{I}_d) (\mathbf{B}_t \mathbf{w}_t - \mathbf{B}_\ast \mathbf{w}_{\ast,t,i})\mathbf{w}_t^\top \del_t \lam_t  \right\|_2  \nonumber \\
    &\leq \left\|\frac{1}{n}\sum_{i=1}^n  \del_t \mathbf{B}_t^\top (\mathbf{B}_t \mathbf{w}_t - \mathbf{B}_\ast \mathbf{w}_{\ast,t,i})\mathbf{w}_t^\top \del_t \lam_t  \right\|_2 
    \nonumber \\
    &\quad + \left\|\frac{1}{n}\sum_{i=1}^n  (\alpha \omega_t +\alpha^2 a_{t,i}) \mathbf{B}_t^\top (\mathbf{B}_t \mathbf{w}_t - \mathbf{B}_\ast \mathbf{w}_{\ast,t,i})\mathbf{w}_t^\top \del_t \lam_t  \right\|_2 \nonumber \\
    &\leq \tfrac{c}{\alpha} \|\del_t\|_2^2 \|\mathbf{w}_t\|_2^2 + \tfrac{c}{\sqrt{\alpha}}\|\del_t\|_2^2\|\mathbf{w}_t\|_2 \eta_\ast +  c\|\del_t\|_2^2 \|\mathbf{w}_t\|_2^4 \nonumber \\
    &\quad +c\sqrt{\alpha}\|\del_t\|_2 \|\mathbf{w}_t\|_2^3 \eta_\ast + c \sqrt{\alpha} \|\del_t\|_2^2\|\mathbf{w}_t\|_2^3 \eta_\ast + c \alpha \|\del_t\|_2 \|\mathbf{w}_t\|_2^2 L_\ast^2 \nonumber
    \end{align}
\begin{align}
    \left\|\frac{1}{n}\sum_{i=1}^n \alpha^2 \mathbf{B}_t^\top\mathbf{B}_t\mathbf{w}_t \mathbf{w}_{\ast,t,i}^\top\mathbf{B}_{\ast}^\top\mathbf{v}_{t,i} \mathbf{w}_t^\top \right\|_2 &\leq c\bigg\|\frac{1}{n}\sum_{i=1}^n \alpha \mathbf{w}_t \mathbf{w}_{\ast,t,i}^\top\mathbf{B}_{\ast}^\top(\deld_t - (\alpha \omega_t +\alpha^2 a_{t,i}) \mathbf{I}_d) \nonumber \\
    &\quad \quad \quad \quad \quad \quad \times(\mathbf{B}_t \mathbf{w}_t - \mathbf{B}_\ast \mathbf{w}_{\ast,t,i}) \mathbf{w}_t^\top \bigg\|_2 \nonumber \\
    &\leq  c \sqrt{\alpha} \|\del_t\|_2  \|\mathbf{w}_t\|_2^3 \eta_\ast +c \alpha (\|\del_t\|_2 + \dist_t^2) \|\mathbf{w}_t\|_2^2 L_{\max}^2  \nonumber \\
    &\quad + c \alpha^{1.5} \|\del_t\|_2 \|\mathbf{w}_t\|_2^5 \eta_\ast + c \alpha^2 \|\del_t\|_2 \|\mathbf{w}\|_2^4 L_{\max}^2 \nonumber \\
    &\quad + c \alpha^{2}\|\mathbf{w}_t\|_2^4 L_{\ast}^2  + c\alpha^{2.5}\|\mathbf{w}_t\|_2^3 L_{\max}^3 \nonumber \\
   \left\|\frac{1}{n}\sum_{i=1}^n \alpha^2 \mathbf{B}_t^\top\mathbf{B}_\ast \mathbf{w}_{\ast,t,i} \mathbf{w}_{\ast,t,i}^\top  \mathbf{B}_\ast^\top \mathbf{v}_{t,i}\mathbf{w}_t^\top \right\|_2 &\leq c {\alpha} \|\del_t\|_2  \|\mathbf{w}_t\|_2^2 L_\ast^2 + c \alpha^{1.5} (\|\del_t\|_2 + \dist_t^2) \|\mathbf{w}_t\|_2 L_{\max}^3 \nonumber \\
   &\quad  + c \alpha^{2} \|\del_t\|_2 \|\mathbf{w}_t\|_2^4 L_\ast^2  + c \alpha^{2.5} \|\del_t\|_2 \|\mathbf{w}\|_2^3 L_{\max}^3  \nonumber \\
    &\quad + c \alpha^{2.5}\|\mathbf{w}_t\|_2^3 L_{\max}^3 + c\alpha^{3}\|\mathbf{w}_t\|_2^2 {L}_{\max}^4  \nonumber 
\end{align}
\begin{align}
  \left\| \frac{1}{n}\sum_{i=1}^n \alpha \mathbf{B}_t^\top\mathbf{B}_t (\mathbf{w}_t \mathbf{v}_{t,i}^\top\mathbf{B}_t \lam_t + \lam_t \mathbf{B}_t^\top \mathbf{v}_{t,i}\mathbf{w}_{t}^\top) \right\|_2 &\leq 2\left\| \frac{1}{n}\sum_{i=1}^n \alpha \mathbf{B}_t^\top\mathbf{B}_t \mathbf{w}_t \mathbf{v}_{t,i}^\top\mathbf{B}_t \lam_t \right\|_2  \nonumber \\
&\leq 2c\bigg\| \frac{1}{n}\sum_{i=1}^n  \mathbf{w}_t  (\mathbf{B}_t \mathbf{w}_t - \mathbf{B}_\ast \mathbf{w}_{\ast,t,i})^\top\nonumber \\
&\quad \quad \quad \quad \quad \quad \times(\deld_t - (\alpha \omega_t +\alpha^2 a_{t,i}) \mathbf{I}_d)\mathbf{B}_t \lam_t \bigg\|_2 \nonumber \\
&\leq \tfrac{2c}{\alpha} \|\del_t\|\|\mathbf{w}_t\|_2^2 + \tfrac{2c}{\sqrt{\alpha}} \|\del_t\|_2 \|\mathbf{w}_t\|_2 \eta_\ast  \nonumber \\
&\quad + 2{c}\|\del_t\|_2 \|\mathbf{w}_t\|_2^4  + 2c\sqrt{\alpha}\|\del_t\|_2 \|\mathbf{w}_t\|_2^3 \eta_\ast  \nonumber \\
&\quad+ 2c \sqrt{\alpha} \|\mathbf{w}_t\|_2^3 \eta_\ast   + 2c {\alpha}  \|\mathbf{w}_t\|_2^2 L_\ast^2  \nonumber \\
   \left\|\frac{1}{n}\sum_{i=1}^n  \alpha\mathbf{B}_t^\top(\mathbf{v}_{t,i} + \deld_t \mathbf{B}_\ast\mathbf{w}_{\ast,t,i})\mathbf{w}_{\ast,t,i}^\top\mathbf{B}_{\ast}^\top\mathbf{B}_t\lam_t  \right\|_2 &\leq \tfrac{c}{\sqrt{\alpha}}\|\del_t\|_2 \|\mathbf{w}_t\|_2 \eta_\ast + c \sqrt{\alpha} \|\del_t\|_2 \|\mathbf{w}_t\|_2^3\eta_\ast  \nonumber \\
   &\quad + c \alpha \|\del_t\|_2 \|\mathbf{w}_t\|_2^2 L_\ast^2 +c \alpha\|\mathbf{w}_t\|_2^2 L_\ast^2  \nonumber \\
   &\quad +  c \alpha^{3/2}\|\mathbf{w}_t\|_2 L_{\max}^3 \nonumber \\
   \left\| \frac{1}{n}\sum_{i=1}^n \alpha\mathbf{B}_t^\top \deld_t \mathbf{B}_\ast\mathbf{w}_{\ast,t,i}\mathbf{w}_{\ast,t,i}^\top\mathbf{B}_{\ast}^\top  \mathbf{B}_t\alpha\mathbf{w}_t\mathbf{w}_t^\top \right\|_2 &\leq  \alpha \|\del_t\|_2 \|\mathbf{w}_t\|_2^2 L_\ast^2 \nonumber 
\end{align}
Let $c_2 \coloneqq E_0/20$. We can combine the above bounds and use inductive hypotheses $A_2(t)$ and $A_3(t)$ to obtain the following bound on 
$\|\mathbf{B}_t^\top\mathbf{N}_t\|_2$:
\begin{align}
    \|\mathbf{B}_t^\top\mathbf{N}_t\|_2 &\leq 2 c c_2^2 
    \alpha^4 L_\ast^4 \mu_\ast^2 \kappa_{\ast,\max}^{-2} + c c_2 
    \alpha^4 L_\ast^4 \eta_\ast \mu_\ast \kappa_{\ast,\max}^{-1} +4 c c_2^4 
    \alpha^6 L^4 \mu_\ast^4 \kappa_{\ast,\max}^{-4} + 4c c_2^3 \alpha^{4}L_{\max}^2\eta_\ast \mu_\ast^3 \kappa_{\ast,\max}^{-3} \nonumber \\
    &\quad + cc_2^3 
    \alpha^{6}L_\ast^4\eta_\ast \mu_\ast^3\kappa_{\ast,\max}^{-3} + 3cc_2^2  \alpha^{4}L_\ast^2 L_{\max}^3 \mu_\ast\kappa_{\ast,\max}^{-1} + 3cc_2^2 \alpha^{2} L_{\max}^3 \mu_\ast\kappa_{\ast,\max}^{-1} \dist_t^2 \nonumber \\
    &\quad  + c c_2^5 \alpha^6 L_\ast^2\eta_\ast \mu_\ast^5 \kappa_{\ast,\max}^{-5}  + 2c c_2^3 \alpha^4 \mu_\ast^3 L_{\max}^3\kappa_{\ast,\max}^{-3} + 2c c_2^3 \alpha^6 \mu_\ast^3 L_{\max}^3 L_\ast^2 \kappa_{\ast,\max}^{-3}   \nonumber \\
    &\quad + c c_2^2 \alpha^4 {L}_{\max}^4 \mu_\ast^2 \kappa_{\ast,\max}^{-2} + 2 c c_2^2 \alpha^{2} L_{\max}^2 \mu_\ast^2 \kappa_{\ast,\max}^{-2}  \nonumber \\
    &\quad  + 3c c_2 \alpha^2  L_{\max}^2 \eta_\ast \mu_\ast \kappa_{\ast,\max}^{-1}  +2 c c_2^4 \alpha^4 L_{\max}^2 \mu_\ast^4 \kappa_{\ast,\max}^{-4}  + 2c c_2^3 \alpha^2 \eta_\ast \mu_\ast^3 \kappa_{\ast,\max}^{-3} \nonumber \\
    &\quad + 4 c c_2 \alpha^2 L_{\max}^3 \mu_\ast \kappa_{\ast,\max}^{-1}  + 2c c_2 \alpha^4 L_{\max}^2 L_\ast^2 \mu_\ast^2 \kappa_{\ast,\max}^{-2} \nonumber \\
    &\leq 21 cc_2 \alpha^{4}(L_\ast^4\eta_\ast + L_{\max}^3L_\ast^2 ) \mu_\ast\kappa_{\ast,\max}^{-1} + 3cc_2^2 \alpha^{2} L_{\max}^3 \mu_\ast\kappa_{\ast,\max}^{-1} \dist_t^2  
    \nonumber \\
    &\quad + 4 c c_2 \alpha^2 (L_{\max}^3+L_{\max}^2(\eta_\ast+\mu_\ast)) \mu_\ast \kappa_{\ast,\max}^{-1}  \nonumber \\
    &\leq 10 c c_2 \alpha^2 L_{\max}^2 \mu_\ast^2 + 3cc_2^2 \alpha^{2} L_{\max}^2 \mu_\ast^2\dist_t^2 \nonumber \\
    &\leq 13 c c_2 \alpha^2 L_{\max}^2 \mu_\ast^2 \nonumber \\
    &\leq 15 c_2 \alpha^2 L_{\max}^2 \mu_\ast^2 \label{silver}
\end{align}
using that $\alpha \leq 1/L_\ast$, $c=1.1$ and combining like terms. We have not optimized constants.
Next we bound $\|\mathbf{G}_t\|_2^2$. First, by \eqref{258} and the triangle and Cauchy-Schwarz inequalities,
\begin{align}
    \|\mathbf{G}_t\|_2 &\leq  \|\deld_t\mathbf{S}_t\|_2 \|\mathbf{B}_t\|_2 +\|\mathbf{S}_t\|_2 \|\mathbf{B}_t\|_2\|\del_t\|_2 + \|\mathbf{N}_t\|_2 \nonumber \\
    &\leq {c}{\sqrt{\alpha}}(\dist_t + 2\|\del_t\|_2)L_\ast^2 +  \|\mathbf{N}_t\|_2.
\end{align}
We have already bounded $\|\mathbf{B}_t^\top\mathbf{N}_t\|$ by separately bounding $\mathbf{B}_t^\top$ times each of the six terms in $\mathbf{N}_t$. We obtain a similar bound on $\|\mathbf{N}_t\|_2$ by separately considering each of the six terms in $\mathbf{N}_t$ (see equation \eqref{NN}).  Of these terms, all but the first and last can be easily bounded by multiplying our previous bounds by $\sqrt{\alpha}$ (to account for no $\mathbf{B}_t$). The other two terms are more complicated because we have previously made the reduction $\|\mathbf{B}_t^\top \deld_t\|_2 = \|\del_t \mathbf{B}_t^\top\|_2\leq \frac{c}{\sqrt{\alpha}} \|\del_t\|_2 $, but now that there is no $\mathbf{B}_t^\top$ to multiply with $\deld_t$, we must control $\deld_t$ via $\|\deld_t \mathbf{B}_\ast\|_2 \leq \|\del_t\|_2 + \dist_t$. Specifically, for the easy four terms we have
\begin{align}
\left\|\frac{1}{n}\sum_{i=1}^n \alpha^2 \mathbf{B}_t\mathbf{w}_t \mathbf{w}_{\ast,t,i}^\top\mathbf{B}_{\ast}^\top\mathbf{v}_{t,i} \mathbf{w}_t^\top \right\|_2 
    &\leq  c {\alpha} \|\del_t\|_2  \|\mathbf{w}_t\|_2^3 \eta_\ast +c \alpha^{1.5} (\|\del_t\|_2 + \dist_t^2) \|\mathbf{w}_t\|_2^2 L_{\max}^2 \nonumber \\
    &\quad + c \alpha^{2} \|\del_t\|_2 \|\mathbf{w}_t\|_2^5 \eta_\ast + c \alpha^{2.5} \|\del_t\|_2 \|\mathbf{w}\|_2^4 L_{\max}^2 \nonumber \\
    &\quad + c \alpha^{2.5}\|\mathbf{w}_t\|_2^4 L_{\ast}^2 + c\alpha^{3}\|\mathbf{w}_t\|_2^3 L_{\max}^3  \nonumber \\
   \left\|\frac{1}{n}\sum_{i=1}^n \alpha^2 \mathbf{B}_\ast \mathbf{w}_{\ast,t,i} \mathbf{w}_{\ast,t,i}^\top  \mathbf{B}_\ast^\top \mathbf{v}_{t,i}\mathbf{w}_t^\top \right\|_2 &\leq c {\alpha}^{1.5} \|\del_t\|_2  \|\mathbf{w}_t\|_2^2 L_\ast^2 + c \alpha^{2} (\|\del_t\|_2 + \dist_t^2) \|\mathbf{w}_t\|_2 L_{\max}^3  \nonumber \\
    &\quad + c \alpha^{2.5} \|\del_t\|_2 \|\mathbf{w}_t\|_2^4 L_\ast^2 + c \alpha^{3} \|\del_t\|_2 \|\mathbf{w}\|_2^3 L_{\max}^3  \nonumber \\
    &\quad + c \alpha^{3}\|\mathbf{w}_t\|_2^3 L_{\max}^3 + c\alpha^{3.5}\|\mathbf{w}_t\|_2^2 {L}_{\max}^4  \nonumber
    \end{align}
\begin{align}
  \left\| \frac{1}{n}\sum_{i=1}^n \alpha \mathbf{B}_t (\mathbf{w}_t \mathbf{v}_{t,i}^\top\mathbf{B}_t \lam_t + \lam_t \mathbf{B}_t^\top \mathbf{v}_{t,i}\mathbf{w}_{t}^\top) \right\|_2 
&\leq \tfrac{2c}{\sqrt{\alpha}} \|\del_t\|\|\mathbf{w}_t\|_2^2 + {2c} \|\del_t\|_2 \|\mathbf{w}_t\|_2 \eta_\ast \nonumber \\
&\quad + 2{c}\sqrt{\alpha}\|\del_t\|_2 \|\mathbf{w}_t\|_2^4 + 2c{\alpha}\|\del_t\|_2 \|\mathbf{w}_t\|_2^3 \eta_\ast   \nonumber \\
&\quad + 2c {\alpha} \|\mathbf{w}_t\|_2^3 \eta_\ast  + 2c {\alpha}^{1.5}  \|\mathbf{w}_t\|_2^2 L_\ast^2  \nonumber \\
   \left\|\frac{1}{n}\sum_{i=1}^n  \alpha(\mathbf{v}_{t,i} + \deld_t \mathbf{B}_\ast\mathbf{w}_{\ast,t,i})\mathbf{w}_{\ast,t,i}^\top\mathbf{B}_{\ast}^\top\mathbf{B}_t\lam_t  \right\|_2 &\leq {c}\|\del_t\|_2 \|\mathbf{w}_t\|_2 \eta_\ast + c {\alpha} \|\del_t\|_2 \|\mathbf{w}_t\|_2^3\eta_\ast  \nonumber \\
   &\quad + c \alpha^{1.5} \|\del_t\|_2 \|\mathbf{w}_t\|_2^2 L_\ast^2  +c \alpha^{1.5}\|\mathbf{w}_t\|_2^2 L_\ast^2\nonumber \\
   &\quad  +  c \alpha^{2}\|\mathbf{w}_t\|_2 L_{\max}^3 \nonumber
 \end{align}
and for the first and last term from \eqref{NN}, we have
\begin{align}
 \left\|\frac{1}{n}\sum_{i=1}^n  \mathbf{v}_{t,i}\mathbf{w}_t^\top \del_t \lam_t  \right\|_2 &\leq\tfrac{c}{\sqrt{\alpha}} \|\del_t\|_2^2 \|\mathbf{w}_t\|_2^2 + {c}{}\|\del_t\|_2(\|\del_t\|_2+\dist_t)\|\mathbf{w}_t\|_2 \eta_\ast  \nonumber \\
 &\quad +  c\sqrt{\alpha}\|\del_t\|_2^2 \|\mathbf{w}_t\|_2^4 +c\alpha\|\del_t\|_2 \|\mathbf{w}_t\|_2^3 \eta_\ast  \nonumber \\
 &\quad + c {\alpha} \|\del_t\|_2^2\|\mathbf{w}_t\|_2^3 \eta_\ast + c \alpha^{1.5} \|\del_t\|_2 \|\mathbf{w}_t\|_2^2 L_\ast^2 \nonumber \\
       \left\| \frac{1}{n}\sum_{i=1}^n \alpha \deld_t \mathbf{B}_\ast\mathbf{w}_{\ast,t,i}\mathbf{w}_{\ast,t,i}^\top\mathbf{B}_{\ast}^\top  \mathbf{B}_t\alpha\mathbf{w}_t\mathbf{w}_t^\top \right\|_2 &\leq  \alpha^{1.5} (\|\del_t\|_2 + \dist_t) \|\mathbf{w}_t\|_2^2 L_\ast^2 \nonumber 
\end{align}
Combining these bounds and applying inductive hypotheses $A_2(t)$ and $A_3(t)$ yields
\begin{align}
    \|\mathbf{N}_t\|_2 &\leq c c_2^3\alpha^{4.5} L_{\max}^2 \eta_\ast \mu_\ast^3 \kappa_{\ast,\max}^{-3} + c c_2^2 \alpha^{4.5} L_\ast^2 L_{\max}^2 \mu_\ast^2\kappa_{\ast,\max}^{-2}  + c c_2^2 \alpha^{2.5} L_{\max}^2\mu_\ast^2 \kappa_{\ast,\max}^{-2}\dist_t^2 \nonumber \\
    &\quad  + c c_2^5 \alpha^{6.5}L_{\ast}^2 \eta_\ast \mu_\ast^5\kappa_{\ast,\max}^{-5} + c c_2^4 \alpha^{6.5}L_{\max}^2 L_{\max}^2 \mu_{\ast}^4 \kappa_{\ast,\max}^{-4} + c c_2 \alpha^{4.5}L_\ast^2 \mu_\ast^4 \kappa_{\ast,\max}^{-4}  \nonumber \\
     &\quad + c c_2^3 \alpha^{4.5}L_{\max}^3 \mu_\ast^3 \kappa_{\ast,\max}^{-3} \nonumber \\
    &\quad + c c_2^2 \alpha^{4.5} L_\ast^2 L_{\max}^2 \mu_\ast^2 \kappa_{\ast,\max}^{-2} + c c_2 \alpha^{4.5}L_{\max}^3 L_{\max}^2 \mu_\ast \kappa_{\ast,\max}^{-1} + c c_2 \alpha^{2.5}L_{\max}^3\mu_\ast \kappa_{\ast,\max}^{-1}\dist_t^2  \nonumber \\
     &\quad + c c_2^4 \alpha^{6.5} L_\ast^2 L_{\max}^2 \mu_\ast^4 \kappa_{\ast,\max}^{-4}  + c c_2^3 \alpha^{6.5} L_{\max}^3 L_{\max}^2 \mu_\ast^3 \kappa_{\ast,\max}^{-3} + c c_2^3 \alpha^{4.5} L_{\max}^3 \mu_\ast^3 \kappa_{\ast,\max}^{-3} \nonumber \\
     &\quad + c c_2^2 \alpha^{4.5} {L}_{\max}^4 \mu_\ast^2 \kappa_{\ast,\max}^{-2} + 2c c_2^2 \alpha^{2.5} L_\ast^2 \mu_\ast^2 \kappa_{\ast,\max}^{-2}  \nonumber \\
    &\quad + 2 c c_2 \alpha^{2.5}L_\ast^2 \eta_\ast \mu_\ast \kappa_{\ast,\max}^{-1} + 2 c c_2^4 \alpha^{4.5} L_{\max}^2 \mu_\ast^4 \kappa_{\ast,\max}^{-4}+ 2 c c_2^3 \alpha^{4.5} L_{\max}^2 \eta_\ast \mu_\ast^3 \kappa_{\ast,\max}^{-3} \nonumber \\
     &\quad + 2 c c_2^3 \alpha^{2.5} \eta_\ast \mu_\ast^3 \kappa_{\ast,\max}^{-3} + 2 c c_2^2 \alpha^{2.5} L_\ast^2 \mu_\ast^2 \kappa_{\ast,\max}^{-2} + c c_2 \alpha^{2.5}L_{\max}^2  \mu_\ast^2 \kappa_{\ast,\max}^{-2}  \nonumber \\
     &\quad + c c_2^3 \alpha^{4.5} L_{\max}^2 \eta_\ast \mu_\ast^3 \kappa_{\ast,\max}^{-3} + c c_2^2 \alpha^{4.5} L_{\max}^2 L_\ast^2 \mu_\ast^2 \kappa_{\ast,\max}^{-2}  \nonumber \\
    &\quad   + c c_2^2 \alpha^{2.5}L_\ast^2 \mu_\ast^2 \kappa_{\ast,\max}^{-2} +  c c_2 \alpha^{2.5}L_{\max}^3 \mu_\ast \kappa_{\ast,\max}^{-1} + c c_2^2 \alpha^{4.5} L^4 \mu_\ast^2\kappa_{\ast,\max}^{-2}   \nonumber \\
    &\quad + c c_2 \alpha^{4.5} L^4 \eta_\ast \mu_\ast \kappa_{\ast,\max}^{-1} + c c_2 \alpha^{2.5}L_{\max}^2 \eta_\ast \mu_\ast \kappa_{\ast,\max}^{-1} \dist_t^2 +  c c_2^4 \alpha^{4.5} L^4 \mu_\ast^4 \kappa_{\ast,\max}^{-4} \nonumber \\
    &\quad + c c_2^3 \alpha^{4.5}L_{\max}^2 \eta_\ast \mu_\ast^3\kappa_{\ast,\max}^{-3} +  c c_2^3 \alpha^{6.5}L^4 \eta_\ast \mu_\ast^3\kappa_{\ast,\max}^{-3} + c c_2^2 \alpha^{4.5}L_{\max}^2 L_\ast^2 \mu_\ast^2 \kappa_{\ast,\max}^{-2}  \nonumber \\
    &\quad + c c_2^2 \alpha^{4.5}L_{\max}^2 L_\ast^2\mu_\ast^2 \kappa_{\ast,\max}^{-2} + c c_2^2 \alpha^{2.5}L_\ast^2\mu_\ast^2 \kappa_{\ast,\max}^{-2} \dist_t \nonumber \\
    &\leq c c_2^2 \alpha^{2.5} L_{\max}^2\mu_\ast^2 \kappa_{\ast,\max}^{-2}  \nonumber \\
    &\quad +   5c c_2^4 \alpha^{6.5} L_{\max}^2  \mu_{\ast}^6   \nonumber \\
    &\quad + 19c c_2 \alpha^{4.5}L_{\max}^2 L_{\max}^2 \mu_\ast^2 + c c_2 \alpha^{2.5}L_{\max}^3\mu_\ast \kappa_{\ast,\max}^{-1}\dist_t^2  \nonumber \\
    &\quad + 2c c_2^2 \alpha^{2.5} L_\ast^2 \mu_\ast^2 \kappa_{\ast,\max}^{-2} + 2 c c_2 \alpha^{2.5}L_\ast^2 \eta_\ast \mu_\ast \kappa_{\ast,\max}^{-1}\nonumber \\
     &\quad + 2 c c_2^3 \alpha^{2.5} \eta_\ast \mu_\ast^3 \kappa_{\ast,\max}^{-3} + 2 c c_2^2 \alpha^{2.5} L_\ast^2 \mu_\ast^2 \kappa_{\ast,\max}^{-2} + c c_2 \alpha^{2.5}L_{\max}^2  \mu_\ast^2 \kappa_{\ast,\max}^{-2} \nonumber \\
    &\quad + c c_2^2 \alpha^{2.5}L_\ast^2 \mu_\ast^2 \kappa_{\ast,\max}^{-2} +  c c_2 \alpha^{2.5}L_{\max}^3 \mu_\ast \kappa_{\ast,\max}^{-1}  \nonumber \\
    &\quad + c c_2 \alpha^{2.5}L_{\max}^2 \eta_\ast \mu_\ast \kappa_{\ast,\max}^{-1} \dist_t^2  \nonumber \\
    &\quad  + c c_2 \alpha^{2.5}L_\ast^2\mu_\ast^2 \kappa_{\ast,\max}^{-2} \dist_t \nonumber \\
    &\leq  24c c_2 \alpha^{4.5}L_{\max}^2 L_{\max}^2 \mu_\ast^2 + 12c c_2 \alpha^{2.5} L_{\max}^2 \mu_\ast^2 + 3 c c_2 \alpha^{2.5}L_{\max}^2 \mu_\ast^2 \dist_t  \nonumber \\
    &\leq 14c c_2 \alpha^{2.5} L_{\max}^2 \mu_\ast^2 + 3 c c_2 \alpha^{2.5}L_{\max}^2 \mu_\ast^2 \dist_t \nonumber \\
    &\leq 17 c c_2 \alpha^{2.5} L_{\max}^2 \mu_\ast^2 \nonumber \\
    &\leq \tfrac{9}{8} c c_2 \sqrt{\alpha} \mu_\ast^2 \label{nitself}
\end{align}
using $\alpha \leq 1/(4L_{\max})$, $\dist_t\leq 1$, and again, $c_2 = E_0/20$. 
Thus,
\begin{align}
    \|\mathbf{G}_t\|_2^2 &\leq \left({c}{\sqrt{\alpha}}(\dist_t + 2\|\del_t\|_2)L_\ast^2 + \tfrac{9}{8} c c_2 \sqrt{\alpha} \mu_\ast^2  \right)^2 \nonumber \\
     &\leq \left(\tfrac{9}{8}{c}{\sqrt{\alpha}}L_\ast^2 + \tfrac{9}{8} c c_2 \sqrt{\alpha} \mu_\ast^2  \right)^2 \nonumber \\
    &\leq \tfrac{3}{2}\alpha \left(L_\ast^2 + c_2 \mu_\ast^2  \right)^2 \label{c} \\
    &\leq 3{{\alpha}}L_\ast^4 \nonumber 
\end{align}
using $c = 1.1$ in \eqref{c}. Returning to \eqref{bylemmagendel} and applying our bounds on $\|\mathbf{B}_t^\top \mathbf{N}_t\|_2$ and $\|\mathbf{G}_t\|_2^2$, along with inductive hypothesis $A_3(t)$, yields
\begin{align}
  \| \del_{t+1}\|_2 &=  \|\del_t\|_2(1 - 2\beta \alpha E_0\mu_\ast^2) + 30 c_2\beta   \alpha^3 L_{\max}^2 \mu_\ast^2  + 3 \beta^2 \alpha^2 L_\ast^4 \nonumber \\
  &\leq \alpha^2 L_{\max}^2 (1 - 2\beta \alpha E_0\mu_\ast^2) + 30c_2 \beta \alpha^3 L_{\max}^2 \mu_\ast^2  + 3 \beta^2 \alpha^2 L_\ast^4 \nonumber \\
  &\leq  \alpha^2 L_{\max}^2 -  2\beta \alpha^3 E_0L_{\max}^2 \mu_\ast^2 + 30c_2 \beta \alpha^3 L_{\max}^2 \mu_\ast^2  + 3 \beta^2 \alpha^2 L_\ast^4 \nonumber \\
  &\leq \alpha^2 L_{\max}^2
\end{align}
where the last inequality follows due to 
$\beta \leq \tfrac{E_0\alpha}{10\kappa_\ast^4}\leq \frac{\alpha E_0 L_{\max}^2 \mu_\ast^2}{6 L_\ast^4}$
and $c_2 = E_0/20$.
\end{proof}

\begin{lemma}[Exact MAML $A_4(t+1)$]  \label{lem:maml_exact_ind4}
Suppose the conditions of Theorem \ref{thm:exact_maml_pop} are satisfied
and $A_2(t)$, $A_3(t)$ and $A_5(t)$ hold. Then $A_4(t+1)$ holds, i.e.
\begin{align}
    \|\mathbf{{B}}_{\ast,\perp}^\top \mathbf{B}_{t+1}\|_2\leq (1-0.5 \beta \alpha E_0 \mu_\ast^2)\|\mathbf{{B}}_{\ast,\perp}^\top \mathbf{B}_{t}\|_2.
\end{align}
\end{lemma}

\begin{proof}
Recall from \eqref{258} that the outer loop gradient for the representation satisfies 
\begin{align}
    \mathbf{G}_t &= -\deld_t\mathbf{S}_t\mathbf{B}_t  - \mathbf{S}_t \mathbf{B}_t\del_t + \mathbf{N}_{t} \label{258-2}
\end{align}
where $\mathbf{S}_t \coloneqq \alpha \tfrac{1}{n}\sum_{i=1}^n\mathbf{B}_\ast\mathbf{w}_{\ast,t,i} \mathbf{w}_{\ast,t,i}^\top\mathbf{B}_{\ast}^\top$ and $\|\mathbf{N}_t\|_2\leq  \tfrac{9}{8} c c_2 \sqrt{\alpha} \mu_\ast^2$, where $c_2 \coloneqq E_0/20$. As a result,
\begin{align}
    \|\mathbf{B}_{\ast,\perp}^\top \mathbf{B}_{t+1}\|_2 &=\|\mathbf{B}_{\ast,\perp}^\top (\mathbf{B}_{t} - \beta (-\deld_t\mathbf{S}_t\mathbf{B}_t  - \mathbf{S}_t \mathbf{B}_t\del_t + \mathbf{N}_{t})) \|_2 \nonumber \\
    &=\|\mathbf{B}_{\ast,\perp}^\top (\mathbf{B}_{t} + \beta\mathbf{S}_t\mathbf{B}_t  - \beta\alpha \mathbf{B}_t\mathbf{B}_t^\top \mathbf{S}_t\mathbf{B}_t + \beta\mathbf{S}_t \mathbf{B}_t  - \beta \alpha \mathbf{S}_t \mathbf{B}_t\mathbf{B}_t^\top \mathbf{B}_t -\beta \mathbf{N}_{t} \|_2 \nonumber \\
    &=\|\mathbf{B}_{\ast,\perp}^\top \mathbf{B}_{t}(\mathbf{I}_k  - \beta\alpha \mathbf{B}_t^\top \mathbf{S}_t\mathbf{B}_t) -\beta \mathbf{B}_{\ast,\perp}^\top\mathbf{N}_{t} \|_2\nonumber \\
    &\leq \|\mathbf{B}_{\ast,\perp}^\top \mathbf{B}_{t}\|_2\|\mathbf{I}_k  - \beta\alpha \mathbf{B}_t^\top \mathbf{S}_t\mathbf{B}_t\|_2 +\beta \|\mathbf{B}_{\ast,\perp}^\top\mathbf{N}_{t} \|_2
\end{align}
where the last equality follows because $\mathbf{B}_{\ast,\perp}^\top \mathbf{S}_t =\alpha \frac{1}{n}\sum_{i=1}^n\mathbf{B}_{\ast,\perp}^\top \mathbf{B}_\ast\mathbf{w}_{\ast,t,i} \mathbf{w}_{\ast,t,i}^\top\mathbf{B}_{\ast}^\top = \mathbf{0}$. Note that due to Lemma \ref{lem:sigminE} and $\|\del_t\|_2\leq \tfrac{1}{10}$, $\sigma_{\min}(\mathbf{B}_t^\top \mathbf{S}_t\mathbf{B}_t) \geq E_0 \mu_\ast^2 $ where $E_0 = 1 - \tfrac{1}{10} - \dist_0^2$.  Therefore, by Weyl's inequality,
\begin{align}
    \|\mathbf{I}_k  - \beta\alpha \mathbf{B}_t^\top \mathbf{S}_t\mathbf{B}_t\|_2 &\leq 1 - \beta \alpha E_0 \mu_\ast^2.
\end{align}
Furthermore, from \eqref{nitself}, we have 
\begin{align}
    \|\mathbf{B}_{\ast,\perp}^\top \mathbf{N}_t\|_2 &\leq \| \mathbf{N}_t\|_2 \leq  \tfrac{9}{8} c c_2 \sqrt{\alpha} \mu_\ast^2 \leq \tfrac{5}{4} c_2 \sqrt{\alpha} \mu_\ast^2  \nonumber \\
    \implies \|\mathbf{B}_{\ast,\perp}^\top \mathbf{B}_{t+1}\|_2 &\leq \|\mathbf{B}_{\ast,\perp}^\top \mathbf{B}_{t}\|_2(1 - \beta \alpha E_0 \mu_\ast^2) +  \tfrac{5}{4} c_2 \sqrt{\alpha} \mu_\ast^2 
\end{align}
Next, recall that $\|\mathbf{B}_{\ast,\perp}^\top\mathbf{B}_t\|_2 \geq\sigma_{\min}(\mathbf{B}_{\ast,\perp}^\top\mathbf{B}_t) \geq \sqrt{\tfrac{9}{10}} \sigma_{\min}(\mathbf{B}_{\ast,\perp}^\top\mathbf{\hat{B}}_t)/\sqrt{\alpha} = \sqrt{\tfrac{9}{10}} \sqrt{1\!-\!\dist^2_t}/\sqrt{\alpha}\geq \sqrt{\tfrac{9}{10}} \sqrt{1 - {\tfrac{10}{9}}\dist_0^2}/\sqrt{\alpha} = \sqrt{E_0}/\sqrt{\alpha}$ due to inductive hypotheses $A_3(t)$ and $A_4(t)$ and $E_0\coloneqq 0.9 - \dist_0^2$. Therefore, using $c_2 \leq 2E_0^{3/2}/5$, we obtain
\begin{align}
     \tfrac{5}{4} c_2 \sqrt{\alpha} \mu_{\ast}^2  &\leq 0.5 \sqrt{\alpha} E_0^{3/2} \mu_\ast^2
     \leq 0.5 \alpha E_0 \mu_\ast^2 \|\mathbf{B}_{\ast,\perp}^\top\mathbf{{B}}_t\|_2
     \nonumber \\
     \implies \|\mathbf{B}_{\ast,\perp}^\top \mathbf{B}_{t+1}\|_2&\leq \|\mathbf{B}_{\ast,\perp}^\top \mathbf{B}_{t}\|_2(1 - 0.5 \beta \alpha E_0 \mu_\ast^2)
\end{align}
\end{proof}

\section{ANIL Finite Samples} \label{app:anil_fs}


First we define the following notations for the finite-sample case.
\begin{table}[h]
\begin{tabular}{|l|l|}
\hline
\textbf{Notation}                                                                                                                                & \textbf{Explanation}                                \\ \hline
$\mathbf{X}_{t,i}^{in} \coloneqq [\mathbf{x}_{t,i,1}^{in}, \dots, \mathbf{x}_{t,n,m_{in}}^{in}]^\top$                                                    & Data for inner loop gradient                        \\ \hline
$\mathbf{\Sigma}_{t,i}^{in} \coloneqq \frac{1}{m_{in}} \sum_{j=1}^{m_{in}} \mathbf{x}_{t,i,j}^{in}(\mathbf{x}_{t,i,j}^{in})^\top$                & Empirical covariance matrix for inner loop gradient \\ \hline
$\mathbf{z}_{t,i}^{in} \coloneqq [z_{t,i,1},\dots, z_{t,i,m_{in}}]$                                                                                      & Additive noise for samples for inner loop gradient  \\ \hline
$\del_{t,i}^{in} \coloneqq \mathbf{I}_k - \alpha \mathbf{B}_t^\top \mathbf{\Sigma}_{t,i}^{in}\mathbf{B}_t$ & Finite-sample analogues of $\del_t$ \\ \hline
$\deld_{t,i}^{in}\coloneqq \mathbf{I}_d - \alpha \mathbf{B}_t \mathbf{B}_t^\top \mathbf{\Sigma}_{t,i}^{in}$ & Finite-sample analogues of $\deld_{t}$ \\ \hline

$\mathbf{X}_{t,i}^{out} \coloneqq [\mathbf{x}_{t,i,1}^{out}, \dots, \mathbf{x}_{t,i,m_{out}}^{out}]^\top$                                        & Data for outer loop gradient                        \\ \hline
$\mathbf{\Sigma}_{t,i}^{out} \coloneqq \frac{1}{m_{out}} \sum_{j=1}^{m_{out}} \mathbf{x}_{t,i,j}^{out}(\mathbf{x}_{t,i,j}^{out})^\top$ & Empirical covariance matrix for outer loop gradient \\ \hline
$\mathbf{z}_{t,i}^{out} \coloneqq [z_{t,i,1},\dots, z_{t,i,m_{out}}]$                                                                                & Additive noise for samples for outer loop gradient  \\ \hline
${\delta}_{m,d_1}\coloneqq \tfrac{\sqrt{d_1}+10\sqrt{\log(n)}}{\sqrt{m}}$                                                                                & Local concentration parameter \\ \hline
$\bar{\delta}_{m,d_2}\coloneqq \tfrac{10 \sqrt{d_2}}{\sqrt{nm}}$                                                                                & Global concentration parameter \\ \hline
\end{tabular}
\end{table}

The inner loop update for the head of the $i$-th task on iteration $t$ is given by:
\begin{align}
    \mathbf{w}_{t,i} 
    &= \mathbf{w}_t - \alpha \nabla_{\mathbf{w}} \hat{\mathcal{L}}_i(\mathbf{B}_t, \mathbf{w}_t, \mathcal{D}_{i}^{in}) \nonumber \\
    &=  (\mathbf{I}_k- \alpha\mathbf{B}_t ^\top \mathbf{\Sigma}_{t,i}^{in} \mathbf{B}_t )\mathbf{{w}}_t + \alpha \mathbf{B}_t ^\top \mathbf{\Sigma}_{t,i}^{in} \mathbf{{B}}_\ast\mathbf{{w}}_{\ast,t,i} + \tfrac{\alpha}{m_{in}} \mathbf{B}_t^\top (\mathbf{X}_{t,i}^{in})^\top \mathbf{z}_{t,i}^{in}. \label{upd_head_anil_exact}
\end{align}
For Exact ANIL, the finite-sample loss after the inner loop update is given by:
\begin{align}
\hat{F}_{t,i}(&\mathbf{B}_t, \mathbf{w}_t; \mathcal{D}_{t,i}^{in}, \mathcal{D}_{t,i}^{out}) \nonumber \\ &\coloneqq \hat{\mathcal{L}}_{t,i}(\mathbf{B}_t, \mathbf{w}_t - \alpha \nabla_{\mathbf{w}}\hat{\mathcal{L}}_{t,i}(\mathbf{B}_t, \mathbf{w}_t ;\mathcal{D}_{t,i}^{in}) ;  \mathcal{D}_{t,i}^{out}  )   \nonumber \\
&=  \tfrac{1}{2m_{out}}\sum_{j=1}^{m_{out}}(\mathbf{x}_{t,i,j}^\top \mathbf{B}_t(\mathbf{w}_t - \alpha \mathbf{B}_t^\top \mathbf{\Sigma}_{t,i}^{in} \mathbf{B}_t\mathbf{w}_t + \alpha \mathbf{B}_t^\top \mathbf{\Sigma}_{t,i}^{in} \mathbf{B}_\ast\mathbf{w}_{\ast,t,i} - \tfrac{\alpha}{m_{in}} \mathbf{B}_t^\top (\mathbf{X}_{t,i}^{in})^\top \mathbf{z}_{t,i}^{in}) \nonumber \\
&\quad \quad \quad \quad \quad  \quad \quad - \mathbf{x}_{t,i,j}^\top \mathbf{B}_\ast \mathbf{w}_{\ast,t,i})  - z_{t,i,j}^{out})^2 \nonumber \\
    &= \tfrac{1}{2m_{out}}\sum_{j=1}^{m_{out}}(\mathbf{x}_{t,i,j}^\top{\deldin} (\mathbf{B}_t \mathbf{w}_t - \mathbf{B}_\ast \mathbf{w}_{\ast,t,i})+ \tfrac{\alpha}{m_{in}} \mathbf{x}_{t,i,j}^\top \b \b^\top (\mathbf{X}_{t,i}^{in})^\top\mathbf{z}_{t,i}^{in} 
 - z_{t,i,j})^2 \nonumber \\
 &= \tfrac{1}{2m_{out}} \left\| \mathbf{\hat{v}}_{t,i}\right\|_2^2 \nonumber \\
  \mathbf{\hat{v}}_{t,i}&\coloneqq \mathbf{X}_{t,i}^{out}\deldin (\mathbf{B}_t \mathbf{w}_t - \mathbf{B}_\ast \mathbf{w}_{\ast,t,i})
 + \tfrac{\alpha}{m_{in}} \mathbf{X}_{t,i}^{out} \b\b^\top (\mathbf{X}_{t,i}^{in})^\top\mathbf{z}_{t,i}^{in} - \zout \nonumber 
 \end{align}
Therefore, using the chain rule,  the exact outer loop gradients for the $i$-th task are:
 \begin{align}
{\nabla}_{\mathbf{B}} \hat{F}_{t,i}(\mathbf{B}_t, \mathbf{w}_t ; \mathcal{D}_{t,i}^{in},\mathcal{D}_{t,i}^{in})&= (\deldin)^\top \tfrac{1}{m_{out}} (\mathbf{X}_{t,i}^{out})^\top \mathbf{\hat{v}}_{t,i}\mathbf{w}_t^\top - \alpha \tfrac{1}{m_{out}} (\mathbf{X}_{t,i}^{out})^\top \mathbf{\hat{v}}_{t,i}\w^\top \bsbin \nonumber \\
&\quad + \alpha  \tfrac{1}{m_{out}} (\mathbf{X}_{t,i}^{out})^\top \mathbf{\hat{v}}_{t,i}\mathbf{w}_{\ast,t,i}^\top \bstin  \nonumber \\
 &\quad - \alpha \sin \b\w \mathbf{\hat{v}}_{t,i}^\top \tfrac{1}{m_{out}} \mathbf{X}_{t,i}^{out}\b + \alpha \sin\mathbf{B}_
 \ast\mathbf{w}_{\ast,t,i} \mathbf{\hat{v}}_{t,i}^\top \tfrac{1}{m_{out}} \mathbf{X}_{t,i}^{out} \b \nonumber \\
 &\quad + \tfrac{\alpha^2}{m_{in}m_{out}}(\mathbf{X}_{t,i}^{out})^\top \mathbf{\hat{v}}_{t,i}(\mathbf{z}_{t,i}^{in})^
 \top \mathbf{X}_{t,i}^{in} \mathbf{B}_t +\tfrac{\alpha^2}{m_{in}m_{out}}(\mathbf{X}_{t,i}^{in})^\top\mathbf{z}_{t,i}^{in} \mathbf{\hat{v}}_{t,i}^\top\mathbf{X}_{t,i}^{out} \mathbf{B}_t   \nonumber \\
 {\nabla}_{\mathbf{w}} \hat{F}_{t,i}(\mathbf{B}_t, \mathbf{w}_t; \mathcal{D}_{t,i}^{in}, \mathcal{D}_{t,i}^{out}) &= \b^\top(\deldin)^\top \tfrac{1}{m_{out}}(\mathbf{X}_{t,i}^{out})^\top \mathbf{\hat{v}}_{t,i}- \tfrac{\alpha}{m_{out}}\mathbf{B}_t^\top (\mathbf{X}_{t,i}^{out})^\top\mathbf{z}_{t,i}^{out} \nonumber 
\end{align}
Meanwhile, the first-order outer loop gradients for the $i$-th task are 
\begin{align}
{\nabla}_{\mathbf{B}}\hat{\mathcal{L}}_{t,i}(&\mathbf{B}_t, \mathbf{w}_{t,i}; \mathcal{D}_{t,i}^{in}, \mathcal{D}_{t,i}^{out})\nonumber \\
&= \mathbf{B}_t^\top \mathbf{\Sigma}_{t,i}^{out} \mathbf{B}_t\mathbf{w}_{t,i} - \mathbf{B}_t^\top \mathbf{\Sigma}_{t,i}^{out} \mathbf{B}_\ast \mathbf{w}_{\ast,t,i}  \nonumber \\
    &= \mathbf{\Sigma}^{out}_{t,i}(\mathbf{B}_t\mathbf{w}_{t,i} - \mathbf{B}_\ast \mathbf{w}_{\ast,t,i})\mathbf{w}_{t,i}^\top  - \tfrac{\alpha}{m_{out}} (\mathbf{X}_{t,i}^{out})^\top\mathbf{z}_{t,i}^{out}\mathbf{w}_{t,i}^\top \nonumber \\
    &= \mathbf{\Sigma}^{out}_{t,i}(\mathbf{B}_t(\del_{t,i}^{in}\w + \alpha \mathbf{B}_t^\top \mathbf{\Sigma}_{t,i}^{in}\mathbf{B}_\ast \mathbf{w}_{\ast,t,i}) - \mathbf{B}_\ast \mathbf{w}_{\ast,t,i})(\del_{t,i}^{in}\w +  \alpha \mathbf{B}_t^\top \mathbf{\Sigma}_{t,i}^{in}\mathbf{B}_\ast \mathbf{w}_{\ast,t,i})^{\top}  \nonumber \\
    &\quad - \tfrac{\alpha}{m_{out}} (\mathbf{X}_{t,i}^{out})^\top\mathbf{z}_{t,i}^{out}\mathbf{w}_{t,i}^\top \nonumber \\
    &= \mathbf{\Sigma}^{out}_{t,i}\deld_{t,i}^{in}(\mathbf{B}_t\w - \mathbf{B}_\ast \mathbf{w}_{\ast,t,i})(\del_{t,i}^{in}\w +  \alpha \mathbf{B}_t^\top \mathbf{\Sigma}_{t,i}^{in}\mathbf{B}_\ast \mathbf{w}_{\ast,t,i})^{\top} - \tfrac{\alpha}{m_{out}} (\mathbf{X}_{t,i}^{out})^\top\mathbf{z}_{t,i}^{out}\mathbf{w}_{t,i}^\top \nonumber 
\end{align}
\begin{align}
    {\nabla}_{\mathbf{w}}\hat{\mathcal{L}}_{t,i}(\mathbf{B}_t, \mathbf{w}_{t,i}; \mathcal{D}_{t,i}^{in}, \mathcal{D}_{t,i}^{out}) &= \mathbf{B}_t^\top \sout \deldin( \b \w -  \mathbf{B}_\ast \mathbf{w}_{\ast,t,i}) - \tfrac{\alpha}{m_{out}}\mathbf{B}_t^\top (\mathbf{X}_{t,i}^{out})^\top\mathbf{z}_{t,i}^{out} \nonumber
\end{align}
Define
\begin{align}
    \mathbf{\hat{G}}_{\mathbf{B},t} &\coloneqq \frac{1}{n}\sum_{i=1}^n \nabla_{\mathbf{B}} \hat{F}_{t,i}(\mathbf{B}_t, \mathbf{w}_t) , \quad\mathbf{{G}}_{\mathbf{B},t} \coloneqq \frac{1}{n}\sum_{i=1}^n \nabla_{\mathbf{B}} {F}_{t,i}(\mathbf{B}_t, \mathbf{w}_t)  \nonumber \\
    \mathbf{\hat{G}}_{\mathbf{w},t} &\coloneqq \frac{1}{n}\sum_{i=1}^n \nabla_{\mathbf{w}} \hat{F}_{t,i}(\mathbf{B}_t, \mathbf{w}_t), \quad \mathbf{{G}}_{\mathbf{w},t} \coloneqq \frac{1}{n}\sum_{i=1}^n \nabla_{\mathbf{w}} \hat{F}_{t,i}(\mathbf{B}_t, \mathbf{w}_t) \nonumber
\end{align}

Now we are ready to state the result.
\begin{thm}[ANIL Finite Samples] \label{thm:anil_fs_app}
Suppose Assumptions \ref{assump:tasks_diverse_main}, \ref{assump:tasks_main} and \ref{assump:data} hold. 
Let  $E_0\!\coloneqq\! 0.9 -\! \dist_0^2 - \delta$ for some $\delta \in (0, 1)$ to be defined shortly and assume  $E_0$ is a positive constant. 
Suppose the initialization further satisfies $\alpha \mathbf{B}_0^\top \mathbf{B}_0 = \mathbf{I}_k$ and $\mathbf{w}_0 = \mathbf{0}$, and let the step sizes be chosen as $\alpha \leq \tfrac{c'}{\sqrt{k}L_{\ast}+\sigma}$, 
and $\beta \leq \frac{c'\alpha E_0^2 }{ \kappa_\ast^4}$ for ANIL and  $\beta\leq \frac{c' \alpha E_0^3 \mu_\ast}{ \kappa_{\ast}^4}\min\left(1, \tfrac{\mu_\ast^2}{\eta_\ast^2}\right)$ for FO-ANIL, for some absolute constant $c'$. 
Then there exists a constant $c>0$ such that, for ANIL, if 
\begin{align}
m_{out} &\geq cT^2 \tfrac{k^2(L_\ast+\sigma)^2}{n\eta_\ast^2 \kappa_\ast^8} +  cT\tfrac{k^3 \kappa_{\ast}^2 (\kappa_{\ast}^2+\sigma^2/\mu_\ast^2)}{n } + c\sqrt{T}(k + \tfrac{kd}{n} +\log(n))\kappa_\ast^{-2}(\tfrac{\sigma^2}{ L_\ast^2}+k) + ck + c\log(n)
   \nonumber \\
    m_{in} &\geq cT^2(k^2+k\log(n))\tfrac{(L_{\ast}+\sigma)^2}{\eta_\ast^2 \kappa_\ast^8} + cT (k^3+k\log(n)){(\kappa_\ast^4 +\tfrac{\sigma^4}{\mu_\ast^4})} + c \sqrt{T}\tfrac{  k^3 {d\log(nm_{in})}}{n }\kappa_\ast^{-2}(\tfrac{\sigma^2}{ L_\ast^2}+1)  
\end{align}
and for FO-ANIL, if 
\begin{align}
m_{out}&\geq  c \tfrac{T  dk}{n\kappa_\ast^{2}}+ c\tfrac{T  dk\sigma^2}{n L_\ast^2 \kappa_\ast^{2}}+ c\tfrac{T^2k^3\kappa_\ast^4}{n} +c\tfrac{T^2k^3\sigma^4}{n\mu_\ast^4} + c\tfrac{k\mu_\ast^2}{\eta_\ast^2 \kappa_{\ast}^6} + c\tfrac{k\sigma^2}{\eta_\ast^2 \kappa_{\ast}^8}
\nonumber \\
    m_{in}&\geq c T (k+\log(n))(k\kappa_\ast^2 + \tfrac{\sigma^2}{ \mu_\ast^{2}})+ c\tfrac{T^2 k^3 \kappa_{\ast}^4}{n} + c \tfrac{T^2k^2 \kappa_\ast^2 \sigma^2}{\mu_\ast^{2}n} + c\tfrac{T^2k^2 (L_\ast^2+\sigma^2)}{\eta_\ast^{2}\kappa_\ast^{8}n} \nonumber
\end{align}
then both ANIL and FO-ANIL  satisfy that after $T$ iterations,
\begin{align}
    \dist(\mathbf{{B}}_T, \mathbf{{B}}_\ast) \leq  \left(1 - 0.5\beta \alpha E_0 \mu_\ast^2   \right)^{T-1} + O(\delta)
\end{align}
with probability at least $1 - O(T\exp(-90k)) - \tfrac{T}{\poly(n)} - \tfrac{T}{\poly(m_{in})}$, where for ANIL,
\begin{align}
\delta&= \tfrac{1}{{m_{in}}}\bigg( \sqrt{k}+ \tfrac{ \sigma }{L_\ast}\bigg)+\tfrac{1}{\sqrt{m_{in}}}\bigg((k \kappa_{\ast}^2+ \sqrt{k}\kappa_\ast \sigma/\mu_\ast)(\sqrt{k}+\sqrt{\log(n)})  \bigg)  \nonumber \\
&\quad +\tfrac{1}{\sqrt{m_{out}}}\bigg( (k\kappa_{\ast}^2+\sqrt{k}\kappa_\ast\sigma/\mu_\ast)(\sqrt{k}+\sqrt{\log(n)})  \bigg) \nonumber \\
&\quad +\tfrac{1}{\sqrt{nm_{in}}}\bigg( (k\kappa_{\ast}^2+\sqrt{k}\kappa_\ast\sigma/\mu_\ast) (k \sqrt{d\log(nm_{in})} + k\log(nm_{in}) + \sqrt{d}\log^{1.5}(nm_{in}) +\log^2(nm_{in}))\nonumber \\
&\quad + \tfrac{\sigma^2}{\mu_\ast^2} (\sqrt{kd} +  \sqrt{d}\log(nm_{in}) +\log^{1.5}(nm_{in}))  + (k\kappa_{\ast}^2+\sqrt{k}\kappa_\ast\sigma/\mu_\ast)\sqrt{d}\bigg) \nonumber \\
&\quad+\tfrac{1}{\sqrt{nm_{out}}}\bigg( (k\kappa_{\ast}^2+\sqrt{k}\kappa_\ast\sigma/\mu_\ast)\sqrt{d} + \tfrac{\sigma^2}{\mu_\ast^2} (\tfrac{\sqrt{d}}{\sqrt{m_{in}}} + \sqrt{k})\bigg)  \nonumber 
\end{align}
and for FO-ANIL,
\begin{align}
    \delta = (\sqrt{k}\kappa_{\ast}^2+\tfrac{\kappa_\ast\sigma}{\mu_\ast}+ \tfrac{\sigma^2}{\mu_\ast^2\sqrt{m_{in}}})\tfrac{\sqrt{dk}}{\sqrt{nm_{out}}} 
\end{align}
\end{thm}

\begin{proof}
 
The proof uses an inductive argument with the following five inductive hypotheses:
  \begin{enumerate}
\item $A_1(t) \coloneqq \{\|\mathbf{w}_{t}\|_2 \leq \frac{\sqrt{\alpha} E_0 }{10} \min(1, \tfrac{\mu_\ast^2}{\eta_\ast^2}) \eta_\ast\}$ 
 \item 
$A_2(t) \coloneqq \{\| \del_t \|_2 \leq (1 - 0.5 \beta \alpha E_0 \mu_\ast^2 )\|\del_{t-1} \|_2   + \tfrac{5}{4}\alpha^2 \beta^2 L_\ast^4 \dist_{t-1}^2+  \beta \alpha \zeta_2\},$
\item $A_3(t) \coloneqq \{\| \del_t\|_2 \leq \frac{1}{10}\}$,
     \item 
      $A_4(t) \coloneqq \{ \|\mathbf{{B}}_{\ast,\perp}^\top \mathbf{B}_{t}\|_2 \leq \left(1 -  0.5 \beta\alpha E_0 \mu_\ast^2  \right) \|\mathbf{{B}}_{\ast,\perp}^\top \mathbf{B}_{t-1}\|_2 + \beta \sqrt{\alpha} \zeta_4\}$,
      \item $A_5(t) \coloneqq \{ \dist_{t} \leq \tfrac{\sqrt{10}}{3}\left(1 - 0.5 \beta  \alpha E_0 \mu_\ast^2  \right)^t\dist_0 +\delta \}$.
 \end{enumerate}
where $\zeta_2$ is defined separately for ANIL and FO-ANIL in Lemmas \ref{lem:exactanil_fs_a2} and \ref{lem:reg_fs_foanil}, respectively, and $\zeta_4$ is defined separately for ANIL and FO-ANIL in Lemmas \ref{lem:exactanil_fs_a4} and \ref{lem:contract_app_fs}, respectively.
These conditions hold for iteration $t=0$ due to the choice of initialization $(\mathbf{B}_0,\mathbf{w}_0)$. We will show that if they hold for all iterations up to and including  iteration $t$ for an arbitrary $t$, then they hold at iteration $t+1$ with probability at least $1-\tfrac{1}{\poly(n)} -\tfrac{1}{\poly(m_{in})} - O(\exp(-90k))$.

\begin{enumerate}
\item $\bigcap_{s=0}^t \{A_2(s) \cap A_6(s)\}   \implies A_1(t+1)$. This is Lemma \ref{lem:w_anil_fs} for FO-ANIL and Lemma \ref{lem:exactanil_fs_a1} for Exact ANIL.

\item $A_1(t) \cap A_3(t) \cap A_5(t)  \implies A_2(t+1)$. This is Lemma \ref{lem:reg_fs_foanil} for FO-ANIL and Lemma \ref{lem:exactanil_fs_a2} for Exact ANIL.

\item $A_1(t) \cap A_2(t\!+\!1) \cap A_3(t) \cap A_5(t)  \implies A_3(t+1)$.
This is Corollary \ref{cor:foanil_rep} for FO-ANIL and Corollary \ref{cor:exactanil_fs_a3} for Exact ANIL.

\item $A_1(t) \cap A_3(t) \cap A_5(t) \implies A_4(t+1)$. This is Lemma \ref{lem:contract_app_fs} for FO-ANIL and Lemma \ref{lem:exactanil_fs_a4} for Exact ANIL.

\item $A_3(t\!+\!1) \cap \left( \cap_{s=1}^{t+1} A_4(s) \right) \implies A_5(t+1) $. 
By $A_3(t+1)$ and $A_4(t+1)$ we have:
\begin{align}
     \|\mathbf{B}_{\ast,\perp}^\top \mathbf{B}_{t+1}\|_2 &\leq (1 - 0.5 \beta \alpha E_0 \mu_\ast^2)\|\mathbf{B}_{\ast,\perp}^\top \mathbf{B}_{t}\|_2 + \beta \sqrt{\alpha}\zeta_4 \nonumber \\
     &\leq (1 - 0.5 \beta \alpha E_0 \mu_\ast^2)^{2}\|\mathbf{B}_{\ast,\perp}^\top \mathbf{B}_{t-1}\|_2 + (1 - 0.5 \beta \alpha E_0 \mu_\ast^2)\beta \sqrt{\alpha}{\zeta_4} + \beta \sqrt{\alpha}\zeta_4 \nonumber \\
     &\quad \vdots \nonumber \\
     &\leq (1 - 0.5 \beta \alpha E_0 \mu_\ast^2)^{t}\|\mathbf{B}_{\ast,\perp}^\top \mathbf{B}_{0}\|_2+ \beta \sqrt{\alpha}\zeta_4 \sum_{s=0}^{t} {(1 - 0.5 \beta \alpha E_0 \mu_\ast^2)}^{s} \nonumber \\
   &\leq (1 - 0.5 \beta \alpha E_0 \mu_\ast^2)^{t}\|\mathbf{B}_{\ast,\perp}^\top \mathbf{B}_{0}\|_2+ \frac{\beta \sqrt{\alpha}\zeta_4 }{ 1-(1 - 0.5 \beta \alpha E_0 \mu_\ast^2)} \nonumber \\
   &= (1 - 0.5 \beta \alpha E_0 \mu_\ast^2)^{t}\|\mathbf{B}_{\ast,\perp}^\top \mathbf{B}_{0}\|_2+ \frac{2 \zeta_4 }{ \sqrt{\alpha} E_0 \mu_\ast^2}. \label{dots}
\end{align}
Now we orthogonalize $\mathbf{B}_t$ and $\mathbf{B}_0$ via the QR-factorization, writing $\mathbf{B}_t = \mathbf{\hat{B}}_t \mathbf{R}_t$ and $\mathbf{B}_0 = \mathbf{\hat{B}}_0 \mathbf{R}_0$. By inductive hypothesis $A_3(t+1)$, we have $ \sigma_{\min}(\mathbf{B}_{t+1})\geq {\frac{\sqrt{0.9}}{\sqrt{\alpha}}}$, and by the initialization  we have $\sigma_{\max}(\mathbf{B}_0) \leq {\frac{1}{\sqrt{\alpha}}}$. Thus, using \eqref{dots} and the definition of the principal angle distance, we have
\begin{align}
     \dist(\mathbf{{B}}_{t+1}, \mathbf{{B}}_\ast ) &\leq \left((1 - 0.5 \beta \alpha E_0 \mu_\ast^2)^t\dist(\mathbf{{B}}_{0}, \mathbf{{B}}_\ast)\|\mathbf{R}_0\|_2  + \frac{2 \zeta_4 }{ \sqrt{\alpha} E_0 \mu_\ast^2}\right)\|\mathbf{R}_{t+1}^{-1}\|_2 \nonumber \\
     &\leq \frac{\sqrt{10}}{3}(1 - 0.5 \beta \alpha E_0 \mu_\ast^2)^t\dist(\mathbf{{B}}_{0}, \mathbf{{B}}_\ast) + \frac{3 \zeta_4 }{ E_0 \mu_\ast^2} \\
     &\leq (1 - 0.5 \beta \alpha E_0 \mu_\ast^2)^t + \delta
\end{align}
where $\varepsilon = O(\tfrac{\zeta_4}{\mu_\ast^2}) 
$. 
\end{enumerate}
After $T$ rounds, we have that the inductive hypotheses hold on every round with probability at least 
\begin{align}
   (1-\tfrac{1}{\poly(n)} -\tfrac{1}{\poly(m_{in})} - O(\exp(-90k)))^{T}\geq  1 - O(T\exp(-90k)) - \tfrac{T}{\poly(n)} - \tfrac{T}{\poly(m_{in})}
\end{align}
where the inequality follows by the Weierstrass Inequality, completing the proof.
\end{proof}

Throughout the proof we will re-use $c$, $c', c''$, etc. to denote absolute constants.

\subsection{General Concentration Lemmas}
We start with generic concentration results for random matrices and vectors  that will be used throughout the proof. 

We use $\chi_{\mathcal{E}}$ to denote the indicator random variable for the event $\mathcal{E}$, i.e. $\chi_{\mathcal{E}} = 1$ if $\mathcal{E}$ holds and $\chi_{\mathcal{E}} = 0$ otherwise.







\begin{lemma}
\label{lem:gen1}
Let $\mathbf{X}_{1} = [\mathbf{x}_{1,1},\dots,\mathbf{x}_{1,m_{1}}]^{\top} \in \mathbb{R}^{m_1 \times d}$ have rows which are i.i.d. samples from a mean-zero, $\mathbf{I}_d$-sub-gaussian distribution, and let $\mathbf{X}_{1,1},\dots,\mathbf{X}_{1,n}$ be independent copies of $\mathbf{X}_1$. Likewise, let  $\mathbf{X}_2= [\mathbf{x}_{2,1},\dots,\mathbf{x}_{2,m_{2}}]^{\top}\in \mathbb{R}^{m_2 \times d}$ have rows which are i.i.d. samples from a mean-zero, $\mathbf{I}_d$-sub-gaussian distribution, and let $\mathbf{X}_{2,1},\dots,\mathbf{X}_{2,n}$ be independent copies of $\mathbf{X}_2$ (and independent of $\mathbf{X}_{1,1},\dots,\mathbf{X}_{1,n}$). Define $\mathbf{\Sigma}_{1,i}\coloneqq \frac{1}{m_1}\mathbf{X}_{1,i}^\top \mathbf{X}_{1,i}$ and $\mathbf{\Sigma}_{2,i}\coloneqq \frac{1}{m_2}\mathbf{X}_{2,i}^\top \mathbf{X}_{2,i}$ for all $i \in [n]$. Let the elements of $\mathbf{z}_1\in \mathbb{R}^{m_1}$ and $\mathbf{z}_2\in \mathbb{R}^{m_2}$ be i.i.d. samples from $\mathcal{N}(0,\sigma^2)$.
Further, let $\mathbf{C}_{\ell,i}\in \mathbb{R}^{d\times d_{\lfloor \ell/2 \rfloor}}$ for $\ell = 1,\dots,6$ be fixed matrices for $i \in [n]$, and let $c_\ell \coloneqq \max_{i\in[n]}\|\mathbf{C}_{\ell,i}\|_2$ for $\ell=1,\dots,6$. Let $\delta_{m,d_{l}} \coloneqq c \frac{\sqrt{d_l}+10\sqrt{\log(n)}}{\sqrt{m}}$ and $\bar{\delta}_{m,d_{l}} \coloneqq c \frac{10\sqrt{d_l}}{\sqrt{nm}}$ for some absolute constant $c$. Assume that in all cases below, each $\delta$ and $\bar{\delta}$ is less than 1.  Then the following hold:
\begin{enumerate}
    \item $\mathbb{P}\left(\left\|\frac{1}{n}\sum_{i=1}^n \mathbf{C}_{1,i}^\top \mathbf{\Sigma}_{1,i} \mathbf{C}_{2,i} -  \mathbf{C}_{1,i}^\top \mathbf{C}_{2,i} \right\|_2\geq c_1c_2\bar{\delta}_{m_1,d_0+d_1} \right)\leq 2e^{-90(d_0+d_1)}$
    \item $\mathbb{P}\bigg(\left\|\frac{1}{n}\sum_{i=1}^n \mathbf{C}_{1,i}^\top \mathbf{\Sigma}_{1,i} \mathbf{C}_{2,i}\mathbf{C}_{3,i}^\top\mathbf{\Sigma}_{2,i}\mathbf{C}_{4,i} -  \mathbf{C}_{1,i}^\top \mathbf{C}_{2,i}\mathbf{C}_{3,i}^\top \mathbf{C}_{4,i}  \right\|_2 \\
    \geq \sigma c_1c_2c_3c_4\left((1+\delta_{m_2,d_1+d_2})\bar{\delta}_{m_1, d_0\!+\!d_2}+ \bar{\delta}_{m_2, d_0\!+\!d_2} \right)\leq 2e^{-90(d_0+d_2)} + 2n^{-99}$ 
        \item $\mathbb{P}\bigg(\left\|\frac{1}{n}\sum_{i=1}^n \mathbf{C}_{1,i}^\top \mathbf{\Sigma}_{1,i} \mathbf{C}_{2,i}\mathbf{C}_{3,i}^\top\mathbf{\Sigma}_{2,i}\mathbf{C}_{4,i} -  \mathbf{C}_{1,i}^\top \mathbf{C}_{2,i}\mathbf{C}_{3,i}^\top \mathbf{C}_{4,i}  \right\|_2 \\
    \geq \sigma c_1c_2c_3c_4\left((1+\delta_{m_2,d_1+d_2})\bar{\delta}_{m_1, d_0\!+\!d_2}+ {\delta}_{m_2, d_1\!+\!d_2} \right)\leq 2e^{-90(d_0+d_2)} + 2n^{-99}$ 
    \item $\mathbb{P}\left(\left\|\frac{1}{n}\sum_{i=1}^n \mathbf{C}_{1,i}^\top\mathbf{X}_{1,i}^\top \mathbf{z}_{1,i} \right\|_2 \geq \sigma c_1 \bar{\delta}_{m_1,d_0} \right)\leq 2e^{-90d_0}$ 
    \item  $ \mathbb{P}\left(\left\|\frac{1}{n}\sum_{i=1}^n \mathbf{C}_{1,i}^\top \mathbf{\Sigma}_{1,i} \mathbf{C}_{2,i}\mathbf{C}_{3,i}^\top\frac{1}{m_2}\mathbf{X}_{2,i}^\top \mathbf{z}_{2,i} \right\|_2 \geq c_1 c_2 c_3(1\! +\!\bar{\delta}_{m_1,d_0}){\delta}_{m_2,d_1}  \right)\leq 2e^{-90d_1} + 2n^{-99}$  
    \item $\mathbb{P}\left(\left\|\frac{1}{n}\sum_{i=1}^n \mathbf{C}_{1,i}^\top\frac{1}{m_1}\mathbf{X}_{1,i}^\top \mathbf{z}_{1,i}\mathbf{c}_{2,i}^\top \mathbf{\Sigma}_{2,i}\mathbf{C}_{3,i} \right\|_2 \geq \sigma c_1 c_2 c_3 (1+ {\delta}_{m_2,d_1}) \bar{\delta}_{m_1,d_0} \right)\leq 2e^{-90d_0} + 2n^{-99}$ 
    \item $\mathbb{P}\left(\left\|\frac{1}{n}\sum_{i=1}^n \mathbf{C}_{1,i}^\top\frac{1}{m_1}\mathbf{X}_{1,i}^\top \mathbf{z}_{1,i} \frac{1}{m_2}\mathbf{z}_{2,i}^\top \mathbf{X}_{2,i}\mathbf{C}_{2,i} \right\|_2 \geq \sigma^2 c_1 c_2 \bar{\delta}_{m_1, d_0}{\delta}_{m_2,d_1} \right)\leq 2e^{-90d_0} + 2e^{-90d_1}$
    \item $\mathbb{P}\bigg(\left\|\frac{1}{n}\sum_{i=1}^n \mathbf{C}_{1,i}^\top \mathbf{\Sigma}_{1,i} \mathbf{C}_{2,i}\mathbf{C}_{3,i}^\top\mathbf{\Sigma}_{2,i}\mathbf{C}_{4,i}\mathbf{C}_{5,i}^\top \mathbf{\Sigma}_{2,i} \mathbf{C}_{6,i}  -  \mathbf{C}_{1,i}^\top \mathbf{C}_{2,i}\mathbf{C}_{3,i}^\top \mathbf{C}_{4,i}\mathbf{C}_{5,i}^\top \mathbf{C}_{6,i} \right\|_2 \\
    \quad \quad \quad \geq c_1c_2c_3c_4 c_5 c_6\left((1+{\delta}_{m_2, d_1 + d_2})(1+ {\delta}_{m_2,d_2 + d_3})\bar{\delta}_{m_1, d_0+d_3}+ {\delta}_{m_2,d_2+d_3}+ (1+\delta_{m_2,d_2+d_3})\delta_{m_2, d_1+d_2} \right) \bigg)\\
    \leq 2e^{-90(d_0+d_3)} + 4n^{-99}$ 
    \item $\mathbb{P}\bigg(\left\|\frac{1}{n}\sum_{i=1}^n \mathbf{C}_{1,i}^\top \mathbf{\Sigma}_{1,i} \mathbf{C}_{2,i}\mathbf{C}_{3,i}^\top\mathbf{\Sigma}_{2,i}\mathbf{C}_{4,i}\mathbf{C}_{5,i}^\top \mathbf{\Sigma}_{1,i} \mathbf{C}_{6,i}  -  \mathbf{C}_{1,i}^\top \mathbf{C}_{2,i}\mathbf{C}_{3,i}^\top \mathbf{C}_{4,i}\mathbf{C}_{5,i}^\top \mathbf{C}_{6,i} \right\|_2 \\
    \quad \quad \quad \geq c_1c_2c_3c_4 c_5 c_6\left((1+{\delta}_{m_1, d_0 + d_1})(1+ {\delta}_{m_1,d_2 + d_3})\bar{\delta}_{m_2, d_1+d_2}+ \bar{\delta}_{m_1,d_0+d_3}+ (1+\delta_{m_1,d_2+d_3})\delta_{m_1, d_0+d_1} \right)  \bigg) \\
    \leq 2e^{-90(d_0+d_3)} + 4n^{-99}$  
    \item $\mathbb{P}\bigg(\left\|\frac{1}{n}\sum_{i=1}^n \mathbf{C}_{1,i}^\top \mathbf{\Sigma}_{1,i} \mathbf{C}_{2,i}\mathbf{C}_{3,i}^\top\mathbf{X}_{2,i}^\top \mathbf{z}_{2,i}\mathbf{c}_{4,i}^\top \mathbf{C}_{5,i}^\top \mathbf{\Sigma}_{2,i}\mathbf{C}_{6,i}   \right\|_2 \\
    \quad \quad \quad \geq \sigma c_1c_2c_3c_4 c_5 c_6 \left(1 +\bar{\delta}_{m_1,d_0+d_1}\right)\left(1\! +\! \delta_{m_2,d_3}\right)\delta_{m_2,d_1} \bigg) \leq 2e^{-90d_1} + 4n^{-99}$
    \item $\mathbb{P}\bigg(\left\|\frac{1}{n}\sum_{i=1}^n \mathbf{C}_{1,i}^\top \mathbf{\Sigma}_{1,i} \mathbf{C}_{2,i}\mathbf{C}_{3,i}^\top\mathbf{\Sigma}_{2,i}\mathbf{C}_{4,i}\mathbf{C}_{5,i}^\top\mathbf{X}_{1,i}^\top \mathbf{z}_{1,i} \right\|_2 \\
    \quad \quad \quad \geq \sigma c_1c_2c_3c_4 c_5 c_6 \left(1 +{\delta}_{m_1,d_0+d_1}\right)\left(1\! +\! \bar{\delta}_{m_2,d_1+d_2}\right)\delta_{m_1,d_2} \bigg) \leq 2e^{-90d_2} + 4n^{-99} $
    \item $\mathbb{P}\bigg(\left\|\frac{1}{n}\sum_{i=1}^n \mathbf{C}_{1,i}^\top \mathbf{X}_{1,i}^\top \mathbf{z}_{1,i}  \mathbf{c}_{2,i}^\top \mathbf{C}_{3,i}^\top \mathbf{\Sigma}_{1,i} \mathbf{C}_{4,i}\mathbf{C}_{5,i}^\top\mathbf{\Sigma}_{2,i}\mathbf{C}_{6,i} \right\|_2 \\
    \quad \quad \quad \geq \sigma c_1c_2c_3c_4 c_5 c_6 \left(1 +{\delta}_{m_2,d_2+d_3}\right)\left(1\! +\! {\delta}_{m_2,d_1+d_2}\right)\delta_{m_1,d_0} \bigg) \leq 2e^{-90d_0} + 6n^{-99} $
    \item $\mathbb{P}\bigg(\left\|\frac{1}{n}\sum_{i=1}^n \mathbf{C}_{1,i}^\top \mathbf{\Sigma}_{1,i} \mathbf{C}_{2,i}\mathbf{C}_{3,i}^\top\mathbf{X}_{2,i}^\top \mathbf{z}_{2,i}\mathbf{z}_{2,i}^\top \mathbf{X}_{2,i}\mathbf{C}_{4,i}   \right\|_2 \\
    \quad \quad \quad \geq \sigma^2 c_1c_2c_3c_4 c_5 \delta_{m_2,d_1}\delta_{m_2,d_2} \left(1\!+\!\bar{\delta}_{m_1, d_0+d_2}\right)  \bigg) \leq 2e^{-90d_1} + 2e^{-90d_2}+ 4n^{-99}$
\end{enumerate}
\end{lemma}
\begin{proof}
We give the proofs for (1), (2), and (8) since the rest of the proofs follow using analogous arguments. In all cases, the proofs are standard applications of Bernstein's inequality.

\begin{enumerate}
    \item  
    For any fixed unit vector $\mathbf{u} \in \mathbb{R}^{d_0}$, 
${r}_{\mathbf{u},i,j}\coloneqq  \mathbf{u}^\top\mathbf{C}_{1,i}^\top  \mathbf{w}_{1,i,j}$ is sub-gaussian with sub-gaussian norm at most $c\|\mathbf{C}_{1,i}\|_2$. Likewise, for any fixed unit vector $\mathbf{v} \in \mathbb{R}^{d_1}$, ${r}_{\mathbf{v},i,j}\coloneqq \mathbf{v}^\top \mathbf{C}_{2,i}^\top \mathbf{x}_{2,i,j}$ is sub-gaussian with norm at most $c\|\mathbf{C}_{2,i}\|_2$ for an absolute constant $c$.
Furthermore, $\mathbb{E}[{r}_{\mathbf{v},i,j}{r}_{\mathbf{u},i,j}] =\mathbf{u}^\top \mathbf{C}_{1,i}^\top \mathbf{x}_{1,i,j} \mathbf{x}_{1,i,j}^\top\mathbf{C}_{2,i}\mathbf{v} = \mathbf{u}^\top \mathbf{C}_{1,i}^\top\mathbf{C}_{2,i}\mathbf{v}$. 
Therefore,
\begin{align}
    \mathbf{v}^\top \left(\frac{1}{n}\sum_{i=1}^n \mathbf{C}_{1,i}^\top \mathbf{\Sigma}_{1,i} \mathbf{C}_{2,i} -  \mathbf{C}_{1,i}^\top \mathbf{C}_{2,i} \right)\mathbf{u}
    &= \frac{1}{n m_1}\sum_{i=1}^n \sum_{j=1}^{m_{1}} \left({r}_{\mathbf{v},i,j}{r}_{\mathbf{u},i,j} - \mathbb{E}[{r}_{\mathbf{v},i,j}{r}_{\mathbf{u},i,j}] \right)
\end{align}
is the sum of $n m_1$ independent, mean-zero, sub-exponential random variables with norm $O(\|\mathbf{C}_{1,i}\|_2\|\mathbf{C}_{2,i}\|_2)$. By Bernstein's inequality we have 
\begin{align}
    \left|\frac{1}{n m_{1}}\sum_{i=1}^n \sum_{j=1}^{m_{1}} \left(\mathbb{E}[{r}_{\mathbf{v},i,j}{r}_{\mathbf{u},i,j}] - {r}_{\mathbf{v},i,j}{r}_{\mathbf{u},i,j} \right)\right| 
    &\leq c \max_{i\in [n]} \|\mathbf{C}_{1,i}\|_2\|\mathbf{C}_{2,i}\|_2 \max \left( \tfrac{\sqrt{d_0+d_1}+ \lambda}{\sqrt{nm_{1}}}, \tfrac{(\sqrt{d_0+d_1}+ \lambda)^2}{nm_{1}} \right) \nonumber
\end{align}
for some absolute constant $c$ and any $\lambda > 0$, with probability at least $1 - 2 e^{-\lambda^2}$ over the outer loop samples. Let $\mathcal{S}^{d_0-1}$ and $\mathcal{S}^{d_1-1}$ denote the unit spheres in $\mathbb{R}^{d_0}$ and $\mathbb{R}^k$, respectively. From Corollary 4.2.13 in \cite{vershynin2018high}, we know that there exists $\frac{1}{4}$-nets $\mathcal{M}_1$ and $\mathcal{M}_2$ on $\mathcal{S}^{d_0-1}$ and $\mathcal{S}^{d_1-1}$ with cardinalities at most $9^{d_0}$ and $9^{d_1}$, respectively. Thus, conditioning on  using the variational definition of the spectral norm, and taking a union bound over the $\frac{1}{4}$-nets, we have
\begin{align}
    \left\|\frac{1}{n}\sum_{i=1}^n \mathbf{C}_{1,i}^\top \mathbf{\Sigma}_{1,i} \mathbf{C}_{2,i} -  \mathbf{C}_{1,i}^\top \mathbf{C}_{2,i} \right\|_2  
    &= \max_{\mathbf{v} \in \mathcal{S}^{d_0-1},\mathbf{u} \in \mathcal{S}^{d_1-1}} \left|\frac{1}{n m_{1}}\sum_{i=1}^n \sum_{j=1}^{m_{1}} \left(\mathbb{E}[{r}_{\mathbf{v},i,j}{r}_{\mathbf{u},i,j}] - {r}_{\mathbf{v},i,j}{r}_{\mathbf{u},i,j} \right)\right| \nonumber \\
    &\leq 2\max_{\mathbf{v} \in \mathcal{M}_1,\mathbf{u} \in \mathcal{M}_2} \left|\frac{1}{n m_{out}}\sum_{i=1}^n \sum_{j=1}^{m_{out}} \left(\mathbb{E}[{r}_{\mathbf{v},i,j}{r}_{\mathbf{u},i,j}] - {r}_{\mathbf{v},i,j}{r}_{\mathbf{u},i,j} \right)\right| \nonumber \\
    &\leq c' \max_{i\in [n]} \|\mathbf{C}_{1,i}\|_2\|\mathbf{C}_{2,i}\|_2 \max \left( \tfrac{\sqrt{d_0+d_1}+ \lambda}{\sqrt{nm_{1}}}, \tfrac{(\sqrt{d_0+d_1}+ \lambda)^2}{nm_{1}} \right) \nonumber
\end{align}
for some absolute constant $c$, with probability at least $1 - 2\times 9^{d_0+d_1}e^{-\lambda^2}$ over the outer loop samples. Choose $\lambda= 10\sqrt{d}$ and let $\sqrt{nm_{1}} \geq 11\sqrt{d_0+d_1}$ to obtain that,
\begin{align}
    \left\|\frac{1}{n}\sum_{i=1}^n \mathbf{C}_{1,i}^\top \mathbf{\Sigma}_{1,i} \mathbf{C}_{2,i} -  \mathbf{C}_{1,i}^\top \mathbf{C}_{2,i} \right\|_2  
    &\leq  c \max_{i\in [n]} \|\mathbf{C}_{1,i}\|_2\|\mathbf{C}_{2,i}\|_2 \bar{\delta}_{m_{1},d_0+d_1} \leq  c c_1 c_2 \bar{\delta}_{m_{1},d_0+d_1}  \nonumber
\end{align}
with probability at least $1-2 e^{-90(d_0+d_1)}$.
    \item Let $\mathbf{E}\coloneqq \frac{1}{n}\sum_{i=1}^n \mathbf{C}_{1,i}^\top \mathbf{\Sigma}_{1,i} \mathbf{C}_{2,i}\mathbf{C}_{3,i}^\top\mathbf{\Sigma}_{2,i}\mathbf{C}_{4,i} -  \mathbf{C}_{1,i}^\top \mathbf{C}_{2,i}\mathbf{C}_{3,i}^\top \mathbf{C}_{4,i}$.
    We have 
    \begin{align}
        \|\mathbf{E}\|_2 &\leq \bigg\|\frac{1}{n}\sum_{i=1}^n \mathbf{C}_{1,i}^\top \mathbf{\Sigma}_{1,i} \mathbf{C}_{2,i}\mathbf{C}_{3,i}^\top\mathbf{\Sigma}_{2,i}\mathbf{C}_{4,i} -  \mathbf{C}_{1,i}^\top \mathbf{C}_{2,i}\mathbf{C}_{3,i}^\top \mathbf{C}_{4,i}\bigg\|_2 \nonumber \\
        &= \bigg\|\frac{1}{n}\sum_{i=1}^n \mathbf{C}_{1,i}^\top (\mathbf{\Sigma}_{1,i}-\mathbf{I}_d) \mathbf{C}_{2,i}\mathbf{C}_{3,i}^\top\mathbf{\Sigma}_{2,i}\mathbf{C}_{4,i} +  \mathbf{C}_{1,i}^\top  \mathbf{C}_{2,i}\mathbf{C}_{3,i}^\top\mathbf{\Sigma}_{2,i}\mathbf{C}_{4,i} -  \mathbf{C}_{1,i}^\top \mathbf{C}_{2,i}\mathbf{C}_{3,i}^\top \mathbf{C}_{4,i}\bigg\|_2 \nonumber \\
        &\leq \bigg\|\underbrace{\frac{1}{n}\sum_{i=1}^n \mathbf{C}_{1,i}^\top (\mathbf{\Sigma}_{1,i} - \mathbf{I}_d) \mathbf{C}_{2,i}\mathbf{C}_{3,i}^\top\mathbf{\Sigma}_{2,i}\mathbf{C}_{4,i}}_{:= \mathbf{E}_1}\bigg\|_2 + \bigg\|\underbrace{\frac{1}{n}\sum_{i=1}^n \mathbf{C}_{1,i}^\top  \mathbf{C}_{2,i}\mathbf{C}_{3,i}^\top(\mathbf{\Sigma}_{2,i}-\mathbf{I}_d)\mathbf{C}_{4,i}}_{:=\mathbf{E}_2}\bigg\|_2 \label{matt} 
    \end{align}
We first consider $\|\mathbf{E}_1\|_2$. For any $i\in [n]$, we have by Theorem 4.6.1 in \citep{vershynin2018high}, 
    \begin{align}
        \left\| \mathbf{C}_{3,i}^\top\mathbf{\Sigma}_{2,i}\mathbf{C}_{4,i} -  \mathbf{C}_{3,i}^\top\mathbf{C}_{4,i}   \right\|_2 \leq c  \max_{i\in[n]}\|\mathbf{C}_{3,i}\|_2\|\mathbf{C}_{4,i}\|_2 \delta_{m_2,d_1+d_2} =  c  c_3c_4 \delta_{m_2,d_1+d_2} 
    \end{align}
    with probability at least $1 - 2 n^{-100}$. Union bounding over all $i\in [n]$ and using the triangle inequality gives 
    \begin{align}\mathbb{P}\left(\mathcal{A}:= \left\{\{\mathbf{\Sigma}_{2,i}\}_{i\in [n]}:\left\| \mathbf{C}_{3,i}^\top\mathbf{\Sigma}_{2,i}\mathbf{C}_{4,i}\right\|_2 \leq cc_3c_4(1+ \delta_{m_2,d_1+d_2}) \; \forall i \in [n] \right\}\right)\geq 1 - 2 n^{-99}. \label{steve}
    \end{align}
    Next, for any fixed set  $\{\mathbf{\Sigma}_{2,i}\}_{i\in [n]}\in \mathcal{A}$, the $d_2$-dimensional  random vectors \\
    $\{\mathbf{x}_{1,i,j}\mathbf{C}_{2,i}\mathbf{C}_{3,i}^\top\mathbf{\Sigma}_{2,i}\mathbf{C}_{4,i}\}_{i\in [n], j\in[m]}$ are sub-gaussian with sub-gaussian norms at most  $c' c_2 c_3c_4(1+\delta_{m_2,d_1+d_2})$. Likewise, the $d_0$-dimensional random vectors $\{\mathbf{C}_{1,i}^\top\mathbf{x}_{i,j'}\}_{i\in [n], j'\in\{2,...,m\}}$ are sub-gaussian with norms at most $c$. Thus using the same argument as in the proof of (1.), we have
\begin{align}
    \mathbb{P}\bigg(&\bigg\|\frac{1}{n}\sum_{i=1}^n \mathbf{C}_{1,i}^\top (\mathbf{\Sigma}_{1,i} - \mathbf{I}_d) \mathbf{C}_{2,i}\mathbf{C}_{3,i}^\top\mathbf{\Sigma}_{2,i}\mathbf{C}_{4,i}\bigg\|_2 \\
    &< c'' c_1 c_2c_3 c_4 (1+\delta_{m_2,d_1+d_2})\bar{\delta}_{m_{1}, d_0+d_2} \big| \{\mathbf{\Sigma}_{2,i}\}_{i\in [n]}, \{\mathbf{\Sigma}_{2,i}\}_{i\in [n]}\in\mathcal{A}\bigg) \nonumber\\
    &\geq 1 - 2e^{-90(d_0+d_2)}. 
\end{align}
for an absolute constant $c''$.
Integrating over all $\{\mathbf{\Sigma}_{2,i}\}_{i\in [n]}\in \mathcal{A}$ and using $\delta_{m_2,d_1+d_2}\leq 1$ yields 
\begin{align}
    \mathbb{P}\left(\bigg\|\frac{1}{n}\sum_{i=1}^n \mathbf{C}_{1,i}^\top (\mathbf{\Sigma}_{1,i} - \mathbf{I}_d) \mathbf{C}_{2,i}\mathbf{C}_{3,i}^\top\mathbf{\Sigma}_{2,i}\mathbf{C}_{4,i}\bigg\|_2< c'' c_1 c_2c_3 c_4 \bar{\delta}_{m_{1}, d_0+d_2} \big| \mathcal{A}\right) 
    &\geq 1 - 2e^{-90(d_0+d_2)}. 
\end{align}
Therefore, by the law of total probability and \eqref{steve}, we have
\begin{align}
    \mathbb{P}\left(\|\mathbf{E}_1\|_2\leq c'' c_1 c_2c_3 c_4 \bar{\delta}_{m_{1}, d_0+d_2}\right) 
    &\leq 2e^{-90(d_0+d_2)} + \mathbb{P}(\mathcal{A}^c) 
    \leq 2e^{-90(d_0+d_2)} + 2 n^{-99}. 
\end{align}
Next, we have from (1.) that $\|\mathbf{E}_2\|_2= \bigg\|\frac{1}{n}\sum_{i=1}^n \mathbf{C}_{1,i}^\top \mathbf{C}_{2,i} \mathbf{C}_{3,i}^\top (\mathbf{\Sigma}_{2,i} - \mathbf{I}_d) \mathbf{C}_{4,i}\bigg\|_2\leq c c_1 c_2 c_3 c_4 \bar{\delta}_{m_2, d_0+d_2}$ with probability at least $1 - 2 e^{-90 (d_0+d_2)}$.
Finally, combining our bounds on the two terms in \eqref{matt} via a union bound yields 
\begin{align}
    \mathbb{P}\left(\|\mathbf{E}\|_2\leq c'' c_1 c_2c_3 c_4 (\bar{\delta}_{m_{1}, d_0+d_2}+(1+\delta_{m_2,d_1+d_2})\bar{\delta}_{m_{2}, d_0+d_2})\right)
\end{align}
as desired. Note that we could instead use \eqref{steve} to bound $\|\mathbf{E}_2\|_2$, which would result in the bound (3.).
    \end{enumerate}
\begin{enumerate}
 \setcounter{enumi}{7}
 \item Let $\mathbf{E}\coloneqq \frac{1}{n}\sum_{i=1}^n \mathbf{C}_{1,i}^\top \mathbf{\Sigma}_{1,i} \mathbf{C}_{2,i}\mathbf{C}_{3,i}^\top\mathbf{\Sigma}_{2,i}\mathbf{C}_{4,i}\mathbf{C}_{5,i}^\top \mathbf{\Sigma}_{2,i} \mathbf{C}_{6,i}  -  \mathbf{C}_{1,i}^\top \mathbf{C}_{2,i}\mathbf{C}_{3,i}^\top \mathbf{C}_{4,i}\mathbf{C}_{5,i}^\top \mathbf{C}_{6,i}$.
 We make a similar argument as in the proof of (2.)
 We have 
 \begin{align}
     \|\mathbf{E}\|_2 &\leq \bigg\|\underbrace{\frac{1}{n}\sum_{i=1}^n \mathbf{C}_{1,i}^\top (\mathbf{\Sigma}_{1,i}-\mathbf{I}_d) \mathbf{C}_{2,i}\mathbf{C}_{3,i}^\top\mathbf{\Sigma}_{2,i}\mathbf{C}_{4,i}\mathbf{C}_{5,i}^\top \mathbf{\Sigma}_{2,i} \mathbf{C}_{6,i} }_{:= \mathbf{E}_1} \bigg\|_2  \nonumber \\
     &\quad + \bigg\|\underbrace{\frac{1}{n}\sum_{i=1}^n \mathbf{C}_{1,i}^\top \mathbf{C}_{2,i}\mathbf{C}_{3,i}^\top(\mathbf{\Sigma}_{2,i}-\mathbf{I}_d)\mathbf{C}_{4,i}\mathbf{C}_{5,i}^\top \mathbf{\Sigma}_{2,i} \mathbf{C}_{6,i}}_{:= \mathbf{E}_2}  \bigg\|_2 \nonumber\\
     &\quad+ \bigg\|\underbrace{\frac{1}{n}\sum_{i=1}^n \mathbf{C}_{1,i}^\top \mathbf{C}_{2,i}\mathbf{C}_{3,i}^\top\mathbf{C}_{4,i}\mathbf{C}_{5,i}^\top (\mathbf{\Sigma}_{2,i}-\mathbf{I}_d) \mathbf{C}_{6,i}}_{:= \mathbf{E}_3}  \bigg\|_2 
 \end{align}
We know from Theorem 4.6.1 in \cite{vershynin2018high} that $\mathbb{P}(\|\mathbf{C}_{3,i}^\top(\mathbf{\Sigma}_{2,i}-\mathbf{I}_d)\mathbf{C}_{4,i}\|_2\leq c  c_3 c_4 {\delta}_{m_2,d_1+d_2})\\ \geq 1 - 2n^{-100}$ and $\mathbb{P}(\|\mathbf{C}_{5,i}^\top\mathbf{\Sigma}_{2,i}\mathbf{C}_{6,i}\|_2\leq c  c_5 c_6 (1+{\delta}_{m_2,d_2+d_3})) \geq 1 - 2n^{-100}$. Union bounding these events over $i\in [n]$ gives $\mathbb{P}(\|\mathbf{E}_2\|_2\leq c c_1 c_2 c_3 c_4 c_5 c_6{\delta}_{m_2,d_1+d_2}(1+ {\delta}_{m_2,d_2+d_3})) \geq 1 - 4e^{-99}$. Union bounding over the same events, we also have $\mathbb{P}(\|\mathbf{E}_3\|_2\leq c c_1 c_2 c_3 c_4 c_5 c_6{\delta}_{m_2,d_2+d_3}) \geq 1 - 2n^{-99}$.  Next, we make a similar argument as in (2.) to control $\|\mathbf{E}_1\|_2$, except that here $\mathcal{A}$ is defined as \begin{align}\mathcal{A}&:= \bigg\{\{\mathbf{\Sigma}_{2,i}\}_{i\in [n]}:\left\| \mathbf{C}_{3,i}^\top\mathbf{\Sigma}_{2,i}\mathbf{C}_{4,i}\right\|_2 \leq cc_3c_4(1+ \delta_{m_2,d_1+d_2}), \nonumber \\
&\quad \quad\quad \quad \; \left\| \mathbf{C}_{5,i}^\top\mathbf{\Sigma}_{2,i}\mathbf{C}_{6,i}\right\|_2 \leq cc_5c_6(1+ \delta_{m_2,d_2+d_3}) \; \forall i \in [n] \bigg\},\end{align} which occurs with probability at least $1-4 n^{-99}$ (which is implied by our discussion of bounding $\|\mathbf{E}_3\|_2$). Thus, following the logic in (2.), we obtain $\mathbb{P}(\|\mathbf{E}_1\|_2 \leq c c_1 c_2 c_3 c_4 c_5 c_6(1+{\delta}_{m_2, d_1 + d_2})(1+ {\delta}_{m_2,d_2 + d_3})\bar{\delta}_{m_1, d_0+d_3}) \geq 1 - 4 n^{-99} - 2e^{-90(d_0+d_3)}$. 
Combining all bounds 
 yields the desired result.
\end{enumerate}
More generally, we add and subtract terms to show concentration through either a $\mathbf{\Sigma} - \mathbf{I}_d$ matrix, or an $\mathbf{X}\mathbf{z}$ matrix, with off terms bounded for each $i$ by sub-gaussianity.
\end{proof}


\begin{lemma} \label{lem:gen2}
Consider the setting described in Lemma \ref{lem:gen1}. Further, suppose $\min(d_1,d_2, d_3) = 1$ and $\max(d_1,d_2, d_3) = k$.  Then  the following events each hold with probability at most $c'(e^{-100} + n^{-99}+  {m_1}^{-99})$ for absolute constants $c,c'$:
\begin{itemize}
    \item $\mathcal{U}_1 \coloneqq \bigg\{ \bigg\|\frac{1}{n}\sum_{i=1}^n  \mathbf{\Sigma}_{1,i} \mathbf{C}_{2,i}\mathbf{C}_{3,i}^\top\mathbf{\Sigma}_{2,i}\mathbf{C}_{4,i}\mathbf{C}_{5,i}^\top \mathbf{\Sigma}_{1,i} \mathbf{C}_{6,i}\!  -\!   \mathbf{C}_{2,i}\mathbf{C}_{3,i}^\top \mathbf{\Sigma}_{2,i}\mathbf{C}_{4,i}\mathbf{C}_{5,i}^\top \mathbf{C}_{6,i} \bigg\|_2 \!\geq\! c c_2c_3c_4 c_5c_6 (\tilde{\delta}\! + \! \tfrac{k}{m_1}) \bigg\}$  
    

    \item $\mathcal{U}_2 \coloneqq \bigg\{\left\|\frac{1}{n}\sum_{i=1}^n \mathbf{\Sigma}_{1,i} \mathbf{C}_{2,i}\mathbf{C}_{3,i}^\top\mathbf{\Sigma}_{2,i}\mathbf{C}_{4,i}\mathbf{C}_{5,i}^\top\mathbf{X}_{1,i}^\top \mathbf{z}_{1,i} \right\|_2 \geq c \sigma  c_2c_3c_4 c_5 \tilde{\delta} \bigg\}$
    
    \item $\mathcal{U}_3 \coloneqq \bigg\{\left\|\frac{1}{n}\sum_{i=1}^n \mathbf{\Sigma}_{1,i} \mathbf{c}_{2,i} \mathbf{z}_{1,i}^\top \mathbf{X}_{1,i} \mathbf{C}_{4,i}\mathbf{C}_{5,i}^\top \mathbf{\Sigma}_{2,i}\mathbf{C}_{6,i}\right\|_2 \geq c \sigma  c_2c_3c_4 c_5 \tilde{\delta} \bigg\}$
    
    \item $\mathcal{U}_4 \coloneqq \bigg\{\left\|\frac{1}{n}\sum_{i=1}^n  \mathbf{X}_{1,i}^\top \mathbf{z}_{1,i}  \mathbf{c}_{2,i}^\top \mathbf{C}_{3,i}^\top \mathbf{\Sigma}_{1,i} \mathbf{C}_{4,i}\mathbf{C}_{5,i}^\top\mathbf{\Sigma}_{2,i}\mathbf{C}_{6,i} \right\|_2 \\\geq c \sigma c_1c_2c_3c_4 c_5 c_6 \tfrac{(\sqrt{kd}+ \sqrt{d\log(nm_1)}+ \log(nm_1))\log(nm_1)}{\sqrt{nm_1}}\bigg\}$ 
    
    \item 
    $\mathcal{U}_5 \coloneqq \bigg\{\big\| {\frac{1}{nm_{1}^2}\sum_{i=1}^n\mathbf{X}_{t,i}^\top\mathbf{z}_{t,i}\mathbf{z}_{t,i}^\top\mathbf{X}_{t,i}  \mathbf{C}_{2,i} \mathbf{C}_{3,i}^\top \mathbf{\Sigma}_{2,i} 
   \mathbf{C}_{4,i} }\big\|_2 \\\geq c\sigma^2 c_2 c_3 c_4 \left(\tfrac{(\sqrt{kd}+ \sqrt{d\log(nm_1)}+ \log(nm_1))\sqrt{\log(nm_1)}}{\sqrt{nm_1}}\!+\! \tfrac{1}{m_1}\right) \bigg\}$
   
\end{itemize}
where \begin{align}
    \tilde{\delta}&\coloneqq  (k\sqrt{d\log(nm_1)} +k \log(nm_1)+ \sqrt{d} \log^{1.5}(nm_1) + \log^{2}(nm_1) )/\sqrt{nm_1}. \nonumber
\end{align}
\end{lemma}
\begin{proof}

\begin{enumerate}
    \item Similarly to previous proofs involving sums of products of independent matrices, the idea is to first use that one set of matrices is small with high probability, then condition on these sets of matrices being small to isolate the randomness of the other matrices. Note that matrix $\mathbf{C}_{3,i}^\top\mathbf{\Sigma}_{2,i}\mathbf{C}_{4,i}$ has maximum dimension at most $k$, so by Lemma \ref{lem:gen1}, for any $i \in [n]$, $\{\|\mathbf{C}_{3,i}^\top\mathbf{\Sigma}_{2,i}\mathbf{C}_{4,i}\|_2 \geq cc_3c_4(1+\delta_{m_1, k})\}$ holds with probability at most $n^{100}$. Applying a union bound over $[n]$ gives that $\mathcal{A}\coloneqq \cap_{i\in[n]}\{\|\mathbf{C}_{3,i}^\top\mathbf{\Sigma}_{2,i}\mathbf{C}_{4,i}\|_2 \leq  cc_3c_4(1+\delta_{m_1, k})\}$ holds with probability at least $1-n^{-99}$. Conditioning on $\mathcal{A}$, and using $\delta_{m_1, k}\leq 1$, we can apply Lemma \ref{lem:fourth} to obtain that 
\begin{align}
    \bigg\|\frac{1}{n}\sum_{i=1}^n  \mathbf{\Sigma}_{1,i} \mathbf{C}_{2,i}\mathbf{C}_{3,i}^\top\mathbf{\Sigma}_{2,i}\mathbf{C}_{4,i}\mathbf{C}_{5,i}^\top \mathbf{\Sigma}_{1,i} \mathbf{C}_{6,i}\!  -\!   \mathbf{C}_{2,i}\mathbf{C}_{3,i}^\top \mathbf{\Sigma}_{2,i}\mathbf{C}_{4,i}\mathbf{C}_{5,i}^\top \mathbf{C}_{6,i} \bigg\|_2 \!\geq\! c c_2c_3c_4 c_5c_6 (\tilde{\delta}\! + \! \tfrac{1+C^2}{m_1})
\end{align}
occurs with probability at most $ $. Since $\mathbb{P}(\mathcal{U}_1) \leq \mathbb{P}(\mathcal{U}_1|\mathcal{A})+ \mathbb{P}(\mathcal{A}^c)$, we obtain the result.

\item We make the same argument as for (1) except that we apply Lemma \ref{lem:fourth1} instead of Lemma \ref{lem:fourth}. 

\item Again, we use Lemma \ref{lem:fourth1} as in (2).

\item Again, we use Lemma \ref{lem:fourth1} as in (2).

\item Here we make the same argument as (1) except that we  apply  Lemma \ref{lem:fourth4} instead of Lemma \ref{lem:fourth}.


\end{enumerate}
\end{proof}







The following is a slightly generalized version of Theorem 1.1 in \citet{magen2011low}: here, the random matrices are not necessarily identically distributed, whereas they are identically distributed in \citet{magen2011low}. However, the proof from \cite{magen2011low} does not rely on the matrices being identically distributed, so the same proof from \citet{magen2011low} holds without modification for the below result.
\begin{thm}[Theorem 1.1 in \cite{magen2011low}] \label{thm:magen}
Let $0< \epsilon< 1$ and $\mathbf{M}_1,\dots, \mathbf{M}_N$ be a sequence of independent symmetric random matrices that satisfy $\|\tfrac{1}{N}\sum_{i=1}^N \mathbb{E}[\mathbf{M}_i]\|_2 \leq 1$ and $\|\mathbf{M}_i\|_2 \leq B$ and $\rank(\mathbf{M}_i)\leq r$ almost surely for all $i\in[N]$. Set $N = \Omega( B \log(B/\epsilon^2)/\epsilon^2)$. If $r \leq N$ almost surely, then
\begin{align}
    \mathbb{P}\left(\left\| \frac{1}{N}\sum_{i=1}^N \mathbf{M}_i - \mathbb{E}[\mathbf{M}_i] \right) \right\|_2\leq \frac{1}{\poly(N)}
\end{align}
\end{thm}

The following lemma again gives generic concentration results but for a more difficult set of matrices. The key technical contribution is a truncated version of Theorem \ref{thm:magen}.
\begin{lemma} \label{lem:fourth}
Suppose that $\mathbf{x}$ is a random vector with $\mathbb{E}[\mathbf{x}]=\mathbf{0}_d$,  and $\Cov(\mathbf{x})=\mathbf{I}_d$, and is  $\mathbf{I}_d$-sub-gaussian. Let $\{ \mathbf{x}_{i,j}\}_{i\in[n], j\in[m]}$ be $nm$ independent copies of $\mathbf{x}$. 
Further, let $\mathbf{C}_{\ell,i}\in \mathbb{R}^{d\times d_{\lfloor \ell/2 \rfloor}}$ for $\ell = 2,3,4$ be fixed matrices for $i \in [n]$, and let $c_\ell \coloneqq \max_{i\in[n]}\|\mathbf{C}_{\ell,i}\|_2$ for $\ell=2,3,4$. 
Denote $\mathbf{\Sigma}_{i} = \frac{1}{m}\sum_{j=1}^{m}\mathbf{x}_{i,j}\mathbf{x}_{i,j}^\top$. Then, if $m\geq \max(1, C^2 c  c_2 c_3 c_4d_1)$, 
\begin{align}
     &\left\|\frac{1}{n}\sum_{i=1}^n  \mathbf{\Sigma}_{i}\mathbf{C}_{2,i}\mathbf{C}_{3,i}^\top \mathbf{\Sigma}_{i} \mathbf{C}_{4,i} - \frac{1}{n}\sum_{i=1}^n  \mathbf{C}_{2,i}\mathbf{C}_{3,i}^\top \mathbf{C}_{4,i}\right\|_2  \nonumber \\
     &\leq c(\sqrt{\log(nm)}+ \sqrt{d_1})(\sqrt{\log(nm)} + \sqrt{d_2})\tfrac{\sqrt{d+d_2} + \sqrt{\log(m)}}{\sqrt{nm}} c_2c_3c_4\nonumber \nonumber \\
     &\quad + \frac{c}{\sqrt{nm}} (\sqrt{\log(nm)}+\sqrt{d})\left(\prod_{\ell=2}^4(\sqrt{\log(nm)}+\sqrt{d_{\lfloor \ell/2 \rfloor}})\right)\max(c_2c_3c_4,1) + c\tfrac{d_1}{m} c_2c_3c_4  \nonumber
\end{align}
for an absolute constant $c$,
with probability at least $1 - 2m^{-99} - \frac{1}{\poly(nm)} - 2e^{-90d_1}$.
\end{lemma}
As in previous cases, in this lemma we would like to show concentration of fourth-order products of sub-gaussian random vectors with only $m =  \poly(k)\tilde{O}(\tfrac{d}{n} +1)$ samples per task. The issue here, unlike in the cases in Lemma \ref{lem:gen1}, is that the leading $\mathbf{\Sigma}_i$  has no dimensionality reduction - there is no product matrix $\mathbf{C}_{1,i}$ to bring the $d$-dimensional random vectors that compose the leftmost $\mathbf{\Sigma}_i$ to a lower dimension. Thus, we would need $m=\Omega(d)$ samples per task to show concentration of each $\mathbf{\Sigma}_i$ (or $\mathbf{\Sigma}_i\mathbf{C}_{2,i}$). We must get around this by averaging over $n$. However, doing so requires dealing with fourth-order products of random vectors instead of bounding each of the two copies of $\mathbf{\Sigma}_i$ in the $i$-th term separately (perhaps along with their dimensionality-reducing products). 

Due to the fourth-order products, we cannot apply standard concentrations based on sub-gaussian and sub-exponential tails. Instead, we leverage the low rank (at most $k$) of the matrices involved by applying a truncated version of the of the concentration result for bounded, low-rank random matrices in \cite{magen2011low}. 

\begin{proof}
Throughout the proof we use $c$ as a generic absolute constant.  First note that by  expanding $\mathbf{\Sigma}_i$ and the triangle inequality,
\begin{align}
    &\left\|\frac{1}{n}\sum_{i=1}^n  \mathbf{\Sigma}_{i}\mathbf{C}_{2,i}\mathbf{C}_{3,i}^\top \mathbf{\Sigma}_{i} \mathbf{C}_{4,i} - \frac{1}{n}\sum_{i=1}^n   \mathbf{C}_{2,i}\mathbf{C}_{3,i}^\top \mathbf{C}_{4,i}\right\|_2\nonumber \\
    &\leq \bigg\|\frac{1}{nm^2}\sum_{i=1}^n \sum_{j,j'\neq j}  \mathbf{x}_{i,j}\mathbf{x}_{i,j}^\top\mathbf{C}_{2,i}\mathbf{C}_{3,i}^\top \mathbf{x}_{i,j'}\mathbf{x}_{i,j'}^\top \mathbf{C}_{4,i} - \frac{m(m-1)}{nm^2}\sum_{i=1}^n  \mathbf{C}_{2,i}\mathbf{C}_{3,i}^\top \mathbf{C}_{4,i}\bigg\|_2 \nonumber \\
    &\quad + \bigg\|{\frac{1}{nm^2}\sum_{i=1}^n \sum_{j} \mathbf{x}_{i,j}\mathbf{x}_{i,j}^\top\mathbf{C}_{2,i}\mathbf{C}_{3,i}^\top \mathbf{x}_{i,j}\mathbf{x}_{i,j}^\top \mathbf{C}_{4,i} - \frac{m}{nm^2}\sum_{i=1}^n  \mathbf{C}_{2,i}\mathbf{C}_{3,i}^\top \mathbf{C}_{4,i}}\bigg\|_2\nonumber \\
    &= \bigg\|\underbrace{\frac{m-1}{nm}\sum_{i=1}^n \frac{1}{m}\sum_{j=1}^m \bigg(\tfrac{1}{m-1}\sum_{j'\neq j} \mathbf{x}_{i,j'}\mathbf{x}_{i,j'}^\top \bigg) \mathbf{C}_{2,i}\mathbf{C}_{3,i}^\top \mathbf{x}_{i,j}\mathbf{x}_{i,j}^\top\mathbf{C}_{4,i} - \frac{(m-1)}{nm}\sum_{i=1}^n   \mathbf{C}_{2,i}\mathbf{C}_{3,i}^\top \mathbf{C}_{4,i}}_{=:\mathbf{E}'}\bigg\|_2 \nonumber \\
    &\quad + \bigg\|\underbrace{\frac{1}{nm^2}\sum_{i=1}^n \sum_{j=1}^m  \mathbf{x}_{i,j}\mathbf{x}_{i,j}^\top\mathbf{C}_{2,i}\mathbf{C}_{3,i}^\top \mathbf{x}_{i,j}\mathbf{x}_{i,j}^\top \mathbf{C}_{4,i} - \frac{1}{nm}\sum_{i=1}^n  \mathbf{C}_{2,i}\mathbf{C}_{3,i}^\top \mathbf{C}_{4,i}}_{=:\mathbf{E}''}\bigg\|_2\nonumber
\end{align}
Note that $\mathbf{E}'$ is unbiased while $\mathbf{E}''$ is biased due to the fourth-order product. 
We first bound $\|\mathbf{E}'\|_2$.

\textbf{Step 1: Bound $\|\mathbf{E}'\|_2$.}
Add and subtract $\frac{(m-1)}{nm}\sum_{i=1}^n   \frac{1}{m}\sum_{j=1}^m \mathbf{C}_{2,i}\mathbf{C}_{3,i}^\top \mathbf{x}_{i,j}\mathbf{x}_{i,j}^\top \mathbf{C}_{4,i} $ to obtain
\begin{align}
    \|\mathbf{E}'\|_2 &\leq \bigg\|\frac{m-1}{nm}\sum_{i=1}^n   \frac{1}{m}\sum_{j=1}^m \bigg(\tfrac{1}{m-1}\sum_{j'\neq j} \mathbf{x}_{i,j'}\mathbf{x}_{i,j'}^\top - \mathbf{I}_d\bigg) \mathbf{C}_{2,i}\mathbf{C}_{3,i}^\top  \mathbf{x}_{i,j}\mathbf{x}_{i,j}^\top\mathbf{C}_{4,i} \bigg\|_2 \nonumber \\
    &\quad +\bigg\|\frac{m-1}{nm}\sum_{i=1}^n   \mathbf{C}_{2,i}\mathbf{C}_{3,i}^\top \bigg(\tfrac{1}{m}\sum_{j=1}^m\mathbf{x}_{i,j}\mathbf{x}_{i,j}^\top - \mathbf{I}_d \bigg) \mathbf{C}_{4,i} \bigg\|_2 \nonumber \\
    &\leq \bigg\|\frac{m-1}{nm}\sum_{i=1}^n  \frac{1}{m}\sum_{j=1}^m \bigg(\tfrac{1}{m-1}\sum_{j'\neq j} \mathbf{x}_{i,j'}\mathbf{x}_{i,j'}^\top - \mathbf{I}_d\bigg) \mathbf{C}_{2,i}\mathbf{C}_{3,i}^\top \mathbf{x}_{i,j}\mathbf{x}_{i,j}^\top \mathbf{C}_{4,i} \bigg\|_2 \nonumber \\
    &\quad +c c_2c_3 c_4 \bar{\delta}_{m,\max(d*, d_2)} \label{hgh}
\end{align}
where \eqref{hgh} follows with probability at least $1 - 2e^{-90(d_1+d_2)}$ by Lemma \ref{lem:gen1}, and $d*$ denotes $d$ if the $\mathbf{C}_{2,i}$'s are distinct, and denotes $d_1$ otherwise (since if these matrices are equal, they can be factored out of the norm, in which case we show concentration of $d_1\times d_2$-dimensional random matrices).
To deal with the first term in \eqref{hgh},
note that as mentioned before, we need to show concentration over $i \in[n]$ to avoid requiring $m = \Omega(d)$. Ideally, we could also concentrate over $j\in [m]$, but we would lose independence of the summands. Thus, we
reorder the sum and use the triangle inequality to write
\begin{align}
 &\bigg\|\frac{m-1}{nm}\sum_{i=1}^n   \frac{1}{m}\sum_{j=1}^m\bigg(\tfrac{1}{m-1}\sum_{j'\neq j} \mathbf{x}_{i,j'}\mathbf{x}_{i,j'}^\top - \mathbf{I}_d\bigg)\mathbf{C}_{2,i}\mathbf{C}_{3,i}^\top \mathbf{x}_{i,j}\mathbf{x}_{i,j}^\top \mathbf{C}_{4,i} \bigg\|_2 \nonumber \\
  &= \bigg\|\frac{1}{m}\sum_{j=1}^m \frac{m-1}{nm}\sum_{i=1}^n  \bigg(\tfrac{1}{m-1}\sum_{j' \neq j} \mathbf{x}_{i,j'}\mathbf{x}_{i,j'}^\top - \mathbf{I}_d\bigg)  \mathbf{C}_{2,i}\mathbf{C}_{3,i}^\top \mathbf{x}_{i,j}\mathbf{x}_{i,j}^\top \mathbf{C}_{4,i} \bigg\|_2 \nonumber \\
    &\leq \frac{1}{m}\sum_{j=1}^m \bigg\|\underbrace{\frac{m-1}{nm}\sum_{i=1}^n  \bigg(\tfrac{1}{m-1}\sum_{j' \neq j} \mathbf{x}_{i,j'}\mathbf{x}_{i,j'}^\top - \mathbf{I}_d\bigg)  \mathbf{C}_{2,i}\mathbf{C}_{3,i}^\top \mathbf{x}_{i,j}\mathbf{x}_{i,j}^\top \mathbf{C}_{4,i}}_{=: \mathbf{E}_j} \bigg\|_2 \nonumber  
\end{align}
For each $j\in [m]$, define $\mathcal{E}_j$ as the event  $\{\|\mathbf{C}_{3,i}^\top\mathbf{x}_{i,j}\|_2 \leq (\gamma+\sqrt{d_1})c_3, \|\mathbf{C}_{4,i}^\top\mathbf{x}_{i,j}\|_2 \leq (\gamma+\sqrt{d_2})c_4 \; \forall i\in [n]\}$ for some  $\gamma>0$. Note that $\mathbf{x}_{i,j}$  and $\mathbf{C}_{2,i}^\top\mathbf{x}_{i,j}$ are $d$ (resp. $d_2$)-dimensional sub-gaussian random vectors with sub-gaussian norm at most $c$ (resp. $c c_2$). Thus $\mathcal{E}_j$ occurs with probability at least  $1 - 2n e^{-c \gamma^2  }$. Then using the law of total probability, for any $\epsilon> 0$, we have
\begin{align}
    \mathbb{P}\left(\|\mathbf{E}_j\|_2\geq \epsilon\right)
    &\leq \mathbb{P}\left(\|\mathbf{E}_j \|_2\geq \epsilon| \mathcal{E}_j\right) + \mathbb{P}\left(\mathcal{E}_j^c \right)  \leq  \mathbb{P}\left(\|\mathbf{E}_j \|_2\geq \epsilon| \mathcal{E}_j\right)  +  2n e^{-c \gamma^2} \label{total-prob}
\end{align}
Consider $\mathbf{E}_1$. For any fixed set $\{\mathbf{x}_{i,1}\}_{i\in [n]}\in \mathcal{E}_1$, 
the $d_2$-dimensional  random vectors \\ $\{\mathbf{x}_{i,j'}\mathbf{C}_{2,i}\mathbf{C}_{3,i}^\top\mathbf{x}_{i,1}\mathbf{x}_{i,1}^\top\mathbf{C}_{4,i}\}_{i\in [n], j'\in\{2,...,m\}}$ are sub-gaussian with norms at most  $c' (\gamma+ \sqrt{d_1})(\gamma + \sqrt{d_2})c_2 c_3c_4$. Likewise, the $d$-dimensional random vectors $\{\mathbf{x}_{i,j'}\}_{i\in [n], j'\in\{2,...,m\}}$ are sub-gaussian with norms at most $c$. Thus using Bernstein's inequality, we can bound 
\begin{align}
    \mathbb{P}\bigg(\|\mathbf{E}_1\|_2&< c'' (\gamma+ \sqrt{d_1})(\gamma + \sqrt{d_2})c_2c_3 c_4 \max\left(\tfrac{\sqrt{d +d_2} + \lambda}{\sqrt{n(m-1)}}, \tfrac{d + d_2 + \lambda^2}{n (m-1)}  \right) \big| \{\mathbf{x}_{i,1}\}_{i\in [n]},\{\mathbf{x}_{i,1}\}_{i\in [n]}\in  \mathcal{E}_1\bigg) \nonumber \\
    &\geq 1 - 2e^{-\lambda^2}. 
\end{align}
for $\lambda> 0$ and an absolute constant $c''$.
Integrating over all $\{\mathbf{x}_{i,1}\}_{i\in [n]}\in \mathcal{E}_1$ yields 
\begin{align}
    \mathbb{P}\left(\|\mathbf{E}_1\|_2< c'' (\gamma+ \sqrt{d_1})(\gamma + \sqrt{d_2})c_2 c_3 c_4 \max\left(\tfrac{\sqrt{d+d_2} + \lambda}{\sqrt{n(m-1)}}, \tfrac{d+d_2 + \lambda^2}{n (m-1)}  \right) \big|  \mathcal{E}_1\right)
    &\geq 1 - 2e^{-\lambda^2}. 
\end{align}
Therefore, using \eqref{total-prob}, we have
\begin{align}
     \mathbb{P}\left(\|\mathbf{E}_1\|_2\geq  c''(\gamma+ \sqrt{d_1})(\gamma + \sqrt{d_2})c_2 c_3 c_4 \max\left(\tfrac{\sqrt{d+d_2} + \lambda}{\sqrt{n(m-1)}}, \tfrac{d+d_2 + \lambda^2}{n (m-1)}  \right) \right) 
     &\leq 2 e^{-\lambda^2} + 2n e^{-c \gamma^2  }.
\end{align}
Repeating the same argument for all $j\in [m]$ and applying a union bound gives
\begin{align}
    \mathbb{P}&\left(\frac{1}{m}\sum_{j=1}^m\|\mathbf{E}_j\|_2\geq c'' (\gamma+ \sqrt{d_1})(\gamma + \sqrt{d_2}) c_2 c_3 c_4\max\left(\tfrac{\sqrt{d+d_2} + \lambda}{\sqrt{n(m-1)}}, \tfrac{d+d_2 + \lambda^2}{n (m-1)}  \right) \right) \nonumber \\
    &\quad \quad \leq  2me^{-\lambda^2} + 2mn e^{-c \gamma^2}. 
\end{align}
Choose $\lambda = 10 \sqrt{\log(m)}$ and $\gamma= 10 \sqrt{\log(mn)}$, and use $ \sqrt{n(m-1)}\geq \sqrt{d+d_2}+10 \sqrt{\log(m)}$ to obtain 
\begin{align} 
    \mathbb{P}&\left(\frac{1}{m}\sum_{j=1}^m\|\mathbf{E}_j\|_2\geq c''' c_2c_3c_4(\sqrt{\log(nm)}+ \sqrt{d_1})(\sqrt{\log(nm)} + \sqrt{d_2})\tfrac{\sqrt{d+d_2} + \sqrt{\log(m)}}{\sqrt{nm}}  \right) \nonumber \\
    &\quad \quad \leq  2m^{-99} + 2(mn)^{-99cd_1 } + 2(mn)^{-99 c d_2} \nonumber  \\
    \implies \mathbb{P}&\bigg(\|\mathbf{E}'\|_2 \geq c'' c_2c_3c_4(\sqrt{\log(nm)}+ \sqrt{d_1})(\sqrt{\log(nm)} + \sqrt{d_2})\tfrac{\sqrt{d+d_2} + \sqrt{\log(m)}}{\sqrt{nm}}  \nonumber \\
    &\quad \quad + c c_2c_3c_4 \bar{\delta}_{m,\max(d*,d_2)}\bigg)  \nonumber \\
    &\quad \quad \leq  2m^{-99} + 2(mn)^{-99cd_1 } + 2(mn)^{-99 c d_2} + 2e^{-90(d_1+d_2) } \label{hghg} \\
    \implies \mathbb{P}&\left(\|\mathbf{E}'\|_2 \geq c''' c_2c_3c_4(\sqrt{\log(nm)}+ \sqrt{d_1})(\sqrt{\log(nm)} + \sqrt{d_2})\tfrac{\sqrt{d+d_2} + \sqrt{\log(m)}}{\sqrt{nm}}\right) \nonumber \\
    &\quad \quad \leq  2m^{-99} + 2(mn)^{-99cd_1 } + 2(mn)^{-99 c d_2} + 2e^{-90(d_1+d_2) } \label{hghgh}
\end{align}
where \eqref{hghg} follows from \eqref{hgh} and \eqref{hghgh} follows by the fact that $\bar{\delta}_{m,\max(d*,d_2)}$ is dominated.

\textbf{Step 2: Bound $\|\mathbf{E}''\|_2$.}
Bounding $\|\mathbf{E}''\|_2$  is challenging because we must deal with fourth-order products in $\mathbf{x}_{i,j}$, which may have heavy tails. However, we can leverage the independence and low-rank of the summands, combined with the sub-gaussian tails of each random vector. Second, we must control the bias in $\mathbf{E}'$, which we achieve by appealing to $C$-L4-L2 hypercontractivity.
First note that by the triangle inequality
\begin{align}
    \left\|\mathbf{E}'' \right\|_2 &\leq \bigg\| \frac{1}{nm^2}\sum_{i=1}^n \sum_{j=1}^m = \mathbf{x}_{i,j}\mathbf{x}_{i,j}^\top\mathbf{C}_{2,i}\mathbf{C}_{3,i}^\top \mathbf{x}_{i,j}\mathbf{x}_{i,j}^\top \mathbf{C}_{4,i}\bigg\|_2  + \bigg\|\frac{1}{nm}\sum_{i=1}^n\mathbf{C}_{2,i}\mathbf{C}_{3,i}^\top \mathbf{C}_{4,i} \bigg\|_2\label{hggh}
\end{align} It remains to control the first norm.   To do so, we employ Theorem \ref{thm:magen} (a.k.a. Theorem 1.1 from \cite{magen2011low}) which characterizes the concentration of low-rank, bounded, symmetric random matrices with small expectation. Thus, in order to apply this theorem, we must truncate and symmetrize the random matrices, and control their expectation. 

Define
$\mathcal{E}_{\ell,i,j} \coloneqq \{\|\mathbf{C}_{\ell,i}\mathbf{x}_{i,j}\|_2 \leq c(\rho +\sqrt{d_{\lfloor \ell/2 \rfloor}}) c_\ell\}$ for some  $\rho>0$ and $\ell=2,3,4$ and all $i,j$,  and $\mathcal{E}_{1,i,j} \coloneqq \{\|\mathbf{x}_{i,j}\|_2 \leq c(\rho +\sqrt{d}) \}$ for some  $\rho>0$ and $\ell=2,3,4$ and all $i,j$. Let $\chi_{\mathcal{E}_{\ell,i,j}}$ be the indicator random variable for the event $\mathcal{E}_{\ell,i,j}$. 
Define the truncated random variables
   $ \mathbf{\bar{x}}_{\ell,i,j}\coloneqq \chi_{\mathcal{E}_{\ell,i,j}}\mathbf{C}_{\ell,i}\mathbf{{x}}_{i,j}$ for $\ell=2,3,4$ and all $i,j$ and  $ \mathbf{\bar{x}}_{1,i,j}\coloneqq \chi_{\mathcal{E}_{\ell,i,j}}\mathbf{{x}}_{i,j}$ for  all $i,j$.  Let $\mathbf{S}_{i,j}\coloneqq  \mathbf{x}_{i,j}\mathbf{x}_{i,j}^\top\mathbf{C}_{2,i}\mathbf{C}_{3,i}^\top \mathbf{x}_{i,j}\mathbf{x}_{i,j}^\top \mathbf{C}_{4,i}/m$ and $\mathbf{\bar{S}}_{i,j}\coloneqq  \mathbf{\bar{x}}_{1,i,j}\mathbf{\bar{x}}_{2,i,j}^\top \mathbf{\bar{x}}_{3,i,j}\mathbf{\bar{x}}_{4,i,j}^\top/m$ for each $i,j$. Note that due to sub-gaussianity and earlier arguments,  $\mathbb{P}(\cup_{i,j}\cup_{\ell=1}^4\mathcal{E}_{\ell,i,j} ) \leq 2mn\sum_{\ell=1}^4 e^{-c \rho^2}  = 8 mne^{-c \rho^2}$.
   Thus, for any $\epsilon>0$,
   \begin{align}
       \mathbb{P}\bigg(\bigg\|\frac{1}{nm}\sum_{i=1}^n \sum_{j=1}^m\mathbf{S}_{i,j}  \bigg\|_2\leq \epsilon\bigg) &\leq\mathbb{P}\bigg(\bigg\|\frac{1}{nm}\sum_{i=1}^n \sum_{j=1}^m\mathbf{\bar{S}}_{i,j}  \bigg\|_2\leq \epsilon\bigg) + 8nm e^{-c \rho^2 } \label{114}
   \end{align}
First,  
  form the lifted, symmetric matrices
   \begin{align}
       \mathbf{\Tilde{\bar{S}}}_{i,j} &\coloneqq \begin{bmatrix}
       \mathbf{0} & \mathbf{{\bar{S}}}_{i,j}  \\
       \mathbf{{\bar{S}}}_{i,j}^\top  &\mathbf{0}
       \end{bmatrix}
   \end{align}
   for all $i,j$, and note that $ \left\|\sum_{i=1}^n \sum_{j=1}^m\mathbf{\Tilde{\bar{S}}}_{i,j}  \right\|_2 = 2\left\|\sum_{i=1}^n \sum_{j=1}^m\mathbf{\bar{S}}_{i,j}  \right\|_2$. Also note that by definition, $\|\mathbf{\Tilde{\bar{S}}}_{i,j} \|_2\leq B\coloneqq 2 (\rho+ \sqrt{d})\prod_{\ell=2}^4(\rho+ \sqrt{d_{\lfloor \ell/2 \rfloor}})\max(c_2 c_3 c_4, 1)$ for all $i,j$ almost surely, and the $\mathbf{\Tilde{\bar{S}}}_{i,j}$'s are independent. 
   
We still must control $\|\mathbb{E}[\mathbf{\Tilde{\bar{S}}}_{i,j} ]\|_2$. We have that $\|\mathbb{E}[\mathbf{\Tilde{\bar{S}}}_{i,j} ]\|_2 = 2\|\mathbb{E}[\mathbf{{\bar{S}}}_{i,j} ]\|_2$. 
   Using Lemma \ref{lem:L4L2} (with $\mathbf{C}_1 = \mathbf{I}_d$), we obtain $m \|\mathbb{E}[\mathbf{{\bar{S}}}_{i,j} ]\|_2\leq m C^2 \|\mathbf{C}_{2,i}\|_2\|\mathbf{C}_{3,i}\|_2\|\mathbf{C}_{4,i}\|_2\leq m C^2 c_2 c_3 c_4 d_1$ for all $i\in [n], j \in [m]$.
      Thus, $\|\mathbb{E}[\mathbf{\Tilde{\bar{S}}}_{i,j} ]\|_2\leq 1$ for all $i,j$ as $m \geq 2   C^2 c_2 c_3 c_4 d_1$.

   Next, note that each $\mathbf{\bar{S}}_{i,j}$ is rank at most $\min(d, d_1, d_2)$, so $\mathbf{\Tilde{\bar{S}}}_{i,j}$ is rank at most $2\min(d, d_1, d_2)$. 
      Now we are ready to apply Theorem \ref{thm:magen}. Doing so,
we obtain:
   \begin{align}
       \mathbb{P}\bigg( \bigg\|\frac{1}{nm}\sum_{i=1}^n \sum_{j=1}^m\mathbf{\Tilde{\bar{S}}}_{i,j}  - \frac{1}{nm}\sum_{i=1}^n \sum_{j=1}^m\mathbb{E}[\mathbf{\Tilde{\bar{S}}}_{i,j}]\bigg\|_2\geq \epsilon\bigg) &\leq \frac{1}{\poly(nm)}
   \end{align}
   as long as $nm \geq cB \log(B/\epsilon^2)/\epsilon^2$ and $nm \geq c\min(d, d_1, d_2)$. Setting $\epsilon = \tfrac{cB}{\sqrt{nm}}$ yields
    \begin{align}
       \mathbb{P}\bigg( \bigg\|\frac{1}{nm}\sum_{i=1}^n \sum_{j=1}^m\mathbf{\Tilde{\bar{S}}}_{i,j}  - \frac{1}{nm}\sum_{i=1}^n \sum_{j=1}^m\mathbb{E}[\mathbf{\Tilde{\bar{S}}}_{i,j}]\bigg\|_2\geq \frac{B'}{\sqrt{nm}}\bigg) &\leq \frac{1}{\poly(nm)}
   \end{align}
as long as $nm \leq Be^{c'B}$, which always holds since we will soon choose $\rho = \sqrt{\log(nm)}$ and we have chosen $B$ appropriately.
Therefore, with probability at least $\frac{1}{\poly(nm)}$, we have
\begin{align}
 \frac{1}{2} \left\|\frac{1}{nm}\sum_{i=1}^n \sum_{j=1}^m\mathbf{\Tilde{\bar{S}}}_{i,j} \right\|_2 &\leq \frac{1}{2}\left\|\frac{1}{nm}\sum_{i=1}^n \sum_{j=1}^m\mathbf{\Tilde{\bar{S}}}_{i,j}  - \frac{1}{nm}\sum_{i=1}^n \sum_{j=1}^m\mathbb{E}[\mathbf{\Tilde{\bar{S}}}_{i,j}]\right\|_2 + \frac{1}{2}\left\|\frac{1}{nm}\sum_{i=1}^n \sum_{j=1}^m\mathbb{E}[\mathbf{\Tilde{\bar{S}}}_{i,j}]\right\|_2 \nonumber \\
 &\leq \frac{{B}}{2\sqrt{nm}} + \tfrac{C^2d_1}{nm}\sum_{i=1}^n\|\mathbf{C}_{2,i}\|_2\|\mathbf{C}_{3,i}\|_2 \|\mathbf{C}_{4,i}\|_2  \nonumber 
 \end{align}
 which implies that
\begin{align}
\left\|\mathbf{E}'' \right\|_2 &\leq  \frac{c}{\sqrt{nm}} (\rho+\sqrt{d})\left(\prod_{\ell=2}^4(\rho+\sqrt{d_{\lfloor \ell/2 \rfloor}})\right)\max(c_2c_3c_4,1) + \frac{(1+C^2)d_1}{m} c_2c_3c_4
\end{align}
with probability at least $1 - \frac{1}{\poly(nm)} - 8 nm e^{-c \rho^2}$ by \eqref{hggh} and \eqref{114}. Choose $\rho = 10 \sqrt{\log(nm)}$ and recall that $C$ is an absolute constant to obtain
\begin{align}
    \left\|\mathbf{E}'' \right\|_2 &\leq  \frac{c}{\sqrt{nm}}(\sqrt{\log(nm)}+\sqrt{d})\left(\prod_{\ell=2}^4(\sqrt{\log(nm)}+\sqrt{d_{\lfloor \ell/2 \rfloor}})\right)\max(c_2c_3c_4,1) + c\tfrac{d_1 }{m}c_2c_3c_4
\end{align}
with probability at least $1 - \frac{1}{\poly(nm)}$. Combining Steps 1 and 2,  we have
\begin{align}
     &\left\|\frac{1}{n}\sum_{i=1}^n   \mathbf{\Sigma}_{i}\mathbf{C}_{2,i}\mathbf{C}_{3,i}^\top \mathbf{\Sigma}_{i} \mathbf{C}_{4,i} - \frac{1}{n}\sum_{i=1}^n   \mathbf{C}_{2,i}\mathbf{C}_{3,i}^\top \mathbf{C}_{4,i}\right\|_2 \nonumber \\
     &\leq   c(\sqrt{\log(nm)}+ \sqrt{d_1})(\sqrt{\log(nm)} + \sqrt{d_2})\tfrac{\sqrt{d+d_2} + \sqrt{\log(m)}}{\sqrt{nm}} c_2c_3c_4 \nonumber \nonumber \\
     &\quad + \frac{c}{\sqrt{nm}} (\sqrt{\log(nm)}+\sqrt{d})\left(\prod_{\ell=2}^4(\sqrt{\log(nm)}+\sqrt{d_{\lfloor \ell/2 \rfloor}})\right)\max(c_2c_3c_4,1) + \tfrac{d_1 }{m} c_2c_3c_4 \nonumber 
\end{align}
for an absolute constant $c$
with probability at least $1 - 2m^{-99} - \frac{1}{\poly(nm)} - 2e^{-90d_1}$.
\end{proof}

\begin{lemma} \label{lem:fourth1}
Suppose that $\mathbf{x}$ is a random vector with mean-zero,  $\mathbf{I}_d$-sub-gaussian distribution over $\mathbb{R}^d$. Let $\{ \mathbf{x}_{i,j}\}_{i\in[n], j\in[m]}$ be $nm$ independent copies of $\mathbf{x}$. Denote $\mathbf{\Sigma}_{i} = \frac{1}{m}\sum_{j=1}^{m}\mathbf{x}_{i,j}\mathbf{x}_{i,j}^\top$ and  $\mathbf{X}_i = [\mathbf{x}_{i,1}, \dots, \mathbf{x}_{i,m}]^\top$ for all $i \in [n]$. Let $\mathbf{z} =[z_1,\dots,z_m] \in \mathbb{R}^m$ be a vector whose elements are i.i.d. draws from $\mathcal{N}(0,\sigma^2)$, and let $\{\mathbf{z}_i\}_{i\in[n]}$ be $n$ independent copies of $\mathbf{z}$. 
Further, let $\mathbf{C}_{\ell,i}\in \mathbb{R}^{d\times d_{\lfloor \ell/2 \rfloor}}$ for $\ell = 2,3,5$ be fixed matrices for $i \in [n]$, and let $\mathbf{c}_{4,i}\in \mathbb{R}^{d}$. Also define $c_\ell \coloneqq \max_{i\in[n]}\|\mathbf{C}_{\ell,i}\|_2$ for $\ell=2,3,5$, $c_4 \coloneqq \max_{i\in[n]}\|\mathbf{c}_{4,i}\|_2$. 
 Then, 
\begin{align}
    (i)\quad \bigg\|\frac{1}{n}\sum_{i=1}^n  \mathbf{\Sigma}_{i}\mathbf{C}_{2,i}\mathbf{C}_{3,i}^\top \mathbf{X}_{i}\mathbf{z}_i\bigg\|_2  
     &\leq  \tfrac{c\sigma c_2 c_3(\sqrt{d}+\sqrt{\log(nm)})(d_1+\log(nm))}{\sqrt{nm}}   \nonumber \\
    (ii)\quad  \bigg\|\frac{1}{n}\sum_{i=1}^n   \mathbf{X}_{i}^\top \mathbf{z}_i \mathbf{c}_{4,i}^\top \mathbf{\Sigma}_{i}\mathbf{C}_{5 ,i}^\top \bigg\|_2  
     &\leq \tfrac{c \sigma c_4 c_5 (\sqrt{d}+\sqrt{\log(mn)})(\sqrt{d_1}+\sqrt{\log(nm)}){\log(nm)}}{\sqrt{nm}} \nonumber 
\end{align}
for an absolute constant $c$, each
with probability at least $1 - 2m^{-99} - \frac{1}{\poly(nm)}$.
\end{lemma}
\begin{proof}
We only show the proof for $(i)$ as the proof for $(ii)$ follows by similar arguments.
We argue similarly to the proof of Lemma \ref{lem:fourth}.
We have
\begin{align}
    \left\|\frac{1}{n}\sum_{i=1}^n  \mathbf{\Sigma}_{i}\mathbf{C}_{2,i}\mathbf{C}_{3,i}^\top \mathbf{X}_{i} \mathbf{z}_{i} \right\|_2
    &\leq \bigg\|\underbrace{\frac{m-1}{nm}\sum_{i=1}^n \frac{1}{m}\sum_{j=1}^m \bigg(\tfrac{1}{m-1}\sum_{j'\neq j} \mathbf{x}_{i,j'}\mathbf{x}_{i,j'}^\top \bigg) \mathbf{C}_{2,i}\mathbf{C}_{3,i}^\top \mathbf{x}_{i,j} {z}_{i,j} }_{=:\mathbf{e}'}\bigg\|_2 \nonumber \\
    &\quad + \bigg\|\underbrace{\frac{1}{nm^2}\sum_{i=1}^n \sum_{j=1}^m  \mathbf{x}_{i,j}\mathbf{x}_{i,j}^\top\mathbf{C}_{2,i}\mathbf{C}_{3,i}^\top \mathbf{x}_{i,j}{z}_{i,j}}_{=:\mathbf{e}''}\bigg\|_2\nonumber
\end{align}
\textbf{Step 1: $\|\mathbf{e}'\|_2$.} Add and subtract $\frac{m-1}{nm}\sum_{i=1}^n \frac{1}{m}\sum_{j=1}^m \mathbf{C}_{2,i}\mathbf{C}_{3,i}^\top \mathbf{x}_{i,j} {z}_{i,j}$ to obtain
\begin{align}
    \|\mathbf{e}'\|_2 &\leq \bigg\|{\frac{m-1}{nm}\sum_{i=1}^n \frac{1}{m}\sum_{j=1}^m \mathbf{C}_{2,i}\mathbf{C}_{3,i}^\top \mathbf{x}_{i,j} {z}_{i,j} }\bigg\|_2 \nonumber \\
    &\quad + \bigg\|{\frac{m-1}{nm}\sum_{i=1}^n \frac{1}{m}\sum_{j=1}^m \bigg(\tfrac{1}{m-1}\sum_{j'\neq j} \mathbf{x}_{i,j'}\mathbf{x}_{i,j'}^\top  - \mathbf{I}_d \bigg) \mathbf{C}_{2,i}\mathbf{C}_{3,i}^\top \mathbf{x}_{i,j} {z}_{i,j} }\bigg\|_2  \nonumber \\
    &\leq c\sigma c_2 c_3 \bar{\delta}_{m, d_1} + \bigg\|{\frac{m-1}{nm}\sum_{i=1}^n \frac{1}{m}\sum_{j=1}^m \bigg(\tfrac{1}{m-1}\sum_{j'\neq j} \mathbf{x}_{i,j'}\mathbf{x}_{i,j'}^\top  - \mathbf{I}_d \bigg) \mathbf{C}_{2,i}\mathbf{C}_{3,i}^\top \mathbf{x}_{i,j} {z}_{i,j} }\bigg\|_2 
\end{align}
where the second inequality follows with probability at least $1-e^{-90d_1}$ by Lemma \ref{lem:gen1}. Next,  
\begin{align}
    &\bigg\|{\frac{m-1}{nm}\sum_{i=1}^n \frac{1}{m}\sum_{j=1}^m \bigg(\tfrac{1}{m-1}\sum_{j'\neq j} \mathbf{x}_{i,j'}\mathbf{x}_{i,j'}^\top  - \mathbf{I}_d \bigg) \mathbf{C}_{2,i}\mathbf{C}_{3,i}^\top \mathbf{x}_{i,j} {z}_{i,j} }\bigg\|_2  \nonumber \\
    &\leq \frac{1}{m}\sum_{j=1}^m \bigg\| \frac{m-1}{nm}\sum_{i=1}^n \bigg(\tfrac{1}{m-1}\sum_{j'\neq j} \mathbf{x}_{i,j'}\mathbf{x}_{i,j'}^\top  - \mathbf{I}_d \bigg) \mathbf{C}_{2,i}\mathbf{C}_{3,i}^\top \mathbf{x}_{i,j} {z}_{i,j}  \bigg\|_2 \nonumber
\end{align}
By sub-gaussianity, we have that with probability at least $ 1- 4 (nm)^{-99}$, $\|\mathbf{C}_{3,i}\mathbf{x}_{i,j}\|_2 \leq c c_3(\sqrt{d_1}+\sqrt{\log(nm)})$ and  $\|{z}_{i,j}\|_2 \leq c \sigma\sqrt{\log(nm)}$ for all $i\in[n],j\in [m]$. Thus, as in previous arguments, we have
\begin{align}
     \bigg\| \frac{m-1}{nm}\sum_{i=1}^n \bigg(\tfrac{1}{m-1}\sum_{j'\neq j} \mathbf{x}_{i,j'}\mathbf{x}_{i,j'}^\top  - \mathbf{I}_d \bigg) \mathbf{C}_{2,i}\mathbf{C}_{3,i}^\top \mathbf{x}_{i,j} {z}_{i,j}  \bigg\|_2 
    &\leq \tfrac{c \sigma c_2 c_3 (\sqrt{d}+\sqrt{\log(m)})(\sqrt{d_1}+\sqrt{\log(nm)})\sqrt{\log(nm)}}{\sqrt{nm}} \nonumber
\end{align}
for all $j\in [m]$ with probability at least $1 - 2m^{-99} - 4 (nm)^{-99}$, resulting in
\begin{align}
    \|\mathbf{e}'\|_2 &\leq  c\sigma c_2 c_3 \bar{\delta}_{m, d_1} + \frac{c \sigma c_2 c_3 (\sqrt{d}+\sqrt{\log(m)})(\sqrt{d_1}+\sqrt{\log(nm)})\sqrt{\log(nm)}}{\sqrt{nm}} \nonumber \\
    &\leq \frac{c' \sigma c_2 c_3 (\sqrt{d}+\sqrt{\log(m)})(\sqrt{d_1}+\sqrt{\log(nm)})\sqrt{\log(nm)}}{\sqrt{nm}} \label{comnnin}
\end{align}
with probability at least $1 - 2m^{-99} - 4 (nm)^{-99}$.

\textbf{Step 2: $\|\mathbf{e}''\|_2$.} For $\mathbf{e}''$, we again use Theorem \ref{thm:magen}. Define $\mathcal{E}_{\ell,i,j}$ and $\mathbf{\bar{x}}_{\ell,i,j}$ as in Lemma \ref{lem:fourth} for $\ell = 1,2,3$ and $i \in [n]$ and $j \in [m]$. Define $\mathcal{E}_{4,i,j}= \{|z_{i,j}|\leq c\sigma \sqrt{\log(nm)}\}$ and  $\bar{z}_{i,j} = \chi_{\mathcal{E}_{4,i,j}}{z}_{i,j}$ for all $i \in [n]$ and $j \in [m]$. Define $\mathbf{s}_{i,j}=  \mathbf{x}_{i,j}\mathbf{x}_{i,j}^\top\mathbf{C}_{2,i}\mathbf{C}_{3,i}^\top \mathbf{x}_{i,j}{z}_{i,j}/m$
and $\mathbf{\bar{s}}_{i,j}=  \mathbf{\bar{x}}_{1,i,j}\mathbf{\bar{x}}_{2,i,j}^\top \mathbf{\bar{x}}_{3,i,j}\bar{z}_{i,j}/m$, then we have $\mathbf{s}_{i,j}=\mathbf{\bar{s}}_{i,j}$ for all $i,j$ with probability at least $ 1- \frac{1}{\poly(nm)}$. Also, 
$\|\mathbf{\bar{s}}_{i,j}\|\leq B \coloneqq c\sigma c_2 c_3(\sqrt{d}+\sqrt{\log(nm)})(d_1+\log(nm))\sqrt{\log(nm)}$. Next, by the symmetry of the Gaussian distribution, $\mathbb{E}[\bar{z}_{i,j}] = 0$, thus $\|\mathbb{E}[\mathbf{\bar{s}}_{i,j}]\|_2 = 0$ by independence. Defining $\mathbf{\Tilde{\bar{s}}}_{i,j}$ as in Lemma \ref{lem:fourth}, we can now apply Theorem \ref{thm:magen} as in Lemma \ref{lem:fourth} to obtain:
\begin{align}
    \mathbb{P}\bigg( \bigg\|\frac{1}{nm}\sum_{i=1}^n \sum_{j=1}^m\mathbf{\Tilde{\bar{s}}}_{i,j} \bigg\|_2\geq \tfrac{c\sigma c_2 c_3(\sqrt{d}+\sqrt{\log(nm)})(d_1+\log(nm))\sqrt{\log(nm)}}{\sqrt{nm}}\bigg) &\leq \tfrac{1}{\poly(nm)}
\end{align}
which, recalling $\|\mathbf{e}''\|_2 = \bigg\|\frac{1}{nm}\sum_{i=1}^n \sum_{j=1}^m\mathbf{{{s}}}_{i,j} \bigg\|_2$, implies
\begin{align}
    \mathbb{P}\big(\|\mathbf{e}''\|_2 \geq \tfrac{c\sigma c_2 c_3(\sqrt{d}+\sqrt{\log(nm)})(d_1+\log(nm))\sqrt{\log(nm)}}{\sqrt{nm}}\big) \leq \tfrac{1}{\poly(nm)} \label{comnin}
\end{align}
Combining  \eqref{comnnin} and \eqref{comnin} completes the proof.
\end{proof}

\begin{lemma} \label{lem:fourth4}
Suppose that $\mathbf{x}$ is a random vector with mean-zero,  $\mathbf{I}_d$-sub-gaussian distribution over $\mathbb{R}^d$. Let $\{ \mathbf{x}_{i,j}\}_{i\in[n], j\in[m]}$ be $nm$ independent copies of $\mathbf{x}$. Denote $\mathbf{\Sigma}_{i} = \frac{1}{m}\sum_{j=1}^{m}\mathbf{x}_{i,j}\mathbf{x}_{i,j}^\top$ and  $\mathbf{X}_i = [\mathbf{x}_{i,1}, \dots, \mathbf{x}_{i,m}]^\top$ for all $i \in [n]$. Let $\mathbf{z} =[z_1,\dots,z_m] \in \mathbb{R}^m$ be a vector whose elements are i.i.d. draws from $\mathcal{N}(0,\sigma^2)$, and let $\{\mathbf{z}_i\}_{i\in[n]}$ be $n$ independent copies of $\mathbf{z}$. 
Further, let $\mathbf{C}_{i}\in \mathbb{R}^{d\times d_1}$ be fixed matrices for $i \in [n]$, and let $\bar{c} \coloneqq \max_{i\in[n]}\|\mathbf{C}_{i}\|_2$. 
 Then, 
\begin{align}
     \bigg\|\frac{1}{n}\sum_{i=1}^n   \mathbf{X}_{i}^\top \mathbf{z}_i \mathbf{z}_i^\top \mathbf{X}_{i} \mathbf{C}_{i} \bigg\|_2  
     &\leq \tfrac{c\sigma^2 \bar{c}(\sqrt{d}+\sqrt{\log(nm)})(\sqrt{d_1}+\sqrt{\log(nm)})\sqrt{\log(nm)}}{\sqrt{nm}}+ \tfrac{\sigma^2 \bar{c}}{m}  \nonumber
\end{align}
for an absolute constant $c$
with probability at least $1 -2 m^{-99} - \frac{1}{\poly(nm)}$.
\end{lemma}
\begin{proof}
We have
\begin{align}
    \left\|\frac{1}{n}\sum_{i=1}^n \mathbf{X}_{i}^\top \mathbf{z}_i \mathbf{z}_i^\top \mathbf{X}_{i} \mathbf{C}_{i}  \right\|_2
    &\leq \bigg\|\underbrace{\frac{m-1}{nm}\sum_{i=1}^n \frac{1}{m}\sum_{j=1}^m \bigg(\tfrac{1}{m-1}\sum_{j'\neq j} \mathbf{x}_{i,j}z_{i,j}\bigg) z_{i,j}\mathbf{x}_{i,j}^\top  \mathbf{C}_{i}}_{=:\mathbf{E}'}\bigg\|_2 \nonumber \\
    &\quad + \bigg\|\underbrace{\frac{1}{nm^2}\sum_{i=1}^n \sum_{j=1}^m  z_{i,j}^2\mathbf{x}_{i,j}\mathbf{x}_{i,j}^\top\mathbf{C}_{i}}_{=:\mathbf{E}''}\bigg\|_2\nonumber
\end{align}
\textbf{Step 1: $\|\mathbf{E}'\|_2$.} 
Note that
\begin{align}
    \bigg\|{\frac{m-1}{nm}\sum_{i=1}^n \frac{1}{m}\sum_{j=1}^m  \mathbf{x}_{i,j}z_{i,j}\bigg(\tfrac{1}{m-1}\sum_{j'\neq j} z_{i,j'}\mathbf{x}_{i,j'}^\top \bigg) \mathbf{C}_{i}}\bigg\|_2 &\leq \frac{1}{m}\sum_{j=1}^m  \bigg\|{\frac{m-1}{nm}\sum_{i=1}^n \mathbf{x}_{i,j}z_{i,j}\bigg(\tfrac{1}{m-1}\sum_{j'\neq j} z_{i,j'}\mathbf{x}_{i,j'}^\top \mathbf{C}_{i}\bigg) }\bigg\|_2 \nonumber 
\end{align}
Next, with probability at least $ 1- 4 (nm)^{-99}$, $\|\mathbf{C}_i\mathbf{x}_{i,j}\|_2 \leq c (\sqrt{d_1}+\sqrt{\log(nm)})$ and  $\|{z}_{i,j}\|_2 \leq c \sigma\sqrt{\log(nm)}$ for all $i\in[n],j\in [m]$. Thus, by conditioning on this event as in previous arguments, we can show
\begin{align}
     \bigg\|{\frac{m-1}{nm}\sum_{i=1}^n \mathbf{x}_{i,j}z_{i,j}\bigg(\tfrac{1}{m-1}\sum_{j'\neq j} z_{i,j'}\mathbf{x}_{i,j'}^\top \mathbf{C}_{i}\bigg) }\bigg\|_2  
    &\leq \frac{c \sigma^2 \bar{c} (\sqrt{d}+\sqrt{\log(m)})(\sqrt{d_1}+\sqrt{\log(nm)})\sqrt{\log(nm)}}{\sqrt{nm}} \nonumber
\end{align}
for all $j\in [m]$ with probability at least $1 - 2m^{-99} - 4 (nm)^{-99}$, resulting in
\begin{align}
    \|\mathbf{E}'\|_2 &\leq   \frac{c \sigma^2 \bar{c} (\sqrt{d}+\sqrt{\log(m)})(\sqrt{d_1}+\sqrt{\log(nm)})\sqrt{\log(nm)}}{\sqrt{nm}} \label{comnninn}
\end{align}
with probability at least $1 - 2 m^{-99} - 4 (nm)^{-99}$.

\textbf{Step 2: $\|\mathbf{E}''\|_2$.}
Define $\mathcal{E}_{1,i,j}= \{\|\mathbf{x}_{i,j}\|\leq c(\sqrt{d}+ \sqrt{\log(nm)})\}$, $\mathcal{E}_{2,i,j}= \{\|\mathbf{C}_{i}\mathbf{x}_{i,j}\|_2 \leq c(\sqrt{d_1}+ \sqrt{\log(nm)})\}$ and $\mathcal{E}_{3,i,j}= \{|{z}_{i,j}|\leq c\sigma\sqrt{\log(nm)}\}$ for all $i\in[n], j\in[m]$.
Define $\mathbf{\bar{x}}_{1,i,j} = \chi_{\mathcal{E}_{1,i,j}}\mathbf{x}_{i,j}$, $\mathbf{\bar{x}}_{2,i,j} = \chi_{\mathcal{E}_{2,i,j}}\mathbf{C}_{i}^\top \mathbf{x}_{i,j}$, and  $\bar{z}_{i,j} = \chi_{\mathcal{E}_{3,i,j}}{z}_{i,j}$ for all $i \in [n]$ and $j \in [m]$. Define $\mathbf{S}_{i,j}= {z}_{i,j}^2 \mathbf{x}_{1,i,j}\mathbf{x}_{2,i,j}^\top\mathbf{C}_i/m$
and $\mathbf{\bar{S}}_{i,j}=  \bar{z}_{i,j}^2 \mathbf{\bar{x}}_{1,i,j}\mathbf{\bar{x}}_{2,i,j}^\top/m$, then we have $\mathbf{S}_{i,j}=\mathbf{\bar{S}}_{i,j}$ for all $i,j$ with probability at least $ 1- \frac{1}{\poly(nm)}$. Also, 
$\|\mathbf{\bar{S}}_{i,j}\|\leq B \coloneqq c\sigma^2 \bar{c}(\sqrt{d}+\sqrt{\log(nm)})(\sqrt{d_1}+\sqrt{\log(nm)})\sqrt{\log(nm)}$.  
Note that by the law of total expectation,
\begin{align}
    \|\mathbb{E}[\mathbf{{\bar{S}}}_{i,j}]\|_2 &= \|\mathbb{E}[\mathbf{{{S}}}_{i,j}\big | \mathcal{E}_{1,i,j}\cap \mathcal{E}_{2,i,j}\cap \mathcal{E}_{3,i,j}]\|_2\mathbb{P}(\mathcal{E}_{1,i,j}, \mathcal{E}_{2,i,j}, \mathcal{E}_{3,i,j}) \nonumber \\
    &\leq \|\mathbb{E}[\mathbf{{{S}}}_{i,j}\big | \mathcal{E}_{1,i,j}\cap \mathcal{E}_{2,i,j}\cap \mathcal{E}_{3,i,j}]\|_2\mathbb{P}(\mathcal{E}_{1,i,j}\cap \mathcal{E}_{2,i,j}\cap \mathcal{E}_{3,i,j}) \nonumber \\
    &\quad + \|\mathbb{E}[\mathbf{{{S}}}_{i,j}\big | \mathcal{E}_{1,i,j}^c\cup \mathcal{E}_{2,i,j}^c\cup \mathcal{E}_{3,i,j}^c]\|_2\mathbb{P}(\mathcal{E}_{1,i,j}^c\cup \mathcal{E}_{2,i,j}^c\cup \mathcal{E}_{3,i,j}) \nonumber \\
    &= \|\mathbb{E}[\mathbf{{{S}}}_{i,j}]\|_2 \nonumber \\
    &= \tfrac{\sigma^2}{m} \|\mathbf{C}_i\|_2 \nonumber 
\end{align}
Now,
defining $\mathbf{\Tilde{\bar{S}}}_{i,j}$ as in Lemma \ref{lem:fourth}, we can now apply Theorem \ref{thm:magen} as in Lemma \ref{lem:fourth} to obtain for $m \geq \sigma^2 \bar{c}$:
\begin{align}
    \mathbb{P}\bigg( \bigg\|\frac{1}{nm}\sum_{i=1}^n \sum_{j=1}^m\mathbf{\Tilde{\bar{S}}}_{i,j} - \mathbb{E}[\mathbf{\Tilde{\bar{S}}}_{i,j}] \bigg\|_2\geq \tfrac{c\sigma^2 \bar{c}(\sqrt{d}+\sqrt{\log(nm)})(\sqrt{d_1}+\sqrt{\log(nm)})\sqrt{\log(nm)}}{\sqrt{nm}}\bigg) &\leq \tfrac{1}{\poly(nm)}
\end{align}
Now, note that
\begin{align}
 \bigg\|\frac{1}{nm}\sum_{i=1}^n \sum_{j=1}^m\mathbf{{\bar{S}}}_{i,j} \bigg\|_2  &\leq \bigg\|\frac{1}{nm}\sum_{i=1}^n \sum_{j=1}^m\mathbf{{\bar{S}}}_{i,j} - \mathbb{E}[\mathbf{{\bar{S}}}_{i,j}] \bigg\|_2 + \bigg\|\frac{1}{nm}\sum_{i=1}^n \sum_{j=1}^m\mathbb{E}[\mathbf{{\bar{S}}}_{i,j}] \bigg\|_2  \nonumber \\
  &\leq \bigg\|\frac{1}{nm}\sum_{i=1}^n \sum_{j=1}^m\mathbf{\Tilde{\bar{S}}}_{i,j} - \mathbb{E}[\mathbf{\Tilde{\bar{S}}}_{i,j}] \bigg\|_2 + \tfrac{\sigma^2 \bar{c}}{m}
\end{align}
Thus, recalling $\|\mathbf{E}''\|_2 = \bigg\|\frac{1}{nm}\sum_{i=1}^n \sum_{j=1}^m\mathbf{{{S}}}_{i,j} \bigg\|_2 $, we have
\begin{align}
    \mathbb{P}\bigg( \|\mathbf{E}''\|_2 \leq \tfrac{c\sigma^2 \bar{c}(\sqrt{d}+\sqrt{\log(nm)})(\sqrt{d_1}+\sqrt{\log(nm)})\sqrt{\log(nm)}}{\sqrt{nm}}+ \tfrac{\sigma^2 \bar{c}}{m} \bigg) \leq 1 - \frac{1}{\poly(nm)} \label{comninn}
\end{align}
Combining  \eqref{comnninn} and \eqref{comninn} completes the proof.
\end{proof}

\begin{fact}
Suppose $\mathbf{x}\sim p$ satisfies $\mathbb{E}[\mathbf{x}]=\mathbf{0}$, $\Cov(\mathbf{x})=\mathbf{I}_d$ and $\mathbf{x}$ is $\mathbf{I}_d$-sub-gaussian, as in Assumption \ref{assump:data}. Then $\mathbf{x}$ is $C$-L4-L2 hypercontractive for an absolute constant $C$, that is for any $\mathbf{u}\in \mathbb{R}^d : \|\mathbf{u}\|_2=1$,
\begin{align}
    \mathbb{E}[\langle \mathbf{u}, \mathbf{x}_{i,j}\rangle^4] \leq C^2 (\mathbb{E}[\langle \mathbf{u}, \mathbf{x}_{i,j}\rangle^2])^2
\end{align}
\end{fact}

\begin{lemma}[L4-L2 hypercontractive implication]\label{lem:L4L2}
Suppose $\mathbf{x}\in \mathbb{R}^d$ is $C$-L4-L2 hypercontractive, $\mathbb{E}[\mathbf{x}]=\mathbf{0}$, and $\Cov(\mathbf{x}) = \mathbf{I}_d$. 
Further, let $\mathbf{C}_{\ell}\in \mathbb{R}^{d\times d_{\lfloor \ell/2 \rfloor}}$ for $\ell = 1,2,3,4$ be fixed matrices for $i \in [n]$, and let $c_\ell \coloneqq \max_{i\in[n]}\|\mathbf{C}_{\ell,i}\|_2$ for $\ell=1,2,3,4$.
Given scalar thresholds $a_\ell$ for $\ell=1,\dots,4$, form the truncated random vectors $\mathbf{\bar{x}}_\ell \coloneqq \chi_{\|\mathbf{C}_{\ell}^\top\mathbf{x}\|_2\leq a_\ell}\mathbf{C}_{\ell}\mathbf{x}$. Then,
\begin{align}
   \| \mathbb{E}[\mathbf{\bar{x}}_1 \mathbf{\bar{x}}_2^\top \mathbf{\bar{x}}_3 \mathbf{\bar{x}}_4^\top]\|_2 &\leq C^2 \|\mathbf{C}_{1}\|_2\|\mathbf{C}_{2}\|_2\|\mathbf{C}_{3}\|_2\|\mathbf{C}_{4}\|_2 d_1.
\end{align}
\end{lemma}

\begin{proof}
First we note that if a random vector $\mathbf{x}$ is $C$-L4-L2 hypercontractive, then for any fixed matrix $\mathbf{C} \in \mathbb{R}^{d \times d_1}$, then the random vector $\mathbf{C}^\top \mathbf{x}\in \mathbb{R}^{d_1}$ is also $C$-L4-L2 hypercontractive, since for any unit vector $\mathbf{u}$,
\begin{align}
    \tfrac{1}{\|\mathbf{Cu}\|_2^4}\mathbb{E}[\langle \mathbf{u}, \mathbf{C}^\top \mathbf{x} \rangle^4] &= \mathbb{E}[\langle \tfrac{\mathbf{C}\mathbf{u}}{\|\mathbf{Cu}\|_2},  \mathbf{x} \rangle^4] \leq C^2(\mathbb{E}[\langle \tfrac{\mathbf{C}\mathbf{u}}{\|\mathbf{Cu}\|_2},  \mathbf{x} \rangle^2])^2 =  \tfrac{1}{\|\mathbf{Cu}\|_2^4} C^2(\mathbb{E}[\langle {\mathbf{u}},  \mathbf{C}^\top\mathbf{x} \rangle^2])^2 \nonumber \\
    \implies \mathbb{E}[\langle \mathbf{u}, \mathbf{C}^\top \mathbf{x} \rangle^4] &\leq C^2(\mathbb{E}[\langle {\mathbf{u}},  \mathbf{C}^\top\mathbf{x} \rangle^2])^2  \nonumber
\end{align}
Also, if the random vector $\mathbf{x}$ is $C$-L4-L2  hypercontractive then the truncated random vector $\mathbf{\bar{x}}\coloneqq \chi_{\|\mathbf{x}\|_2 \leq a}\|\mathbf{x}\|_2$ is also $C$-L4-L2  hypercontractive. To see this, observe that by the law of total expectation,
\begin{align}
    \mathbb{E}\big[\langle \mathbf{u},  \mathbf{\bar{x}} \rangle^4\big] &= \mathbb{E}\big[\langle \mathbf{u}, \chi_{\|\mathbf{x}\|_2 \leq a} \mathbf{{x}} \rangle^4\big] = \mathbb{E}\big[\langle \mathbf{u}, \mathbf{{x}} \rangle^4 \big| \|\mathbf{x}\|_2 \leq a\big] \mathbb{P}(\|\mathbf{x}\|_2 \leq a) \leq \mathbb{E}[\langle \mathbf{u},  \mathbf{{x}} \rangle^4] \leq C^2(\mathbb{E}[\langle \mathbf{u},  \mathbf{{x}} \rangle^2])^2
\end{align}
So we have that the truncated random vectors $\{\mathbf{\bar{x}}_h\}_{h=1}^4$ are $C$-L4-L2 hypercontractive. 
 Next, pick some $\mathbf{u}\in \mathbb{R}^{d_1}:\|\mathbf{u}\|_2\leq 1$ and $\mathbf{v}\in \mathbb{R}^{d_4}:\|\mathbf{v}\|\leq 1$. By the Cauchy-Schwarz inequality and $C$-L4-L2 hypercontractivity, we have 
 \begin{align}
       \mathbb{E}[\mathbf{u}^\top&\mathbf{\bar{x}}_1\mathbf{\bar{x}}_2^\top \mathbf{\bar{x}}_3 \mathbf{\bar{x}}_4^\top \mathbf{v} ] \nonumber \\
       &\leq (\mathbb{E}[(\mathbf{u}^\top\mathbf{\bar{x}}_1 \mathbf{\bar{x}}_4^\top \mathbf{v})^2 ] \mathbb{E}[(\mathbf{\bar{x}}_2^\top \mathbf{\bar{x}}_3 )^2])^{1/2} \nonumber \\
       &\leq (\mathbb{E}[(\mathbf{u}^\top\mathbf{\bar{x}}_1)^4]\mathbb{E}[(\mathbf{\bar{x}}_4^\top \mathbf{v})^4 ] )^{1/4}(\mathbb{E}[\Tr(\mathbf{\bar{x}}_2 \mathbf{\bar{x}}_3^\top )^2])^{1/2} \nonumber \\
       &\leq C(\mathbb{E}[(\mathbf{u}^\top \mathbf{\bar{x}}_1)^2]\mathbb{E}[( \mathbf{\bar{x}}_4^\top \mathbf{v})^2 ] )^{1/2}\bigg(\mathbb{E}\bigg[\bigg(\sum_{\ell=1}^d\mathbf{e}_\ell^\top\mathbf{\bar{x}}_2 \mathbf{\bar{x}}_3^\top \mathbf{e}_\ell\bigg)^2\bigg]\bigg)^{1/2} \nonumber \\
       &\leq C(\mathbb{E}[(\mathbf{u}^\top \mathbf{\bar{x}}_1)^2]\mathbb{E}[( \mathbf{\bar{x}}_4^\top \mathbf{v})^2 ] )^{1/2}\bigg(\sum_{\ell,\ell'}\mathbb{E}\left[\mathbf{e}_\ell^\top\mathbf{\bar{x}}_{2}\mathbf{\bar{x}}_{3}^\top \mathbf{e}_\ell \mathbf{e}_{\ell'}^\top\mathbf{\bar{x}}_{2}\mathbf{\bar{x}}_{3}^\top \mathbf{e}_{\ell'} \right]\bigg)^{1/2}  \nonumber \\
       &\leq C(\mathbb{E}[(\mathbf{u}^\top \mathbf{\bar{x}}_1)^2]\mathbb{E}[( \mathbf{\bar{x}}_4^\top \mathbf{v})^2 ] )^{1/2} \bigg(\sum_{\ell,\ell'}\big(\mathbb{E}\left[(\mathbf{e}_\ell^\top\mathbf{\bar{x}}_{2})^4]\mathbb{E}[(\mathbf{\bar{x}}_{3}^\top\mathbf{e}_\ell)^4]\mathbb{E}[( \mathbf{e}_{\ell'}^\top\mathbf{\bar{x}}_{2})^4]\mathbb{E}[(\mathbf{\bar{x}}_{3}^\top  \mathbf{e}_{\ell'})^{4} \right]\big)^{1/4}\bigg)^{1/2}  \nonumber \\
       &\leq  C^2(\mathbb{E}[(\mathbf{u}^\top \mathbf{\bar{x}}_1)^2]\mathbb{E}[( \mathbf{\bar{x}}_4^\top \mathbf{v})^2 ] )^{1/2} \nonumber \\
       &\quad \times\bigg(\sum_{\ell,\ell'}\big((\mathbb{E}[(\mathbf{e}_\ell^\top\mathbf{\bar{x}}_{2})^2])^2(\mathbb{E}[(\mathbf{\bar{x}}_{3}^\top\mathbf{e}_\ell)^2])^2(\mathbb{E}[( \mathbf{e}_{\ell'}^\top\mathbf{\bar{x}}_{2})^2])^2(\mathbb{E}[(\mathbf{\bar{x}}_{3}^\top  \mathbf{e}_{\ell'})^{2} ])^2\big)^{1/4}\bigg)^{1/2} \label{lgl}
   \end{align}
   where $\mathbf{e}_{\ell}$ is the $\ell$-th standard basis vector in $\mathbb{R}^{d_1}$. Note that by the law of total expectation and the nonnegativity of $\mathbf{U}^\top\mathbf{{C}}_1^\top \mathbf{{x}}\mathbf{{x}}^\top\mathbf{C}_1 \mathbf{U}$,
   \begin{align}
       \mathbb{E}[( \mathbf{\bar{x}}_1^\top \mathbf{u})^2 ] &= \mathbb{E}\left[\mathbf{u}^\top\mathbf{{C}}_1^\top \mathbf{{x}}\mathbf{{x}}^\top\mathbf{C}_1 \mathbf{u} \big| \|\mathbf{C}_1^\top\mathbf{{x}}\|_2\leq a \right] \mathbb{P}( \|\mathbf{C}_1^\top\mathbf{{x}}\|_2\leq a) \nonumber \\
       &\leq  \mathbb{E}[\mathbf{u}^\top\mathbf{{C}}_1^\top \mathbf{{x}}\mathbf{{x}}^\top\mathbf{C}_1 \mathbf{u} ] = \mathbf{u}^\top\mathbf{{C}}_1^\top \mathbf{C}_1 \mathbf{u}  \leq \|\mathbf{C}_1\|_2^2 \nonumber
   \end{align}
   Therefore, applying the same logic for $ \mathbb{E}[(\mathbf{e}_\ell^\top\mathbf{\bar{x}}_{2})^2]$,  $\mathbb{E}[(\mathbf{e}_\ell^\top\mathbf{\bar{x}}_{3})^2]$, and $\mathbb{E}[(\mathbf{e}_\ell^\top\mathbf{\bar{x}}_{4})^2]$, and using \eqref{lgl}, we obtain
 \begin{align}
       \mathbb{E}[\mathbf{u}^\top\mathbf{\bar{x}}_1\mathbf{\bar{x}}_2^\top \mathbf{\bar{x}}_3 \mathbf{\bar{x}}_4^\top \mathbf{v} ] 
       &\leq  C^2\|\mathbf{C}_{1}\|_2\|\mathbf{C}_{4}\|_2 \bigg(\sum_{\ell,\ell'}\|\mathbf{C}_{2}\|_2^2 \|\mathbf{C}_{3}\|_2^2 \bigg)^{1/2} = C^2\|\mathbf{C}_{1}\|_2\|\mathbf{C}_{2}\|_2 \|\mathbf{C}_{3}\|_2 \|\mathbf{C}_{4}\|_2 d_1 \nonumber
   \end{align}   
   Repeating this argument over all unit vectors $\mathbf{u}, \mathbf{v}$ completes the proof. 
\end{proof}

Next, we characterize the diversity of the inner loop-updated heads for both ANIL and FO-ANIL. Note that now we are analyzing  ANIL and FO-ANIL specifically rather than studying generic matrix concentration.
\begin{lemma} \label{lem:cond_app_exact}
Let $\mathbf{w}_{t,i}$ be the inner loop-updated head for the $i$-th task at iteration $t$ for ANIL and FO-ANIL for all $i\in [n]$. 
Define $\mu^2 \coloneqq \sigma_{\min}\left(\frac{1}{n} \sum_{i=1}^n \mathbf{w}_{t,i}\mathbf{w}_{t,i}^\top \right)$ and $L^2 \coloneqq \sigma_{\max}\left(\frac{1}{n} \sum_{i=1}^n \mathbf{w}_{t,i}\mathbf{w}_{t,i}^\top \right)$. 
Assume $\|\del_t\|_2 \leq \tfrac{1}{10}$ and Assumption \ref{assump:tasks_diverse_main}, \ref{assump:tasks_main}, and \ref{assump:data}  hold. Then 
\begin{align}
     \sigma_{\max}\left(\frac{1}{n} \sum_{i=1}^n \mathbf{w}_{t,i}\mathbf{w}_{t,i}^\top \right)\leq L^2 &\coloneqq   2\left(\|\del_t\|_2\|\mathbf{w}_t\|_2 +  \sqrt{\alpha} L_\ast + {\delta}_{m_{in},k}(\|\w\|_2  + \sqrt{{\alpha}}L_{\max}  + \sqrt{{{\alpha}}}\sigma ) \right)^2 \\
  \sigma_{\min}\left(\frac{1}{n} \sum_{i=1}^n \mathbf{w}_{t,i}\mathbf{w}_{t,i}^\top \right) \geq \mu^2 &\coloneqq 0.9 \alpha E_0 \mu_\ast^2 - 2.2 \sqrt{\alpha}\|\mathbf{w}_t\|_2 \|\del_t\|_2 \eta_\ast  \nonumber \\
  &\quad - 2 \|\del_t\|_2 \|\mathbf{w}_t\|_2\bar{\delta}_{m_{in},k}(\|\w\|_2  + \sqrt{{\alpha}}L_{\max}  + \sqrt{{{\alpha}}}\sigma ) \nonumber \\
   &\quad - 2.2 \sqrt{\alpha} \bar{\delta}_{m_{in},k}(\|\w\|_2 + \sqrt{{\alpha}}L_{\ast}  + \sqrt{{{\alpha}}}\sigma )L_{\max}  
 \end{align}
 with probability at least $1 - 4 n^{-99} - 6e^{-90k}$.
\end{lemma}

\begin{proof}
Note that $\mathbf{w}_{t,i}$ can be written as:
\begin{align}
    \mathbf{w}_{t,i} &= \w - \alpha\b^\top \sin \b \w + \alpha \b^\top \sin \mathbf{B}_\ast \mathbf{w}_{\ast,t,i} =  \mathbf{r} + \mathbf{s}_i   + \mathbf{p}_{1,i} + \mathbf{p}_{2,i} + \mathbf{p}_{3,i}
\end{align}
where $\mathbf{r} = \del_t\mathbf{w}_t$, $\mathbf{s}_i = \alpha \mathbf{B}_t^\top \mathbf{{B}}_\ast \mathbf{w}_{\ast,t,i}$,  $\mathbf{p}_{1,i} \coloneqq \alpha(\mathbf{B}_t^\top \mathbf{B}_t - \mathbf{B}_t^\top\mathbf{\Sigma}_{t,i}^{in} \mathbf{B}_t )  \mathbf{{w}}_{t}$,  $\mathbf{p}_{2,i} \coloneqq -\alpha(\mathbf{B}_t^\top \mathbf{{B}}_\ast - \mathbf{B}_t ^\top  \mathbf{\Sigma}_{t,i}^{in} \mathbf{{B}}_\ast) \mathbf{{w}}_{\ast,t,i}$, and  $\mathbf{p}_{3,i} \coloneqq\frac{\alpha}{m_{in}} \mathbf{B}_t ^\top (\mathbf{X}_{t,i}^{in})^\top \mathbf{z}_{t,i}^{in}$ for all $i \in [n]$ (for ease of notation we drop the iteration index $t$). Note that since $\|\del_t\|_2 \leq \tfrac{1}{10}$, $\|\mathbf{B}_t\|_2 \leq \tfrac{\sqrt{11/10}}{\sqrt{\alpha}}$.
As a result, for any $i \in [n]$, from Lemma \ref{lem:gen1} 
we have 
\begin{align}
    \| \mathbf{p}_{1,i}\|_2 &\leq  1.1 \|\w\|_2 {\delta}_{m_{in},k}, \quad
    \| \mathbf{p}_{2,i}\|_2 \leq  \sqrt{1.1{\alpha}}L_{\max} {\delta}_{m_{in},k}, \quad 
    \| \mathbf{p}_{3,i}\|_2 \leq   \sqrt{1.1{{\alpha}}}\sigma {\delta}_{m_{in},k}  \label{basicbounds}
\end{align}
each with probability at least $1 - 2 n^{-100}$. Thus, all of these events  happen simultaneously with probability at least $1- 6 n^{-100}$ via a union bound. Further, a union bound over all $i \in [n]$ shows that $A \coloneqq\left\{\cap_{i\in [n]}\left\{\| \mathbf{p}_{1,i}\|_2 \leq 1.1 \|\w\|_2 {\delta}_{m_{in},k} \cap \| \mathbf{p}_{2,i}\|_2 \leq \sqrt{1.1{\alpha}}L_{\max} {\delta}_{m_{in},k} \cap \| \mathbf{p}_{3,i}\|_2 \leq \sqrt{1.1{{\alpha}}}\sigma {\delta}_{m_{in},k} \right\} \right\}$ occurs with probability at least $ 1 - 6 n^{-99}$. Thus by the triangle inequality, a
    \begin{align}
       \left\|\frac{1}{n}\sum_{i=1}^n \mathbf{p}_{1,i} +  \mathbf{p}_{2,i} + \mathbf{p}_{3,i}\right\|_2 
       &\leq 1.1{\delta}_{m_{in},k}(\|\w\|_2  + \sqrt{{\alpha}}L_{\max}  + \sqrt{{{\alpha}}}\sigma )
    \end{align}
    with probability at least $1 - 6 n^{-99}$,
    and
\begin{align}
    \left\|\frac{1}{n}\sum_{i=1}^n (\mathbf{r}+\mathbf{s}_i)(\mathbf{r}+ \mathbf{s}_i)^\top \right\|_2 
    &\leq \|\del_t\|_2^2\|\mathbf{w}_t\|_2^2 + 2.2 \sqrt{\alpha} \|\del_t\|_2\|\mathbf{w}_t\|_2  \eta_\ast  + 1.1\alpha L_\ast^2 
    \end{align}
    So, 
\begin{align}
   &\left\| \frac{1}{n} \sum_{i = 1}^n \mathbf{w}_{t,i}\mathbf{w}_{t,i}^\top\right\|_2  \nonumber \\
   &\leq \left\|\frac{1}{n} \sum_{i = 1}^n (\mathbf{r}+\mathbf{s}_i)(\mathbf{r}+\mathbf{s}_i)^\top\right\|_2 + 2 \left\|\frac{1}{n} \sum_{i = 1}^n (\mathbf{r}+\mathbf{s}_i)(\mathbf{p}_{1,i} + \mathbf{p}_{2,i} + \mathbf{p}_{3,i} )^\top \right\|_2 \nonumber \\
   &\quad \quad \quad + \left\|\frac{1}{n} \sum_{i = 1}^n (\mathbf{p}_{1,i} + \mathbf{p}_{2,i} + \mathbf{p}_{3,i})(\mathbf{p}_{1,i} + \mathbf{p}_{2,i} + \mathbf{p}_{3,i})^\top \right\|_2 \nonumber \\
   &\leq \|\del_t\|_2^2\|\mathbf{w}_t\|_2^2 + 2.2 \sqrt{\alpha} \|\del_t\|_2\|\mathbf{w}_t\|_2  \eta_\ast  + 1.1\alpha L_\ast^2  + 2 \left\|\frac{1}{n} \sum_{i = 1}^n \mathbf{r}(\mathbf{p}_{1,i} + \mathbf{p}_{2,i} + \mathbf{p}_{3,i} )^\top \right\|_2 \nonumber\\
   &\quad + 2\alpha\|\mathbf{B}_t ^\top \mathbf{{B}}_\ast\|_2 \left\| \frac{1}{n}\mathbf{W}_{\ast,t}^\top \begin{bmatrix}\vdots \\
(\mathbf{p}_{1,i} + \mathbf{p}_{2,i} + \mathbf{p}_{3,i} )^\top    \\
   \vdots \\
   \end{bmatrix}\right\|_2  + \max_{i\in [n]} \|\mathbf{p}_{1,i} + \mathbf{p}_{2,i} + \mathbf{p}_{3,i} \|_2^2 \label{tri32}\\
   &\leq \|\del_t\|_2^2\|\mathbf{w}_t\|_2^2 + 2.2 \sqrt{\alpha} \|\del_t\|_2\|\mathbf{w}_t\|_2  \eta_\ast  + 1.1\alpha L_\ast^2  \nonumber \\
   &\quad + 2.2 \|\del_t\|_2^2\|\mathbf{w}_t\|_2{\delta}_{m_{in},k}(\|\w\|_2  + \sqrt{{\alpha}}L_{\max}  + \sqrt{{{\alpha}}}\sigma ) \nonumber \\ 
  &\quad   + 2.2 \sqrt{\alpha} L_\ast {\delta}_{m_{in},k}(\|\w\|_2  + \sqrt{{\alpha}}L_{\max}  + \sqrt{{{\alpha}}}\sigma ) 
 + 1.1^2(\|\w\|_2  + \sqrt{{\alpha}}L_{\max}  + \sqrt{{{\alpha}}}\sigma )^2 {\delta}_{m_{in},k}^2  \nonumber\\
   &\leq 2\left(\|\del_t\|_2\|\mathbf{w}_t\|_2 +  \sqrt{\alpha} L_\ast + {\delta}_{m_{in},k}(\|\w\|_2  + \sqrt{{\alpha}}L_{\max}  + \sqrt{{{\alpha}}}\sigma ) \right)^2
\label{hp1_exact}
\end{align}
where $\mathbf{W}_{\ast,t} = [\mathbf{w}_{\ast,t,1},\dots,\mathbf{w}_{\ast,t,n}]^\top$, \eqref{tri32} follows from the triangle inequality, and \eqref{hp1_exact} follows with probability at least $1 - 6 n^{-99}$ from the discussion above.

We make an analogous argument to lower bound $\sigma_{\min}\left(\frac{1}{n} \sum_{i = 1}^n \mathbf{w}_{t,i}\mathbf{w}_{t,i}^\top\right)$. This time, we only need to bound first-order products of the $\mathbf{p}$ matrices, which concentrate around zero as $n$ becomes large. So now we are able to obtain finite-sample dependence on $\bar{\delta}_{m_{in},k}$ (which decays with $\tfrac{1}{\sqrt{n}}$) instead of ${\delta}_{m_{in},k}$ (which does not), as follows.
\begin{align}
  \sigma_{\min} \left( \frac{1}{n} \sum_{i = 1}^n \mathbf{w}_{t,i}\mathbf{w}_{t,i}^\top\right) 
   &=  \sigma_{\min}  \Bigg(\frac{1}{n} \sum_{i = 1}^n (\mathbf{r}+\mathbf{s}_i)(\mathbf{r}+\mathbf{s}_i)^\top + (\mathbf{r}+\mathbf{s}_i)(\mathbf{p}_{1,i} + \mathbf{p}_{2,i} + \mathbf{p}_{3,i} )^\top \nonumber \\
   &\quad \quad \quad + (\mathbf{p}_{1,i} + \mathbf{p}_{2,i} + \mathbf{p}_{3,i})(\mathbf{r}+\mathbf{s}_i)^\top \nonumber \\
   &\quad \quad \quad + (\mathbf{p}_{1,i} + \mathbf{p}_{2,i} + \mathbf{p}_{3,i})(\mathbf{p}_{1,i} + \mathbf{p}_{2,i} + \mathbf{p}_{3,i})^\top \Bigg) \nonumber \\
   &\geq  \sigma_{\min}  \left(\frac{1}{n} \sum_{i = 1}^n (\mathbf{r}+\mathbf{s}_i)(\mathbf{r}+\mathbf{s}_i)^\top\right) - 2    \left\|\frac{1}{n}\sum_{i=1}^n(\mathbf{r}+\mathbf{s}_i)(\mathbf{p}_{1,i} + \mathbf{p}_{2,i} + \mathbf{p}_{3,i} )^\top\right\|_2  \nonumber \\
   &\geq \sigma_{\min}  \left(\frac{1}{n} \sum_{i = 1}^n \mathbf{s}_i\mathbf{s}_i^\top\right) - 2 \left\|\frac{1}{n} \sum_{i = 1}^n \mathbf{r}\mathbf{s}_i^\top \right\|_2  - 2    \left\|\frac{1}{n}\sum_{i=1}^n(\mathbf{r}+\mathbf{s}_i)(\mathbf{p}_{1,i} + \mathbf{p}_{2,i} + \mathbf{p}_{3,i} )^\top\right\|_2 
   \nonumber \\
   &\geq 0.9 \alpha E_0 \mu_\ast^2 - 2.2 \sqrt{\alpha}\|\mathbf{w}_t\|_2 \|\del_t\|_2 \eta_\ast \nonumber \\
   &\quad - 2 \|\del_t\|_2 \|\mathbf{w}_t\|_2\bar{\delta}_{m_{in},k}(\|\w\|_2  + \sqrt{{\alpha}}L_{\max}  + \sqrt{{{\alpha}}}\sigma ) \nonumber \\
   &\quad - 2.2 \sqrt{\alpha} \bar{\delta}_{m_{in},k}(\|\w\|_2 + \sqrt{{\alpha}}L_{\ast}  + \sqrt{{{\alpha}}}\sigma )L_{\max}   \nonumber 
\end{align}
where the last inequality follows with probability at least $1 - 6e^{-90k}$.
\end{proof}


\subsection{FO-ANIL}

For FO-ANIL, inner loop update for the head of the $i$-th task on iteration $t$ is given by:
\begin{align}
    \mathbf{w}_{t,i} 
    &= \mathbf{w}_t - \alpha \nabla_{\mathbf{w}} \hat{\mathcal{L}}_i(\mathbf{B}_t, \mathbf{w}_t, \mathcal{D}_{i}^{in}) \nonumber \\
    &=  (\mathbf{I}_k- \alpha\mathbf{B}_t ^\top \mathbf{\Sigma}_{t,i}^{in} \mathbf{B}_t )\mathbf{{w}}_t + \alpha \mathbf{B}_t ^\top \mathbf{\Sigma}_{t,i}^{in} \mathbf{{B}}_\ast\mathbf{{w}}_{\ast,t,i} + \tfrac{\alpha}{m_{in}} \mathbf{B}_t^\top (\mathbf{X}_{t,i}^{in})^\top \mathbf{z}_{t,i}^{in}. \label{upd_head_anil_foo}
\end{align}
The outer loop updates for the head and representation are:
\begin{align}
\mathbf{w}_{t+1} &= \mathbf{w}_t - \frac{\beta}{n}\sum_{i=1}^n \nabla_{\mathbf{w}} \hat{\mathcal{L}}_i(\mathbf{B}_t , \mathbf{w}_{t,i}, \mathcal{D}_{t,i}^{out}) \nonumber \\
&=  \mathbf{{w}}_{t} - \frac{\beta}{n}\sum_{i=1}^{n}\left(\mathbf{B}_t ^\top \mathbf{\Sigma}_{t,i}^{out} \mathbf{B}_t \mathbf{{w}}_{t,i} - \mathbf{B}_t ^\top \mathbf{\Sigma}_{t,i}^{out} \mathbf{{B}}_\ast\mathbf{{w}}_{\ast,i} - \tfrac{2}{m_{out}} \mathbf{B}_t ^\top (\mathbf{X}_{t,i}^{out,g})^\top\mathbf{z}_{t,i}^{out,g}\right) \label{upd_head_out_anil_foo} \\
    \mathbf{{B}}_{t+1} &= \mathbf{B}_t  - \frac{\beta}{n}\sum_{i=1}^n \nabla_{\mathbf{B}} \hat{\mathcal{L}}_i(\mathbf{B}_t , \mathbf{w}_{t,i}, \mathcal{D}_{t,i}^{out}) \nonumber \\
    &= \mathbf{{B}}_t  - \frac{\beta}{n}\sum_{i=1}^n \left( \mathbf{\Sigma}_{t,i}^{out} \mathbf{B}_t  \mathbf{w}_{t,i}\mathbf{w}_{t,i}^\top - \mathbf{\Sigma}_{t,i}^{out} \mathbf{{B}}_\ast\mathbf{w}_{\ast,t,i} \mathbf{w}_{t,i}^\top  - \tfrac{2}{m_{out}}(\mathbf{X}^{out}_{t,i})^\top \mathbf{z}^{out}_{t,i}\mathbf{w}_{t,i}^\top \right) \label{upd_rep_anil_foo}
\end{align}

\begin{lemma}[FO-ANIL, Finite samples $A_1(t+1)$]\label{lem:w_anil_fs}
For any $t$, suppose that $A_2(s),A_3(s)$ and $A_4(s)$ occur for all $s \in [t]$. Then
\begin{align}
    \|\mathbf{w}_{t+1}\|_2 \leq  \tfrac{1}{10}\sqrt{\alpha} E_0 \min(1,\tfrac{\mu_\ast^2}{\eta_\ast^2})\eta_\ast
\end{align}
with probability at least $1 - \frac{1}{\poly(n)}$ .
\end{lemma}

\begin{proof}
The proof follows similar structure as in the analogous proof for the infinite-sample case. Recall the outer loop updates for ANIL (here we replace $t$ with $s$): 
\begin{align}
\mathbf{w}_{s+1} &= (\mathbf{I}_{k} - \beta  \mathbf{B}_s^\top\mathbf{B}_s (\mathbf{I}_k-\alpha\mathbf{B}_s^\top\mathbf{B}_s)) \mathbf{{w}}_{s} + \beta (\mathbf{I}_k - \alpha\mathbf{B}_s^\top\mathbf{B}_s )  \mathbf{B}_s^\top \mathbf{{B}}_\ast\frac{1}{n}\sum_{i=1}^n\mathbf{{w}}_{\ast,s,i}  \nonumber \\
    &\quad + \alpha\beta \mathbf{B}_s^\top\mathbf{B}_s
    \frac{1}{n}\sum_{i=1}^n\left( \mathbf{B}_s^\top\mathbf{B}_s -  \mathbf{B}_s^\top\mathbf{\Sigma}_{s,i}^{in}\mathbf{B}_s\right)\mathbf{w}_s  - \alpha \beta \mathbf{B}_s^\top\mathbf{B}_s \frac{1}{n}\sum_{i=1}^n (\mathbf{B}_s^\top \mathbf{B}_\ast - \mathbf{B}_s^\top\mathbf{\Sigma}_{s,i}^{in}\mathbf{B}_\ast) \mathbf{w}_{\ast,s,i}\nonumber \\
    &\quad +  \alpha \beta \mathbf{B}_s^\top\mathbf{B}_s\frac{1}{n}\sum_{i=1}^n \mathbf{B}_s^\top (\mathbf{X}_{s,i}^{in})^\top \mathbf{z}_{s,i}^{in} + \frac{\beta}{n}\sum_{i=1}^n (\mathbf{B}_s^\top\mathbf{B}_s - \mathbf{B}_s ^\top\mathbf{\Sigma}_{s,i}^{out} \mathbf{B}_s )  \mathbf{{w}}_{s,i} \nonumber \\
&\quad - \frac{\beta}{n}\sum_{i=1}^n  (\mathbf{B}_s^\top \mathbf{{B}}_\ast - \mathbf{B}_s^\top  \mathbf{\Sigma}_{s,i}^{out} \mathbf{{B}}_\ast) \mathbf{{w}}_{\ast,s,i} +  \frac{2\beta}{nm_{out}} \sum_{i=1}^n \mathbf{B}_s^\top (\mathbf{X}_{s,i}^{out})^\top \mathbf{z}_{s,i}^{out} \label{anil_fo_outer_head}
\end{align}
Note that $\bigcup_{s=0}^t A_3(s)$ implies $\sigma_{\max}(\mathbf{B}_s^\top \mathbf{B}_s) \leq \frac{1+\|\del_s\|_2}{\alpha}< \frac{1.1}{\alpha}$ for all $s\in\{0,\dots,t\!+\!1\}$. Also, we can straightforwardly use Lemma \ref{lem:gen1} with the Cauchy-Schwartz inequality to obtain, for some absolute constant $c$,
\begin{align}
       \Bigg\| \alpha\beta \mathbf{B}_s^\top\mathbf{B}_s
    \frac{1}{n}\sum_{i=1}^n\left( \mathbf{B}_s^\top\mathbf{B}_s -  \mathbf{B}_s^\top\mathbf{\Sigma}_{s,i}^{in}\mathbf{B}_s\right)\mathbf{w}_t \Bigg\|_2
   &\leq c\tfrac{\beta}{\alpha} \|\mathbf{w}_s\|_2 \bar{\delta}_{m_{in},k} \nonumber \\
    \Bigg\|\alpha \beta \mathbf{B}_s^\top\mathbf{B}_s \frac{1}{n}\sum_{i=1}^n (\mathbf{B}_s^\top \mathbf{B}_\ast - \mathbf{B}_t^\top\mathbf{\Sigma}_{s,i}^{in}\mathbf{B}_s) \mathbf{w}_{\ast,s,i}\Bigg\|_2 &\leq c\tfrac{\beta}{\sqrt{\alpha}}   L_{\max} \bar{\delta}_{m_{in},k}
\end{align}
\begin{align}
\Bigg\|\tfrac{\beta}{n}\sum_{i=1}^n  (\mathbf{B}_s ^\top \mathbf{{B}}_s - \mathbf{B}_s ^\top  \mathbf{\Sigma}_{s,i}^{out} \mathbf{{B}}_s) \mathbf{{w}}_{s,i}\Bigg\|_2
&\leq c\tfrac{\beta}{\alpha}  \bar{\delta}_{m_{out},k}\max_{i\in[n]}\|\mathbf{w}_{s,i}\|_2 \nonumber \\
&\leq c\tfrac{\beta}{\alpha}  \bar{\delta}_{m_{out},k}(\|\del_s\|_2 \|\mathbf{w}_s\|_2\!+\!\sqrt{\alpha} L_{\max} + {\delta}_{m_{in},k}\|\mathbf{w}_s\|_2\nonumber \\
&\quad +\!\sqrt{\alpha} {\delta}_{m_{in},k} L_{\max} +\sqrt{\alpha} \sigma {\delta}_{m_{in},k}) \nonumber \\
&\leq c\tfrac{\beta}{\alpha}  \bar{\delta}_{m_{out},k}(\|\del_s\|_2 \|\mathbf{w}_s\|_2\!+\!c'\sqrt{\alpha} L_{\max} + \sqrt{\alpha} \sigma {\delta}_{m_{in},k}) \label{inin} \\
\Bigg\|\frac{\beta}{n}\sum_{i=1}^n  (\mathbf{B}_s ^\top \mathbf{{B}}_\ast - \mathbf{B}_s ^\top  \mathbf{\Sigma}_{s,i}^{out} \mathbf{{B}}_\ast) \mathbf{{w}}_{\ast,s,i}\Bigg\|_2 &\leq c\tfrac{\beta}{\sqrt{\alpha}} L_{\max} \bar{\delta}_{m_{out},k} \nonumber \\
\Bigg\|  \tfrac{2\beta}{nm_{out}} \sum_{i=1}^n \mathbf{B}_s^\top (\mathbf{X}_{s,i}^{out})^\top \mathbf{z}_{s,i}^{out}\Bigg\|_2 &\leq  c\tfrac{\beta}{\sqrt{\alpha}}  \sigma {\bar{\delta}_{m_{out},k}}{} \nonumber 
\end{align}
using that $\delta_{m_{in},k} < 1$ and $\|\mathbf{w}_s\|_2 \leq c\sqrt{\alpha} \eta_\ast$ in \eqref{inin}.
Thus using \eqref{anil_fo_outer_head} and the Cauchy-Schwarz and triangle inequalities, we have for an absolute constant $c$:
\begin{align}
    \|\mathbf{w}_{s+1}\|_2&\leq (1+c\tfrac{\beta}{\alpha} \|\del_s\|_2)\|\mathbf{w}_s\|_2 + c\tfrac{\beta}{\sqrt{\alpha}}\|\del_s\|_2  \eta_{\ast} + c\tfrac{\beta}{\alpha} \|\mathbf{w}_s\|_2{\bar{\delta}_{m_{in},k}} \nonumber   \\
    &\quad + c\tfrac{\beta}{\sqrt{\alpha}} {\bar{\delta}_{m_{in},k}}(L_{\max}+\sigma) + c\tfrac{\beta}{\alpha}{\bar{\delta}_{m_{out},k}} \sqrt{\alpha} L_{\max} + c\tfrac{\beta}{\sqrt{\alpha}}L_{\max} {\bar{\delta}_{m_{out},k}} + c\tfrac{\beta}{\sqrt{\alpha}}\sigma {\bar{\delta}_{m_{out},k}}\nonumber \\
    &\leq (1+c\tfrac{\beta}{\alpha} \|\del_s\|_2)\|\mathbf{w}_s\|_2 + t\tfrac{\beta}{\sqrt{\alpha}}\|\del_s\|_2  \eta_{\ast} + t\tfrac{\beta}{\sqrt{\alpha}}\left(L_{\max} +\sigma \right) ({\bar{\delta}_{m_{in},k}+\bar{\delta}_{m_{out},k}}) \nonumber  \\
    &= (1+c\tfrac{\beta}{\alpha} \|\del_s\|_2)\|\mathbf{w}_s\|_2 + c\tfrac{\beta}{\sqrt{\alpha}}\|\del_s\|_2  \eta_{\ast} + \tfrac{\beta}{\sqrt{\alpha}}\zeta_1
\end{align}
using  $\|\mathbf{w}_{s,i}\|_2 \leq \sqrt{\alpha}L_{\max}$, where $\zeta_1 = c(L_{\max}+\sigma)({\bar{\delta}_{m_{in},k}+ \bar{\delta}_{m_{out},k}})$.
Thus, by Lemma \ref{lem:gen_w}, we have 
\begin{align}
    \|\mathbf{w}_{t+1}\|_2 &\leq c\tfrac{\beta}{\sqrt{\alpha}}\sum_{s=1}^t  (\|\del_s\|_2\eta_\ast + \zeta_1)\left(1 + 2 \sum_{r=s}^t \frac{\beta}{\alpha} \|\del_r\|_2 \right) \label{158l}
\end{align}
Next, let $\rho \coloneqq 1 - 0.5\beta \alpha E_0 \mu_\ast^2
$ and $\zeta_2 =  O\big((L_{\max}+\sigma)^2 \bar{\delta}_{m_{out},k}+  (L_{\max}^2+ L_{\max}\sigma) (\bar{\delta}_{m_{in},k} +\delta^2_{m_{in},k})  +    \sigma^2 \delta^2_{m_{in},k} + \beta \alpha(L_{\max}+ \sigma)^4 \bar{\delta}^2_{m_{out},d}\big) $ as defined in \eqref{zeta2}. 
By $\bigcup_{r=0}^s A_2(r)$, we have
\begin{align}
    \|\del_{s+1}\|_2 &\leq \rho\|\del_{s}\|_2 + c \alpha^{2}\beta^2 L_\ast^4 \dist_s^2 + \beta \alpha \zeta_2 \nonumber \\
    &\leq \rho^2\|\del_{s-1}\|_2 + \rho(c\alpha^{2}\beta^2 \dist_{s-1}^2 + \beta \alpha \zeta_2) +  c\alpha^{2}\beta^2 \dist_s^2+ \beta \alpha \zeta_2 \nonumber \\
    &\;\;\vdots \nonumber \\
    &\leq  \rho^{s+1}\|\del_{0}\|_2 +   \sum_{r=0}^s \rho^{s-r} (c \alpha^{2}\beta^2 L_\ast^4\dist_r^2 + \beta \alpha \zeta_2) \nonumber \\
    &= \sum_{r=0}^s \rho^{s-r} ( c\alpha^{2}\beta^2 L_\ast^4 \dist_r^2 +\beta \alpha \zeta_2)
\end{align}
since $\|\mathbf{I}-\alpha \mathbf{B}_{0}^\top \mathbf{B}_{0}\|_2 = 0$ by choice of initialization.
Now, we have that $\dist_s \leq \rho^s + \varepsilon $ for all $s\in \{0,...,t\}$ by $\bigcup_{s=0}^t A_5(s)$.
Thus, for any $s\in \{0,...,t\}$, we have
\begin{align}
    \|\del_{s+1}\|_2 &\leq  \sum_{r=0}^s \rho^{s-r} ( c\alpha^{2}\beta^2 L_\ast^4 \dist_r^2 +\beta \alpha \zeta_2) \nonumber \\
    &\leq   \sum_{r=0}^s \rho^{s-r} ( c\alpha^{2}\beta^2 L_\ast^4 (2\rho^{2r}+ 2 \varepsilon^2) +\beta \alpha\zeta_2) \nonumber \\
    &\leq 2c\alpha^{2}\beta^2 L_\ast^4  \sum_{r=0}^s \rho^{s-r}\rho^{2r}+ 2c\alpha^{2}\beta^2 L_\ast^4  \sum_{r=0}^s \rho^{s-r} \varepsilon^2 + \beta \alpha\sum_{r=0}^s \rho^{s-r} \zeta_2 \nonumber \\
    &\leq 2c\rho^{s} \tfrac{\alpha^{2}\beta^2 L_\ast^4 }{1-\rho}  + (2c\alpha^{2}\beta^2 L_\ast^4 \varepsilon^2+\beta \alpha\zeta_2 ) \tfrac{1}{1-\rho}  \nonumber \\
    &\leq 2c \rho^s\beta\alpha L_\ast^2 \kappa_{\ast}^2 /E_0 + 2c\varepsilon^2 \beta\alpha L_\ast^2 \kappa_{\ast}^2 /E_0 + \zeta_2/(E_0\mu_\ast^2)  \nonumber \\
    &=: \epsilon_s
\end{align}
Now, applying equation \eqref{158l} yields
\begin{align}
    \|\mathbf{w}_{t+1}\|_2 &\leq \frac{\beta}{\sqrt{\alpha}}\sum_{s=1}^t  (\epsilon_s\eta_\ast + \zeta_1)\left(1 + 2 \sum_{r=s}^t \frac{\beta}{\alpha} \epsilon_r \right) \nonumber \\
    &\leq \frac{\beta}{\sqrt{\alpha}}\sum_{s=1}^t  (\epsilon_s\eta_\ast + \zeta_1)\left(1 +  4c \beta \rho^s \kappa_\ast^4 /(\alpha E_0^2) +  4c(t-s) {\beta}^2 \varepsilon^2 \kappa_{\ast}^2 L_\ast^2 /E_0 + 2(t-s)\beta \zeta_2/(\alpha E_0 \mu_\ast^2) \right) \label{rkhs} \\
    &\leq \frac{\beta}{\sqrt{\alpha}}\sum_{s=1}^t  (\epsilon_s\eta_\ast + \zeta_1)\left(1 +  4c \beta  \kappa_\ast^4 /(\alpha E_0^2) +  4cT {\beta}^2 \varepsilon^2 \kappa_{\ast}^2 L_\ast^2 /E_0 + 2T\beta \zeta_2/(\alpha E_0 \mu_\ast^2) \right) \label{rkhss} 
\end{align}
where \eqref{rkhs} follows by plugging in the definition of $\epsilon_r$ and using the sum of a geometric series. 
In order for the RHS of \eqref{rkhss} to be at most $\tfrac{1}{10}\sqrt{\alpha}E_0\min(1, \tfrac{\mu_\ast^2}{\eta_\ast^2}) \eta_\ast $ as desired, we can ensure that $\left( 4c \beta  \kappa_\ast^4 /(\alpha E_0^2) +  4cT {\beta}^2 \varepsilon^2 \kappa_{\ast}^2 L_\ast^2 /E_0 + 2T\beta \zeta_2/(\alpha E_0 \mu_\ast^2) \right)\leq 1$ for all $s$ and  $\frac{\beta}{\sqrt{\alpha}}\sum_{s=1}^t  (\epsilon_s\eta_\ast + \zeta_1) \leq \tfrac{1}{10}\sqrt{\alpha}E_0\min(1, \tfrac{\mu_\ast^2}{\eta_\ast^2}) \eta_\ast $. To satisfy the first condition, it is sufficient to have 
\begin{align}
\beta &\leq  c'\tfrac{\alpha E_0^2 }{\kappa_\ast^4} \nonumber \\
\varepsilon^2 &\leq c' \tfrac{E_0}{T\beta^2 \kappa_\ast^2 L_\ast^2  } \nonumber \\
\zeta_2 &\leq c'\tfrac{\alpha E_0 \mu_\ast^2}{T\beta } \nonumber 
\end{align}
For the second condition, it is sufficient to have
\begin{align}
\beta &\leq c \alpha E_0^3 \kappa_{\ast}^{-4} \min(1, \tfrac{\mu_\ast^2}{\eta_\ast^2})  \nonumber \\
    \zeta_1 &\leq c'\tfrac{\kappa_\ast^4 \eta_\ast}{T E_0^2 } \nonumber \\
    \sum_{s=1}^t \epsilon_s &\leq c'\tfrac{\kappa_\ast^4}{E_0^2}\nonumber \\
    \implies \varepsilon^2 &\leq c'' \tfrac{1}{ T\beta\alpha\mu_\ast^2E_0} \nonumber \\
    \zeta_2 &\leq c''\tfrac{L_\ast^2 \kappa_\ast^2}{T E_0}
\end{align}
However, for  Corollary \ref{cor:foanil_rep}, will need a tighter bound on $\zeta_2$, namely $ \zeta_2 \leq \frac{c' E_0 \mu_\ast^2}{ T}$. In summary, the tightest bounds are:
\begin{align} 
\beta &\leq c \alpha E_0^3 \kappa_{\ast}^{-4} \min(1, \tfrac{\mu_\ast^2}{\eta_\ast^2})   \label{biggie} \\
\varepsilon^2 &\leq  c \tfrac{1}{ T\beta\alpha\mu_\ast^2E_0}   \\
\zeta_1 &\leq   c\tfrac{\kappa_\ast^4 \eta_\ast}{T E_0^2 }   \\
    \zeta_2 &\leq c \tfrac{ E_0 \mu_\ast^2}{ T} \label{tupac}
\end{align}


To determine when these conditions hold, we must recall the scaling of $\varepsilon, \zeta_1, \zeta_2$.
\begin{align}
\varepsilon &= O(\tfrac{L_{\max}(L_{\max}+\sigma)}{\mu_\ast^2}\bar{\delta}_{m_{out},d} ) \nonumber \\
\zeta_1 &= O((L_{\max}+\sigma)({\bar{\delta}_{m_{in},k}+ \bar{\delta}_{m_{out},k}})) \nonumber \\
    \zeta_2 &= O\big(\left(  (L_{\max}^2+ L_{\max}\sigma) ({\bar{\delta}_{m_{in},k}} +{\bar{\delta}_{m_{out},k}}+{\delta}_{m_{in},k}^2)  +    \sigma^2 {\delta}_{m_{in},k}^2\right) \nonumber \\
    &\quad \quad + \beta \alpha\bar{\delta}_{m_{out},d}^2(L_{\max}+ \sigma)^2(L_{\max}+\sigma {\delta}_{m_{in},k})^2\big) \nonumber 
\end{align}
Thus, in order to satisfy \eqref{biggie}-\eqref{tupac}, we can choose:
\begin{align}
    m_{out} \geq c'\left(\beta \alpha \tfrac{d T E_0(L_{\max}+\sigma)^4}{n\mu_\ast^2} +\tfrac{T^2 k E_0^2L_{\max}^2(L_{\max}+\sigma)^2}{n\mu_\ast^4 } +\tfrac{T^2 E_0^4 k(L_{\max}+\sigma)^2}{n \eta_\ast^2\kappa_\ast^8} + \beta\alpha \tfrac{ T L_{\max}^2(L_{\max}+\sigma)^2 E_0 {d}}{n\mu_\ast^2}\right). \nonumber
\end{align}
Recalling that $L_{\max}\leq c \sqrt{k}L_\ast$, $\beta \leq \alpha \kappa_\ast^{-4}$, and $\alpha \leq \tfrac{1}{L_{\max}+\sigma}$, we see that our choice of $m_{out}$ as
\begin{align}
m_{out} &\geq c \left( \frac{T  dk}{n\kappa_\ast^{2}}+ \frac{T  dk\sigma^2}{n L_\ast^2 \kappa_\ast^{2}}+ \frac{T^2k^3\kappa_\ast^4}{n} +\frac{T^2k^3\sigma^4}{n\mu_\ast^4} + \frac{k\mu_\ast^2}{\eta_\ast^2 \kappa_{\ast}^6} + \frac{k\sigma^2}{\eta_\ast^2 \kappa_{\ast}^8} \right)
\nonumber 
\end{align}
is sufficient, where we have treated $E_0$ as a constant. 
For $m_{in}$, we can choose:
\begin{align}
    m_{in} &\geq c'\frac{T (k+\log(n))E_0(L_{\max}+\sigma)^2}{\mu_\ast^2 } +c' \frac{T^2kE_0^2L_{\max}^2(L_{\max}+\sigma)^2}{n\mu_\ast^4} + c'\frac{T^2k E_0^4 (L_{\max}+\sigma)^2}{n\eta_\ast^2 \kappa_\ast^8} \nonumber 
    \end{align}
    which is satisfied by
\begin{align}
 m_{in}&\geq c T (k+\log(n))(k\kappa_\ast^2 + \tfrac{\sigma^2}{ \mu_\ast^{2}})+ c\tfrac{T^2 k^3 \kappa_{\ast}^4}{n} + c \tfrac{T^2k^2 \kappa_\ast^2 \sigma^2}{\mu_\ast^{2}n} + c\tfrac{T^2k^2 (L_\ast^2+\sigma^2)}{\eta_\ast^{2}\kappa_\ast^{8}n}  \nonumber
\end{align}
Since $m_{in}$ and $m_{out}$ satisfy these conditions, we have completed the proof.

\end{proof}

\begin{lemma}[FO-ANIL, Finite samples, $A_2(t+1)$]\label{lem:reg_fs_foanil}
Suppose the conditions of Theorem \ref{thm:anil_fs_app} are satisfied and inductive hypotheses $A_1(t)$, $A_3(t)$ and $A_5(t)$ hold. Then $A_2(t+1)$ holds with high probability, i.e.
\begin{align}
    \|\del_{t+1}\|_2 &\leq(1 - 0.5\beta \alpha E_0 \mu_{\ast}^2)\|\del_t\|_2 + c\beta^2 \alpha^2 L_\ast^4 \dist_t^2 + \beta \alpha \zeta_2
\end{align}
for an absolute constant $c$ and $\zeta_2 = O\big((L_{\max}+\sigma)^2 \bar{\delta}_{m_{out},k}+  (L_{\max}^2+ L_{\max}\sigma) (\bar{\delta}_{m_{in},k} +\delta^2_{m_{in},k})  +    \sigma^2 \delta^2_{m_{in},k} + \beta \alpha(L_{\max}+ \sigma)^4 \bar{\delta}^2_{m_{out},d}\big)$, with probability at least $1-\tfrac{1}{\poly(n)}$.
\end{lemma}
\begin{proof}
Note that we can write:
\begin{align}
    \mathbf{B}_{t+1} &=\mathbf{{B}}_t  - \beta\deld_t\tfrac{1}{n}\sum_{i=1}^n(  \mathbf{{B}}_t  \mathbf{w}_{t} -  \mathbf{{B}}_\ast\mathbf{w}_{\ast,t,i} )( (\mathbf{I}_k - \alpha \mathbf{B}_t^\top \mathbf{B}_t)\mathbf{w}_{t} + \alpha \mathbf{B}_t^\top \mathbf{B}_{\ast}\mathbf{w}_{\ast,t,i})^\top  \nonumber \\
    &\quad + \beta\deld_t\tfrac{1}{n}\sum_{i=1}^n(  \mathbf{{B}}_t  \mathbf{w}_{t} -  \mathbf{{B}}_\ast\mathbf{w}_{\ast,t,i} )(\alpha (\mathbf{B}_t^\top \mathbf{B}_t - \mathbf{B}_t^\top \mathbf{\Sigma}_{t,i}^{in} \mathbf{B}_t)\mathbf{w}_{t} + \alpha (\mathbf{B}_t^\top \mathbf{B}_{\ast} - \mathbf{B}_t^\top \mathbf{\Sigma}_{t,i}^{in} \mathbf{B}_{\ast} )\mathbf{w}_{\ast,t,i})^\top \label{ttt1}  \\
    &\quad + \beta\deld_t\tfrac{1}{n}\sum_{i=1}^n(  \mathbf{{B}}_t  \mathbf{w}_{t} -  \mathbf{{B}}_\ast\mathbf{w}_{\ast,t,i} ) (\alpha \mathbf{B}_t^\top \tfrac{1}{m_{in}}(\mathbf{X}_{t,i}^{in})^\top \mathbf{z}_{t,i} )^\top  \label{ttt2}  \\
    &\quad + \tfrac{\beta}{n} \sum_{i=1}^n \left(\mathbf{I}_d - \mathbf{\Sigma}_{t,i}^{out}\right) \left(\mathbf{B}_t  \mathbf{w}_{t,i} - \mathbf{{B}}_\ast \mathbf{w}_{\ast,t,i}\right)\mathbf{w}_{t,i}^\top  + \tfrac{\beta}{n m_{out}}\sum_{i=1}^n(\mathbf{X}^{out}_{t,i})^\top \mathbf{z}^{out}_{t,i}\mathbf{w}_{t,i}^\top \label{ttt3}  \\
    &\quad + \tfrac{\beta\alpha}{n} \sum_{i=1}^n \mathbf{B}_t \mathbf{B}_t^\top \left(\mathbf{I}_d - \mathbf{\Sigma}_{t,i}^{in}\right) \left(\mathbf{B}_t  \mathbf{w}_{t} - \mathbf{{B}}_\ast \mathbf{w}_{\ast,t,i}\right)\mathbf{w}_{t,i}^\top  + \tfrac{\beta \alpha}{n m_{in}}\sum_{i=1}^n\mathbf{B}_t\mathbf{B}_t^\top (\mathbf{X}^{in}_{t,i})^\top \mathbf{z}^{in}_{t,i}\mathbf{w}_{t,i}^\top  \label{ttt4} \\
    &= \mathbf{B}_{t+1}^{pop} + \beta(\mathbf{E}_1 +  \mathbf{E}_2 + \mathbf{E}_3 + \mathbf{E}_4) \label{tutt}
\end{align}
where $\mathbf{B}_{t,pop} \coloneqq \mathbf{{B}}_t  - \beta\deld_t\frac{1}{n}\sum_{i=1}^n(  \mathbf{{B}}_t  \mathbf{w}_{t} -  \mathbf{{B}}_\ast\mathbf{w}_{\ast,t,i} )(\del_t\w+ \alpha \mathbf{B}_t^\top\mathbf{B}_\ast \mathbf{w}_{\ast,t, i} )^\top $ denotes the update of the representation in the infinite sample case, and
$\mathbf{E}_1, \mathbf{E}_2, \mathbf{E}_3$ and $\mathbf{E}_4$ are the finite-sample error terms in lines \eqref{ttt1}, \eqref{ttt2}, \eqref{ttt3} and \eqref{ttt4}, respectively. 
From \eqref{tutt} and the triangle inequality, we can compute the final bound.
\begin{align}
    \|\del_{t+1}\|_2 &\leq  \|\mathbf{I}_k- \alpha \mathbf{B}_{t,pop}^\top \mathbf{B}_{t,pop}\|_2 + 2\beta \alpha \|\mathbf{B}_{t,pop}^\top(\mathbf{E}_1+ \mathbf{E}_2+\mathbf{E}_3+\mathbf{E}_4)\|_2 + \beta^2 \alpha \|\mathbf{E}_1+ \mathbf{E}_2+\mathbf{E}_3+\mathbf{E}_4\|_2^2  \label{tuttt}
\end{align}
Note that from Corollary \ref{cor:reg} and the fact that $\|\b\|_2 \leq 1.1/\sqrt{\alpha}$ by $A_3(t)$, and $\beta, \alpha$ are sufficiently small, we have that $\|\mathbf{B}_{t+1}^{pop}\|_2\leq \frac{1.1}{\sqrt{\alpha}}$. Also, clearly $\mathbf{B}_{t+1}^{pop}\in \mathbb{R}^{d \times k}$. Therefore by  the concentration results in 
Lemma \ref{lem:gen1} and the triangle and Cauchy-Schwarz inequalities, we have, for an absolute constant $c$,
\begin{align}
\max_{i\in[n]}\|\mathbf{w}_{t,i}\|_2 &\leq \|\del_t\|_2 \|\mathbf{w}_t\|_2 +c \sqrt{\alpha}L_{\max} + {\delta}_{m_{in},k} \|\mathbf{w}_t\|_2 + c\sqrt{\alpha}\sigma {\delta}_{m_{in},k}  \nonumber \\
    \|\mathbf{B}_{t,pop}^\top \mathbf{E}_1\|_2 &\leq \tfrac{c}{\alpha}\|\del_t\|_2 {\|\mathbf{w}_t\|_2^2\bar{\delta}_{m_{in},k}} + \tfrac{c}{\sqrt{\alpha}}\|\del_t\|_2 {L_{\max}\|\mathbf{w}_t\|_2\bar{\delta}_{m_{in},k}}{}+  \tfrac{c}{\sqrt{\alpha}}\|\del_t\|_2 {\|\mathbf{w}_t\|_2 L_{\max}\bar{\delta}_{m_{in},k}} \nonumber \\
    &\quad + \|\del_t\|_2 {L_{\max}^2\bar{\delta}_{m_{in},k}} \nonumber \\
   & \leq  c L_{\max}^2 {\bar{\delta}_{m_{in},k}} \nonumber \\
    \|\mathbf{B}_{t,pop}^\top \mathbf{E}_2\|_2 &\leq \tfrac{c}{\sqrt{\alpha}}\|\del_t\|_2  \|\mathbf{w}_t\|_2  \sigma {\bar{\delta}_{m_{in},k}}{} + c\|\del_t\|_2 L_{\max}\sigma {\bar{\delta}_{m_{in},k}}{} \leq c L_{\max}\sigma {\bar{\delta}_{m_{in},k}}{} \nonumber \\
   \|\mathbf{B}_{t,pop}^\top \mathbf{E}_3\|_2  &\leq  \tfrac{c}{\alpha}\bar{\delta}_{m_{out},k}\left(\|\del_t\|_2 \|\mathbf{w}_t\|_2 + \sqrt{\alpha}L_{\max} + {\delta}_{m_{in},k} \|\mathbf{w}_t\|_2 + \sqrt{\alpha}\sigma {\delta}_{m_{in},k} \right)^2 \nonumber \\
   &\quad + \tfrac{c}{\sqrt{\alpha}}{\bar{\delta}_{m_{out},k}}{}L_{\max} \left(\|\del_t\|_2 \|\mathbf{w}_t\|_2 + \sqrt{\alpha}L_{\max} + {\delta}_{m_{in},k} \|\mathbf{w}_t\|_2 + \sqrt{\alpha}\sigma {\delta}_{m_{in},k} \right) \nonumber \\
   &\quad + \tfrac{c}{\sqrt{\alpha}}\sigma {\bar{\delta}_{m_{out},k}}{}\left(\|\del_t\|_2 \|\mathbf{w}_t\|_2 + \sqrt{\alpha}L_{\max} + {\delta}_{m_{in},k} \|\mathbf{w}_t\|_2 + \sqrt{\alpha}\sigma {\delta}_{m_{in},k} \right) \nonumber \\
   &\leq  {c\bar{\delta}_{m_{out},k}}{}(L_{\max}+\sigma) \left(L_{\max} + \sigma {\delta}_{m_{in},k} \right) \nonumber \\
  \|\mathbf{B}_{t,pop}^\top \mathbf{E}_4\|_2 &\leq c{\delta}_{m_{in},k} \left( \tfrac{\|\mathbf{w}_t\|_2}{{\alpha}} + \tfrac{L_{\max}}{\sqrt{\alpha}}\right)\left(\tfrac{\|\del_t\|_2 \|\mathbf{w}_t\|_2}{\sqrt{n}} + L_{\max}\tfrac{\sqrt{\alpha}}{\sqrt{n}} + {\delta}_{m_{in},k} \|\mathbf{w}_t\|_2 + \sqrt{\alpha}\sigma {\delta}_{m_{in},k}\right) \nonumber \\
  &\quad + \tfrac{c}{\sqrt{\alpha}}( \|\del_t\|_2 \|\mathbf{w}_t\|_2 \sigma {\bar{\delta}_{m_{in},k}}{} + L_{\max}\sigma {\bar{\delta}_{m_{in},k}} + \sigma \delta^2(m_{in},k) \tfrac{\|\mathbf{w}_t\|_2}{\sqrt{\alpha}} + \sigma^2 \delta^2(m_{in},k) ) \nonumber \\
  &\leq c{\delta}_{m_{in},k}  {L_{\max}}{}( \tfrac{L_{\max}}{\sqrt{n}} + L_{\max}{\delta}_{m_{in},k}  + \sigma {\delta}_{m_{in},k}) \nonumber \\
  &\quad + c( L_{\max}\sigma {\bar{\delta}_{m_{in},k}} + L_{\max} \sigma \delta^2(m_{in},k) + \sigma^2 \delta^2(m_{in},k) ) \nonumber \\
  &\leq   c {\bar{\delta}_{m_{out},k}}{}(L_{\max}+\sigma) \left(L_{\max} + \sigma {\delta}_{m_{in},k} \right) + L_{\max}^2(\delta^2(m_{in},k) +  \bar{\delta}_{m_{in},k})  \nonumber \\
  &\quad + cL_{\max}\sigma ({\delta}_{m_{in},k}^2  +  {\bar{\delta}_{m_{in},k}}) +  c\sigma^2 {\delta}_{m_{in},k}^2 
 \end{align}
with probability at least $1 - \frac{1}{\poly(n)}$. Thus
\begin{align}
    \|\mathbf{B}_{t,pop}^\top(\mathbf{E}_1+ \mathbf{E}_2+ \mathbf{E}_3 + \mathbf{E}_4)\|_2 &\leq \frac{c\bar{\delta}_{m_{out},k}}{\sqrt{ n}}(L_{\max}+\sigma) \left(L_{\max} + \sigma {\delta}_{m_{in},k} \right)+ c (L_{\max}^2+ L_{\max}\sigma) \frac{{\delta}_{m_{in},k}}{\sqrt{n}} \nonumber \\
    &\quad + L_{\max}^2{\delta}_{m_{in},k}^2 + L_{\max}\sigma {\delta}_{m_{in},k}^2  +    \sigma^2 {\delta}_{m_{in},k}^2 
\end{align}


Similarly,
\begin{align}
    \|\mathbf{E}_1\|_2 &\leq  (\tfrac{\|\del_t\|_2 \|\mathbf{w}_t\|_2}{\sqrt{\alpha}} + \dist_t L_{\max} + \|\del_t\|_2L_{\max})({\delta}_{m_{in},k} \|\mathbf{w}_t\|_2 + \sqrt{\alpha} {\delta}_{m_{in},k}L_{\max})\tfrac{1}{\sqrt{n}}  \nonumber \\
    &\leq \sqrt{\alpha}L_{\max}^2 \frac{{\delta}_{m_{in},k}}{\sqrt{n}}  \nonumber \\
    \|\mathbf{E}_2\|_2 &\leq  (\tfrac{\|\del_t\|_2 \|\mathbf{w}_t\|_2}{\sqrt{\alpha}} + \dist_t L_{\max} + \|\del_t\|_2 L_{\max})\sqrt{\alpha}\sigma \tfrac{{\delta}_{m_{in},k}}{\sqrt{n}}  \nonumber \\
    &\leq \sqrt{\alpha} L_{\max}\sigma \frac{{\delta}_{m_{in},k}}{\sqrt{n}} \nonumber \\
    \|\mathbf{E}_3\|_2 &\leq  \bar{\delta}_{m_{out},d}\sqrt{\alpha}(L_{\max}+ \sigma)(L_{\max}+\sigma {\delta}_{m_{in},k})\nonumber \\
    \|\mathbf{E}_4\|_2 
    &\leq 
    \sqrt{\alpha}\bigg({\delta}_{m_{in},k} \left( \tfrac{\|\mathbf{w}_t\|_2}{{\alpha}} + \tfrac{L_{\max}}{\sqrt{\alpha}}\right)\big(\tfrac{\|\del_t\|_2 \|\mathbf{w}_t\|_2}{\sqrt{n}} + \tfrac{\sqrt{\alpha}L_{\max}}{\sqrt{n}} + {\delta}_{m_{in},k} \|\mathbf{w}_t\|_2 + \sqrt{\alpha}\sigma {\delta}_{m_{in},k}\big) \nonumber \\
  &\quad + \big( \|\del_t\|_2 \|\mathbf{w}_t\|_2 \sigma \tfrac{\bar{\delta}_{m_{in},k}}{\sqrt{\alpha}} + L_{\max}\sigma {\bar{\delta}_{m_{in},k}} + \sigma {\delta}_{m_{in},k}^2 \tfrac{\|\mathbf{w}_t\|_2}{\sqrt{\alpha}} + \sigma^2 {\delta}_{m_{in},k}^2 \big)\bigg) \nonumber \\
  &\leq   \sqrt{\alpha} {\delta}_{m_{in},k}  {L_{\max}}{}\big( \tfrac{L_{\max}}{\sqrt{n}} + L_{\max}{\delta}_{m_{in},k}  + \sigma {\delta}_{m_{in},k}\big)  +   \sqrt{\alpha}\big( L_{\max}\sigma \frac{{\delta}_{m_{in},k}}{\sqrt{n}} + L_{\max} \sigma {\delta}_{m_{in},k}^2 + \sigma^2 {\delta}_{m_{in},k}^2 \big) \nonumber \\
  &\leq     c\sqrt{\alpha}{\bar{\delta}_{m_{out},k}}(L_{\max}+\sigma) \left(L_{\max} + \sigma {\delta}_{m_{in},k} \right) +   \sqrt{\alpha} L_{\max}^2({\delta}_{m_{in},k}^2+ \bar{\delta}_{m_{in},k})  \nonumber \\
  &\quad +   \sqrt{\alpha} L_{\max}\sigma ({\delta}_{m_{in},k}^2  +  {\bar{\delta}_{m_{in},k}}) +   \sqrt{\alpha} \sigma^2 {\delta}_{m_{in},k}^2
    \nonumber 
\end{align}
thus
\begin{align}
    \|\mathbf{E}_1+ \mathbf{E}_2+\mathbf{E}_3+\mathbf{E}_4\|_2 &\leq
    \sqrt{\alpha}(L_{\max}^2 + L_{\max}\sigma)({\bar{\delta}_{m_{in},k}}+{\delta}_{m_{in},k}^2) \nonumber \\
    &\quad + \bar{\delta}_{m_{out},d}\sqrt{\alpha}(L_{\max}+ \sigma)(L_{\max}+\sigma {\delta}_{m_{in},k}) + \sqrt{\alpha }\sigma^2{\delta}_{m_{in},k}^2  \nonumber 
\end{align}
Now, from \eqref{tuttt} and the triangle inequality, we can compute the final bound.
\begin{align}
    \|&\del_{t+1}\|_2 \nonumber \\
    &\leq \|\mathbf{I}_k- \alpha \mathbf{B}_{t,pop}^\top \mathbf{B}_{t,pop}\|_2  \nonumber \\
    &\quad + c \beta \alpha \left({\bar{\delta}_{m_{out},k}}(L_{\max}+\sigma) \left(L_{\max} + \sigma {\delta}_{m_{in},k} \right)+  (L_{\max}^2+ L_{\max}\sigma) ({\bar{\delta}_{m_{in},k}}{} +{\delta}^2_{m_{in},k)}  +    \sigma^2 {\delta}_{m_{in},k}^2\right) \nonumber \\
    &\quad +c \beta^2 \alpha^2 \left((L_{\max}^2 + L_{\max}\sigma)({\bar{\delta}_{m_{in},k}} +{\delta}_{m_{in},k}^2) + \bar{\delta}_{m_{out},d}(L_{\max}+ \sigma)(L_{\max}+\sigma {\delta}_{m_{in},k}) + \sigma^2{\delta}_{m_{in},k}^2 \right)^2 \nonumber \\
    &\leq \|\mathbf{I}_k- \alpha \mathbf{B}_{t,pop}^\top \mathbf{B}_{t,pop}\|_2  \nonumber \\
    &\quad + c \beta \alpha \left({\bar{\delta}_{m_{out},k}}(L_{\max}+\sigma) \left(L_{\max} + \sigma {\delta}_{m_{in},k} \right)+  (L_{\max}^2+ L_{\max}\sigma) ({\bar{\delta}_{m_{in},k}} +{\delta}_{m_{in},k}^2)  +    \sigma^2 {\delta}_{m_{in},k}^2\right) \nonumber \\
    &\quad + c\beta^2 \alpha^2  \bar{\delta}_{m_{out},d}^2(L_{\max}+ \sigma)^2(L_{\max}+\sigma {\delta}_{m_{in},k})^2 \nonumber \\
    &= \|\mathbf{I}_k- \alpha \mathbf{B}_{t,pop}^\top \mathbf{B}_{t,pop}\|_2  + \beta \alpha \zeta_2 \nonumber \\
    &\leq (1 - 0.5\beta \alpha E_0 \mu_{\ast}^2)\|\del_t\|_2 + c\beta^2 \alpha^2 L_\ast^4 \dist_t^2 + \beta \alpha \zeta_2
\end{align}
where the last line follows from Lemma 
\ref{lem:reg} (note that all conditions for that lemma are satisfied by $\|\mathbf{I}_k- \alpha \mathbf{B}_{t,pop}^\top \mathbf{B}_{t,pop}\|_2$), and \begin{align}
    \zeta_2 &=  O\big((L_{\max}+\sigma)^2 \bar{\delta}_{m_{out},k}+  (L_{\max}^2+ L_{\max}\sigma) (\bar{\delta}_{m_{in},k} +\delta^2_{m_{in},k})  +    \sigma^2 \delta^2_{m_{in},k} + \beta \alpha(L_{\max}+ \sigma)^4 \bar{\delta}^2_{m_{out},d}\big) \label{zeta2}
\end{align}
\end{proof}

\begin{cor}[FO-ANIL, Finite samples $A_3(t+1)$]\label{cor:foanil_rep}
Suppose that $A_2(t+1)$ and $A_3(t)$ hold. Then
\begin{align}
    \|\del_{t+1}\|_2\leq \tfrac{1}{10}
\end{align}
\end{cor}
\begin{proof}
From $A_2(t+1)$ we have
\begin{align}
   \|\del_{t+1}\|_2 &\leq  (1 - 0.5\beta \alpha E_0 \mu_{\ast}^2)\|\del_t\|_2 + c\beta^2 \alpha^2 L_\ast^4 \dist_t^2 + \beta \alpha \zeta_2 \nonumber \\
   &\leq (1 - 0.5\beta \alpha E_0 \mu_{\ast}^2)\tfrac{1}{10}+ c\beta^2 \alpha^2 L_\ast^4  + \beta \alpha \zeta_2 \nonumber \\
   &\leq  \tfrac{1}{10} - 0.25\beta \alpha E_0 \mu_{\ast}^2 + c\beta^2 \alpha^2 L_\ast^4   \label{zeta}\\
   &\leq \tfrac{1}{10} \label{tnth}
\end{align}
where \eqref{zeta} follows as long as $\zeta_2 \leq 0.25 E_0 \mu_{\ast}^2$, and  \eqref{tnth} follows since $\beta\leq c'\alpha E_0^3 \kappa^{-4}$.
\end{proof}

\begin{lemma}[FO-ANIL, Finite samples, $A_4(t+1)$]\label{lem:contract_app_fs}
Suppose $A_1(t),A_3(t)$ and $A_5(t)$ hold. Then $A_4(t+1)$ holds, i.e.
\begin{align}
\| \mathbf{{B}}_{\ast,\perp}^\top \mathbf{{B}}_{t+1} \|_2 &\leq  (1 - 0.5 \beta\alpha E_0 \mu^2 
)
\|\mathbf{{B}}_{\ast,\perp}^\top \mathbf{{B}}_{t} \|_2 +  \beta \sqrt{\alpha} \zeta_4
\nonumber 
\end{align} 
where $\zeta_4 = O( (L_{\max}+\sigma)(L_{\max}+ \sigma \delta_{m_{in},k})\bar{\delta}_{m_{out},d})$
    with probability at least $1- \frac{1}{\poly(n)}$.
\end{lemma}

\begin{proof} 
Using \eqref{ttt4}, we have
\begin{align}
    \mathbf{\hat{B}}_{\ast,\perp}^\top \mathbf{B}_{t+1} &= \mathbf{\hat{B}}_{\ast,\perp}^\top\mathbf{B}_t \left(\mathbf{I}_k - \frac{\beta}{n}\sum_{i=1}^n \mathbf{w}_{t,i}\mathbf{w}_{t,i}^\top\right) \nonumber \\
    &\quad + \beta\underbrace{\frac{1}{n} \sum_{i=1}^n \mathbf{{B}}_{\ast,\perp}^\top\left(\mathbf{I}_d -\mathbf{\Sigma}_{t,i}^{out}\right)( \mathbf{B}_t  \mathbf{w}_{t,i} - \mathbf{{B}}_\ast \mathbf{w}_{\ast,t,i}) \mathbf{w}_{t,i}^\top}_{=:\mathbf{E}_1} + \beta\underbrace{\frac{1}{nm_{out}}\sum_{i=1}^n \mathbf{{B}}_{\ast,\perp}^\top (\mathbf{X}^{out}_{t,i})^\top \mathbf{z}^{out}_{t,i}\mathbf{w}_{t,i}^\top}_{=:\mathbf{E}_2} 
\end{align}
Next, we can use the concentration results in Lemma \ref{lem:gen1} to show that all of the following inequalities hold with probability at least $1 - \frac{1}{\poly(n)}$
\begin{align}
    \max_{i\in [n]} \|\mathbf{w}_{t,i}\|_2&\leq \|\del_t\|_2 \|\w\|_2 + c \sqrt{\alpha}L_{\max} + \delta_{m_{in},k}\|\w\|_2 + c \sqrt{\alpha}\sigma \delta_{m_{in},k} \leq c' \sqrt{\alpha}L_{\max}+ c \sqrt{\alpha}\sigma \delta_{m_{in},k} \nonumber \\
    \|\mathbf{E}_1\|_2 &\leq c L_{\max}\max_{i\in [n]} \|\mathbf{w}_{t,i}\|_2  \bar{\delta}_{m_{out},d}\\
     \|\mathbf{E}_2\|_2  &\leq c \sigma \max_{i\in [n]} \|\mathbf{w}_{t,i}\|_2  \bar{\delta}_{m_{out},d}
\end{align}
Thus we have 
\begin{align}
    \| \mathbf{\hat{B}}_{\ast,\perp}^\top \mathbf{{B}}_{t+1} \|_2 &\leq \bigg\| \mathbf{\hat{B}}_{\ast,\perp}^\top \mathbf{{B}}_{t}(\mathbf{I}_k - \frac{\beta}{n}\sum_{i=1}^n \mathbf{w}_{t,i}\mathbf{w}_{t,i}^\top) \bigg\|_2 + \beta \sqrt{\alpha} \zeta_4 \nonumber 
\end{align}
where $\zeta_4 = O( (L_{\max}+\sigma)(L_{\max}+ \sigma \delta_{m_{in},k})\bar{\delta}_{m_{out},d})$
with probability at least $1 -\frac{1}{\poly(n)}$. Next, recall from Lemma \ref{lem:contract_app_fs} that
\begin{align*}
     \sigma_{\max}\left(\frac{1}{n} \sum_{i=1}^n \mathbf{w}_{t,i}\mathbf{w}_{t,i}^\top \right)\leq L^2 &\coloneqq  2\left(\|\del_t\|_2\|\mathbf{w}_t\|_2 +  \sqrt{\alpha} L_\ast + {\delta}_{m_{in},k}(\|\w\|_2  + \sqrt{{\alpha}}L_{\max}  + \sqrt{{{\alpha}}}\sigma ) \right)^2 \\
  \sigma_{\min}\left(\frac{1}{n} \sum_{i=1}^n \mathbf{w}_{t,i}\mathbf{w}_{t,i}^\top \right)  \geq \mu^2 &\coloneqq  0.9 \alpha E_0 \mu_\ast^2 - 2.2 \sqrt{\alpha}\|\mathbf{w}_t\|_2 \|\del_t\|_2 \eta_\ast  \nonumber \\
  &\quad - 2 \|\del_t\|_2 \|\mathbf{w}_t\|_2\bar{\delta}_{m_{in},k}(\|\w\|_2  + \sqrt{{\alpha}}L_{\max}  + \sqrt{{{\alpha}}}\sigma ) \nonumber \\
   &\quad - 2.2 \sqrt{\alpha} \bar{\delta}_{m_{in},k}(\|\w\|_2 + \sqrt{{\alpha}}L_{\ast}  + \sqrt{{{\alpha}}}\sigma )L_{\max}  
 \end{align*}
 with probability at least $1 -\frac{1}{\poly(n)}$. Apply inductive hypotheses $A_1(t)$ and $A_3(t)$ to obtain
 \begin{align}
     \sigma_{\max}\left(\frac{1}{n} \sum_{i=1}^n \mathbf{w}_{t,i}\mathbf{w}_{t,i}^\top \right)&\leq  4 \alpha L_\ast^2 + 4 \alpha( L_{\max}+\sigma)^2{\delta}_{m_{in},k}^2 \leq 12 \alpha L_\ast^2\nonumber \end{align}
     by choice of $m_{in}= \Omega((k+\log(n))(L_{\max}+\sigma)^2)$ . This means that we have $\beta \leq  \sigma_{\max}\left(\frac{1}{n} \sum_{i=1}^n \mathbf{w}_{t,i}\mathbf{w}_{t,i}^\top \right)^{-1}$ since we have chosen $\beta = O(\alpha \kappa_\ast^{-4})$. Also, we have 
     \begin{align}
  \sigma_{\min}\left(\frac{1}{n} \sum_{i=1}^n \mathbf{w}_{t,i}\mathbf{w}_{t,i}^\top \right)  &\geq 0.8 \alpha E_0 \mu_\ast^2    
   - 2.3 {\alpha} L_{\max}( L_{\max}  + \sigma ) \bar{\delta}_{m_{in},k} \geq 0.5 \alpha E_0 \mu_\ast^2  \label{eqref}
 \end{align}
 where the last inequality follows since $m_{in}= \Omega\left({\tfrac{k^2}{n}(k\kappa_{\ast}^4 + \kappa_{\ast}^2 \sigma^2 \mu_\ast^{-2}} )\right)$, recalling that $L_{\max}\leq c \sqrt{k}L_\ast$.
 Thus, using the above and Weyl's inequality with $\beta \leq  \sigma_{\max}\left(\frac{1}{n} \sum_{i=1}^n \mathbf{w}_{t,i}\mathbf{w}_{t,i}^\top \right)^{-1}$, we obtain:
 \begin{align}
   \| \mathbf{\hat{B}}_{\ast,\perp}^\top \mathbf{{B}}_{t+1} \|_2 &\leq   \bigg\| \mathbf{\hat{B}}_{\ast,\perp}^\top \mathbf{{B}}_{t}\bigg\|_2(1 - 0.5 \beta \alpha E_0 \alpha^2) + \beta \sqrt{\alpha} \zeta_4 
 \end{align}
\end{proof}

\subsection{Exact ANIL}

\begin{lemma}[Exact ANIL FS representation concentration I]\label{lem:exactanil_rep1}
For Exact ANIL, consider any $t \in [T]$. With probability at least $ 1 - \frac{1}{\poly(n)}- \frac{1}{\poly(m_{in})} - c e^{-90k}$,
\begin{align}
    \|\mathbf{\hat{G}}_{\mathbf{B},t} - \mathbf{{G}}_{\mathbf{B},t}\|_2 &=\sqrt{\alpha}\zeta_{2,a},
\end{align}
where 
\begin{align}
    \zeta_{2,a}= O\Bigg(\bigg( \tfrac{1}{{m_{in}}}\bigg(&  \tfrac{(L_{\max}+\sigma) L_\ast}{\kappa_\ast^2}\bigg)+\tfrac{1}{\sqrt{m_{in}}}\bigg( L_{\max} (L_{\max}+\sigma)(\sqrt{k}+\sqrt{\log(n)})  \bigg)  \nonumber \\
&\quad +\tfrac{1}{\sqrt{m_{out}}}\bigg( L_{\max}(L_{\max}+\sigma)(\sqrt{k}+\sqrt{\log(n)})  \bigg) \nonumber \\
+\tfrac{1}{\sqrt{nm_{in}}}\bigg(&L_{\max} (L_{\max}+\sigma)(k \sqrt{d\log(nm_{in})} + k\log(nm_{in}) + \sqrt{d}\log^{1.5}(nm_{in}) +\log^2(nm_{in}))\nonumber \\
&\quad + \sigma^2 (\sqrt{kd} +  \sqrt{d}\log(nm_{in}) +\log^{1.5}(nm_{in}))  + L_{\max}(L_{\max}+\sigma)\sqrt{d}\bigg) \nonumber \\
+\tfrac{1}{\sqrt{nm_{out}}}\bigg(& L_{\max}(L_{\max} + \sigma)\sqrt{d} + \sigma^2 (\tfrac{\sqrt{d}}{\sqrt{m_{in}}} + \sqrt{k})\bigg) \bigg)\Bigg) \nonumber 
\end{align}
\end{lemma}
\begin{proof}
Let $\mathbf{q}_{t,i}\coloneqq \b \w - \mathbf{B}_\ast \mathbf{w}_{\ast,t,i}$.
First recall that $\mathbf{\hat{G}}_{\mathbf{B},t} = \frac{1}{n}\sum_{i=1}^n\nabla_{\mathbf{B}} \hat{F}_{t,i}(\mathbf{B}_t, \mathbf{w}_t)$, where
\begin{align}
 \nabla_{\mathbf{B}} \hat{F}_{t,i}(\mathbf{B}_t, \mathbf{w}_t)&= (\deldin)^\top \tfrac{1}{m_{out}} (\mathbf{X}_{t,i}^{out})^\top \mathbf{\hat{v}}_{t,i}\mathbf{w}_t^\top - \alpha \tfrac{1}{m_{out}} (\mathbf{X}_{t,i}^{out})^\top \mathbf{\hat{v}}_{t,i}\mathbf{q}_{t,i}^\top \sin \mathbf{B}_t    \nonumber \\
 &\quad - \alpha \sin \mathbf{q}_{t,i} \mathbf{\hat{v}}_{t,i}^\top \tfrac{1}{m_{out}} \mathbf{X}_{t,i}^{out}\b + \tfrac{\alpha^2}{m_{in}m_{out}}(\mathbf{X}_{t,i}^{out})^\top \mathbf{\hat{v}}_{t,i}(\mathbf{z}_{t,i}^{in})^
 \top \mathbf{X}_{t,i}^{in} \mathbf{B}_t  \nonumber \\
 &\quad  +\tfrac{\alpha^2}{m_{in} m_{out}}(\mathbf{X}_{t,i}^{in})^\top\mathbf{z}_{t,i}^{in} \mathbf{\hat{v}}_{t,i}^\top\mathbf{X}_{t,i}^{out} \mathbf{B}_t \nonumber
\end{align}
where $\mathbf{\hat{v}}_{t,i} = \mathbf{X}_{t,i}^{out}\deldin\mathbf{q}_{t,i}
 + \tfrac{\alpha}{m_{in}} \mathbf{X}_{t,i}^{out} \b\b^\top \mathbf{X}_{t,i}^{in}\mathbf{z}_{t,i}^{in} - \zout$.
Also, $\mathbf{{G}}_{\mathbf{B},t} = \frac{1}{n}\sum_{i=1}^n\nabla_{\mathbf{B}} {F}_{t,i}(\mathbf{B}_t, \mathbf{w}_t)$, where
\begin{align}
 \nabla_{\mathbf{B}}{F}_{t,i}(\mathbf{B}_t, \mathbf{w}_t) &=\deld_t  \mathbf{v}_{t,i}\mathbf{w}_t^\top - \alpha  \mathbf{v}_{t,i}\mathbf{q}_{t,i}^\top \b - \alpha  \mathbf{q}_{t,i} \mathbf{v}_{t,i}^\top \b  \nonumber
\end{align}
and $\mathbf{v}_{t,i}=\deld_t \mathbf{q}_{t,i}$. Thus, \begin{align}
    \|&\mathbf{\hat{G}}_{\mathbf{B},t} - \mathbf{{G}}_{\mathbf{B},t}\|_2 \nonumber \\
   &\leq \bigg\|\underbrace{\frac{1}{n}\sum_{i=1}^n (\deldin)^\top \tfrac{1}{m_{out}} \mathbf{X}_{t,i}^\top \mathbf{\hat{v}}_{t,i}\mathbf{w}_t^\top- \deld_t  \mathbf{v}_{t,i}\mathbf{w}_t^\top}_{=: \mathbf{E}_1} \bigg\|_2  \nonumber \\
   &\quad + \alpha \bigg\|\underbrace{\frac{1}{n}\sum_{i=1}^n  \tfrac{1}{m_{out}} (\mathbf{X}_{t,i}^{out})^\top \mathbf{\hat{v}}_{t,i}\mathbf{q}_{t,i}^\top \sin \mathbf{B}_t -  \mathbf{v}_{t,i}\mathbf{q}_{t,i}^\top \b}_{=: \mathbf{E}_2}   \bigg\|_2  \nonumber \\
    &\quad + \alpha \bigg\|\underbrace{\frac{1}{n}\sum_{i=1}^n  \sin \mathbf{q}_{t,i} \mathbf{\hat{v}}_{t,i}^\top \tfrac{1}{m_{out}} \mathbf{X}_{t,i}^{out}\b   -   \mathbf{q}_{t,i} \mathbf{v}_{t,i}^\top \b }_{=: \mathbf{E}_3}   \bigg\|_2 +\alpha \bigg\|\underbrace{\frac{1}{nm_{in}m_{out}}\sum_{i=1}^n  (\mathbf{X}_{t,i}^{out})^\top \mathbf{\hat{v}}_{t,i}(\mathbf{z}_{t,i}^{in})^\top \mathbf{X}_{t,i}^{in} \mathbf{B}_t}_{=: \mathbf{E}_4}   \bigg\|_2 \nonumber \\
 &\quad + {\alpha}\bigg\| \underbrace{\frac{1}{nm_{in}m_{out}}\sum_{i=1}^n (\mathbf{X}_{t,i}^{in})^\top\mathbf{z}_{t,i}^{in} \mathbf{\hat{v}}_{t,i}^\top\mathbf{X}_{t,i}^{out} \mathbf{B}_t}_{=: \mathbf{E}_5}  \bigg\|_2
\end{align}

We will further decompose each of the above terms into terms for which we can apply concentration results from Lemmas \ref{lem:gen1} and \ref{lem:gen2}.
First we bound $\|\mathbf{E}_1\|_2$. We have
\begin{align}
    \|\mathbf{E}_1\|_2&= \bigg\|{\frac{1}{n}\sum_{i=1}^n (\deldin)^\top \tfrac{1}{m_{out}} \mathbf{X}_{t,i}^\top \mathbf{\hat{v}}_{t,i}\mathbf{w}_t^\top- \deld_t  \mathbf{v}_{t,i}\mathbf{w}_t^\top} \bigg\|_2 \nonumber \\
    &= \bigg\|{\frac{1}{n}\sum_{i=1}^n (\deldin)^\top \sout \deldin \mathbf{q}_{t,i}
    \mathbf{w}_t^\top- \deld_t \deld_t \mathbf{q}_{t,i}\mathbf{w}_t^\top} \bigg\|_2 \nonumber \\
    &\quad+\bigg\|{\frac{1}{n}\sum_{i=1}^n \frac{\alpha}{m_{in}}(\deldin)^\top \sout \b \b^\top (\mathbf{X}_{t,i}^{in})^\top\mathbf{z}_{t,i}^{in}
    \mathbf{w}_t^\top} \bigg\|_2 +
     \bigg\|\frac{1}{n}\sum_{i=1}^n (\deldin)^\top \tfrac{1}{m_{out}}(\mathbf{X}_{t,i}^{out})^\top\mathbf{z}_{t,i}^{out}\mathbf{w}_t^\top \bigg\|_2
    \nonumber \\
    &\leq \bigg\|\underbrace{\frac{1}{n}\sum_{i=1}^n \sout \mathbf{q}_{t,i}
    \mathbf{w}_t^\top-  \mathbf{q}_{t,i}\mathbf{w}_t^\top}_{=:\mathbf{E}_{1,1}} \bigg\|_2 + \bigg\|\underbrace{\frac{1}{n}\sum_{i=1}^n \alpha \sin \b \b^\top \sout \mathbf{q}_{t,i}
    \mathbf{w}_t^\top-  \alpha \b\b^\top\mathbf{q}_{t,i}\mathbf{w}_t^\top}_{=:\mathbf{E}_{1,2}} \bigg\|_2 \nonumber \\ 
    &\quad + \bigg\|\underbrace{\frac{1}{n}\sum_{i=1}^n \alpha  \sout \b \b^\top \sin \mathbf{q}_{t,i}
    \mathbf{w}_t^\top-  \alpha \b\b^\top\mathbf{q}_{t,i}\mathbf{w}_t^\top}_{=:\mathbf{E}_{1,3}} \bigg\|_2   \nonumber \\
    &\quad +  \bigg\|\underbrace{\frac{1}{n}\sum_{i=1}^n \alpha^2 \sin\b \b^\top  \sout \b \b^\top \sin \mathbf{q}_{t,i}
    \mathbf{w}_t^\top-  \alpha^2 \b\b^\top \b\b^\top\mathbf{q}_{t,i}\mathbf{w}_t^\top}_{=:\mathbf{E}_{1,4}} \bigg\|_2 \nonumber \\
    &\quad+\bigg\|\underbrace{\frac{1}{n}\sum_{i=1}^n \frac{\alpha}{m_{in}} \sout \b \b^\top (\mathbf{X}_{t,i}^{in})^\top\mathbf{z}_{t,i}^{in}
    \mathbf{w}_t^\top}_{=:\mathbf{E}_{1,5}} \bigg\|_2 \nonumber \\
    &\quad +\bigg\| \underbrace{\frac{1}{n}\sum_{i=1}^n \frac{\alpha^2}{m_{in}}\sin \b \b^\top \sout \b \b^\top (\mathbf{X}_{t,i}^{in})^\top\mathbf{z}_{t,i}^{in}
    \mathbf{w}_t^\top }_{=:\mathbf{E}_{1,6}} \bigg\|_2 \nonumber \\
    &\quad +
     \bigg\|\underbrace{\frac{1}{n}\sum_{i=1}^n  \tfrac{1}{m_{out}}(\mathbf{X}_{t,i}^{out})^\top\mathbf{z}_{t,i}^{out}\mathbf{w}_t^\top}_{=:\mathbf{E}_{1,7}} \bigg\|_2+
    \alpha \bigg\|\underbrace{\frac{1}{n}\sum_{i=1}^n \sin \b \b^\top \tfrac{1}{m_{out}}(\mathbf{X}_{t,i}^{out})^\top\mathbf{z}_{t,i}^{out}\mathbf{w}_t^\top}_{=:\mathbf{E}_{1,8}} \bigg\|_2
    \nonumber 
\end{align}
    Note that after factoring out trailing $\mathbf{w}_t$'s where necessary, each of the above matrices is in the form that is bounded in Lemma \ref{lem:gen1} or Lemma \ref{lem:gen2}. We apply the bounds from those lemmas and use $\alpha \|\b\|_2^2 = O(1)$, $\|\mathbf{w}_t\|_2 = O(\sqrt{\alpha}\min(1, \eta_\ast^2/\mu_\ast^2)\eta_\ast)$, and $\max_{i\in [n]}\|\mathbf{q}_{t,i}\|_2 = O(L_{\max})$ to obtain that each of the following bounds hold with probability at least $1 - \frac{1}{\poly(n)} - \frac{1}{\poly(m_{in})} - c'e^{-90k}$, for some absolute constants $c,c'$.
\begin{align}
    \|\mathbf{E}_{1,1}\|_2 &\leq  c\sqrt{\alpha} \tfrac{L_{\max} L_\ast}{\kappa_\ast^2} \bar{\delta}_{m_{out},d} \nonumber \\
    \|\mathbf{E}_{1,2}\|_2 +  \|\mathbf{E}_{1,3}\|_2 &\leq c\sqrt{\alpha} \tfrac{L_{\max} L_\ast}{\kappa_\ast^2}(\bar{\delta}_{m_{out},d}+\bar{\delta}_{m_{in},d}) \nonumber \\
     \|\mathbf{E}_{1,4}\|_2 &\leq  c\sqrt{\alpha} \tfrac{L_{\max} L_\ast}{\kappa_\ast^2}\left( \tfrac{k\sqrt{d\log(nm_{in})}  + \sqrt{d}\log^{1.5}(nm_{in})+  \log^{2}(nm_{in})}{\sqrt{nm_{in}}} + \tfrac{1+C^2k}{m_{in}}+ \bar{\delta}_{m_{out},k}\right) \nonumber \\
     \|\mathbf{E}_{1,5}\|_2 &\leq c\sqrt{\alpha} \tfrac{\sigma L_\ast}{\kappa_\ast^2}\delta_{m_{in},k} \nonumber \\
     \|\mathbf{E}_{1,6}\|_2 &\leq c\sqrt{\alpha} \tfrac{\sigma L_\ast}{\kappa_\ast^2}\left( \tfrac{k\sqrt{d\log(nm_{in})}  + \sqrt{d}\log^{1.5}(nm_{in})+  \log^{2}(nm_{in})}{\sqrt{nm_{in}}} + \bar{\delta}_{m_{out},k}\right) \nonumber \\
     \|\mathbf{E}_{1,7}\|_2  &\leq  c\sqrt{\alpha} \tfrac{\sigma L_\ast}{\kappa_\ast^2}\bar{\delta}_{m_{out}, d} \nonumber \\
     \|\mathbf{E}_{1,8}\|_2 &\leq  c\sqrt{\alpha} \tfrac{\sigma L_\ast}{\kappa_\ast^2}{\delta}_{m_{out}, k}
\end{align}

For $\|\mathbf{E}_2\|_2$, we have
\begin{align}
   \| \mathbf{E}_2\|_2 & \leq \bigg\|{\frac{1}{n}\sum_{i=1}^n  \sout \deldin \mathbf{q}_{t,i} \mathbf{q}_{t,i}^\top \sin \mathbf{B}_t -  \deld_t \mathbf{q}_{t,i}\mathbf{q}_{t,i}^\top \b}   \bigg\|_2   \nonumber \\
   &\quad + \bigg\| {\frac{1}{n}\sum_{i=1}^n  \sout \tfrac{\alpha}{m_{in}}\b \b^\top (\mathbf{X}_{t,i}^{in})^\top\mathbf{z}_{t,i}^{in}}\mathbf{q}_{t,i}^\top \sin \b \bigg\|_2 + \bigg\|\frac{1}{n}\sum_{i=1}^n \tfrac{1}{m_{out}}(\mathbf{X}_{t,i}^{out})^\top \mathbf{z}_{t,i}^{out} \mathbf{q}_{t,i}^\top \sin \b \bigg\|_2 \nonumber \\
   &\leq \bigg\|\underbrace{\frac{1}{n}\sum_{i=1}^n  \sout \mathbf{q}_{t,i} \mathbf{q}_{t,i}^\top \sin \mathbf{B}_t -   \mathbf{q}_{t,i}\mathbf{q}_{t,i}^\top \b }_{=:\mathbf{E}_{2,1}}  \bigg\|_2\nonumber 
   \\
   &\quad +\alpha\bigg\|\underbrace{{\frac{1}{n}\sum_{i=1}^n  \sout \b \b^\top \sin \mathbf{q}_{t,i} \mathbf{q}_{t,i}^\top \sin \mathbf{B}_t -  \b \b^\top \mathbf{q}_{t,i}\mathbf{q}_{t,i}^\top \b}}_{=:\mathbf{E}_{2,2}}   \bigg\|_2\nonumber \\
   &\quad   + \alpha\bigg\| \underbrace{{\frac{1}{n}\sum_{i=1}^n  \sout \b \b^\top \tfrac{1}{m_{in}}(\mathbf{X}_{t,i}^{in})^{\top}\mathbf{z}_{t,i}^{in}}\mathbf{q}_{t,i}^\top \sin \b}_{=:\mathbf{E}_{2,3}} \bigg\|_2 + \bigg\|\underbrace{\frac{1}{n}\sum_{i=1}^n \tfrac{1}{m_{out}}(\mathbf{X}_{t,i}^{out})^\top \mathbf{z}_{t,i}^{out} \mathbf{q}_{t,i}^\top \sin \b }_{=:\mathbf{E}_{2,4}} \bigg\|_2 \nonumber 
\end{align}
As before, we apply the bounds from Lemmas \ref{lem:gen1} and \ref{lem:gen2} and use $\alpha \|\b\|_2^2 = O(1)$, $\|\mathbf{w}_t\|_2 = O(\sqrt{\alpha}\eta_\ast/\kappa_\ast^2)$, and $\max_{i\in [n]}\|\mathbf{q}_{t,i}\|_2 = O(L_{\max})$ to obtain that each of the following bounds hold  with probability at least $1 - \frac{1}{\poly(n)} - \frac{1}{\poly(m_{in})} - c'e^{-90k}$, for some absolute constants $c,c'$.
\begin{align}
    \|\mathbf{E}_{2,1}\|_2 &\leq  \tfrac{c L_{\max}^2}{\sqrt{\alpha}} ( \bar{\delta}_{m_{out},d}+ {\delta}_{m_{in},k}) \nonumber \\
    \|\mathbf{E}_{2,2}\|_2  &\leq \tfrac{c L_{\max}^2}{\sqrt{\alpha}} ( \bar{\delta}_{m_{out},d}+ {\delta}_{m_{in},k}) \nonumber \\
     \|\mathbf{E}_{2,3}\|_2 &\leq  \tfrac{c L_{\max}\sigma}{\sqrt{\alpha}} {\delta}_{m_{in},k}  \nonumber \\
     \|\mathbf{E}_{2,4}\|_2 &\leq \tfrac{c L_{\max}\sigma}{\sqrt{\alpha}} \bar{\delta}_{m_{out},d} \nonumber
\end{align}







For $\|\mathbf{E}_3\|_2$, we have
\begin{align}
   \| \mathbf{E}_3\|_2 & \leq \bigg\|{\frac{1}{n}\sum_{i=1}^n  \sin  \mathbf{q}_{t,i} \mathbf{q}_{t,i}^\top (\deldin)^\top\sout \mathbf{B}_t -   \mathbf{q}_{t,i}\mathbf{q}_{t,i}^\top \deld_t\b}   \bigg\|_2 \nonumber \\
   &\quad + \bigg\| {\frac{1}{n}\sum_{i=1}^n  \sin \mathbf{q}_{t,i}\tfrac{\alpha}{m_{in}}  (\mathbf{z}_{t,i}^{in}})^\top(\mathbf{X}_{t,i}^{in})
   \b \b^\top \sout \b \bigg\|_2 \nonumber \\
    &\quad + \bigg\|\frac{1}{n}\sum_{i=1}^n \sin \mathbf{q}_{t,i} \tfrac{1}{m_{out}} (\mathbf{z}_{t,i}^{out})^\top\mathbf{X}_{t,i}^{out}  \b \bigg\|_2 \nonumber \\
   &\leq \bigg\|\underbrace{\frac{1}{n}\sum_{i=1}^n  \sin  \mathbf{q}_{t,i} \mathbf{q}_{t,i}^\top \sout \mathbf{B}_t -   \mathbf{q}_{t,i}\mathbf{q}_{t,i}^\top \b} _{=:\mathbf{E}_{3,1}}  \bigg\|_2\nonumber \\
   &\quad +\bigg\|\underbrace{\frac{\alpha}{n}\sum_{i=1}^n  \sin  \mathbf{q}_{t,i} \mathbf{q}_{t,i}^\top \sin \b \b^\top \sout \mathbf{B}_t -   \mathbf{q}_{t,i}\mathbf{q}_{t,i}^\top \deld_t\b} _{=:\mathbf{E}_{3,2}}   \bigg\|_2\nonumber \\
   &\quad   + \bigg\| \underbrace{{\frac{1}{n}\sum_{i=1}^n  \sin \mathbf{q}_{t,i}\tfrac{\alpha}{m_{in}}  (\mathbf{z}_{t,i}^{in}})^\top(\mathbf{X}_{t,i}^{in})
   \b \b^\top \sout \b }_{=:\mathbf{E}_{3,3}} \bigg\|_2 + \bigg\|\underbrace{\frac{1}{n}\sum_{i=1}^n \sin \mathbf{q}_{t,i} \tfrac{1}{m_{out}} (\mathbf{z}_{t,i}^{out})^\top \mathbf{X}_{t,i}^{out} \b }_{=:\mathbf{E}_{3,4}} \bigg\|_2 \nonumber
\end{align}
Each term is bounded as follows with probability at least $1 - \frac{1}{\poly(n)} - \frac{1}{\poly(m_{in})} - c'e^{-90k}$, for some absolute constants $c,c'$.
\begin{align}
    \|\mathbf{E}_{3,1}\|_2 &\leq  \tfrac{c L_{\max}^2}{\sqrt{\alpha}} (\bar{\delta}_{m_{in},d}+{\delta}_{m_{out},k}) \nonumber \\
    \|\mathbf{E}_{3,2}\|_2  &\leq \tfrac{c L_{\max}^2}{\sqrt{\alpha}}\left( \tfrac{\sqrt{kd\log(nm_{in})}+ \sqrt{k}\log(nm_{in})  + \sqrt{d}\log^{1.5}(nm_{in})+  \log^{2}(nm_{in})}{\sqrt{n m_{in}}}    + \tfrac{\sqrt{k}}{\sqrt{n m_{out}}}\right) \nonumber \\
     \|\mathbf{E}_{3,3}\|_2 &\leq \tfrac{c L_{\max}\sigma }{\sqrt{\alpha}}\left( \tfrac{k\sqrt{d\log(nm_{in})} + k\log(nm_{in}) + \sqrt{d}\log^{1.5}(nm_{in})+  \log^{2}(nm_{in})}{\sqrt{n m_{in}}}  +   \tfrac{\sqrt{k}}{\sqrt{n m_{out}}} \right) \nonumber \\
     \|\mathbf{E}_{3,4}\|_2 &\leq \tfrac{c L_{\max}\sigma}{\sqrt{\alpha}} {\delta}_{m_{out},k} \nonumber
\end{align}





For $\|\mathbf{E}_4\|_2$, we have
\begin{align}
   \| \mathbf{E}_4\|_2 & \leq \bigg\|{\frac{1}{nm_{in}}\sum_{i=1}^n  \sout \deldin \mathbf{q}_{t,i} (\mathbf{z}_{t,i}^{in})^\top \mathbf{X}_{t,i}^{in} \mathbf{B}_t}   \bigg\|_2  + \alpha \bigg\| {\frac{1}{n}\sum_{i=1}^n  \sout   \b \b^\top \tfrac{1}{m_{in}^2}(\mathbf{X}_{t,i}^{in})^\top\mathbf{z}_{t,i}^{in}(\mathbf{z}_{t,i}^{in}})^\top\mathbf{X}_{t,i}^{in}
   \b \bigg\|_2 \nonumber \\
    &\quad + \bigg\|\frac{1}{nm_{in}m_{out}}\sum_{i=1}^n  (\mathbf{X}_{t,i}^{out})^\top \mathbf{z}_{t,i}^{out}(\mathbf{z}_{t,i}^{in})^\top \mathbf{X}_{t,i}^{in} \mathbf{B}_t\bigg\|_2 \nonumber \\
   &\leq \bigg\|\underbrace{{\frac{1}{nm_{in}}\sum_{i=1}^n  \sout \mathbf{q}_{t,i} (\mathbf{z}_{t,i}^{in})^\top \mathbf{X}_{t,i}^{in} \mathbf{B}_t}} _{=:\mathbf{E}_{4,1}}  \bigg\|_2+\alpha\bigg\|\underbrace{{\frac{1}{nm_{in}}\sum_{i=1}^n  \sout \b \b^\top \sin \mathbf{q}_{t,i} (\mathbf{z}_{t,i}^{in})^\top \mathbf{X}_{t,i}^{in} \mathbf{B}_t}} _{=:\mathbf{E}_{4,2}}   \bigg\|_2\nonumber \\
   &\quad   + \alpha\bigg\| \underbrace{{\frac{1}{n}\sum_{i=1}^n  \sout   \b \b^\top \tfrac{1}{m_{in}^2}(\mathbf{X}_{t,i}^{in})^\top\mathbf{z}_{t,i}^{in}(\mathbf{z}_{t,i}^{in}})^\top\mathbf{X}_{t,i}^{in}
   \b }_{=:\mathbf{E}_{4,3}} \bigg\|_2 \nonumber \\
   &\quad + \bigg\|\underbrace{\frac{1}{nm_{in}m_{out}}\sum_{i=1}^n  (\mathbf{X}_{t,i}^{out})^\top \mathbf{z}_{t,i}^{out}(\mathbf{z}_{t,i}^{in})^\top \mathbf{X}_{t,i}^{in} \mathbf{B}_t}_{=:\mathbf{E}_{4,4}} \bigg\|_2 \nonumber
\end{align}
Each term is bounded as follows with probability at least $1 - \frac{1}{\poly(n)} - \frac{1}{\poly(m_{in})} - c'e^{-90k}$, for some absolute constants $c,c'$.
\begin{align}
    \|\mathbf{E}_{4,1}\|_2 &\leq  \tfrac{c L_{\max} \sigma}{\sqrt{\alpha}} {\delta}_{m_{in},k} \nonumber \\
    \|\mathbf{E}_{4,2}\|_2  &\leq \tfrac{c L_{\max} \sigma}{\sqrt{\alpha}} {\delta}_{m_{in},k} \nonumber \\
     \|\mathbf{E}_{4,3}\|_2 &\leq \tfrac{c \sigma^2 }{\sqrt{\alpha}}{\delta}^2_{m_{in},k}  \nonumber \\
     \|\mathbf{E}_{4,4}\|_2 &\leq \tfrac{c \sigma^2}{\sqrt{\alpha}} \bar{\delta}_{m_{out},d} {\delta}_{m_{in},k} \nonumber
\end{align}
For $\|\mathbf{E}_5\|_2$, we have 
\begin{align}
   \| \mathbf{E}_5\|_2 
   &\leq \bigg\|\underbrace{{\frac{1}{nm_{in}}\sum_{i=1}^n    (\mathbf{X}_{t,i}^{in})^\top \mathbf{z}_{t,i}^{in}\mathbf{q}_{t,i}^\top \sout \mathbf{B}_t}}_{=:\mathbf{E}_{5,1}}  \bigg\|_2+\bigg\|\underbrace{{\frac{\alpha}{nm_{in}}\sum_{i=1}^n     (\mathbf{X}_{t,i}^{in})^\top\mathbf{z}_{t,i}^{in}\mathbf{q}_{t,i}^\top \sin \b \b^\top \sout \mathbf{B}_t}}_{=:\mathbf{E}_{5,2}}   \bigg\|_2\nonumber \\
   &\quad   + \bigg\| \underbrace{{\frac{\alpha}{nm_{in}^2}\sum_{i=1}^n     (\mathbf{X}_{t,i}^{in})^\top\mathbf{z}_{t,i}^{in}}(\mathbf{z}_{t,i}^{in})^\top\mathbf{X}_{t,i}^{in}  \b \b^\top \sout 
   \b }_{=:\mathbf{E}_{5,3}} \bigg\|_2 + \bigg\|\underbrace{\frac{1}{nm_{in}m_{out}}\sum_{i=1}^n  (\mathbf{X}_{t,i}^{in})^\top \mathbf{z}_{t,i}^{in}(\mathbf{z}_{t,i}^{out})^\top \mathbf{X}_{t,i}^{out} \mathbf{B}_t}_{=:\mathbf{E}_{5,4}} \bigg\|_2 \nonumber
\end{align}
Each term is bounded as follows with probability at least $1 - \frac{1}{\poly(n)} - \frac{1}{\poly(m_{in})} - c'e^{-90k}$, for some absolute constants $c,c'$.
\begin{align}
    \|\mathbf{E}_{5,1}\|_2 &\leq  \tfrac{c L_{\max}\sigma}{\sqrt{\alpha}} \bar{\delta}_{m_{in},d} \nonumber \\
    \|\mathbf{E}_{5,2}\|_2  &\leq \tfrac{c L_{\max} \sigma}{\sqrt{\alpha}}\left( \tfrac{\sqrt{kd}\log(nm_{in})+ \sqrt{d}\log^{1.5}(nm_{in})+ \log^2(nm_{in})}{\sqrt{n m_{in}}}    + \tfrac{\sqrt{k}}{\sqrt{n m_{out}}}\right) \nonumber \\
     \|\mathbf{E}_{5,3}\|_2 &\leq \tfrac{c \sigma^2 }{\sqrt{\alpha}}\left( \tfrac{\sqrt{kd}\sqrt{\log(nm_{in})} + \sqrt{d}\log(nm_{in}) + \log^{1.5}(nm_{in})}{\sqrt{n m_{in}}}  +   \tfrac{\sqrt{k}}{\sqrt{n m_{out}}} \right) \nonumber \\
     \|\mathbf{E}_{5,4}\|_2 &\leq \tfrac{c \sigma^2}{\sqrt{\alpha}} \bar{\delta}_{m_{in},d}{\delta}_{m_{out},k} \nonumber
\end{align}

Applying a union bound over these events yields that
\begin{align}
   \| &\mathbf{\hat{G}}_{\mathbf{B},t} - \mathbf{{G}}_{\mathbf{B},t}\|_2 \nonumber \\
   &\leq c\sqrt\alpha\Bigg( \tfrac{1}{{m_{in}}}\bigg( \tfrac{k L_{\max} L_\ast}{\kappa_\ast^2}\bigg) 
+\tfrac{1}{\sqrt{m_{in}}}\bigg( L_{\max} (L_{\max}+\sigma)(\sqrt{k}+\sqrt{\log(n)})  \bigg) \nonumber \\
&\quad +\tfrac{1}{\sqrt{m_{out}}}\bigg( L_{\max}(L_{\max}+\sigma)(\sqrt{k}+\sqrt{\log(n)})  \bigg) \nonumber \\
&\quad +\tfrac{1}{\sqrt{nm_{in}}}\bigg( L_{\max} (L_{\max}+\sigma)(k \sqrt{d\log(nm_{in})} + k\log(nm_{in}) + \sqrt{d}\log^{1.5}(nm_{in}) +\log^2(nm_{in}))\nonumber \\
&\quad + \sigma^2 (\sqrt{kd} +  \sqrt{d}\log(nm_{in}) +\log^{1.5}(nm_{in})) 
\bigg) \nonumber \\
&\quad + \tfrac{1}{\sqrt{nm_{out}}}\bigg( L_{\max}(L_{\max} + \sigma)\sqrt{d} + \sigma^2 (\tfrac{\sqrt{d}}{\sqrt{m_{in}}} + \sqrt{k})\bigg) \Bigg) \nonumber \\
&:=\sqrt{\alpha} \zeta_{2,a} \nonumber
\end{align}
with probability at least $ 1 - \frac{1}{\poly(n)}- \frac{1}{\poly(m_{in})} - c' e^{-90k}$ for absolute constants $c,c'$.
\end{proof}

\begin{lemma}[Exact ANIL FS representation concentration II]\label{lem:exactanil_rep2}
For Exact ANIL, consider any $t \in [T]$. With probability at least $1 - c e^{-100k} - \tfrac{1}{\poly(n)}$ for an absolute constant $c$:
\begin{align}
    \|\mathbf{B}_t^\top\mathbf{\hat{G}}_{\mathbf{B},t} - \mathbf{B}_t^\top \mathbf{{G}}_{\mathbf{B},t}\|_2 \leq  
    \zeta_{2,b} \nonumber
\end{align}
where 
\begin{align}
    \zeta_{2,b}\coloneqq O\left( \tfrac{\sqrt{k}+\sqrt{\log(n)}}{\sqrt{m_{in}}}\big(L_{\max}(L_{\max}+\sigma) +\sigma^2(\tfrac{\sqrt{k}+\sqrt{\log(n)}}{\sqrt{m_{in}}} + \tfrac{\sqrt{k}}{\sqrt{nm_{out}}}  ) \big)  +
    \tfrac{\sqrt{k}}{\sqrt{nm_{out}}}\big(  L_{\max}(L_{\max}+\sigma) \big) \right).
\end{align}
\end{lemma}
\begin{proof}
We adapt the proof of Lemma \ref{lem:exactanil_rep1}.
Multiplying $\mathbf{\hat{G}}_{\mathbf{B},t}\! -\! \mathbf{{G}}_{\mathbf{B},t} $ on the left by $\mathbf{B}_t^\top$ serves to reduce the dimensionality of $\mathbf{\hat{G}}_{\mathbf{B},t}\! -\! \mathbf{{G}}_{\mathbf{B},t}$ from $\mathbb{R}^{d\times k}$ to $\mathbb{R}^{k\times k}$. This means that all of the $d$ dependence in the previous concentration result for $\|\mathbf{\hat{G}}_{\mathbf{B},t}\! -\! \mathbf{{G}}_{\mathbf{B},t} \|_2$ is reduced to $k$. Moreover, we no longer need to apply the complicated bounds on sums of fourth-order products (Lemma \ref{lem:gen2}) to show concentration at a rate of $\tfrac{\sqrt{d}}{\sqrt{nm_{in}}}$, since we can afford to show concentration of each second order product at a rate $\tfrac{\sqrt{k}+\sqrt{\log(n)}}{\sqrt{m_{in}}}$ (see Lemma \ref{lem:gen1}). Finally, we must divide the remaining bound from Lemma \ref{lem:exactanil_rep1} by $\sqrt{\alpha}$ since $\|\b\|_2 = \Theta(\tfrac{1}{\sqrt{\alpha}})$. Making these changes yields the result.

\end{proof}

\begin{lemma}[Exact ANIL FS head concentration]\label{lem:exactanil_head}
For Exact ANIL, consider any $t \in [T]$. With probability at least  $1 - ce^{-100k} - \frac{1}{\poly(n)}$ for an absolute constant $c$, we have
\begin{align}
    \|\mathbf{\hat{G}}_{\mathbf{w},t} -  \mathbf{{G}}_{\mathbf{w},t}\|_2 &\leq \tfrac{1}{\sqrt{\alpha}}\zeta_1,
\end{align}
where $\zeta_1 = O(\tfrac{(L_{\max}+\sigma)(\sqrt{k}+\sqrt{\log(n))}}{\sqrt{m_{in}}}  + \tfrac{(L_{\max}+\sigma)\sqrt{k}}{\sqrt{nm_{out}}} )$.
\end{lemma}
\begin{proof}
We have:
\begin{align}
    \|\mathbf{\hat{G}}_{\mathbf{w},t} -  \mathbf{{G}}_{\mathbf{w},t}\|_2 &= \bigg\|\frac{1}{n}\sum_{i=1}^n\b^\top(\deldin)^\top \tfrac{1}{m_{out}}(\mathbf{X}_{t,i}^{out})^\top \mathbf{\hat{v}}_{t,i}- \tfrac{1}{m_{out}}\mathbf{B}_t^\top (\mathbf{X}_{t,i}^{out})^\top\mathbf{z}_{t,i}^{out}  - \b^\top \deld_t \deld_t \mathbf{q}_{t,i} \bigg\|_2 \nonumber \\
    &\leq \bigg\|\frac{1}{n}\sum_{i=1}^n \b^\top(\deldin)^\top \sout \deldin \mathbf{q}_{t,i}-   \del_t \del_t \b^\top \mathbf{q}_{t,i} \bigg\|_2 +  \bigg\|\frac{1}{n}\sum_{i=1}^n \tfrac{1}{m_{out}}\mathbf{B}_t^\top (\mathbf{X}_{t,i}^{out})^\top\mathbf{z}_{t,i}^{out} \bigg\|_2 \nonumber \\
    &\quad +  \bigg\|\frac{1}{n}\sum_{i=1}^n \tfrac{\alpha}{m_{in}}\b^\top(\deldin)^\top \sout \b \b^\top (\mathbf{X}_{t,i}^{in})^\top \mathbf{z}_{t,i}^{in}\bigg\|_2 \nonumber \\
    &\quad +   \bigg\|\frac{1}{n}\sum_{i=1}^n \b^\top(\deldin)^\top \tfrac{1}{m_{out}}(\mathbf{X}_{t,i}^{out})^\top \mathbf{z}_{t,i}^{out} \bigg\|_2 \nonumber \\
    &\leq \bigg\|\underbrace{\frac{1}{n}\sum_{i=1}^n\del_t \b^\top \sout \deld_t \mathbf{q}_{t,i}-    \del_t \b^\top \deld_t \mathbf{q}_{t,i}}_{=:\mathbf{E}_1} \bigg\|_2  + \bigg\|\underbrace{\frac{1}{n}\sum_{i=1}^n \tfrac{1}{m_{out}}\mathbf{B}_t^\top (\mathbf{X}_{t,i}^{out})^\top\mathbf{z}_{t,i}^{out}}_{=:\mathbf{E}_2} \bigg\|_2 \nonumber \\
    &\quad + \bigg\|\underbrace{ \frac{\alpha^2}{n}\sum_{i=1}^n \b^\top \sin \b \b^\top \sout \b \b^\top \sin \mathbf{q}_{t,i}  - \frac{\alpha^2}{n}\sum_{i=1}^n \b^\top  \b \b^\top \sout \b \b^\top  \mathbf{q}_{t,i} }_{=:\mathbf{E}_3} \bigg\|_2 \nonumber \\
    &\quad +  \bigg\|\underbrace{\frac{1}{n}\sum_{i=1}^n \tfrac{\alpha}{m_{in}}\b^\top\deld_t \sout \b \b^\top (\mathbf{X}_{t,i}^{in})^\top \mathbf{z}_{t,i}^{in}}_{=:\mathbf{E}_4}\bigg\|_2 \nonumber \\
    &\quad +  \bigg\|\underbrace{\frac{1}{n}\sum_{i=1}^n \tfrac{\alpha^2}{m_{in}}\b^\top \sin \b \b^\top\sout \b \b^\top (\mathbf{X}_{t,i}^{in})^\top \mathbf{z}_{t,i}^{in}}_{=:\mathbf{E}_5}\bigg\|_2 \nonumber \\
    &\quad +  \bigg\|\underbrace{\frac{1}{n}\sum_{i=1}^n \tfrac{\alpha^2}{m_{in}}\b^\top \b \b^\top\sout \b \b^\top (\mathbf{X}_{t,i}^{in})^\top \mathbf{z}_{t,i}^{in}}_{=:\mathbf{E}_6}\bigg\|_2 + \bigg\|\underbrace{\frac{1}{n}\sum_{i=1}^n \b^\top\deld \tfrac{1}{m_{out}}(\mathbf{X}_{t,i}^{out})^\top \mathbf{z}_{t,i}^{out}}_{=: \mathbf{E}_7} \bigg\|_2  \nonumber \\
    &\quad  + \bigg\|\underbrace{\frac{\alpha}{n}\sum_{i=1}^n \b^\top(\sin - \mathbf{I}_d) \b \b^\top \tfrac{1}{m_{out}}(\mathbf{X}_{t,i}^{out})^\top \mathbf{z}_{t,i}^{out}}_{=: \mathbf{E}_8} \bigg\|_2 \nonumber
\end{align}
By Lemma \ref{lem:gen1} and the facts that $\|\b\|_2 = \tfrac{1}{\sqrt{\alpha}}$, $\|\del_t\|_2=O(1)$, and $\max_{i}\|\mathbf{q}_{t,i}\|_2=L_{\max}$ we have
\begin{align}
    \mathbb{P}(\|\mathbf{E}_1\|_2 \geq \|\del_t\|_2 \|\b\|_2 \max_i\|\deld_t \mathbf{q}_{t,i}\|_2 \bar{\delta}_{m_{out},k}) &\leq 2e^{-90k} \nonumber \\
    \mathbb{P}(\|\mathbf{E}_2\|_2 \geq \sigma \|\b\|_2 \bar{\delta}_{m_{out},k}) &\leq 2e^{-90k} \nonumber \\
    \mathbb{P}(\|\mathbf{E}_3\|_2 \geq \alpha^2 \|\b\|_2^5 \max_i\|\mathbf{q}_{t,i}\|_2 {\delta}_{m_{in},k} &\leq 8 n^{-99} \nonumber \\
     \mathbb{P}(\|\mathbf{E}_4\|_2 \geq \alpha \|\b\|_2^3 \sigma \|\del_t\|_2 \bar{\delta}_{m_{in},k}) &\leq 4n^{-99} \nonumber \\
      \mathbb{P}(\|\mathbf{E}_5\|_2 \geq \alpha^2 \|\b\|_2^5 \sigma  {\delta}_{m_{in},k}&\leq  6n^{-99} \nonumber \\
      \mathbb{P}(\|\mathbf{E}_6\|_2 \geq \alpha^2 \|\b\|_2^5 \sigma  \bar{\delta}_{m_{in},k}) &\leq 4 n^{-99} \nonumber \\
    \mathbb{P}(\|\mathbf{E}_7\|_2 \geq  \|\b\|_2\sigma \|\del_t\|_2  \bar{\delta}_{m_{out},k}) &\leq 2e^{-90k} \nonumber \\
     \mathbb{P}(\|\mathbf{E}_8\|_2 \geq \alpha \|\b\|_2^3 \sigma  {\delta}_{m_{in},k}\bar{\delta}_{m_{out},k}) &\leq 2e^{-90k}+  2n^{-99} \nonumber 
\end{align}


    

Combining these bounds with a union bound yields:
\begin{align}
    \|\mathbf{\hat{G}}_{\mathbf{w},t} -  \mathbf{{G}}_{\mathbf{w},t}\|_2 &\leq 
    \tfrac{c}{\sqrt{\alpha}}\tfrac{(L_{\max}+\sigma)(\sqrt{k}+\sqrt{\log(n))}}{\sqrt{m_{in}}}  +  \tfrac{c}{\sqrt{\alpha}}\tfrac{(L_{\max}+\sigma)\sqrt{k}}{\sqrt{nm_{out}}}
\end{align}
with probability at least $1 - c'e^{-90k} - \frac{1}{\poly(n)}$ for absolute constants $c,c'$. 
\end{proof}

\begin{lemma}[Exact ANIL, Finite samples, $A_1(t+1)$] \label{lem:exactanil_fs_a1}
For Exact ANIL, suppose $A_2(s)$ and $A_5(s)$ hold for all $s\in [t]$. Then 
\begin{align}
    \|\mathbf{w}_{t+1}\|_2 &\leq  \tfrac{1}{10} \sqrt{\alpha}E_0\min\left(1, \tfrac{\mu_\ast^2}{\eta_\ast^2}\right)\eta_\ast
\end{align}
with probability at least $1 - ce^{-100k} - \frac{1}{\poly(n)}$ for an absolute constant $c$.
\end{lemma}
\begin{proof}
For any $s \in [t]$, we have
\begin{align}
   \| \mathbf{w}_{s+1}\|_2 &= \|\mathbf{w}_s - \beta \mathbf{G}_{\mathbf{w}, s} + \beta (\mathbf{G}_{\mathbf{w}, s} - \mathbf{\hat{G}}_{\mathbf{w}, s} )\|_2 \nonumber \\
    &\leq \|\mathbf{w}_s - \beta \mathbf{G}_{\mathbf{w}, s}  \|_2 + \beta \|\mathbf{G}_{\mathbf{w}, s} - \mathbf{\hat{G}}_{\mathbf{w}, s} \|_2\nonumber \\
        &\leq \|\mathbf{w}_s\|_2  +  c \tfrac{\beta}{\sqrt{\alpha}}\|\del_s\|_2^2\eta_\ast   +  \beta \|\mathbf{G}_{\mathbf{w}, s} - \mathbf{\hat{G}}_{\mathbf{w}, s} \|_2 \label{wsws}   \\
        &\leq \|\mathbf{w}_s\|_2  +  c \tfrac{\beta}{\sqrt{\alpha}}\|\del_s\|_2^2\eta_\ast   + \tfrac{\beta}{\sqrt{\alpha}}\zeta_1  \label{wwwu} 
\end{align}
where $\zeta_1$ is defined as in Lemma \ref{lem:exactanil_head}, \eqref{wsws} follows from equation \eqref{eq64} and \eqref{wwwu} follows from Lemma \ref{lem:exactanil_head}. This will allow us to apply Lemma \ref{lem:gen_w} with $\xi_{1,s}=0$ and $\xi_{2,s} = \frac{c\beta }{\sqrt{\alpha}} (\|\del_s\|_2^2\eta_\ast  +\zeta_1) $.

Before doing so, let 
$\zeta_2$ be defined as in Lemma \ref{lem:exactanil_fs_a2} and $\zeta_4\coloneqq \zeta_{2,a}$, corresponding to Lemma \ref{lem:exactanil_fs_a4}.
Observe that for any $s\in[t]$, we can recursively apply $A_2(s), A_2(s-1), \dots$ to obtain
\begin{align}
    \|\del_s\|_2 &\leq (1 - 0.5 \beta \alpha E_0 \mu_\ast^2)   \|\del_{s-1}\|_2 +  \beta^2 \alpha^2 L_\ast^4 \dist^2_{s-1} + \beta \alpha \zeta_2 \nonumber \\
    &\quad \vdots \nonumber \\
&\leq \sum_{r = 1}^{s-1} ( 1 - 0.5 \beta \alpha E_0 \mu_\ast^2)^{s-1-r}(c\beta^2 \alpha^2 L_\ast^4 \dist_{r}^2 + \beta {\alpha} \zeta_2) \nonumber \\
&\leq \sum_{r = 1}^{s-1} ( 1 - 0.5 \beta \alpha E_0 \mu_\ast^2)^{s-1-r}(c\beta^2 \alpha^2 L_\ast^4 (2 \rho^{2r} +  \beta^2 {\alpha} \zeta_4^2) + \beta \alpha \zeta_2)\nonumber \\
&\leq c'\alpha^{2}\beta^2 L_\ast^4  \sum_{r=1}^s ( 1 - 0.5 \beta \alpha E_0 \mu_\ast^2)^{s-1-r}\rho^{2r}+ c'\alpha^{3}\beta^4 L_\ast^4  \sum_{r=1}^{s-1} ( 1 - 0.5 \beta \alpha E_0 \mu_\ast^2)^{s-1-r} \zeta_4^2 \nonumber \\
&\quad + \beta \alpha \sum_{r=1}^{s-1} ( 1 - 0.5 \beta \alpha E_0 \mu_\ast^2)^{s-1-r} \zeta_2 \nonumber \\
&\leq  \frac{c'\alpha^{2}\beta^2 L_\ast^4  ( 1 - 0.5 \beta \alpha E_0 \mu_\ast^2)^{s-1} }{0.5 \beta \alpha \mu_\ast^2}+   \frac{c'\alpha^{3}\beta^4 L_\ast^4 \zeta_4^2}{0.5 \beta \alpha \mu_\ast^2} + \frac{\beta \alpha \zeta_2 }{0.5 \beta \alpha \mu_\ast^2} \nonumber \\
&\leq  {c''\alpha\beta \kappa_\ast^2 L_\ast^2  ( 1 - 0.5 \beta \alpha E_0 \mu_\ast^2)^{s-1} }+  {c''\alpha^{2} \beta^3 \kappa_\ast^2 L_\ast^2 \zeta_4^2} + \frac{\zeta_2 }{\mu_\ast^2} 
\end{align}

Therefore, via Lemma \ref{lem:gen_w},
\begin{align}
    \|\mathbf{w}_{t+1}\|_2 &\leq \sum_{s=1}^t c  \tfrac{\beta}{\sqrt{\alpha}}\|\del_s\|_2^2\eta_\ast  +\tfrac{c\beta}{\sqrt{\alpha}}\zeta_1 \nonumber \\
    &\leq c'\sum_{s=1}^t   \tfrac{\beta^3 \alpha^{1.5}  L_\ast^8}{E_0^2 \mu_\ast^4}( 1 - 0.5 \beta \alpha E_0 \mu_\ast^2)^{2s-2} \eta_\ast + t\beta^7 \alpha^{3.5} \kappa_\ast^4 L_\ast^4 \eta_\ast \zeta_4^4 + t\tfrac{\beta}{\sqrt{\alpha}}  \tfrac{\eta_\ast}{\mu_\ast^4} \zeta_2^2 + t\tfrac{\beta}{\sqrt{\alpha}}\zeta_1  \nonumber \\
    &\leq  c'' \tfrac{\beta^2 \sqrt{\alpha}  L_\ast^8}{ E_0^3 \mu_\ast^6}\eta_\ast  + c'' t\beta^7 \alpha^{3.5} \kappa_\ast^4 L_\ast^4 \eta_\ast \zeta_4^4 + c'' t\tfrac{\beta}{\sqrt{\alpha}}  \tfrac{\eta_\ast}{\mu_\ast^4} \zeta_2^2 + c'' t\tfrac{\beta}{\sqrt{\alpha}}\zeta_1 \label{geomm}
    \\
    &\leq c'''\tfrac{\sqrt{\alpha}E_0 }{ \kappa_{\ast}^2} \min(1 , \tfrac{\mu_\ast^2}{\eta_\ast^2})\eta_\ast + t\beta^7 \alpha^{3.5} \kappa_\ast^4 L_\ast^4 \eta_\ast \zeta_4^4 + t \tfrac{{\sqrt{\alpha}} E_0}{L_\ast^4} \min(1 , \tfrac{\mu_\ast^2}{\eta_\ast^2})\eta_\ast \zeta_2^2 + t\tfrac{\sqrt{\alpha}E_0}{\kappa_\ast^4} \min(1 , \tfrac{\mu_\ast^2}{\eta_\ast^2})\zeta_1  \label{ssssss} \\
    &\leq c''''\big(\tfrac{\sqrt{\alpha}E_0 }{ \kappa_{\ast}^2} \min(1 , \tfrac{\mu_\ast^2}{\eta_\ast^2})\eta_\ast  + t \tfrac{{\sqrt{\alpha}} E_0}{L_\ast^4} \min(1 , \tfrac{\mu_\ast^2}{\eta_\ast^2})\eta_\ast \zeta_{2,b}^2 + t \beta^2 \alpha^{2.5} \tfrac{ E_0}{L_\ast^4} \min(1 , \tfrac{\mu_\ast^2}{\eta_\ast^2})\eta_\ast \zeta_{2,a}^4 \nonumber \\
    &\quad \quad + t\tfrac{\sqrt{\alpha}E_0}{\kappa_\ast^4} \min(1, \tfrac{\mu_\ast^2}{\eta_\ast^2})\zeta_1\big)     \label{ssssssss}
\end{align}
where \eqref{geomm} follows by the sum of a geometric series and \eqref{ssssss} follows by choice of $\beta \leq c\frac{\alpha E_0^2 }{ \kappa_\ast^4}$ for a sufficiently small constant $c$, \eqref{ssssssss} follows by using the definitions of $\zeta_4$ and $\zeta_2$, the numerical inequality $(a+b)^2 \leq 2a^2 + 2b^2$, and
 subsuming the dominated term.

In order for the RHS \eqref{ssssssss} to be at most $\tfrac{\sqrt{\alpha}E_0}{10}\min(1,\tfrac{\mu_\ast^2}{\eta_\ast^2})\eta_\ast$, we require the following:
\begin{align}
    \zeta_1 &\leq \tfrac{c}{T},\quad
    \zeta_{2,b} \leq \tfrac{L_\ast^2}{\sqrt{T}}  ,\quad \zeta_{2,a} \leq  \tfrac{L_\ast}{\sqrt{\beta \alpha} T^{0.25}}
\end{align}
However, from Corollary \ref{cor:exactanil_fs_a3} we require tighter bounds on $\zeta_{2,b}$ and $\zeta_{2,a}$ when $T$ is small. Accounting for these, it is sufficient to choose
\begin{align}
\zeta_{1} &\leq \tfrac{c\kappa_\ast^4 \eta_\ast}{T} \nonumber \\
    \zeta_{2,b}&\leq c  \tfrac{E_0 \mu_\ast^2}{\sqrt{T}}, \nonumber \\ \zeta_{2,a} &\leq c  \tfrac{\sqrt{E_0} \mu_\ast}{\sqrt{\beta \alpha} T^{0.25}}
\end{align}
We also require $m_{out}\geq c k+c \log(n)$ so that the concentration results hold. This implies that we need
\begin{align}
    m_{out} &\geq cT^2\tfrac{k(L_{\max}+\sigma)^2}{n\eta_\ast^2\kappa_\ast^8} + cT\tfrac{k L_{\max}^2 (L_{\max}+\sigma)^2}{nE_0^2\mu_\ast^4  } + c\sqrt{T}\tfrac{\beta \alpha }{E_0\mu_\ast^2}(L_{\max}^2(L_{\max}+ \sigma)^2 (k+\log(n)) + (L_{\max}+\sigma)^4\tfrac{ d}{n}  )  \nonumber\\
    &\quad + c k+c \log(n) \nonumber
\end{align}
For $m_{in}$, we need
\begin{align}
    m_{in} &\geq cT^2\tfrac{(L_{\max}+\sigma)^2(k+\log(n))}{\eta_\ast^2 \kappa_\ast^8} + cT \tfrac{(L_{\max}+\sigma)^4(k+\log(n))}{E_0^2\mu_\ast^4} + c {T}^{0.25}\tfrac{
    \sqrt{\beta \alpha} L_{\max} k}{\sqrt{E_0} \kappa_\ast} \nonumber \\
    &\quad + c \sqrt{T}\tfrac{\beta \alpha L_{\max}^2(L_{\max}+\sigma)^2 (k+\log(n))}{E_0 \mu_\ast^2} + c \sqrt{T}\tfrac{\beta \alpha (L_{\max}+\sigma)^4 k^2 {d\log(nm_{in})}}{nE_0 \mu_\ast^2} 
\end{align}
under the natural assumption that $k= \Omega(\log(nm_{in}))$. Note that if $m_{in}$ satisfies the above lower bound, this implies $m_{in}>> k+\log(n)$, as needed. Using our upper bounds on $\beta$ and $\alpha$, replacing  $L_{\max}$ with $\sqrt{k}L_\ast$, and treating $E_0$ as a constant gives the final results:
\begin{align}
    m_{out} &\geq cT^2 \tfrac{k^2(L_\ast+\sigma)^2}{n\eta_\ast^2 \kappa_\ast^8} +  cT\tfrac{k^3 \kappa_{\ast}^2 (\kappa_{\ast}^2+\sigma^2/\mu_\ast^2)}{n } + c\sqrt{T}(k + \tfrac{kd}{n} +\log(n))\kappa_\ast^{-2}(\tfrac{\sigma^2}{ L_\ast^2}+k) +ck +c\log(n)
   \nonumber \\
    m_{in} &\geq cT^2(k^2+k\log(n))\tfrac{(L_{\ast}+\sigma)^2}{\eta_\ast^2 \kappa_\ast^8} + cT (k^3+k\log(n)){(\kappa_\ast^4 +\tfrac{\sigma^4}{\mu_\ast^4})} + c \sqrt{T}\tfrac{  k^3 {d\log(nm_{in})}}{n }\kappa_\ast^{-2}(\tfrac{\sigma^2}{ L_\ast^2}+1) \nonumber
\end{align}






















\end{proof}

\begin{lemma}[Exact ANIL Finite samples $A_2(t+1)$] \label{lem:exactanil_fs_a2}
Suppose the conditions of Theorem \ref{thm:anil_fs_app} are satisfied and $A_1(t),A_3(t)$ and $A_5(t)$ hold. Then with probability at least $1 - ce^{-90k} - \tfrac{1}{\poly(n)} - \tfrac{1}{\poly(m_{in})}$ for an absolute constant $c$,
\begin{align}
    \|\del_{t+1}\|_2 &\leq (1 - 0.5 \beta \alpha E_0 \mu_\ast^2)\|\del_t\|_2 + \tfrac{5}{4}\beta^2 \alpha^2 L_\ast^4 \dist_t^2 + \beta \alpha \zeta_2,
\end{align}
$\zeta_2 \coloneqq 2 \zeta_{2,b}+ \beta \alpha \zeta_{2,a}^2$, and $\zeta_{2,a}$ and $\zeta_{2,b}$ are defined in Lemmas \ref{lem:exactanil_rep1} and \ref{lem:exactanil_rep2}, respectively.
\end{lemma}
\begin{proof}
As in Lemma \ref{lem:reg_fs_foanil}, let  $\mathbf{B}_{t+1} = \mathbf{B}_{t+1}^{pop} + \beta (\mathbf{G}_{\mathbf{B},t} - \mathbf{\hat{G}}_{\mathbf{B},t})$, and let $\del_{t+1}^{pop} = \mathbf{I}_k - \alpha  (\mathbf{B}_{t+1}^{pop})^\top \mathbf{B}_{t+1}^{pop}$. Note that the bound from Lemma \ref{lem:ea_pop_a} applies to $\|\del_{t+1}^{pop}\|_2$ This results in
\begin{align}
    \|\del_{t+1}\|_2 &= \|\del_{t+1}^{pop} - \beta \alpha \mathbf{B}_t^\top (\mathbf{G}_{\mathbf{B},t} - \mathbf{\hat{G}}_{\mathbf{B},t}) - \beta\alpha (\mathbf{G}_{\mathbf{B},t} - \mathbf{\hat{G}}_{\mathbf{B},t})^\top \b \nonumber \\
    &\quad \quad + \beta^2 \alpha (\mathbf{G}_{\mathbf{B},t} - \mathbf{\hat{G}}_{\mathbf{B},t})^\top  (\mathbf{G}_{\mathbf{B},t} - \mathbf{\hat{G}}_{\mathbf{B},t})\|_2 \nonumber \\
    &\leq \|\del_{t+1}^{pop}\|_2 + 2 \beta \alpha \|\mathbf{B}_t^\top (\mathbf{G}_{\mathbf{B},t} - \mathbf{\hat{G}}_{\mathbf{B},t})\|_2 + \beta^2 \alpha \|\mathbf{G}_{\mathbf{B},t} - \mathbf{\hat{G}}_{\mathbf{B},t})\|_2^2 \nonumber \\
     &\leq (1-0.5 \beta \alpha E_0 \mu_\ast^2)\|\del_{t}\|_2 +\tfrac{5}{4}\beta^2 \alpha^2 L_\ast^4 \dist_t^2 + 2 \beta \alpha \|\mathbf{B}_t^\top (\mathbf{G}_{\mathbf{B},t} - \mathbf{\hat{G}}_{\mathbf{B},t})\|_2 \nonumber \\
     &\quad \quad \quad \quad + \beta^2 \alpha \|\mathbf{G}_{\mathbf{B},t} - \mathbf{\hat{G}}_{\mathbf{B},t}\|_2^2 \label{prince} \\
    &\leq (1-0.5 \beta \alpha E_0 \mu_\ast^2)\|\del_{t}\|_2 +\tfrac{5}{4}\beta^2 \alpha^2 L_\ast^4 \dist_t^2 + 2 \beta \alpha \zeta_{2,b}+ \beta^2 \alpha^2 \zeta_{2,a}^2 \label{usn} 
\end{align}
where \eqref{prince} follows from Lemma \ref{lem:ea_pop_a} and \eqref{usn}  $\zeta_{2,a}$ and $\zeta_{2,b}$ are defined in Lemmas \ref{lem:exactanil_rep1} and \ref{lem:exactanil_rep2}, respectively. Define $\zeta_2 \coloneqq 2\zeta_{2,b}+ \beta \alpha\zeta_{2,a}^2$ to complete the proof.

\end{proof}

\begin{cor}[Exact ANIL, Finite samples, $A_3(t+1)$] \label{cor:exactanil_fs_a3}
Suppose the conditions of Theorem \ref{thm:anil_fs_app} are satisfied and $A_2(t+1)$ and $A_3(t)$ hold. Then
\begin{align}
    \|\del_{t+1}\|_2 &\leq \tfrac{1}{10}
\end{align}
\end{cor}
\begin{proof}
By $A_2(t+1)$ and $A_3(t)$, we have
\begin{align}
    \|\del_{t+1}\|_2 &\leq  (1-0.5 \beta \alpha E_0 \mu_\ast^2)\|\del_{t}\|_2 +\tfrac{5}{4}\beta^2 \alpha^2 L_\ast^2 \dist_t^2 + 2 \beta \alpha \zeta_{2,b}+ \beta^2 \alpha^2 \zeta_{2,a}^2 \nonumber \\
    &\leq  \tfrac{1}{10} -0.05 \beta \alpha E_0 \mu_\ast^2 +\tfrac{5}{4}\beta^2 \alpha^2 L_\ast^4  + 2 \beta \alpha \zeta_{2,b}+ \beta^2 \alpha^2 \zeta_{2,a}^2  \nonumber \\
    &\leq \tfrac{1}{10} -0.04 \beta \alpha E_0 \mu_\ast^2  + 2 \beta \alpha \zeta_{2,b}+ \beta^2 \alpha^2 \zeta_{2,a}^2 \label{better} \\
    &\leq \tfrac{1}{10}   \label{bet}
\end{align}
where \eqref{better} follows by choice of $\beta = c \tfrac{\alpha E_0}{\kappa_\ast^4}$  and \eqref{bet} follows by $\zeta_{2,b}\leq c E_0 \mu_\ast^2$ and $\zeta_{2,a}^2 \leq c \tfrac{E_0 \mu_\ast^2}{\beta \alpha}$
for a sufficiently small constant $c$.
\end{proof}

\begin{lemma}[Exact-ANIL, Finite samples, $A_4(t+1)$] \label{lem:exactanil_fs_a4}
Suppose the conditions of Theorem \ref{thm:anil_fs_app} are satisfied and $A_1(t)$, $A_3(t)$ and $A_5(t)$ hold. Then $A_{4}(t+1)$ holds with high probability , i.e.
\begin{align}
    \|\mathbf{B}_{\ast,\perp}^\top\mathbf{B}_{t+1}\|_2 &\leq (1 - 0.5 \beta \alpha E_0 \mu_\ast^2 )\|\mathbf{B}_{\ast,\perp}^\top\mathbf{B}_{t}\|_2 + \beta \sqrt{\alpha} \zeta_{4}
\end{align}
where $\zeta_4 =\zeta_{2,a}$ where $\zeta_{2,a}$ is defined in Lemma \ref{lem:exactanil_rep1},
with probability at least $1- ce^{-90k}- \tfrac{1}{\poly(n)} - \tfrac{1}{\poly(m_{in})}$ for an absolute constant $c$.
\end{lemma}
\begin{proof}
We have
\begin{align}
   \| \mathbf{\hat{B}}_{\ast,\perp}^\top \mathbf{B}_{t+1}\|_2 &= \|\mathbf{\hat{B}}_{\ast,\perp}^\top(\mathbf{B}_t - \beta \mathbf{G}_{\mathbf{B},t}) + \beta \mathbf{\hat{B}}_{\ast,\perp}^\top (\mathbf{G}_{\mathbf{B},t} - \mathbf{\hat{G}}_{\mathbf{B},t} )\|_2 \nonumber \\
   &\leq \|\mathbf{\hat{B}}_{\ast,\perp}^\top(\mathbf{B}_t - \beta \mathbf{G}_{\mathbf{B},t})\|_2 + \beta \| \mathbf{G}_{\mathbf{B},t} - \mathbf{\hat{G}}_{\mathbf{B},t} \|_2  \nonumber \\
   &\leq  (1 - 0.5\beta \alpha E_0 \mu_\ast^2 )\|\mathbf{\hat{B}}_{\ast,\perp}^\top\mathbf{B}_t\|_2 + \beta \| \mathbf{G}_{\mathbf{B},t} - \mathbf{\hat{G}}_{\mathbf{B},t} \|_2  \label{eqqe} \\
   &\leq  (1 - 0.5\beta \alpha E_0 \mu_\ast^2 )\|\mathbf{\hat{B}}_{\ast,\perp}^\top\mathbf{B}_t\|_2  + \beta \sqrt{\alpha} \zeta_{2,a} \label{ethio}
\end{align}
where \eqref{eqqe} follows by Lemma \ref{lem:ea_pop_a4} (note that all the required conditions are satisfied) and 
\eqref{ethio} holds with probability at least $1- ce^{-90k}- \tfrac{1}{\poly(n)} - \tfrac{1}{\poly(m_{in})}$ for an absolute constant $c$ according to Lemma \ref{lem:exactanil_rep1}, where $\zeta_{2,a}$ is defined therein.


\end{proof}

\section{Additional simulation and details} \label{app:sims}

In all experiments, we  generated $\mathbf{B}_\ast$ by sampling a matrix in $\mathbb{R}^{d\times k}$ with i.i.d. standard normal elements, then orthogonalizing this matrix by computing its QR-factorization. The same procedure was used to generate $\mathbf{B}_0$ in cases with random initialization, except that the result of the QR-factorization was scaled by $\tfrac{1}{\sqrt{\alpha}}$ such that $\del_0 = \mathbf{0}$, and for the case of methodical initialization (Figure \ref{fig:sims1} (right)), we initialized with an orthogonalized and scaled linear combination of Gaussian noise and  $\mathbf{B}_\ast$ such that $\dist_0\in[0.65,0.7]$ and $\|\del_0\|=0$. Meanwhile, we set $\mathbf{w}_0= \mathbf{0}$.
We used step sizes $\beta=\alpha=0.05$ in all cases for Figure \ref{fig:sims1}, which were tuned optimally. Figure \ref{fig:intro} uses the same setting of $d=20$, $n=k=3$, and Gaussian ground-truth heads as in Figure \ref{fig:sims1}, except that the mean of the ground-truth heads is shifted to zero. We are therefore able to use the larger step sizes  of $\alpha=\beta=0.1$ and observe faster convergence in this case, as task diversity is larger since the ground-truth heads are isotropic, and $L_\ast$ and $L_{\max}$ are smaller.
 Additionally, in Figure \ref{fig:intro}, Avg. Risk Min. is the algorithm that tries to minimize $\mathbb{E}_{\mathbf{w}_{\ast,t,i}}[\mathcal{L}_{t,i}(\mathbf{B},\mathbf{w})]$ via standard mini-batch SGD.  It is equivalent to ANIL and MAML with no inner loop ($\alpha=0$). 
  All results are averaged over 5 random trials. 

\end{document}